%% file: main.tex
\documentclass{article}

\usepackage[final]{neurips_2024}

\usepackage{wrapfig}

\usepackage{multirow}

\usepackage[utf8]{inputenc} 
\usepackage[T1]{fontenc}    
\usepackage{hyperref}       
\usepackage{url}            
\usepackage{booktabs}       
\usepackage{amsfonts}       
\usepackage{nicefrac}       
\usepackage{microtype}      
\usepackage{xcolor}         

\usepackage{amsmath}
\usepackage{amssymb}
\usepackage{mathtools}
\usepackage{amsthm}

\usepackage{times}
\usepackage{soul}
\usepackage{url}
\usepackage[utf8]{inputenc}
\usepackage[small]{caption}
\usepackage{graphicx}
\usepackage{amsmath}
\usepackage{amsthm}
\usepackage{booktabs}
\usepackage{algorithm}
\usepackage{algorithmic}
\usepackage{bbm}
\usepackage{comment}
\usepackage{multirow}
\usepackage{siunitx}
\usepackage{tabularx}
\usepackage{nicematrix}
\usepackage{subfigure}
\usepackage{footnote}
\usepackage{wrapfig}

\def\UCR{\text{UCR}}
\def\UCG{\text{UCG}}
\def\APSS{\text{APSS}}

\def\indicator{\mathbbm{1}}

\usepackage{multirow,paralist,mathrsfs,amsfonts,dsfont,tikz}

\usepackage{enumitem}

\def\EXP{\text{EXP}}
\def\POLY{\text{POLY}}
\def\MAJ{\text{MAJ}}

\def\calX{\mathcal X}
\def\calY{\mathcal Y}

\def\calC{\mathcal C}
\def\calI{\mathcal I}
\def\calD{\mathcal D}

\def\calP{\mathcal P}

\def\E{\mathbb E}
\def\P{\mathbb P}
\def\R{\mathbb R}

\def\newCP{\text{RC3P}}

\def\CCP{\text{CCP}}

\def\Xtest{{X_{n+1}}}
\def\Ytest{{Y_{n+1}}}
\def\Xtest{{X_\text{test}}}
\def\Ytest{{Y_\text{test}}}

\def\class{\text{class}}

\def\min{\text{min}}
\def\max{\text{max}}
\def\tr{\text{tr}}
\def\cal{\text{cal}}
\def\test{\text{test}}

\theoremstyle{plain}
\newtheorem{theorem}{Theorem}[section]

\newtheorem{lemma}[theorem]{Lemma}

\theoremstyle{definition}

\theoremstyle{remark}

\title{Conformal Prediction for Class-wise Coverage \\ via Augmented Label Rank Calibration}

\author{
  Yuanjie Shi \\
  Washington State University\\
  \And
  Subhankar Ghosh \\
  Washington State University\\
  \And
  Taha Belkhouja \\
  Washington State University\\
  \And
  Janardhan Rao Doppa \\
  Washington State University\\
  \And
  Yan Yan \\
  Washington State University\\
}

\begin{document}

\maketitle

\begin{abstract}
Conformal prediction (CP) is an emerging uncertainty quantification framework that allows us to construct a prediction set to cover the true label with a pre-specified marginal or conditional probability.
Although the valid coverage guarantee has been extensively studied for classification problems, CP often produces large prediction sets which may not be practically useful.
This issue is exacerbated for the setting of class-conditional coverage on classification tasks with many and/or imbalanced classes.
This paper proposes the Rank Calibrated Class-conditional CP (\newCP) algorithm to reduce the prediction set sizes to achieve class-conditional coverage, where the valid coverage holds for each class.
In contrast to the standard class-conditional CP (CCP) method that uniformly thresholds the class-wise conformity score for each class, the augmented label rank calibration step allows \newCP~to selectively iterate this class-wise thresholding subroutine only for a subset of classes whose class-wise top-$k$ error is small.
We prove that agnostic to the classifier and data distribution, \newCP~achieves class-wise coverage. We also show that \newCP~reduces the size of prediction sets compared to the CCP method. 
Comprehensive experiments on multiple real-world datasets demonstrate that \newCP~achieves class-wise coverage and  $26.25\%$ $\downarrow$ reduction in prediction set sizes on average. 
\end{abstract}

\section{Introduction}
\label{section:introduction}

Safe deployment of machine learning (ML) models in high stakes applications such as medical diagnosis requires theoretically-sound uncertainty estimates. 
Conformal prediction (CP) \cite{vovk2005algorithmic} is an emerging uncertainty quantification framework that constructs a prediction set of candidate output values such that the true output is present with a pre-specified level (e.g., $\geq 90\%$) of the marginal or conditional probability 
\cite{zaffran2023conformal,fisch2021few}. A promising property of CP is the model-agnostic and distribution-free {\it coverage validity} under certain notions \cite{fontana2023conformal}.
For example, marginal coverage is the commonly studied validity notion \cite{romano2020classification,angelopoulos2020uncertainty,zaffran2023conformal}, while conditional coverage is a stronger notion.  
There is a general taxonomy to group data (i.e., input-output pairs) into categories and to study the valid coverage for each group (i.e., the group-wise validity) \cite{vovk2003mondrian,vovk2005algorithmic}. 
This paper focuses on the specific notion of class-conditional coverage that guarantees coverage for each class individually, which is important for imbalanced classification tasks with many and/or imbalanced classes (e.g., medical applications) \cite{lu2022fair,vazquez2022conformal,lu2022improving}.

\vspace{1.0ex}

In addition to the coverage validity, {\it predictive efficiency} is another important criterion for CP \cite{fontana2023conformal,vovk2016criteria}, which refers to the size of the prediction sets.
Both coverage validity and predictive efficiency are used together to measure the performance of CP methods \cite{angelopoulos2020uncertainty,romano2020malice,romano2020classification,ding2024class,ghosh2023probabilistically,fisch2020efficient}.
Since the two measures are competing \cite{angelopoulos2020uncertainty}, our goal is to guarantee the coverage validity with high predictive efficiency, i.e., small prediction sets \cite{fontana2023conformal,romano2020classification,fisch2020efficient}.
Some studies improved the predictive efficiency under the marginal coverage setting using new conformity score function \cite{angelopoulos2020uncertainty,huang2023conformal} and new calibration procedures \cite{fisch2021few,fisch2020efficient,guan2023localized,ghosh2023improving}.
However, it is not known if these methods will benefit the predictive efficiency for the class-conditional coverage setting.
A very recent work \cite{ding2024class} proposed the cluster CP method to achieve {\em approximate} class-conditional coverage.
It empirically improves predictive efficiency over the baseline class-wise CP method (i.e., each class is one cluster) \cite{vovk2012conditional}, but the approximation guarantee for class-wise coverage is {\em model-dependent} (i.e., requires certain assumptions on the model).
The main question of this paper is: {\it how can we develop a model-agnostic CP algorithm that guarantees the class-wise coverage with improved predictive efficiency (i.e., small prediction sets)?}

\vspace{1.0ex}

To answer this question, we propose a novel approach referred to as {\em Rank Calibrated Class-conditional CP (\newCP)} that guarantees the class-wise coverage with small expected prediction sets. 
The class-conditional coverage validity of \newCP~is agnostic to the data distribution and the underlying ML model, while the improved predictive efficiency depends on very mild conditions of the given trained classifier. The main ingredient behind the 
\newCP~method is the label rank calibration strategy augmented with the standard conformal score calibration from the class-wise CP (CCP) \cite{vovk2012conditional,angelopoulos2021gentle}.

\vspace{1.0ex}

The CCP method finds the class-wise quantiles of non-conformity scores on calibration data. To produce the prediction set for a new test input $X_\test$, it pairs $X_\test$ with each candidate class label $y$ and includes the label $y$ if the non-conformity score of the pair $(X_\test, y)$ is less than or equal to the corresponding class-wise quantile associated with $y$.
Thus, CCP constructs the prediction set by uniformly iterating over {\em all} candidate labels.
In contrast, the label rank calibration allows \newCP~to selectively iterate this class-wise thresholding subroutine only if the label $y$ is ranked by the classifier $f(X_\test)$ (e.g., $f(\cdot)$ denotes the softmax prediction) in the top $k_y$ candidates, where the value of $k_y$ is calibrated for each label $y$ individually according to the class-wise top-$k_y$ error.
In other words, given $X_\test$, \newCP~enables standard class-wise conformal thresholding for the sufficiently certain class labels only (as opposed to all labels).
Our theory shows that the class-wise coverage provided by \newCP~is agnostic to the data distribution and the underlying ML model.
Moreover, under a very mild condition, \newCP~guarantees improved predictive efficiency over the baseline CCP method.

\vspace{1.0ex}

{\bf Contributions.}
The main contributions of this paper are: 
\begin{itemize}
\item We design a novel algorithm \newCP~that augments the label rank calibration strategy to the standard conformal score calibration step.
To produce prediction sets for new inputs, it selectively performs class-wise conformal thresholding only on a subset of classes based on their corresponding calibrated label ranks.

\item We develop theoretical analysis to show that \newCP~guarantees class-wise coverage, which is agnostic to the data distribution and trained classifier.
Moreover, it provably produces smaller average prediction sets over the baseline CCP method \cite{vovk2012conditional}.

\item We perform extensive experiments on multiple imbalanced classification datasets and show that \newCP~achieves the class-wise coverage with significantly improved predictive efficiency over the existing class-conditional CP baselines ($26.25\%$ reduction in the prediction size on average on all four datasets or $35\%$ reduction excluding CIFAR-10). The code is available at \href{https://github.com/YuanjieSh/RC3P}{https://github.com/YuanjieSh/RC3P}.
\end{itemize}

\section{Related Work}
\label{section:related_work}

Precise uncertainty quantification of machine learning based predictions is necessary in high-stakes decision-making applications.
It is especially challenging for imbalanced classification tasks. 
Although many imbalanced classification learning algorithms \cite{cao2019learning,gottlieb2021apportioned} are proposed, e.g., re-sampling \cite{chawla2002smote,mohammed2020machine,krawczyk2019radial,tsai2019under,vuttipittayamongkol2020neighbourhood} and re-weighting \cite{huang2019deep,madabushi2020cost},
they do not provide uncertainty quantification with rigorous guarantees over predictions for each class.

\vspace{1.0ex}

Conformal prediction \cite{vovk1999machine,vovk2005algorithmic} is a model-agnostic and distribution-free framework for uncertainty quantification by producing prediction sets that cover the true output with a pre-specified probability,
which means CP could provide valid coverage guarantee with any underlying model and data distribution \cite{jockel2023conformal,sun2024conformal,dunn2018distribution}. 
Many CP algorithms are proposed for regression \cite{lei2018distribution,romano2019conformalized,gibbs2021adaptive,feldman2023calibrated}, classification \cite{romano2020malice,angelopoulos2020uncertainty,xu2023conformal,lei2021conformal}, structured prediction \cite{bates2021distribution,angelopoulos2022conformal,cohen2024cross,huang2024uncertainty}, online learning \cite{gibbs2024conformal,bhatnagar2023improved}, and covariate shift \cite{jin2023model,tibshirani2019conformal,barber2023conformal} settings.
{\it Coverage validity} and {\it predictive efficiency} are two common and competing desiderata for CP methods \cite{angelopoulos2020uncertainty}.
Thus, small prediction sets are favorable whenever the coverage validity is guaranteed \cite{fontana2023conformal,romano2020classification,fisch2020efficient}, e.g., human and machine learning collaborative systems \cite{lu2022fair,vazquez2022conformal,lu2022improving}.
Recent work\footnote{A concurrent work by Huang and colleagues~\cite{huang2024conformal} studied a method named sorted adaptive prediction sets which uses label ranking information to improve the predictive efficiency in the marginal coverage setting.} improved the predictive efficiency for marginal coverage by designing new conformity score \cite{angelopoulos2020uncertainty,huang2023conformal} and calibration procedures \cite{fisch2021few,fisch2020efficient,guan2023localized,ghosh2023improving}.
These methods can be combined with class-conditional CP methods including \newCP~as we demonstrate in our experiments, but the effect on predictive efficiency is not clear.

\vspace{1.0ex}

In general, the methods designed for a specific coverage validity notion are not necessarily compatible with another notion of coverage, such as object-conditional coverage \cite{vovk2012conditional}, class-conditional coverage \cite{vovk2012conditional}, local coverage \cite{lei2014distribution} which are introduced and studied in the prior CP literature \cite{vovk2003mondrian,vovk2005algorithmic,fontana2023conformal,ding2024class,bostrom2021mondrian}.
The standard class-conditional CP method in \cite{vovk2012conditional,sadinle2019least} guarantees the class-wise coverage, but does not particularly aim to reduce the size of prediction sets.
The cluster CP method \cite{ding2024class} which performs CP over clusters of labels achieves a cluster-conditional coverage that approximates the class-conditional guarantee, but requires some assumptions on the underlying clustering model.

\vspace{1.0ex}

Our goal is to develop a provable class-conditional CP algorithm with small prediction sets to guarantee the class-wise coverage that is agnostic to the underlying model.

\section{Notations, Background, and Problem Setup}
\label{section:problem_setup}

{\bf Notations.}
Suppose $(X, Y)$ is a data sample where $X \in \calX$ is an input from the input space $\mathcal{X}$, and $Y \in \calY = \{ 1,2,\cdots, K\}$ is the ground-truth label with $K$ candidate classes. 
Assume $(X, Y)$ is randomly drawn from an underlying distribution $\calP$ defined on $\calX \times \calY$, where we denote $p_y = \P_{XY} [ Y = y ]$.
Let $f: \calX \rightarrow \Delta_+^K$ denote a soft classifier (e.g., a soft-max classifier) that produces prediction scores for all candidate classes on any given input $X$, where $\Delta_+^K$ denote the $K$-dimensional probability simplex and $f(X)_y$ denotes the predicted confidence for class $y$.
We define the class-wise top-$k$ error for class $y$ from the trained classifier $f$ as $\epsilon_y^k = \P \{ r_f(X, Y) > k | Y = y \}$,
where $r_f(X, Y) = \sum_{l=1}^K \indicator[ f(X)_l \geq f(X)_Y ]$ returns the rank of $Y$ predicted by $f(X)$ in a descending order, and $\indicator[\cdot]$ is an indicator function.
We are provided with a training set $\calD_\tr$ for training the classifier $f$, and a calibration set $\calD_\cal = \{ (X_i, Y_i) \}_{i=1}^n$ for CP. 
Let $\calI_y = \{i : Y_i = y, \text{ for all } (X_i, Y_i) \in \calD_\cal \}$ and $n_y = | \calI_y |$ denote the number of calibration examples for class $y$. 

\begin{table}[!t]
\centering
\caption{\textbf{Key notations used in this paper.}}
\label{tab:notation}
\begin{tabular}{|l|l|}

    \hline
     {\bf Notation} & {\bf Meaning} \\
     \hline
     $X \in \mathcal{X} $ & Input example \\
     \hline
     $Y \in \mathcal{Y} $ & The ground-truth label \\
     \hline
     $f$ & The soft classifier \\
     \hline
     $\Delta_+^K$ & The $K$-dimensional probability simplex\\
    \hline
    $f(X)_y$ & The predicted confidence on class $y$\\
    \hline
    $\epsilon_y^k$ & The class-wise top-$k$ error for class $y$ from $f$\\
    \hline
    $r_f(X, Y)$ & The rank of $Y$ predicted by $f(X)$\\
    \hline
    $\mathcal{D}_{\tr}$ & Training data\\
    \hline
    $\mathcal{D}_{\cal}$& Calibration data\\
    \hline
    $\mathcal{D}_{\test}$& Test data\\
    \hline
    $n_y $ & The number of calibration examples for class $y$\\
    \hline
    $V(X, Y)$ & Non-conformity scoring function \\
    \hline
    $\calC_{1-\alpha}(X_\test)$ & Prediction set for input $X_\test$\\
    \hline
    $\alpha$ & Target mis-coverage rate \\
    \hline
    $\widehat \alpha_y$ & Nominal mis-coverage rate for class $y$ \\
   \hline
\end{tabular}
\end{table}

\vspace{1.0ex}

\noindent {\bf Problem Setup of CP.}  
Let $V: \calX \times \calY \rightarrow \R$ denote a {\em non-conformity} scoring function to measure how different a new example is from old ones \cite{vovk2005algorithmic}.
It is employed to compare a given testing sample $(X_\test, Y_\test)$ with a set of calibration data $\mathcal{D}_\cal$: 
if the non-conformity score is large, 
then $(X_\test, Y_\test)$ conforms less to calibration samples. 
Prior work has considered the design of good non-conformity scoring functions, e.g., \cite{angelopoulos2021gentle,shafer2008tutorial,romano2020classification}.
In this paper, we focus on the scoring functions of {\it Adaptive Prediction Sets} (APS) proposed in \cite{romano2020classification} 
and {\it Regularized APS} (RAPS) proposed in \cite{angelopoulos2020uncertainty}
for classification based on the ordered probabilities of $f$ and true label rank $r_f(X, Y)$.
For the simplicity of notation, we denote the non-conformity score of the $i$-th calibration example as $V_i = V(X_i, Y_i)$. 

\vspace{1.0ex}

Given a input $X$, we sort the predicted probability for all classes \{$1, \cdots, K$\} of the classifier $f$ such that 
$1 \geq f(X)_{(1)} \geq \cdots \geq f(X)_{(K)} \geq 0$ are ordered statistics, where $f(X)_{(k)}$ denotes the $k$-th largest prediction.
The APS \cite{romano2020classification} score for a sample $(X, Y)$ is computed as follows: 
\begin{align*}
V(X, Y) 
=  
\sum_{l=1}^{r_f(X, Y)-1} f(X)_{(l)} + U \cdot f(X)_{(r_f(X, Y))} 
,
\end{align*}
where $U \in [0,1]$ is a uniform random variable to break ties. 
We also consider its regularized variant RAPS \cite{angelopoulos2020uncertainty}, which additionally includes a rank-based regularization $\lambda (r_f(X, Y) - k_{reg})^+$ to the above equation, where $(\cdot)^+ = \max\{0, \cdot\}$ denotes the hinge loss,  $\lambda$ and $ k_{reg}$ are two hyper-parameters.

\vspace{1.0ex}

For a target coverage $1-\alpha$, we find the corresponding empirical quantile on calibration data $\calD_\cal$ defined as
$$
\widehat Q_{1-\alpha}
= 
\min\Big\{ t : \sum_{i=1}^n \frac{1}{n} \cdot \indicator[ V_i \leq t ] \geq 1 - \alpha \Big\}
,
$$
which can be determined by finding the $\lceil (1-\alpha) (1+n) \rceil$-th smallest value of $\{V_i\}_{i=1}^n$.
The prediction set of a testing input $X_\test$ can be constructed by thresholding with $\widehat Q_{1-\alpha}$:
$$
\widehat \calC_{1-\alpha}(X_\test)
=
\{ y \in \calY : V(X_\test, y) \leq \widehat Q_{1-\alpha} \}.
$$
Therefore, $\widehat \calC_{1-\alpha}$ gives a {\it marginal coverage} guarantee \cite{romano2020classification,angelopoulos2020uncertainty}:
$
\P_{(X, Y) \sim \calP}\{ Y \in \widehat \calC_{1-\alpha}(X) \} 
\geq 
1 - \alpha 
.
$
To achieve the {\it class-conditional coverage}, standard CCP \cite{vovk2012conditional} uniformly iterates the class-wise thresholding subroutine with the class-wise quantiles $\{ \widehat Q_{1-\alpha}^\class(y) \}_{y \in \calY}$:
\begin{align}\label{eq:CCP}
& ~~~
\widehat \calC^\CCP_{1-\alpha}(X_\test)
=
\{ y \in \calY: V(X_\test, y) \leq \widehat Q_{1-\alpha}^\class(y) \}, 
\\
& \qquad \qquad
\text{where }
\widehat Q_{1-\alpha}^\class(y)
=
\min\Big\{ t : \sum_{i \in \calI_y} \frac{1}{n_y} \cdot \indicator[ V_i \leq t ] \geq 1 - \alpha \Big\}
.
\nonumber
\end{align}
Specifically, CCP pairs $X_\test$ with each candidate class label $y$, and includes $y$ in the prediction set $\widehat \calC^\CCP_{1-\alpha}(X_\test)$ if $V(X_\test, y) \leq \widehat Q_{1-\alpha}^\class(y)$ holds.
After going through all candidate class labels $y \in \calY$, it achieves the class-wise coverage for any $y \in \calY$ \cite{vovk2012conditional,angelopoulos2021gentle}:
\begin{align}\label{eq:class_conditional_coverage}
\P_{ (X, Y) \sim \calP } \{ Y \in \widehat \calC^\CCP_{1-\alpha}(X) | Y = y \} 
\geq 
1 - \alpha
.
\end{align}

CCP produces large prediction sets which are not useful in practice. Therefore, our goal is to develop a provable CP method that provides class-conditional coverage and constructs smaller prediction sets than those from CCP. We summarize all the notations in Table \ref{tab:notation}.

\section{ Rank Calibrated Class-Conditional CP }
\label{section:RC3P}

We first explain the proposed {\em Rank Calibrated Class-conditional Conformal Prediction (\newCP)} algorithm and present its model-agnostic coverage guarantee.
Next, we provide the theoretical analysis for the provable improvement of predictive efficiency of \newCP~over the CCP method. 

\subsection{ Algorithm and Model-Agnostic Coverage Analysis }
\label{subsection:RC3P_algorithm}

We start with the motivating discussion about the potential drawback of the standard CCP method in terms of {\it predictive efficiency}.
Equation (\ref{eq:CCP}) shows that, for a given test input $X_\test$, CCP likely contains some uncertain labels due to the uniform iteration over each class label $y \in \calY$ to check if $y$ should be included into the prediction set or not.
For example, given a class label $y$ and two test samples $X_1, X_2$, suppose their APS scores are $V(X_1, y) = 0.9, V(X_2, y) = 0.8$, with ranks $r_f(X_1, y) = 1, r_f(X_2, y) = 5$.
Furthermore, if $\widehat Q_{1-\alpha}^\class(y) = 0.85$, then by (\ref{eq:CCP}) for CCP, we know that $y \notin \widehat \calC^\CCP_{1-\alpha}(X_1)$ and $y \in \widehat \calC^\CCP_{1-\alpha}(X_2)$, even though $f(X_1)$ ranks $y$ at the \#1 class label for $X_1$ with a very high confidence $f(X_1)_y = 0.9$ and CCP can still achieve the valid class-conditional coverage.
We argue that, the principle of CCP to scans all $y \in \calY$ uniformly can easily result in large prediction sets, which is detrimental to the effectiveness of human-ML collaborative systems \cite{babbarijcai2022utility,straitouriicml2023improving}.

\vspace{1.0ex}

Consequently, to improve the predictive efficiency of CCP (i.e., reduce prediction set sizes), it is reasonable to include label rank information in the calibration procedure to adjust the distribution of non-conformity scores for predictive efficiency.
As mentioned in the previous sections, better scoring functions have been proposed to improve the predictive efficiency for marginal coverage, e.g., RAPS. 
However, directly applying RAPS for class-wise coverage presents challenges: 1) tuning its hyper-parameters for each class requires extra computational overhead, and 2) fixing its hyper-parameters for all classes overlooks the difference between distributions of different classes.
Moreover, for the approximate class-conditional coverage achieved by cluster CP \cite{ding2024class}, it still requires some assumptions on the underlying model (i.e., it is not fully model-agnostic).

\begin{algorithm}[t]

    \caption{ \newCP~ Method for Class-Conditional CP} 
    \label{alg:RC3P}
    \begin{algorithmic}[1]
    
    \STATE \textbf{Input}: Mis-coverage rate $\alpha \in (0, 1)$., top-$k$ errors $\epsilon_y^k$ for all classes and ranks $y, k \in \{1,\cdots, K\}$
    
    \STATE Randomly split data into train $\calD_\tr$ and calibration $\calD_\cal$
    and train the classifier $f$ on $\calD_\tr$
    
    \FOR{ $y \in \{1,\cdots, K\}$ }
    \label{alg:CCP:line:for_class}

        \STATE Compute $\{V_i\}_{i=1}^{n_y}$ for all $(X_i, Y_i) \in \calD_\cal$ such that $Y_i = y$

        \STATE Configure calibrated label rank $\widehat k(y)$ and nominal error $\widehat \alpha_y$:

        \STATE \quad Option I (model-agnostic coverage): 
        \\
        \quad \qquad
        $\widehat k(y) \in \{k: \epsilon_y^k < \alpha \}, ~~
        0 \leq \widehat \alpha_y \leq \alpha - \epsilon_y^{\widehat k(y)}$, 
        as per Eq (\ref{eq:feasible_configuration})

        \STATE \quad Option II (model-agnostic coverage + improved predictive efficiency): 
        \\
        \quad \qquad
        $\widehat k(y) = \min\{k: \epsilon_y^k < \alpha \}
        ,~~
        \widehat \alpha_y = \alpha - \epsilon_y^{\widehat k(y)}$,
        as per Eq (\ref{eq:configure_k_for_efficient_RC3P})
        
      \STATE $\widehat Q^\class_{1-\widehat \alpha_y}(y) \leftarrow \lceil (1-\widehat \alpha_y)(1 + n_y ) \rceil$-th smallest value in $\{V_i\}_{i =1}^{n_y}$ according to Eq (\ref{eq:CCP}) 

    \ENDFOR  
    \label{alg:CCP:line:endfor_class}  

    \STATE Construct $\widehat \calC^\newCP_{1-\alpha}(X_\test)$ with $\widehat Q^\class_{1-\widehat \alpha_y}(y)$ and $\widehat k(y)$ for a test input $X_\test$ using Eq (\ref{eq:RC3P_set_predictor})
    \label{alg:CCP:line:construct_prediction_set} 

    \end{algorithmic}
   
\end{algorithm}

\vspace{1.0ex}

Therefore, the key idea of our proposed \newCP~algorithm (outlined in Algorithm~\ref{alg:RC3P}) is to refine the class-wise calibration procedure using a label rank calibration strategy augmented to the standard conformal score calibration, to enable adaptivity to various classes.
Specifically, in contrast to CCP, \newCP~selectively activates the class-wise thresholding subroutine in (\ref{eq:CCP}) according to their class-wise top-$k$ error $\epsilon_y^k$ for class $y$.
\newCP~produces the prediction set for a given test input $X_\test$ with two calibration schemes (one for conformal score and another for label rank) as shown below:
\begin{align}
\label{eq:RC3P_set_predictor}
&
\widehat \calC^\newCP_{1-\alpha}(X_\test)
= \big\{ y \in \calY :
\underbrace{ 
V(X_\test, y) \leq \widehat Q^\class_{1-\widehat \alpha_y}(y) 
}_{ \text{conformal score calibration } }, ~
\underbrace{ 
r_f(X_\test, y) \leq \widehat k(y) 
}_{ \text{ label rank calibration } } 
\big\} 
,
\end{align} 
where $\widehat Q_{ 1 - \widehat \alpha_y}^\class(y)$ and $\widehat k(y)$ are score and label rank threshold for class $y$, respectively.
In particular, $\widehat k(y)$ controls the class-wise uncertainty adaptive to each class $y$ based on the top-$k$ error $\epsilon_y^{\widehat k(y)}$ of the classifier.
By determining $\widehat k(y)$, the top $k$ predicted class labels of $f(X_\test)$ will more likely cover the true label $Y_\test$,
making the augmented label rank calibration filter out the class labels $y$ that have a high rank (larger $r_f(X,y)$).
As a result, given all test input and label pairs $\{ (X_\test, y) \}_{y \in \calY}$, \newCP~performs score thresholding using class-wise quantiles only on a subset of reliable test pairs.

\vspace{1.0ex}

{\bf Determining $\widehat k(y)$ and $\widehat \alpha_y$ for model-agnostic valid coverage.}
For class $y$, intuitively, we would like a value for $\widehat k(y)$ such that the corresponding top-$\widehat k(y)$ error is smaller than $\alpha$, so that it is possible to guarantee valid coverage (recall $\P\{A, B\} = \P\{ A \} \cdot \P\{ B | A \}$).
Since a larger $\widehat k(y)$ gives a smaller $\epsilon_y^{\widehat k(y)}$ untill $\epsilon_y^{K} = 0$, 
it is guaranteed to find a value for $\widehat k(y)$, in which the corresponding $\epsilon_y^{\widehat k(y)} < \alpha$.
As a result, given all test input and label pairs $\{ (X_\test, y) \}_{y \in \calY}$, \newCP~performs score thresholding using class-wise quantiles only on a subset of reliable test pairs and filters out the class labels $y$ that have a high rank (larger $r_f(X,y)$).
The following result formally shows the principle to configure $\widehat k(y)$ and $\widehat \alpha_y$ to guarantee the class-wise coverage that is agnostic to the underlying model.
\begin{theorem}
\label{theorem:class_conditional_coverage_RC3P}
(Class-conditional coverage of \newCP) 
Suppose that selecting $\widehat k(y)$ values result in the class-wise top-$k$ error $\epsilon_y^{\widehat k(y)}$ for each class $y \in \calY$.
For a target class-conditional coverage $1-\alpha$, if we set $\widehat \alpha_y$ and $\widehat k(y)$ in \newCP~ (\ref{eq:RC3P_set_predictor}) in the following ranges:
\begin{align}\label{eq:feasible_configuration}
\widehat k(y) \in \{k: \epsilon_y^k < \alpha \}, \quad
0 \leq \widehat \alpha_y \leq \alpha - \epsilon_y^{\widehat k(y)}, 
\end{align} 
then \newCP~can achieve the class-conditional coverage for every $y \in \calY$:
\begin{align*}
\P_{ (X, Y) \sim \calP } \{ Y \in \widehat \calC^\newCP_{1-\alpha}(X) | Y = y \} 
\geq 
1 - \alpha
.
\end{align*}
\end{theorem}

\subsection{ Analysis of Predictive Efficiency for RC3P }
\label{subsection:RC3P_efficiency_analysis}

We further analyze the predictive efficiency of \newCP: under what conditions \newCP~can produce a smaller expected prediction set size compared to CCP, when both achieve the same ($1-\alpha$)-class-conditional coverage.
We investigate how to choose the value of $\widehat \alpha_y$ and $\widehat k(y)$ from the feasible ranges in (\ref{eq:feasible_configuration}) to achieve the best predictive efficiency using \newCP.

\begin{lemma}
\label{lemma:RC3P_improved_efficiency}
(Trade-off condition for improved predictive efficiency of \newCP)
Suppose $\widehat \alpha_y$ and $\widehat k(y)$ satisfy (\ref{eq:feasible_configuration}) in Theorem \ref{theorem:class_conditional_coverage_RC3P}.
If the following inequality holds for any $y \in \calY$:
\begin{align}\label{eq:smaller_ps_condition}
&
\P_\Xtest \big[ V(\Xtest, y) \leq \widehat Q^\class_{1-\widehat\alpha}(y), ~ r_f(\Xtest, y) \leq \widehat k(y) \big]  
\leq
\P_\Xtest \big[ V(\Xtest, y) \leq \widehat Q^\class_{1-\alpha}(y) \big] 
,
\end{align}
then \newCP~produces smaller expected prediction sets than CCP, i.e., 
\begin{align*}
\E_\Xtest [ | \widehat \calC^\newCP_{1-\widehat \alpha}(\Xtest) | ]
\leq 
\E_\Xtest [ | \widehat \calC^\CCP_{1-\alpha}(\Xtest) |  ] 
.
\end{align*}
\end{lemma}

\vspace{1.0ex}

{\bf Remark.}
The above result demonstrates that when both \newCP~and CCP achieve the target $1-\alpha$ class-conditional coverage, under the condition of (\ref{eq:smaller_ps_condition}), \newCP~produces smaller prediction sets than CCP.
In fact, this condition implies that the combined (conformity score and label rank) calibration of \newCP~tends to include less labels with high rank or low confidence from the classifier.
In contrast, the CCP method tends to include relatively more uncertain labels into the  prediction set, where their ranks are high and the confidence of the classifier is low.
Now we can interpret the condition (\ref{eq:smaller_ps_condition}) by defining a condition number, termed as 
$\sigma_y$:
\begin{align}
\label{eq:sigma_y_defination}
\sigma_y 
= \frac{ \P_\Xtest \Big[ V(\Xtest, y) \leq \widehat Q^\class_{1-\hat\alpha}(y), ~ r_f(\Xtest, y) \leq \widehat k(y) \Big] }{ \P_\Xtest \Big[ V(\Xtest, y) \leq \widehat Q^\class_{1-\alpha}(y) \Big] }
.
\end{align}

\vspace{1.0ex}

In other words, if we can verify that $\sigma_y \leq 1$ for all $y$, then \newCP~can improve the predictive efficiency over CCP.
Furthermore, if $\sigma_y$ is fairly small, then the efficiency improvement can be even more significant.
To verify this condition, our comprehensive experiments (Section \ref{sec:results_and_discussion}, Figure \ref{fig:condition_number_sigma_y_exp_0.1}) show that $\sigma_y$ values are much smaller than $1$ on real-world data. These results demonstrate the practical utility of our theoretical analysis to produce small prediction sets using \newCP. 
Note that the reduction in prediction set size of \newCP~over CCP is proportional to how small the $\sigma_y$ values are. 

\begin{theorem}
\label{theorem:RC3P_efficiency_condition}
(Conditions of improved predictive efficiency for \newCP)
Define $D = \P[ r_f(X, y) \leq \widehat k(y) | Y \neq y ]$,
and $\bar r_f(X, y) = \lfloor \frac{ r_f(X, y) + 1 }{ 2 } \rfloor$.
Denote $B = \P[ f(X)_{ ( \bar r_f(X, y) ) } \leq \widehat Q_{1-\alpha}^\class(y) | Y \neq y ]$ if $V$ is APS, 
or $B = \P[ f(X)_{ ( \bar r_f(X, y) ) } + \lambda \leq \widehat Q_{1-\alpha}^\class(y) | Y \neq y ]$ if $V$ is RAPS.
If $B - D \geq \frac{ p_y }{ 1 - p_y } ( \alpha - \epsilon_y^{\widehat k(y)} )$,
then $\sigma_y \leq 1$.
\end{theorem}

{\bf Remark.}
The above result further analyzes when the condition in Eq (\ref{eq:smaller_ps_condition}) of Lemma \ref{lemma:RC3P_improved_efficiency} (or equivalently, $\sigma_y \leq 1$) holds to guarantee the improved predictive efficiency.
Specifically, the condition $B - D \geq \frac{ p_y }{ 1 - p_y } ( \alpha - \epsilon_y^{\widehat k(y)} )$ of Theorem \ref{theorem:RC3P_efficiency_condition} can be realized by two ways: (i) making LHS $B - D$ as large as possible; (ii) making the RHS $\frac{ p_y }{ 1 - p_y } ( \alpha - \epsilon_y^{\widehat k(y)} )$ as small as possible.
To this end, we can set Line 7 in Algorithm \ref{alg:RC3P} in the following way:
\begin{align}\label{eq:configure_k_for_efficient_RC3P}
\widehat k(y) = \min\{k: \epsilon_y^k < \alpha \}
,
\quad 
\widehat \alpha_y = \alpha - \epsilon_y^{\widehat k(y)}.
\end{align}
Therefore, this setting ensures $\sigma_y \leq 1$ and as a result improved predictive efficiency.

\section{Experiments and Results}
\label{section:experiments}
\vspace{-1.0ex}

We present the empirical evaluation of the  \newCP~algorithm and demonstrate its effectiveness in achieving class-conditional coverage to produce small prediction sets. We conduct experiments using two baselines (\texttt{CCP} and \texttt{Cluster-CP}), four datasets (each with three imbalance types and five imbalance ratios), and two machine learning models (trained for $50$ epochs and $200$ epochs, with $200$ epochs being our main experimental setting). Additionally, we use two scoring functions (APS and RAPS) and set three different $\alpha$ values ($\alpha \in {0.1, 0.05, 0.01}$, with $\alpha = 0.1$ as our main setting).

\begin{table*}[!t]
\centering
\caption{
\textbf{
Imalanced classification data experiment on CIFAR-10, CIFAR-100, mini-ImageNet, Food-101.}
APSS results comparing \texttt{CCP}, \texttt{Cluster-CP}, and \texttt{\newCP}~with ResNet-20 model trained with $200 $ epochs under different imbalance types and ratios when $\alpha = 0.1$.
For a fair comparison of APSS, we set UCR of \newCP~the same as or smaller (more restrictive) than that of \texttt{CCP} and \texttt{Cluster-CP} under $0.16$ on CIFAR-10 and $0.03$ on other datasets. The specified UCR values are in Table \ref{tab:overall_comparison_cluster_four_datasets} and \ref{tab:overall_comparison_RAPS_four_datasets} of Appendix \ref{appendix:aps} and \ref{appendix:raps}.
The APSS results show that \texttt{\newCP}~significantly outperforms \texttt{CCP} and \texttt{Cluster-CP} in terms of average prediction set size with $24.47\%$ (four datasets) or $32.63\%$ (excluding CIFAR-10) reduction over $\min \{ \texttt{CCP}, \texttt{cluster-CP}\}$. 
}
\label{tab:overall_comparison_four_imbalanced_datasets}
\resizebox{\textwidth}{!}{
\begin{NiceTabular}{@{}cc!{~}cc!{~}cc!{~}cc@{}} \toprule 
\Block{2-1}{Conformity\\ Score}  & \Block{2-1}{Methods} & \Block{1-2}{\EXP} & & \Block{1-2}{\POLY} & & \Block{1-2}{\MAJ} \\ 
 & & $\rho$ = 0.5 & $\rho$ = 0.1 & $\rho$ = 0.5 & $\rho$ = 0.1 & $\rho$ = 0.5 & $\rho$ = 0.1\\ 
\midrule 
\Block{1-*}{CIFAR-10}
\\
\midrule
\Block{3-1}{APS} 
& CCP     & \textbf{1.555 $\pm$ 0.010} & \textbf{1.855 $\pm$ 0.014}
& \textbf{1.538 $\pm$ 0.010} & \textbf{1.776 $\pm$ 0.012}
& \textbf{1.840 $\pm$ 0.020} & \textbf{2.629 $\pm$ 0.013}
\\ 
& Cluster-CP    & 1.714 $\pm$ 0.018 & 2.162 $\pm$ 0.015
& 1.706 $\pm$ 0.014 & 1.928 $\pm$ 0.013
& 1.948 $\pm$ 0.023 & 3.220 $\pm$ 0.020
\\ 
& \textbf{\newCP} & \textbf{1.555 $\pm$ 0.010} & \textbf{1.855 $\pm$ 0.014}
& \textbf{1.538 $\pm$ 0.010} & \textbf{1.776 $\pm$ 0.012}
& \textbf{1.840 $\pm$ 0.020} & \textbf{2.629 $\pm$ 0.013}
\\
\midrule
\Block{3-1}{RAPS}
& CCP     & \textbf{1.555 $\pm$ 0.010} & \textbf{1.855 $\pm$ 0.014}
& \textbf{1.538 $\pm$ 0.010} & \textbf{1.776 $\pm$ 0.012}
& \textbf{1.840 $\pm$ 0.020} & \textbf{2.632 $\pm$ 0.012}
\\ 
& Cluster-CP     & 1.714 $\pm$ 0.018 & 2.162 $\pm$ 0.015
& 1.706 $\pm$ 0.014 & 1.929 $\pm$ 0.013
& 1.787 $\pm$ 0.019 & 2.968 $\pm$ 0.024
\\ 
& \textbf{\newCP} & \textbf{1.555 $\pm$ 0.010} & \textbf{1.855 $\pm$ 0.014}
& \textbf{1.538 $\pm$ 0.010} & \textbf{1.776 $\pm$ 0.012}
& \textbf{1.840 $\pm$ 0.020} & \textbf{2.632 $\pm$ 0.012}
\\
\midrule
\Block{3-1}{HPS}
& CCP     & \textbf{1.144 $\pm$ 0.005} & \textbf{1.324 $\pm$ 0.007}
& \textbf{1.137 $\pm$ 0.003} & \textbf{1.243 $\pm$ 0.005}
& \textbf{1.272 $\pm$ 0.008} & \textbf{1.936 $\pm$ 0.010}
\\ 
& Cluster-CP    & 1.214 $\pm$ 0.008 & 1.508 $\pm$ 0.010
& 1.211 $\pm$ 0.004 & 1.354 $\pm$ 0.005
& 1.336 $\pm$ 0.009 & 2.312 $\pm$ 0.025
\\ 
& \textbf{\newCP} & \textbf{1.144 $\pm$ 0.005} & \textbf{1.324 $\pm$ 0.007}
& \textbf{1.137 $\pm$ 0.003} & \textbf{1.243 $\pm$ 0.005}
& \textbf{1.272 $\pm$ 0.008} & \textbf{1.936 $\pm$ 0.010}
\\
\midrule
\Block{1-*}{CIFAR-100}
\\
\midrule
\Block{3-1}{APS} 
& CCP     & 44.224 $\pm$ 0.341 & 50.969 $\pm$ 0.345
& 49.889 $\pm$ 0.353 & 64.343 $\pm$ 0.237
& 44.194 $\pm$ 0.514 & 64.642 $\pm$ 0.535
\\ 
& Cluster-CP     & 29.238 $\pm$ 0.609 & 37.592 $\pm$ 0.857
& 38.252 $\pm$ 0.353 & 52.391 $\pm$ 0.595
& 31.518 $\pm$ 0.335 & 50.883 $\pm$ 0.673
\\ 
& \textbf{\newCP} & \textbf{17.705 $\pm$ 0.004} & \textbf{21.954 $\pm$ 0.005}
& \textbf{23.048 $\pm$ 0.008} & \textbf{33.185 $\pm$ 0.005}
& \textbf{18.581 $\pm$ 0.007} & \textbf{32.699 $\pm$ 0.005}
\\
\midrule
\Block{3-1}{RAPS}
& CCP     & 44.250 $\pm$ 0.342 & 50.970 $\pm$ 0.345
& 49.886 $\pm$ 0.353 & 64.332 $\pm$ 0.236
& 48.343 $\pm$ 0.353 & 64.663 $\pm$ 0.535
\\ 
& Cluster-CP     & 29.267 $\pm$ 0.612 & 37.795 $\pm$ 0.862
& 38.258 $\pm$ 0.320 & 52.374 $\pm$ 0.592
& 31.513 $\pm$ 0.325 & 50.379 $\pm$ 0.684
\\ 
& \textbf{\newCP} & \textbf{17.705 $\pm$ 0.004} & \textbf{21.954 $\pm$ 0.005}
& \textbf{23.048 $\pm$ 0.008} & \textbf{33.185 $\pm$ 0.005}
& \textbf{18.581 $\pm$ 0.006} & \textbf{32.699 $\pm$ 0.006}
\\
\midrule
\Block{3-1}{HPS}
& CCP     & 41.351 $\pm$ 0.242 & 49.469 $\pm$ 0.344
& 48.063 $\pm$ 0.376 & 63.963 $\pm$ 0.277
& 46.125 $\pm$ 0.351 & 64.371 $\pm$ 0.564
\\ 
& Cluster-CP     & 27.566 $\pm$ 0.555 & 35.528 $\pm$ 0.979
& 36.101 $\pm$ 0.565 & 51.333 $\pm$ 0.776
& 29.323 $\pm$ 0.363 & 50.519 $\pm$ 0.679
\\ 
& \textbf{\newCP} & \textbf{20.363 $\pm$ 0.006} & \textbf{25.212 $\pm$ 0.010}
& \textbf{25.908 $\pm$ 0.007} & \textbf{36.951 $\pm$ 0.018}
& \textbf{21.149 $\pm$ 0.006} & \textbf{35.606 $\pm$ 0.005}
\\
\midrule
\Block{1-*}{mini-ImageNet}
\\
\midrule 
\Block{3-1}{APS} 
& CCP     & 26.676 $\pm$ 0.171 & 26.111 $\pm$ 0.194
& 26.626 $\pm$ 0.133 & 26.159 $\pm$ 0.208
& 27.313 $\pm$ 0.154 & 25.629 $\pm$ 0.207
\\ 
& Cluster-CP     & 25.889 $\pm$ 0.301 & 25.253 $\pm$ 0.346
& 26.150 $\pm$ 0.393 & 25.633 $\pm$ 0.268
& 26.918 $\pm$ 0.241 & 25.348 $\pm$ 0.334
\\ 
& \textbf{\newCP} & \textbf{18.129 $\pm$ 0.003} & \textbf{17.082 $\pm$ 0.002}
& \textbf{17.784 $\pm$ 0.003} & \textbf{17.465 $\pm$ 0.003}
& \textbf{18.111 $\pm$ 0.002} & \textbf{17.167 $\pm$ 0.004}
\\
\midrule
\Block{3-1}{RAPS}
& CCP     & 26.756 $\pm$ 0.178 & 26.212 $\pm$ 0.199
& 26.689 $\pm$ 0.142 & 26.248 $\pm$ 0.219
& 27.397 $\pm$ 0.162 & 25.725 $\pm$ 0.214
\\ 
& Cluster-CP     & 26.027 $\pm$ 0.325 & 25.415 $\pm$ 0.289
& 26.288 $\pm$ 0.407 & 25.712 $\pm$ 0.315
& 26.969 $\pm$ 0.305 & 25.532 $\pm$ 0.350
\\ 
& \textbf{\newCP} & \textbf{18.129 $\pm$ 0.003} & \textbf{17.082 $\pm$ 0.002}
& \textbf{17.784 $\pm$ 0.003} & \textbf{17.465 $\pm$ 0.003}
& \textbf{18.111 $\pm$ 0.002} & \textbf{17.167 $\pm$ 0.004}
\\
\midrule
\Block{3-1}{HPS}
& CCP     & 24.633 $\pm$ 0.212 & 24.467 $\pm$ 0.149
& 24.379 $\pm$ 0.152 & 24.472 $\pm$ 0.167
& 25.449 $\pm$ 0.196 & 23.885 $\pm$ 0.159
\\ 
& Cluster-CP     & 23.911 $\pm$ 0.322 & 24.023 $\pm$ 0.195
& 24.233 $\pm$ 0.428 & 23.263 $\pm$ 0.295
& 24.987 $\pm$ 0.319 & 23.323 $\pm$ 0.378
\\ 
& \textbf{\newCP} & \textbf{17.830 $\pm$ 0.104} & \textbf{17.036 $\pm$ 0.014}
& \textbf{17.684 $\pm$ 0.062} & \textbf{17.393 $\pm$ 0.013}
& \textbf{18.024 $\pm$ 0.049} & \textbf{17.086 $\pm$ 0.059}
\\
\midrule
\Block{1-*}{Food-101}
\\
\midrule 
\Block{3-1}{APS} 
& CCP     & 27.022 $\pm$ 0.192 & 30.900 $\pm$ 0.170
& 30.943 $\pm$ 0.119 & 35.912 $\pm$ 0.105
& 27.415 $\pm$ 0.194 & 36.776 $\pm$ 0.132
\\ 
& Cluster-CP     & 28.953 $\pm$ 0.333 & 33.375 $\pm$ 0.377
& 33.079 $\pm$ 0.393 & 38.301 $\pm$ 0.232
& 30.071 $\pm$ 0.412 & 39.632 $\pm$ 0.342
\\ 
& \textbf{\newCP} & \textbf{18.369 $\pm$ 0.004} & \textbf{21.556 $\pm$ 0.006}
& \textbf{21.499 $\pm$ 0.003} & \textbf{25.853 $\pm$ 0.004}
& \textbf{19.398 $\pm$ 0.006} & \textbf{26.585 $\pm$ 0.004}
\\
\midrule
\Block{3-1}{RAPS}
& CCP    & 27.022 $\pm$ 0.192 & 30.900 $\pm$ 0.170
& 30.966 $\pm$ 0.125 & 35.940 $\pm$ 0.111
& 27.439 $\pm$ 0.203 & 36.802 $\pm$ 0.138
\\ 
& Cluster-CP     & 28.953 $\pm$ 0.333 & 33.375 $\pm$ 0.377
& 33.337 $\pm$ 0.409 & 38.499 $\pm$ 0.216
& 29.946 $\pm$ 0.407 & 39.529 $\pm$ 0.306
\\ 
& \textbf{\newCP} & \textbf{18.369 $\pm$ 0.004} & \textbf{21.556$\pm$ 0.006}
& \textbf{21.499 $\pm$ 0.003} & \textbf{25.853 $\pm$ 0.004}
& \textbf{19.397 $\pm$ 0.006} & \textbf{26.585 $\pm$ 0.004}
\\
\midrule
\Block{3-1}{HPS}
& CCP     & 26.481 $\pm$ 0.142 & 30.524 $\pm$ 0.152
& 30.787 $\pm$ 0.099 & 35.657 $\pm$ 0.107
& 26.826 $\pm$ 0.163 & 36.518 $\pm$ 0.122
\\ 
& Cluster-CP     & 29.347 $\pm$ 0.288 & 33.806 $\pm$ 0.513
& 33.407 $\pm$ 0.345 & 38.956 $\pm$ 0.242
& 29.606 $\pm$ 0.436 & 39.880 $\pm$ 0.318
\\ 
& \textbf{\newCP} & \textbf{18.337 $\pm$ 0.004} & \textbf{21.558$\pm$ 0.006}
& \textbf{21.477 $\pm$ 0.003} & \textbf{25.853 $\pm$ 0.005}
& \textbf{19.396 $\pm$ 0.008} & \textbf{26.584 $\pm$ 0.003}
\\
\bottomrule 
\end{NiceTabular}
}
\end{table*}

\vspace{-1.5ex}
\subsection{Experimental Setup}
\label{section: Experimental setup}

\noindent \textbf{Classification datasets.}
We consider four datasets: CIFAR-10, CIFAR-100 \cite{krizhevsky2009learning}, mini-ImageNet \cite{vinyals2016matching}, and Food-101 \cite{bossard2014food} by using the standard training and validation split. 
We employ the same methodology as \cite{menon2020long,cao2019learning,cui2019class} to create an imbalanced long-tail setting for each dataset as a harder challenge: 1) We use the original training split as a training set for training $f$ with training samples ($n_{tr}$ is defined as the number of training samples), and randomly split the original (balanced) validation set into calibration samples and testing samples. 
2) We define an imbalance ratio $\rho$, the ratio between the sample size of the smallest and largest class: $\rho = \frac{\min_i \{\text{\# samples in class}~ i\}}{\max_i \{\text{\# samples in class}~ i\}}$. 3) For each training set, we create three different imbalanced distributions using three decay types over the class indices $c\in\{1, \cdots, K\}$: 
(a) An exponential-based decay (\EXP) with $\frac{n_{tr}}{K}\times \rho^{\frac{c}{K}}$  examples in class $c$, 
(b) A polynomial-based decay (\POLY) with $\frac{n_{tr}}{K}\times\frac{1}{\sqrt{\frac{c}{10\rho}+1}}$ examples in class $c$, and 
(c) A majority-based decay (\MAJ) with $\frac{n_{tr}}{K}\times \rho$ examples in classes $c>1$. 
We keep the calibration and test set balanced and unchanged. 
We provide an illustrative example of the three decay types in Appendix (Section \ref{subsection:appendix:illustration_imbalanced_data}, Figure \ref{fig:imbalanced_decay_examples}).
Towards a more complete comparison, we also employ balanced datasets. 
Following Cluster-CP\footnote{https://github.com/tiffanyding/class-conditional-conformal/tree/main}, we employ CIFAR-100, Places365 \cite{zhou2017places}, iNaturalist\cite{van2018inaturalist}, and ImageNet\cite{russakovsky2015imagenet}.

\noindent \textbf{Deep neural network models.} 
We consider ResNet-20 \cite{he2016deep} as the main architecture to train classifiers for imbalanced classification datasets. 
To 
handle imbalanced data, we employ the training algorithm ``LDAM'' proposed by \cite{cao2019learning} that assigns different margins to classes, where larger margins are assigned to minority classes in the loss function.
We follow the training strategy in \cite{cao2019learning} where all models are trained with $200$ epochs.
The class-wise performance with three imbalance types and imbalance ratios $\rho = 0.5$ and $\rho = 0.1$ on four datasets are evaluated (see Appendix \ref{appendix:training}).
We also train models with $50$ epochs and the corresponding APSS results are reported in Appendix \ref{appendix:epochs}. 

For balanced datsets, we follow the same settings from Cluster-CP, which uses IMAGENET1K\_V2 as pre-trained weights from PyTorch \cite{paszke2019pytorch} and then fine-tune models with ResNet-50 for all datasets except ImageNet. 
For ImageNet, we use SimCLR-v2 \cite{chen2020big} 
as training models.

\begin{table*}[!ht]
\centering
\caption{
\textbf{Balanced experiment on CIFAR-100, Places365, iNaturalist, ImageNet.}
The models are pre-trained.
UCR is controlled to $\leq 0.05$.
RC3P significantly outperforms the best baseline with $32.826\%$ reduction in APSS ($\downarrow$ better) on average over $\min \{ \texttt{CCP}, \texttt{cluster-CP} \}$.
}
\label{tab:overall_comparison_balanced_datasets}
\vspace{-0.1in}
\resizebox{\textwidth}{!}{
\begin{NiceTabular}{@{}ccc!{~}c!{~}c!{~}c!{~}c@{}} \toprule 
Conformity Score  & Measure & Methods & CIFAR-100 & Places365 & iNaturalist & ImageNet\\ 
\midrule 
\Block{6-1}{APS} 
& \Block{3-1}{UCR} 
& CCP     
& 0.045 $\pm$ 0.008 
& 0.012 $\pm$ 0.002 
& 0.016 $\pm$ 0.001 
& 0.036 $\pm$ 0.001 
\\ 
& & Cluster-CP    
& 0.023 $\pm$ 0.006 
& 0.029 $\pm$ 0.003 
& 0.026 $\pm$ 0.003 
& 0.031 $\pm$ 0.002
\\ 
& & \textbf{\newCP} 
& \textbf{0.006 $\pm$ 0.003} 
& \textbf{0.003 $\pm$ 0.001} 
& \textbf{0.008 $\pm$ 0.001} 
& \textbf{0.023 $\pm$ 0.001} 
\\
\cline{2-7}
& \Block{3-1}{APSS} 
& CCP 
& 30.467 $\pm$ 0.307 
& 19.698 $\pm$ 0.050
& 18.802 $\pm$ 0.023
& 101.993 $\pm$ 0.812
\\
& & Cluster-CP  
& 32.628 $\pm$ 0.720
& 20.818 $\pm$ 0.173
& 23.467 $\pm$ 0.494
& 66.285 $\pm$ 1.433
\\
& & \textbf{\newCP} 
& \textbf{12.551 $\pm$ 0.005}
& \textbf{13.772 $\pm$ 0.005}
& \textbf{12.736 $\pm$ 0.006}
& \textbf{6.518 $\pm$ 0.001}
\\
\midrule
\Block{6-1}{RAPS}
& \Block{3-1}{UCR} 
& CCP     
& 0.043 $\pm$ 0.006 
& 0.013 $\pm$ 0.002 
& 0.016 $\pm$ 0.020 
& 0.038 $\pm$ 0.020 
\\ 
& & Cluster-CP     
& 0.016 $\pm$ 0.005 
& 0.036 $\pm$ 0.002 
& 0.027 $\pm$ 0.003 
& 0.046 $\pm$ 0.004 
\\ 
& & \textbf{\newCP} 
& \textbf{0.002 $\pm$ 0.001} 
& \textbf{0.002 $\pm$ 0.001} 
& \textbf{0.006 $\pm$ 0.001} 
& \textbf{0.017 $\pm$ 0.001} 
\\
\cline{2-7}
& \Block{3-1}{APSS} 
& CCP 
& 26.135 $\pm$ 0.308
& 15.694 $\pm$ 0.049
& 14.812 $\pm$ 0.042
& 37.748 $\pm$ 0.304
\\
& & Cluster-CP  
& 28.084 $\pm$ 0.609
& 16.750 $\pm$ 0.143
& 23.964 $\pm$ 0.419
& 16.155 $\pm$ 1.241
\\
& & \textbf{\newCP} 
& \textbf{12.586 $\pm$ 0.002}
& \textbf{14.192 $\pm$ 0.001}
& \textbf{13.251 $\pm$ 0.001}
& \textbf{6.560 $\pm$ 0.002}
\\
\midrule
\Block{6-1}{HPS} 
& \Block{3-1}{UCR} 
& CCP     
& 0.034 $\pm$ 0.006 
& 0.015 $\pm$ 0.002 
& \textbf{0.018 $\pm$ 0.002} 
& 0.036 $\pm$ 0.002 
\\ 
& & Cluster-CP    
& 0.006 $\pm$ 0.003 
& 0.029 $\pm$ 0.004
& 0.035 $\pm$ 0.002 
& 0.039 $\pm$ 0.005 
\\ 
& & \textbf{\newCP} 
& \textbf{0.003 $\pm$ 0.002}
& \textbf{0.002 $\pm$ 0.001} 
& \textbf{0.018 $\pm$ 0.002} 
& \textbf{0.006 $\pm$ 0.000} 
\\
\cline{2-7}
& \Block{3-1}{APSS} 
& CCP 
& 25.898 $\pm$ 0.321
& 14.020 $\pm$ 0.044
& \textbf{9.751 $\pm$ 0.033}
& 24.384 $\pm$ 0.249
\\
& & Cluster-CP    
& 27.165 $\pm$ 0.600
& 14.530 $\pm$ 0.143
& 13.080 $\pm$ 0.374
& 8.810 $\pm$ 0.046
\\
& & \textbf{\newCP} 
& \textbf{12.558 $\pm$ 0.004}
& \textbf{13.919 $\pm$ 0.004}
& \textbf{9.751 $\pm$ 0.033}
& \textbf{6.533 $\pm$ 0.001}
\\
\bottomrule 
\end{NiceTabular}
}
\end{table*}

\noindent {\bf CP baselines.} 
We consider three CP methods: 
{\bf 1)} \texttt{CCP} which estimates class-wise score thresholds and produces prediction set using Equation (\ref{eq:CCP}); 
{\bf 2)} \texttt{Cluster-CP} \cite{ding2024class} 
that performs calibration over clusters to reduce prediction set sizes; and 
{\bf 3)} \texttt{\newCP}~that produces prediction set using Equation (\ref{eq:RC3P_set_predictor}).
All CP methods are built on the same classifier and non-conformity scoring function for a fair comparison.
We employ the three common scoring functions:
APS \cite{romano2020classification}, RAPS \cite{angelopoulos2020uncertainty}, and HPS \cite{sadinle2019least}. 
We set $\alpha = 0.1$ as our main experiment setting and also report other experiment results of different $\alpha$ values (See Appendix \ref{appendix:alphas}).
Meanwhile, the hyper-parameters for each baseline are tuned according to their recommended ranges based on the same criterion (see Appendix \ref{subsection:appendix:calibration_details}).
We repeat experiments over $10$ different random calibration-testing splits and report the mean and standard deviation.

\noindent \textbf{Evaluation methodology.}
We use the target coverage $1-\alpha$ = 90\% 
class-conditional coverage for \texttt{CCP}, \texttt{Cluster-CP}, and \newCP. 
We compute three evaluation metrics on the testing set: 

\vspace{0.75ex}

\textbullet~{\em Under Coverage Ratio (UCR).} 
\begin{align*}
\UCR 
:= 
\sum_{c \in [K]} \indicator \Big[\frac{
\E_{X_\test} \indicator [y \in \widehat \calC_{1-\alpha}(x) \text{ s.t. } y = c] 
}{
\E_{X_\test} \indicator [y = c] }  < 1- \alpha\Big] / K .
\end{align*}

\vspace{0.75ex}

\textbullet~{\em Average Prediction Set Size (APSS).} 
\begin{align*}
\APSS := \sum_{c \in [K]} \frac{ 
\E_{X_\test} \indicator [ y = c ] \cdot |  \widehat \calC_{1-\alpha}(x) | 
}{
\E_{X_\test} \indicator [y = c]
} / K .
\end{align*}

Note that coverage and predictive efficiency are two competing metrics in CP \cite{angelopoulos2020uncertainty}, e.g., achieving better coverage (resp. predictive efficiency) degenerates predictive efficiency (resp. coverage). 
Therefore, following the same strategy in \cite{fontana2023conformal}, 
we choose to control their UCR as the same level that is close to $0$ for a fair comparison over three class-conditional CP algorithms in terms of APSS.
Meanwhile, to address the gap between population values and empirical ones (e.g., quantiles with $\tilde O(1 / \sqrt{n_y})$ error bound, common to all CP methods \cite{vovk2012conditional,ghosh2023probabilistically,angelopoulos2021gentle}, or class-wise top-$k$ error $\epsilon_y^k$ with $\tilde O(1 / \sqrt{n_y})$ error bound \cite{mohri2018foundations}), we uniformly add $g/\sqrt{n_y}$ (the same order with the standard concentration gap) to inflate the nominal coverage $1-\alpha$ on each baseline and tune $g \in \{0.25, 0.5, 0.75, 1\}$ on the calibration dataset in terms of UCR. 
The detailed $g$ values of each method are displayed in Appendix \ref{subsection:appendix:calibration_details}.
In addition, the actual achieved UCR values are shown in the complete results (see Appendix \ref{appendix:aps}, \ref{appendix:raps}, and \ref{appendix:hps}).
For a complete evaluation, 
we add the experiments without controlling coverage on imbalanced datasets under the same setting and use the total under coverage gap (UCG) metric: 
\\
\vspace{0.75ex}
\textbullet~{\em Under Coverage Gap (UCG).} 
\begin{align*}
\UCG 
:= \sum_{c \in [K]} \max \bigg \{1-\alpha-\frac{\P [ Y \in \hat \calC(X), \text{s.t. Y=c} ] }{ \P[Y=c] },0 \bigg \}.
\end{align*}
Experiments with UCG metric evaluation are shown in the Appendix \ref{appendix:UCG}.

\subsection{Results and Discussion}
\label{sec:results_and_discussion}
\vspace{-0.75ex}

We list empirical results in Table \ref{tab:overall_comparison_four_imbalanced_datasets} for an overall comparison on four imbalanced datasets with $\rho=0.5, 0.1$ using all three training distributions (\EXP, \POLY and \MAJ) based on the considered APS, RAPS and HPS scoring functions.
Complete experiment results under more values of $\rho$ are in Appendix \ref{section:appendix:experiment_result}).
Results with APS, RAPS, and HPS scoring functions on balanced datasets are also summarized in Table \ref{tab:overall_comparison_balanced_datasets}.
We make the following two key observations: 
(i) \texttt{CCP}, \texttt{Cluster-CP}, and \texttt{\newCP}~can guarantee the class-conditional coverage (their UCRs are all close to $0$) 
for all settings; 
(ii) \texttt{\newCP}~significantly outperforms \texttt{CCP} and \texttt{Cluster-CP} in APSS on almost all imbalanced settings by reducing APSS with $24.47\%$ on all four datasets and $32.63\%$ on three datasets excluding CIFAR-10 compared with $\min \{\texttt{CCP}, \texttt{Cluster-CP} \}$ on average,
while for balanced settings, \texttt{\newCP}~still significantly outperforms the best baselines in terms of APSS with $32.826\%$ APSS reduction.

To investigate the challenge of imbalanced data and more importantly, how \newCP~significantly improves the APSS, we further conduct three careful experiments on imbalanced datasets.
First,
we report the histograms of class-conditional coverage and the corresponding histograms of prediction set size.
This experiment verifies that \texttt{\newCP}~derives significantly more class-conditional coverage above $1-\alpha$ and thus reduces the prediction set size.
Second, 
we visualize the normalized frequency of label rank included in prediction sets on testing datasets for all class-wise algorithms: \texttt{CCP}, \texttt{Cluster-CP}, and \texttt{\newCP}. 
The normalized frequency is defined as: 
$\P(k):=\frac{\E_{X_\test} \indicator[r_f(X_\test, y)=k, y \in \widehat \calC(x)]}{\sum^K_{k=1} \E_{X_\test} \indicator[r_f(X_\test, y)=k, y \in \widehat \calC(x)]}$.
Finally, 
we empirically verify the trade-off condition number $\{\sigma_y\}_{y=1}^K$ of Equation \ref{eq:sigma_y_defination} on calibration dataset to reveal the underlying reason for \texttt{\newCP}~producing smaller prediction sets over \texttt{CCP} with our standard training models (epoch $=200$). 
We also evaluate $\{\sigma_y\}_{y=1}^K$ with less trained models (epoch $=50$) on imbalanced datasets in Appendix \ref{subsection:appendix:complete experiment_results}.
Additionally, we also repeat all three experiments on balanced datasets (i.e., the histograms of class-conditional coverage and prediction set size, the normalized frequency of label rank included in prediction sets, and $\{\sigma_y\}_{y=1}^K$) in Appendix \ref{subsection:appendix:complete experiment_results_balanced}.
Below we discuss our experimental results and findings in detail.

\begin{figure*}[!ht]
    \centering
    \begin{minipage}{.24\textwidth}
        \centering
        (a) CIFAR-10
    \end{minipage}%
    \begin{minipage}{.24\textwidth}
        \centering
        (b) CIFAR-100
    \end{minipage}%
    \begin{minipage}{.24\textwidth}
        \centering
        (c) mini-ImageNet
    \end{minipage}%
    \begin{minipage}{.24\textwidth}
        \centering
        (d) Food-101
    \end{minipage}

    \subfigure{
        \begin{minipage}{0.23\linewidth}
            \includegraphics[width=\linewidth]{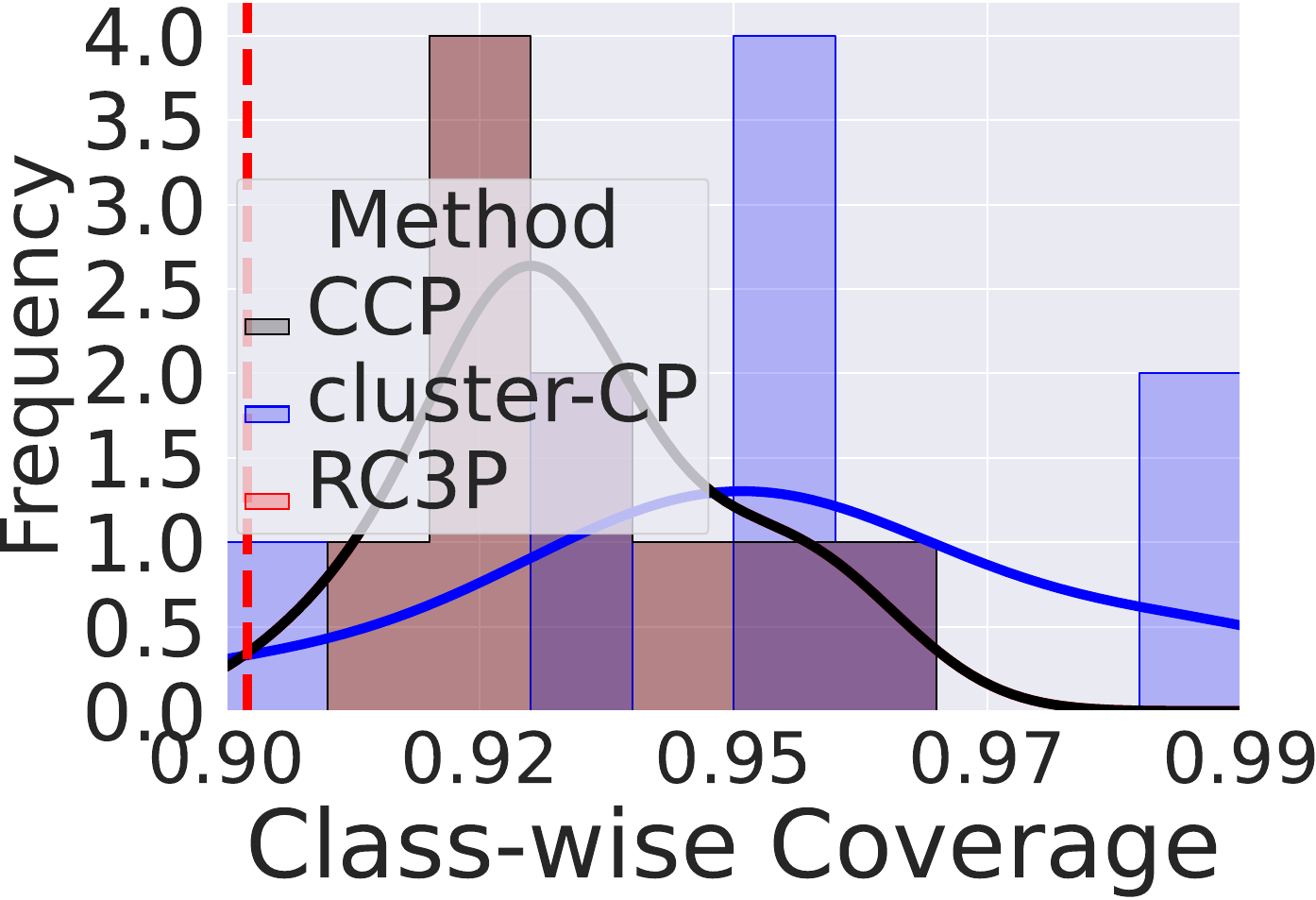}
        \end{minipage}
    }
    \subfigure{
        \begin{minipage}{0.23\linewidth}
            \includegraphics[width=\linewidth]{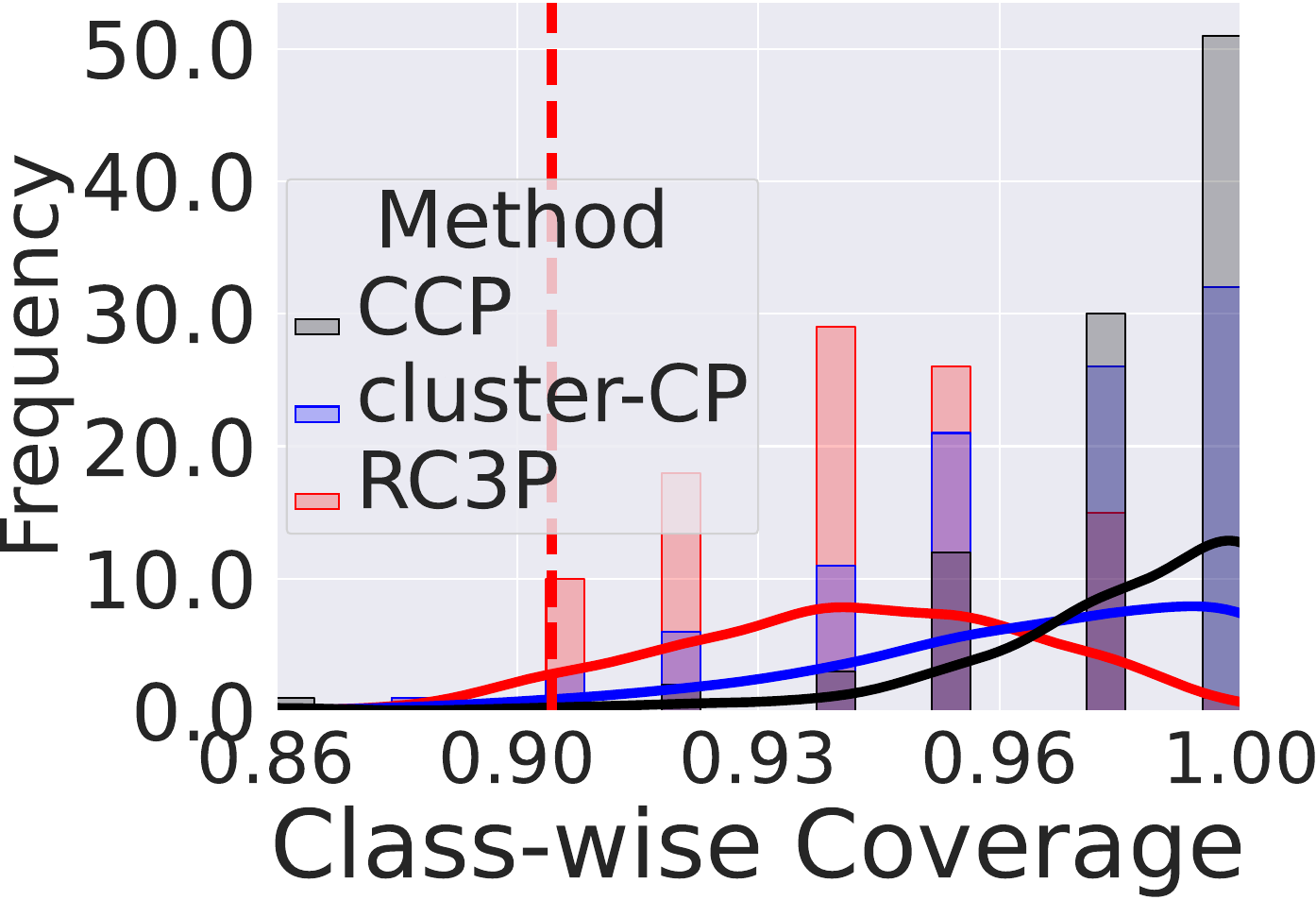}
        \end{minipage}
    }
    \subfigure{
        \begin{minipage}{0.23\linewidth}
            \includegraphics[width=\linewidth]{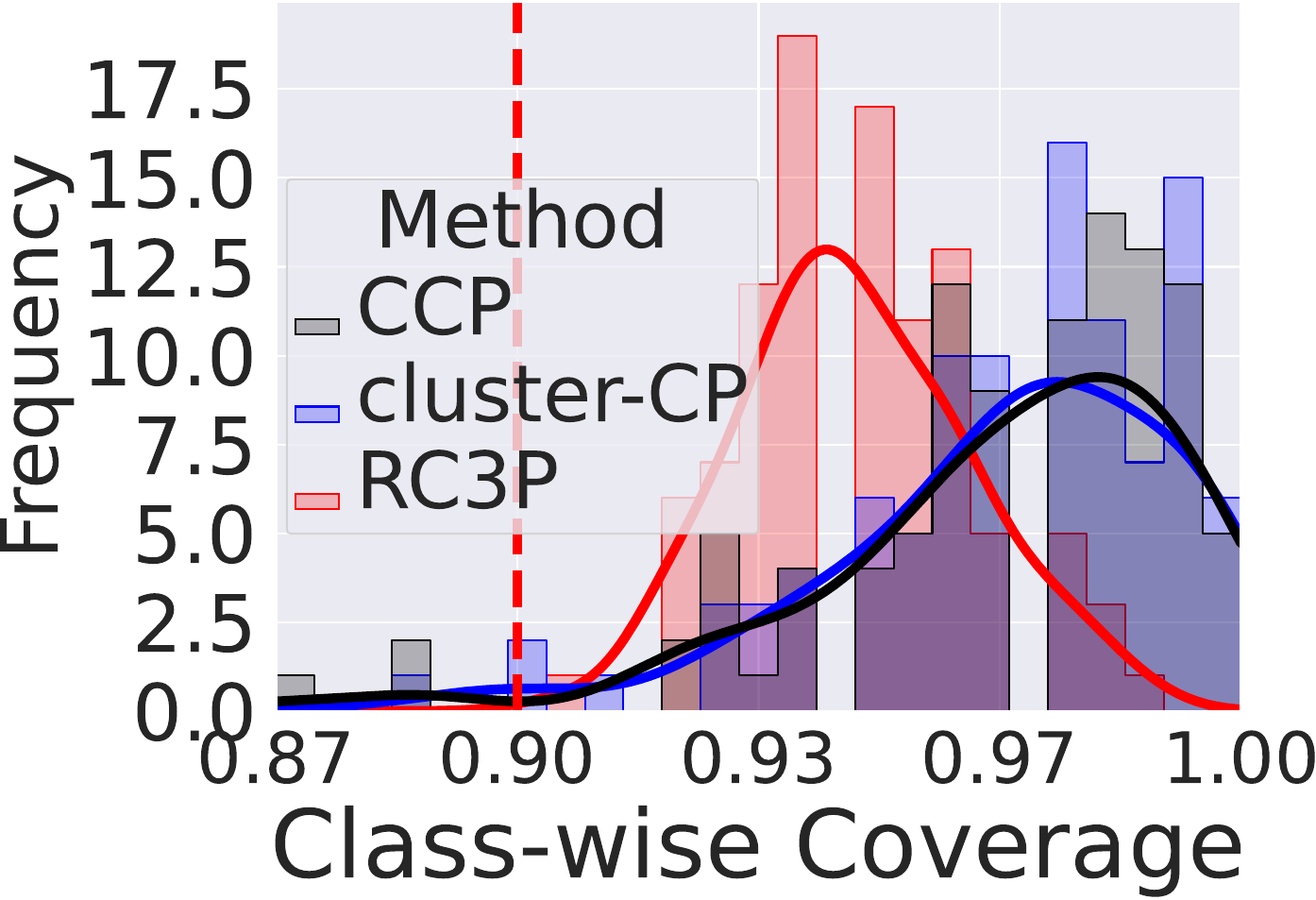}
        \end{minipage}
    }
    \subfigure{
        \begin{minipage}{0.23\linewidth}
            \includegraphics[width=\linewidth]{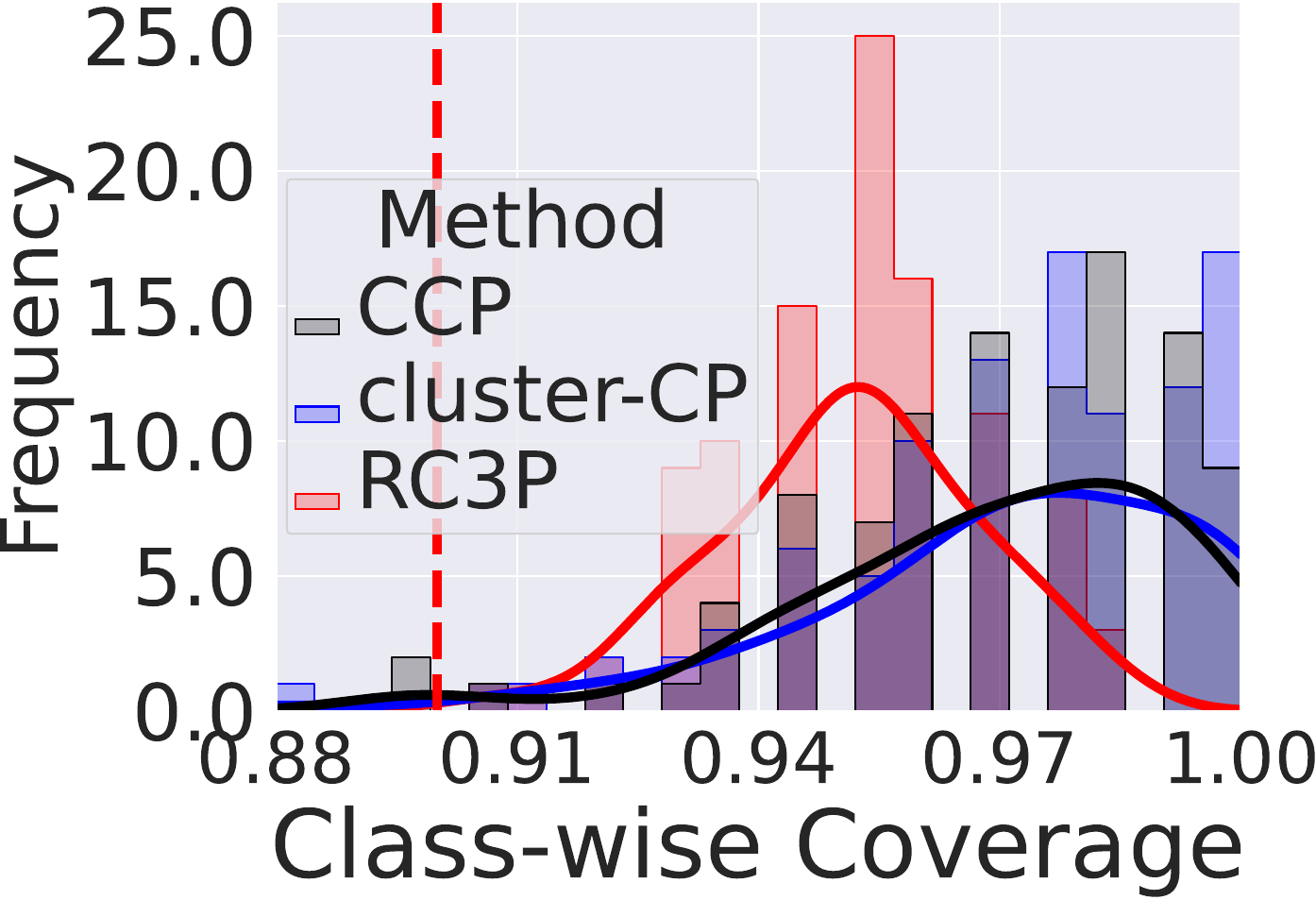}
        \end{minipage}
    }
    \subfigure{
        \begin{minipage}{0.23\linewidth}
            \includegraphics[width=\linewidth]{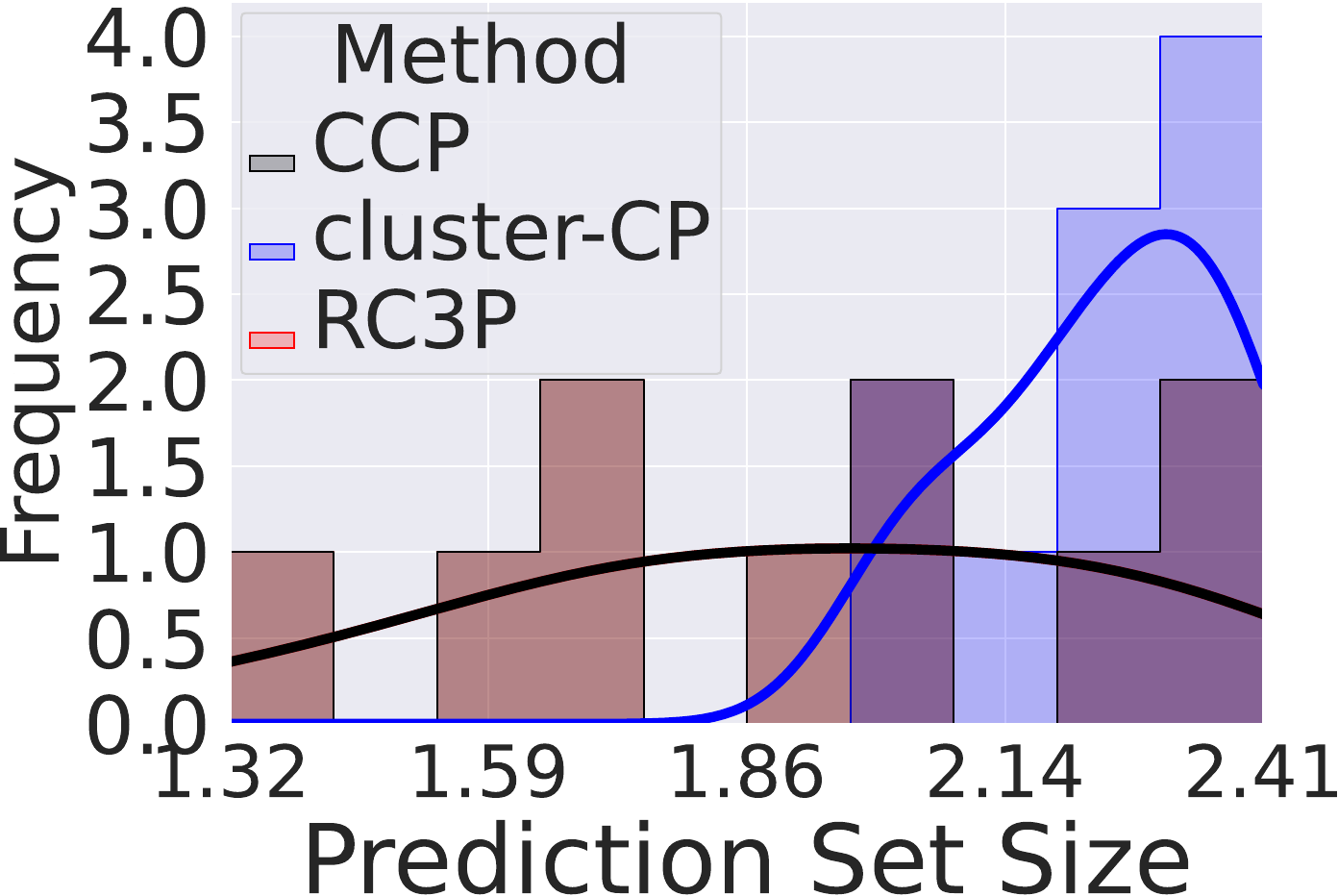}
        \end{minipage}
    }
    \subfigure{
        \begin{minipage}{0.23\linewidth}
            \includegraphics[width=\linewidth]{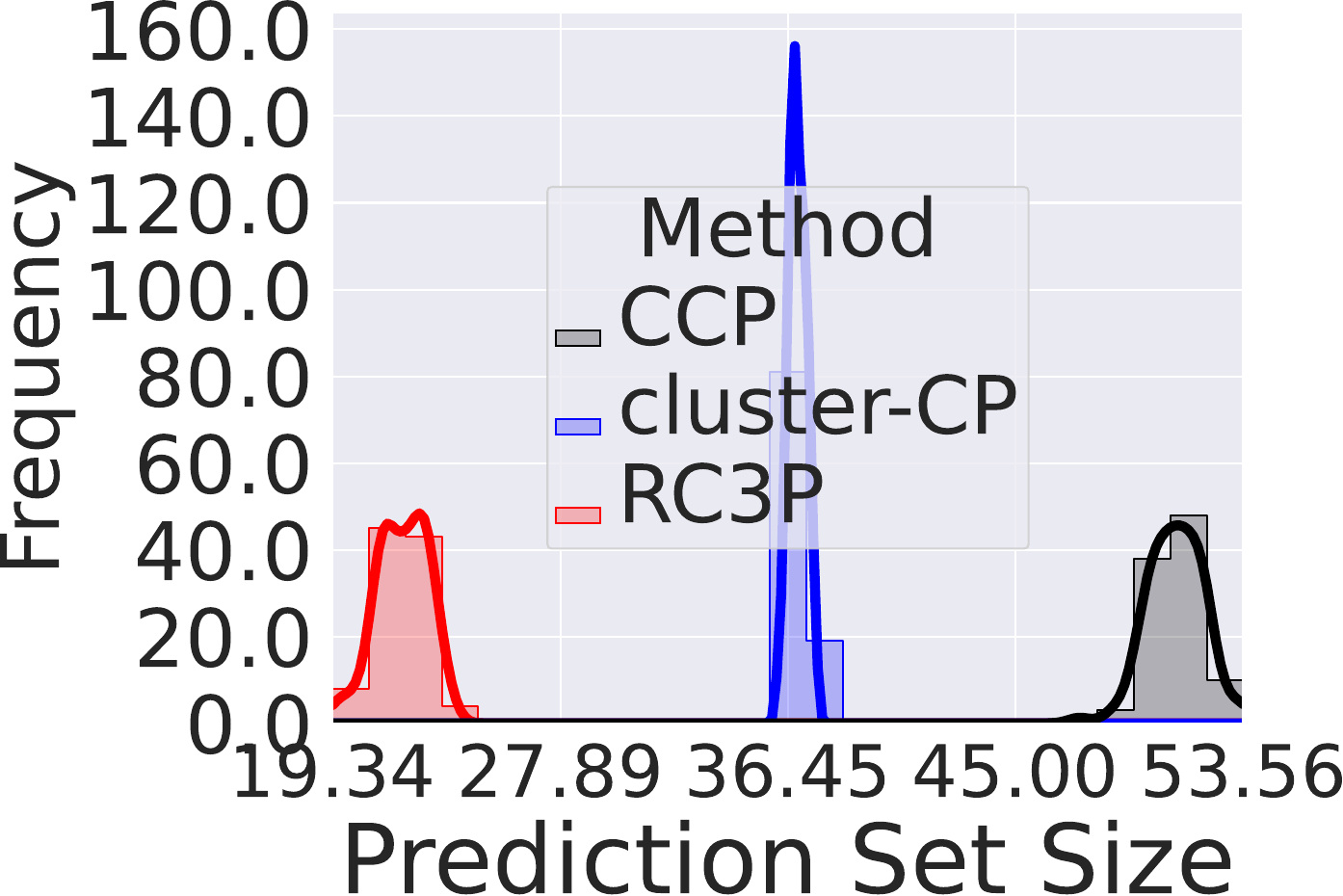}
        \end{minipage}
    }
    \subfigure{
        \begin{minipage}{0.23\linewidth}
            \includegraphics[width=\linewidth]{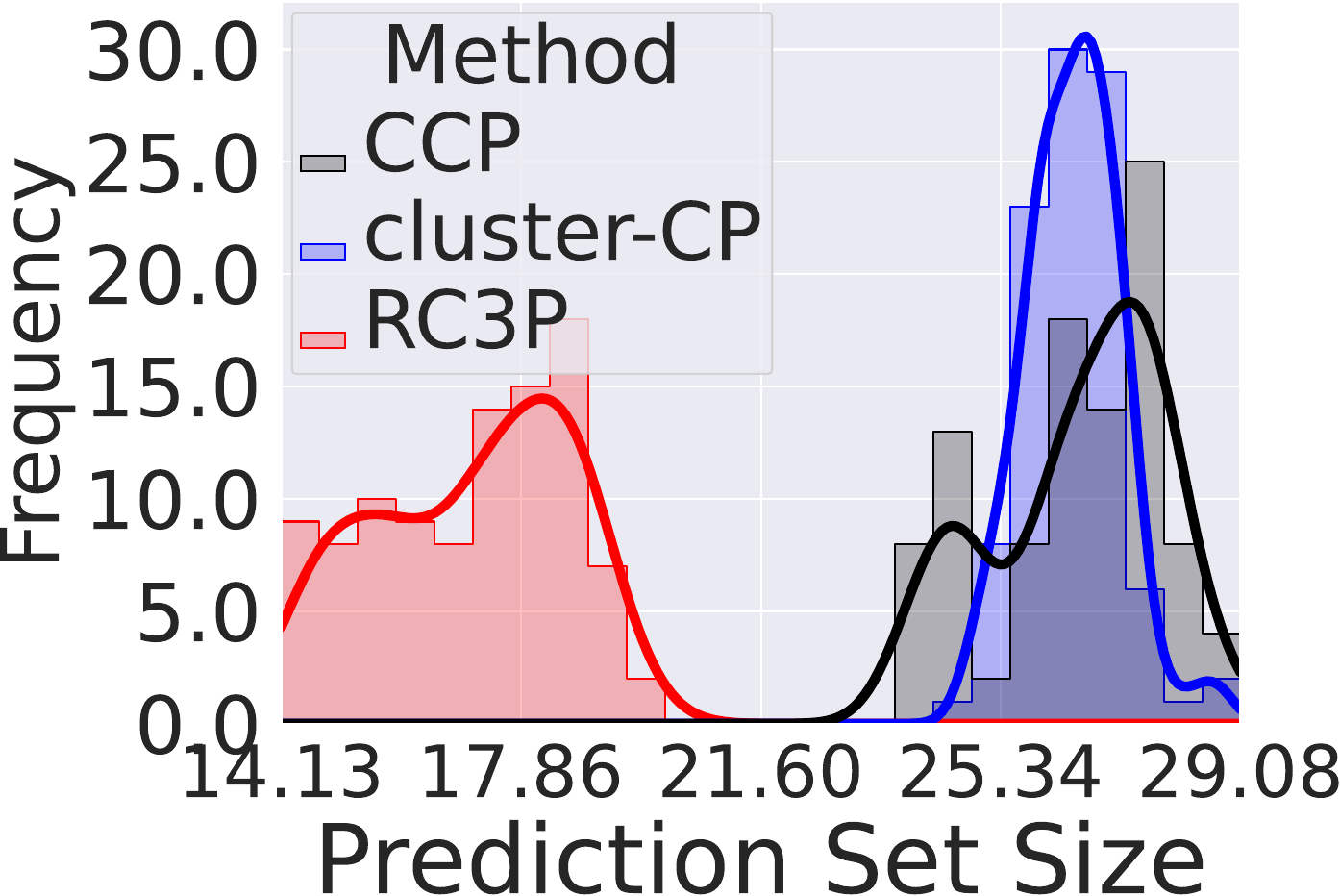}
        \end{minipage}
    }
    \subfigure{
        \begin{minipage}{0.23\linewidth}
            \includegraphics[width=\linewidth]{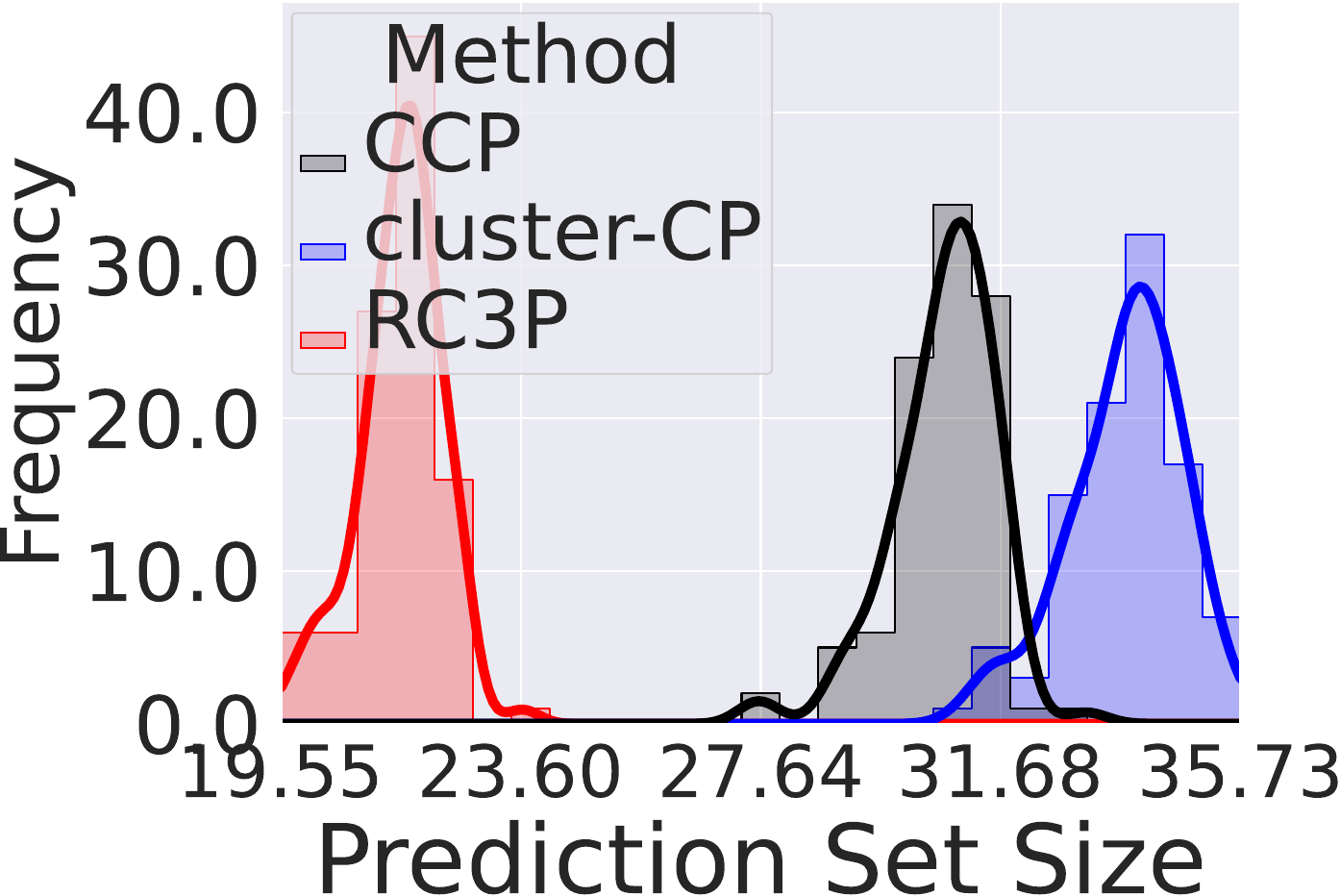}
        \end{minipage}
    }
    \vspace{-0.07in}
    \caption{
    Class-conditional coverage (Top row) and prediction set size (Bottom row) achieved by \texttt{CCP}, \texttt{Cluster-CP}, and \texttt{\newCP} methods when $\alpha = 0.1$ and models are trained with $200$ epochs on four imbalanced datasets with imbalance type \EXP~ $\rho=0.1$.
    We clarify that \texttt{\newCP} overlaps with \texttt{CCP} on CIFAR-10.
    It is clear that \texttt{\newCP} has more densely distributed class-conditional coverage above $0.9$ (the target $1-\alpha$ class-conditional coverage) than \texttt{CCP} and \texttt{Cluster-CP} with significantly smaller prediction sets on CIFAR-100, mini-ImageNet and Food-101.
    }
    \label{fig:overall_comparison_four_datasets_exp_0.1}
\end{figure*}

\begin{figure*}[!ht]
    \centering
    \begin{minipage}{.24\textwidth}
        \centering
        (a) CIFAR-10
    \end{minipage}%
    \begin{minipage}{.24\textwidth}
        \centering
        (b) CIFAR-100
    \end{minipage}%
    \begin{minipage}{.24\textwidth}
        \centering
        (c) mini-ImageNet
    \end{minipage}%
    \begin{minipage}{.24\textwidth}
        \centering
        (d) Food-101
    \end{minipage}
    \subfigure{
        \begin{minipage}{0.23\linewidth}
            \includegraphics[width=\linewidth]{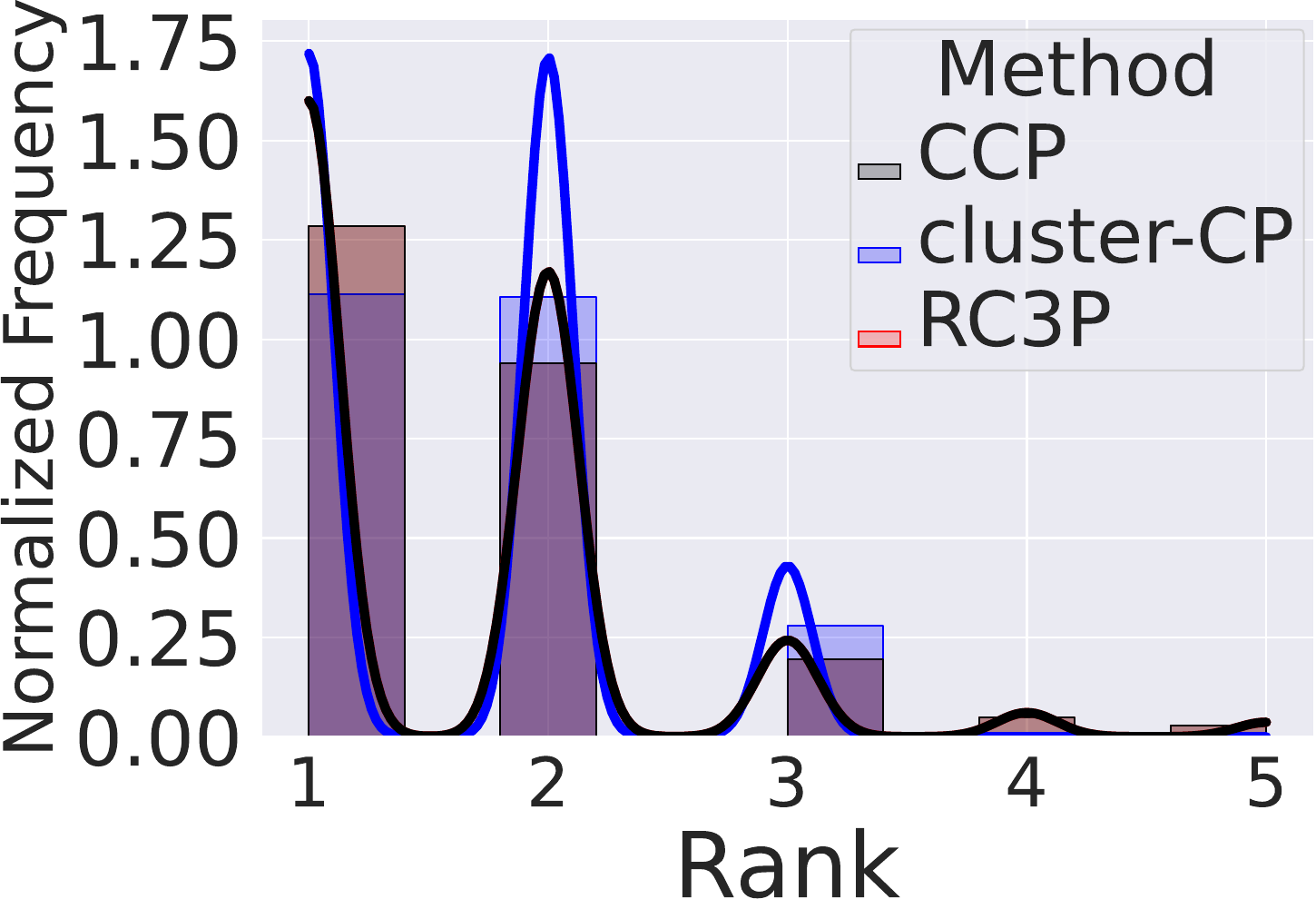}
        \end{minipage}
    }
    \subfigure{
        \begin{minipage}{0.23\linewidth}
            \includegraphics[width=\linewidth]{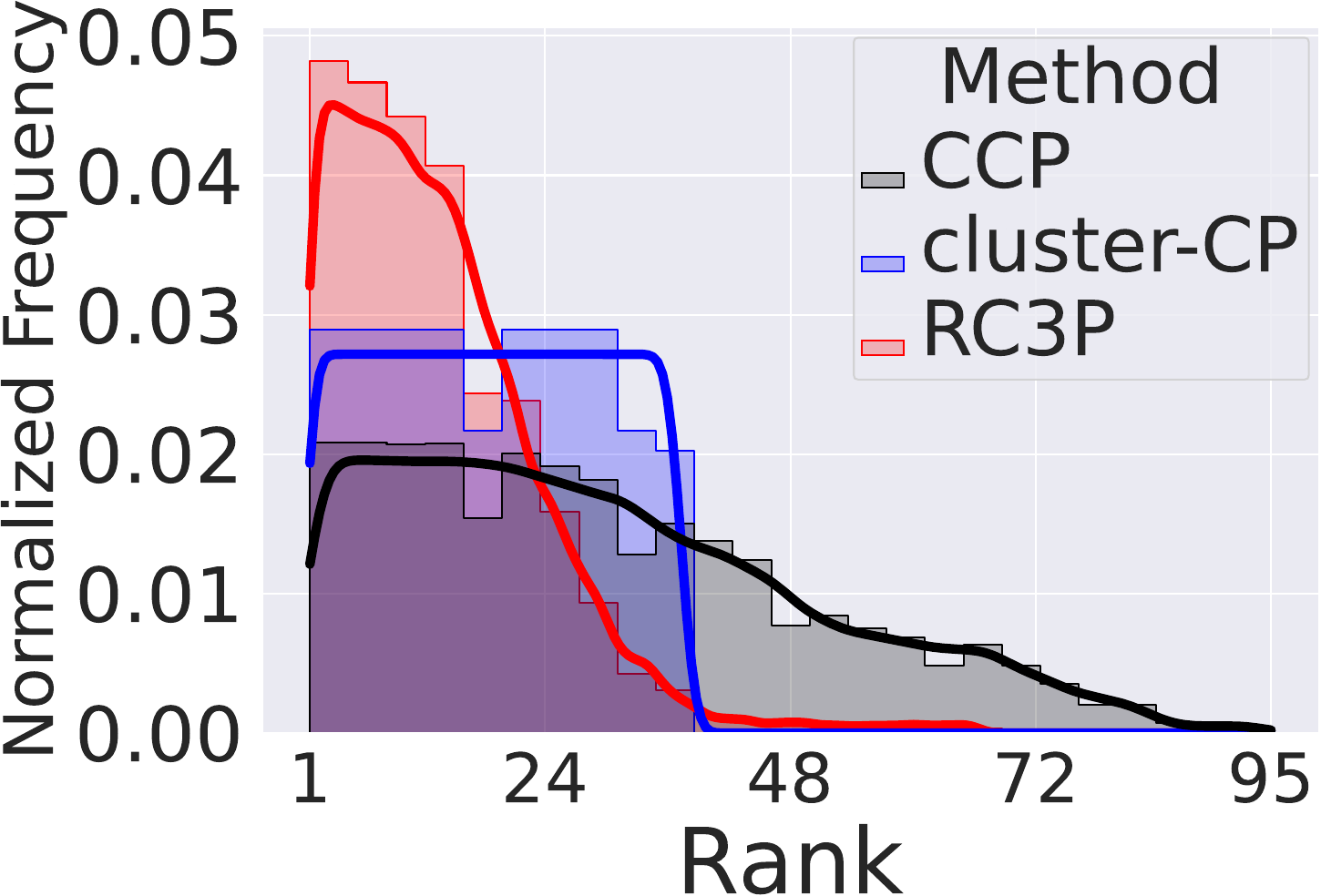}
        \end{minipage}
    }
    \subfigure{
        \begin{minipage}{0.23\linewidth}
            \includegraphics[width=\linewidth]{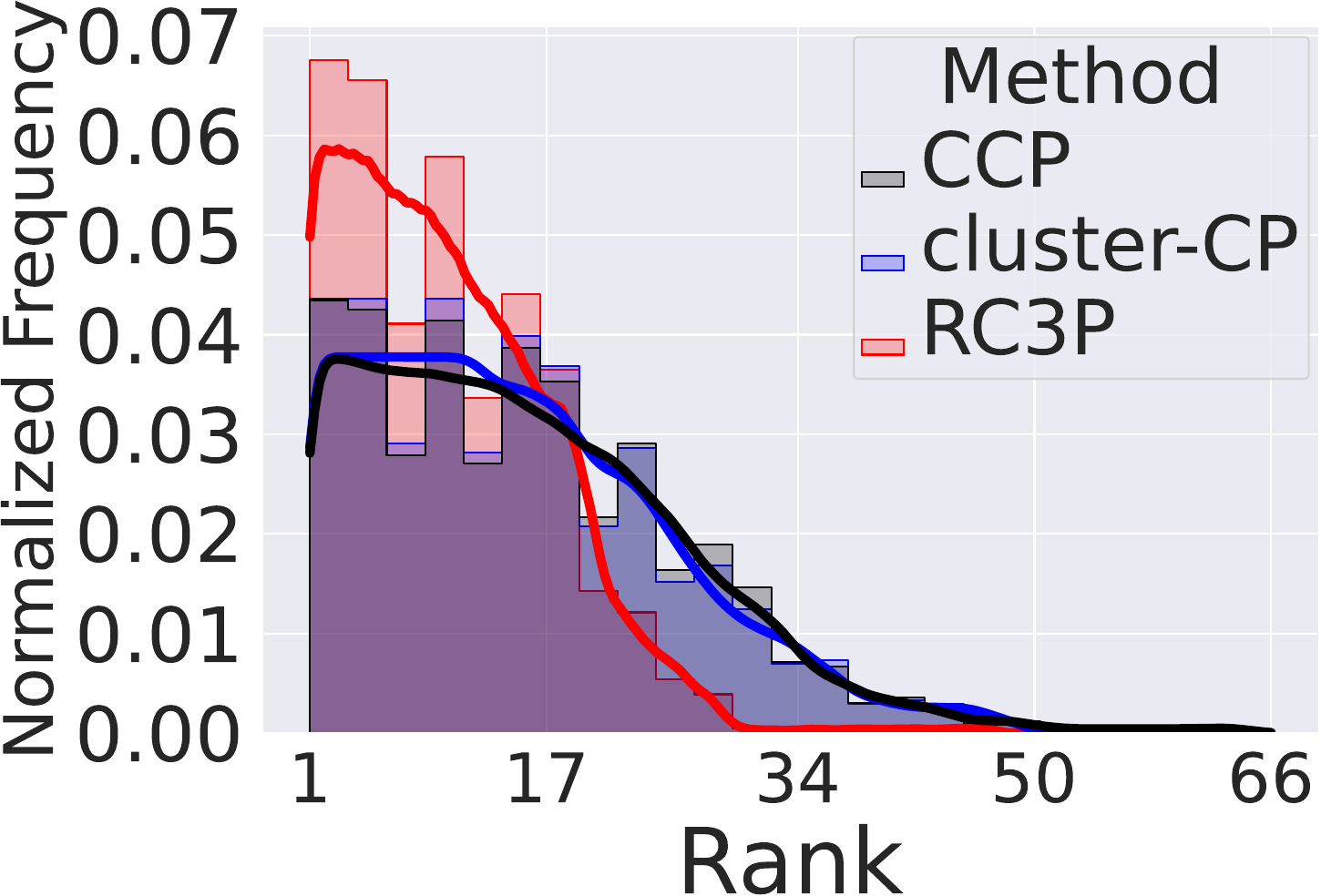}
        \end{minipage}
    }
    \subfigure{
        \begin{minipage}{0.23\linewidth}
            \includegraphics[width=\linewidth]{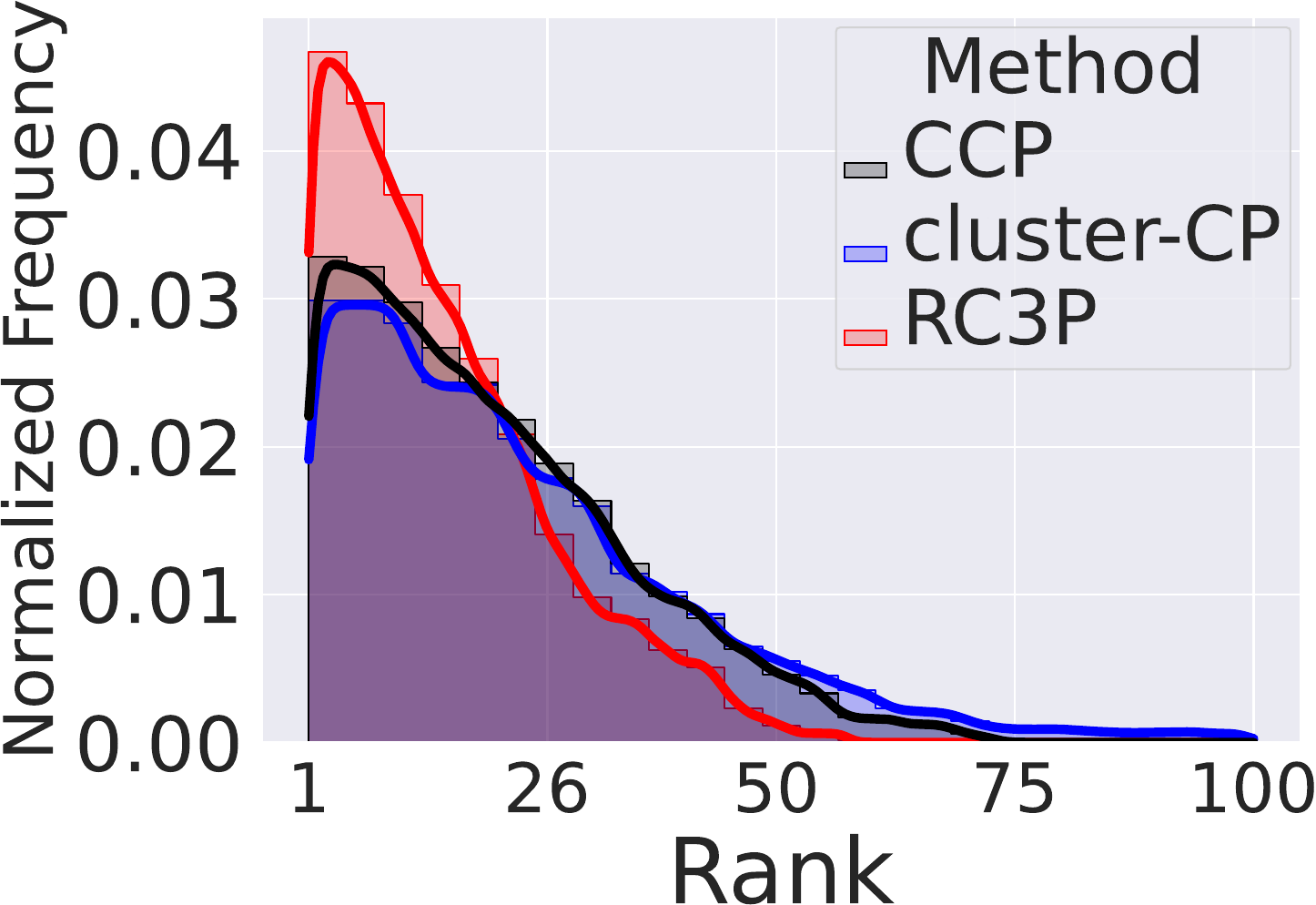}
        \end{minipage}
    }
    \vspace{-0.1in}
    \caption{
    Visualization for the normalized frequency distribution of label ranks included in the prediction set of \texttt{CCP}, \texttt{Cluster-CP}, and \texttt{\newCP} with $\rho=0.1$ for imbalance type \EXP~when $\alpha = 0.1$ and models are trained with $200$ epochs.
    It is clear that the distribution of normalized frequency generated by \texttt{\newCP} tends to be lower compared to those produced by \texttt{CCP} and \texttt{Cluster-CP}.  
    Furthermore, the probability density function tail for label ranks in the \texttt{\newCP} prediction set is notably shorter than that of other methods.
    }
    \label{fig:condition_number_rank_exp_0.1}
\end{figure*}

\noindent \textbf{\newCP~significantly outperforms CCP and Cluster-CP}. 
First, it is clear from Table \ref{tab:overall_comparison_cluster_four_datasets},  \ref{tab:overall_comparison_HPS_four_datasets}, and \ref{tab:overall_comparison_RAPS_four_datasets},
and \ref{tab:overall_comparison_balanced_datasets} that \texttt{\newCP}, \texttt{CCP}, and \texttt{Cluster-CP} guarantee class-conditional coverage on all settings.
This can also be observed by the first row of Fig \ref{fig:overall_comparison_four_datasets_exp_0.1}, where the class-wise coverage bars of \texttt{CCP} and \texttt{\newCP}~distribute on the right-hand side of the target probability $1-\alpha$ (red dashed line).
Second, \texttt{\newCP}~outperforms \texttt{CCP} and \texttt{Cluster-CP} with $24.47\%$ (four datasets) or $32.63\%$ (excluding CIFAR-10) on imbalanced datasets and $32.63\%$ on balanced datasets decrease in terms of average prediction set size for the same class-wise coverage.
We also report the histograms of the corresponding prediction set sizes in the second row of Figure \ref{fig:overall_comparison_four_datasets_exp_0.1}, which shows
(i) \texttt{\newCP}~has more concentrated class-wise coverage distribution than \texttt{CCP} and \texttt{Cluster-CP};
(ii) the distribution of prediction set sizes produced by \texttt{\newCP}~is globally smaller than that produced by \texttt{CCP} and \texttt{Cluster-CP}, which is justified by a better trade-off number of $\{\sigma_y\}_{y=1}^K$ as shown in Figure \ref{fig:condition_number_sigma_y_exp_0.1}.
Note that the class-wise coverage and the corresponding prediction set sizes \texttt{\newCP} overlap with \texttt{CCP} on CIFAR-10 in Figure \ref{fig:overall_comparison_four_datasets_exp_0.1}.

{\bf Visualization of normalized frequency.}
Figure \ref{fig:condition_number_rank_exp_0.1} illustrates the normalized frequency distribution of label ranks included in the prediction sets across various testing datasets. 
It is evident that the distribution of label ranks in the prediction set generated by \texttt{\newCP}~tends to be lower compared to those produced by \texttt{CCP} and \texttt{Cluster-CP}. 
Furthermore, the probability density function tail for label ranks in the \texttt{\newCP}~prediction set is notably shorter than that of other methods. This indicates that \texttt{\newCP}~more effectively incorporates lower-ranked labels into prediction sets, as a result of its augmented rank calibration scheme.

\begin{figure*}[!ht]
    \centering
    \begin{minipage}{.24\textwidth}
        \centering
        (a) CIFAR-10
    \end{minipage}%
    \begin{minipage}{.24\textwidth}
        \centering
        (b) CIFAR-100
    \end{minipage}%
    \begin{minipage}{.24\textwidth}
        \centering
        (c) mini-ImageNet
    \end{minipage}%
    \begin{minipage}{.24\textwidth}
        \centering
        (d) Food-101
    \end{minipage}
    \subfigure{
        \begin{minipage}{0.23\linewidth}
            \includegraphics[width=\linewidth]{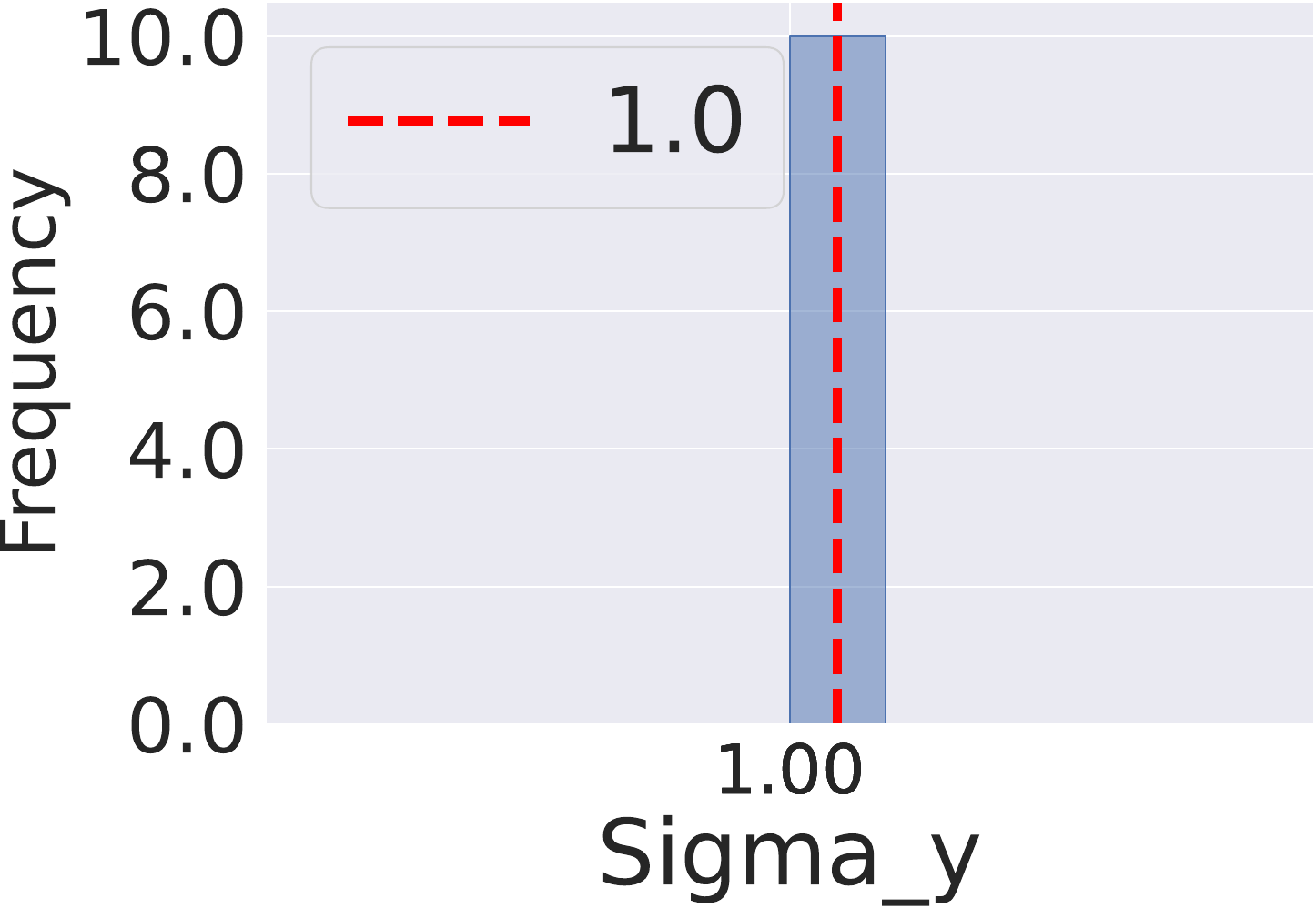}
        \end{minipage}
    }
    \subfigure{
        \begin{minipage}{0.23\linewidth}
            \includegraphics[width=\linewidth]{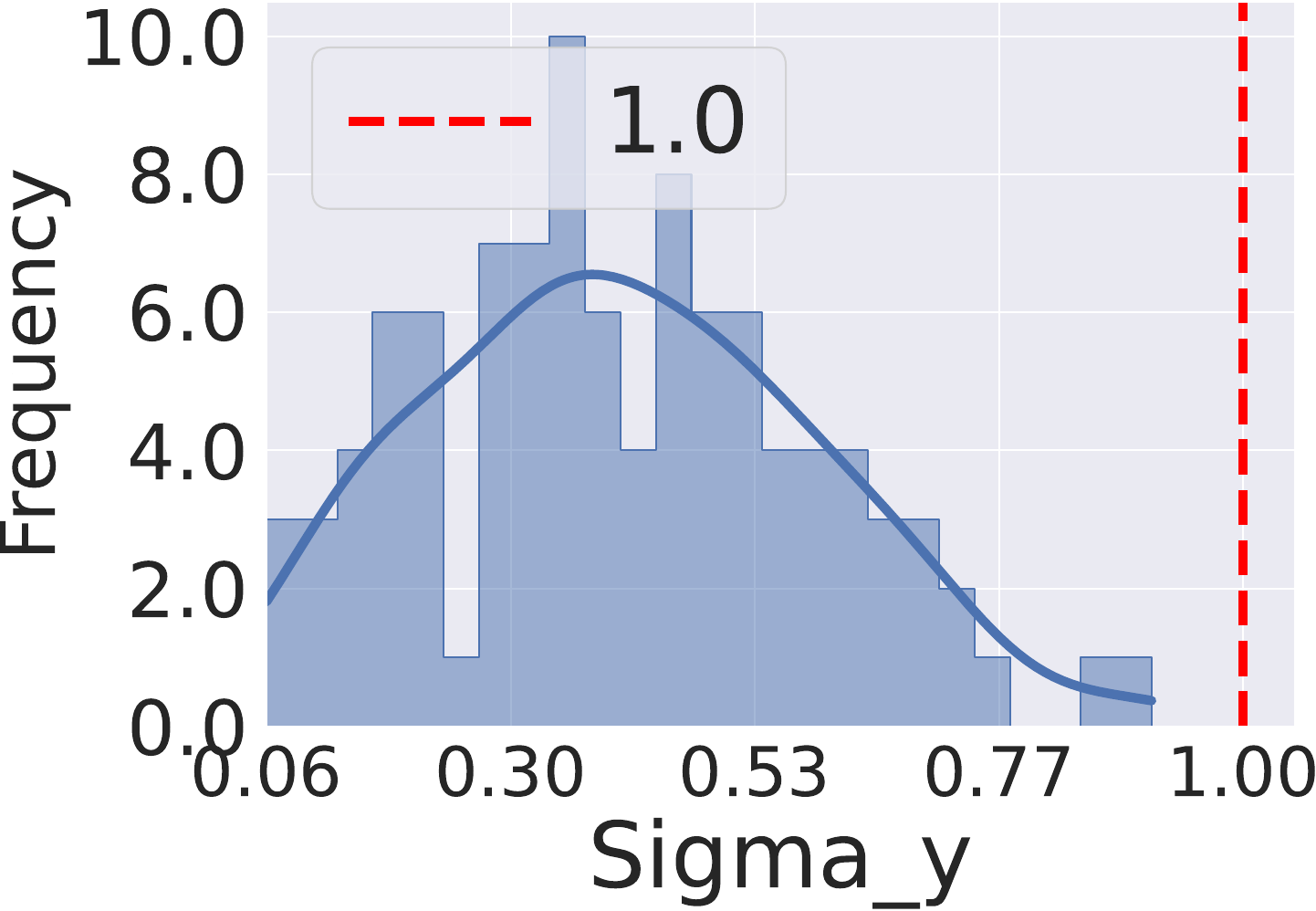}
        \end{minipage}
    }
    \subfigure{
        \begin{minipage}{0.23\linewidth}
            \includegraphics[width=\linewidth]{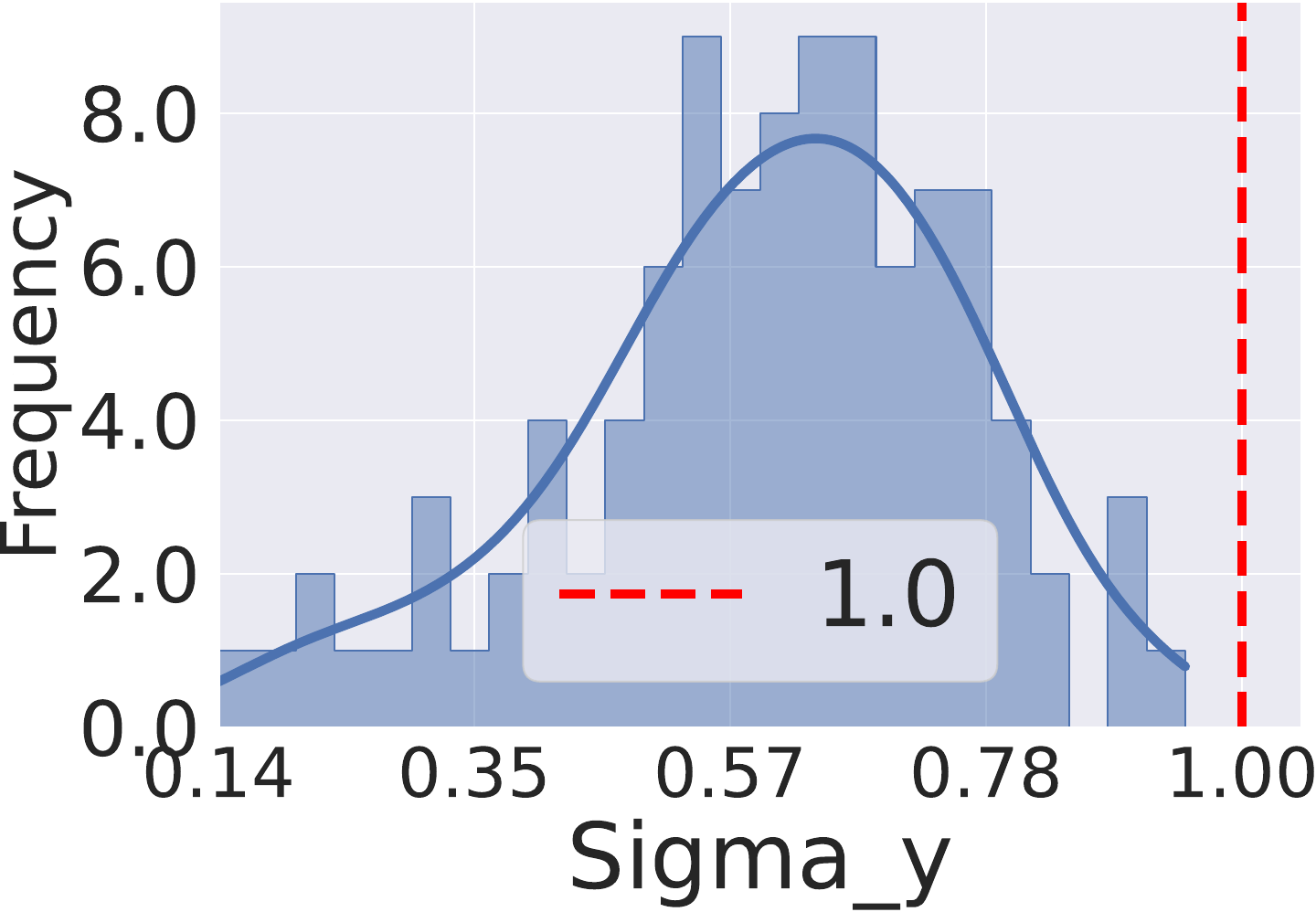}
        \end{minipage}
    }
    \subfigure{
        \begin{minipage}{0.23\linewidth}
            \includegraphics[width=\linewidth]{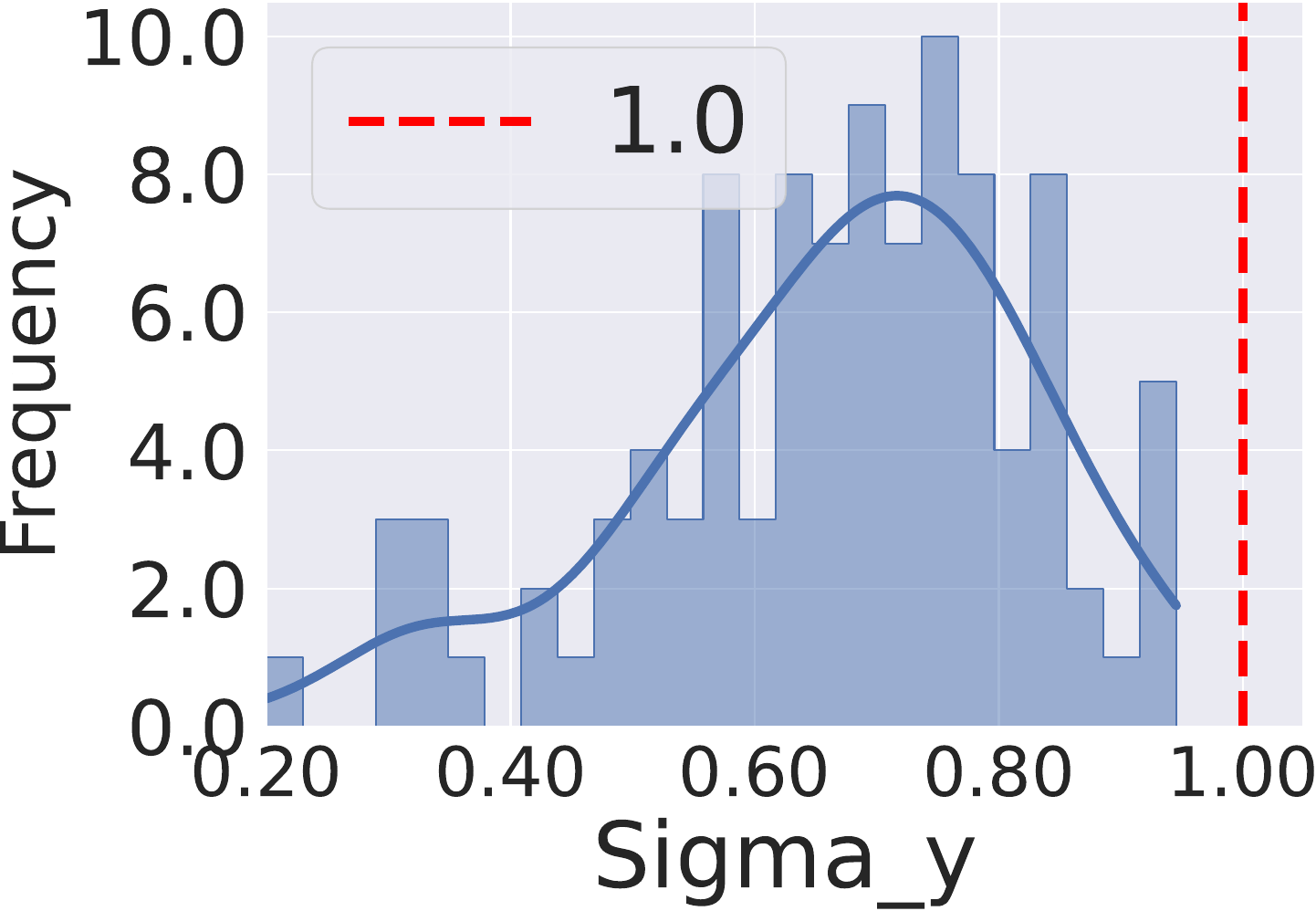}
        \end{minipage}
    }
    \vspace{-0.10in}
    \caption{
    Verification of condition numbers $\{\sigma_y\}_{y=1}^K$ in Equation \ref{eq:sigma_y_defination} with imbalance type \EXP, $\rho=0.1$ when $\alpha = 0.1$ and models are trained with $200$ epochs.
    Vertical dashed lines represent the value $1$, and we observe that all the condition numbers are smaller than $1$.
    This verifies the validity of the condition for Lemma \ref{lemma:RC3P_improved_efficiency}, and thus confirms that \texttt{\newCP}~produces smaller prediction sets than \texttt{CCP} using calibration on both non-conformity scores and label ranks.
    }
    \label{fig:condition_number_sigma_y_exp_0.1}
\end{figure*}

{\bf Verification of $\sigma_y$.}
Figure \ref{fig:condition_number_sigma_y_exp_0.1} verifies the validity of Equation (\ref{eq:sigma_y_defination}) on testing datasets and confirms the optimized trade-off between the coverage with inflated quantile and the constraint with calibrated label rank leads to smaller prediction sets.
It also confirms that the condition number $\{\sigma_y\}_{y=1}^K$ could be evaluated on calibration datasets without testing datasets and thus decreases the overall computation cost.
We verify that $\sigma_y \leq 1$ for all settings and $\sigma_y$ is much smaller than $1$ on all datasets with large number of classes.

\section{Summary}
\vspace{-1.0ex}

This paper studies a provable conformal prediction (CP) algorithm that aims to provide class-conditional coverage guarantee and to produce small prediction sets for  classification tasks with many and/or imbalanced classes. 
Our proposed \newCP~algorithm performs double-calibration, one over conformity score and one over label rank for each class separately, to achieve this goal.
Our experiments clearly demonstrate the significant efficacy of \newCP~over the baseline class-conditional CP algorithms on both balanced and imbalanced classification data settings.

\vspace{1.0ex}

\noindent{\bf Acknowledgments.} This research was supported in part by United States Department of Agriculture (USDA) NIFA award No. 2021-67021-35344 (AgAID AI Institute) and by NSF CNS-2312125 grant.
\clearpage

\bibliographystyle{plainnat}
\bibliography{icp_ref}

\newpage

\appendix

\input{appendix}

\clearpage

\end{document}

%% file: appendix.tex
\section{Technical Proofs of Theoretical Results}
\label{section:appendix:proofs}

\subsection{Proof of Theorem \ref{theorem:class_conditional_coverage_RC3P} }
\label{subsection:appendix:proof_of_class_conditional_coverage_RC3P}

\begin{theorem}
\label{theorem:appendix:class_conditional_coverage_RC3P}
(Theorem \ref{theorem:class_conditional_coverage_RC3P} restated, class-conditional coverage of \newCP) 
Suppose that selecting $\widehat k(y)$ values result in the class-wise top-$k$ error $\epsilon_y^{\widehat k(y)}$ for each class $y \in \calY$.
For a target class-conditional coverage $1-\alpha$, if we set $\widehat \alpha_y$ and $\widehat k(y)$ in \newCP~ (\ref{eq:RC3P_set_predictor}) in the following ranges:
\begin{align}
\widehat k(y) \in \{k: \epsilon_y^k < \alpha \}, \quad
0 \leq \widehat \alpha_y \leq \alpha - \epsilon_y^{\widehat k(y)}, 
\end{align} 
then \newCP~can achieve the class-conditional coverage for every $y \in \calY$:
\begin{align*}
\P_{ (X, Y) \sim \calP } \{ Y \in \widehat \calC^\newCP_{1-\alpha}(X) | Y = y \} 
\geq 
1 - \alpha
.
\end{align*}
\end{theorem}

\begin{proof}
(of Theorem \ref{theorem:class_conditional_coverage_RC3P})

Let $y \in \calY$ denote any class label.
In this proof, we omit the superscript $k$ in the top-$k$ error notation $\epsilon_y^k$ for simplicity.

With the lower bound of the coverage on class $y$ (Theorem 1 in \cite{romano2020classification}), we have
\begin{align*}
&
1 - \widehat \alpha 
\leq 
\P\{ \Ytest \in \widehat \calC^\CCP_{1 - \widehat \alpha }(\Xtest) | Y = y \}
\\
= &
\P\{ V(\Xtest, \Ytest) \leq \widehat Q^\class_{1 - \widehat \alpha } (y) | Y = y \}
\\
= &
\P\{ V(\Xtest, \Ytest) \leq \widehat Q^\class_{1-\hat  \alpha} (y), r_f(\Xtest, \Ytest) \leq \widehat k(y) | Y = y \}
\\
& 
+ \P\{ V(\Xtest, \Ytest) \leq \widehat Q^\class_{1-\hat  \alpha} (y), r_f(\Xtest, \Ytest) > \widehat k(y) | Y = y \}
\\
\leq &
\P\{ V(\Xtest, \Ytest) \leq \widehat Q^\class_{1-\hat  \alpha} (y), r_f(\Xtest, \Ytest) \leq \widehat k(y) | Y = y \}
\\
& 
+ \underbrace{ \P\{ r_f(\Xtest, \Ytest) > \widehat k(y) | Y = y \} }_{ \leq \epsilon^{\widehat k(y)}_y }
\\
\leq &
\P\{ \Ytest \in \widehat \calC^\newCP_{1-\widehat \alpha}(y) | Y = y \} 
+ \epsilon^{\widehat k(y)}_y 
.
\end{align*}

Re-arranging the above inequality, we have
\begin{align*}
\P\{ \Ytest \in \widehat \calC^\newCP_{1-\widehat \alpha}(y) | Y = y \} 
\geq 
1 - \widehat \alpha - \epsilon^{\widehat k(y)}_y 
\geq
1 - \alpha,
\end{align*}
where the last inequality is due to $\widehat \alpha_y \leq \alpha - \epsilon^{\widehat k(y)}_y$.
This implies that \texttt{\newCP} guarantees the class-conditional coverage on any class $y$.
This completes the proof for Theorem \ref{theorem:class_conditional_coverage_RC3P}.
\end{proof}

\subsection{Proof of Lemma \ref{lemma:RC3P_improved_efficiency} }
\label{subsection:appendix:proof_of_RC3P_efficiency}

\begin{theorem}
\label{theorem:appendix:RC3P_improved_efficiency}
(Lemma \ref{lemma:RC3P_improved_efficiency} restated, improved predictive efficiency of \newCP)
Let $\widehat \alpha_y$ and $\widehat k(y)$ satisfy Theorem \ref{theorem:class_conditional_coverage_RC3P}.
If the following inequality holds for any $y \in \calY$:
\begin{align}\label{eq:appendix:smaller_ps_condition}
&
\P_\Xtest \big[ V(\Xtest, y) \leq \widehat Q^\class_{1-\widehat\alpha}(y), ~ r_f(\Xtest, y) \leq \widehat k(y) \big]  
\leq
\P_\Xtest \big[ V(\Xtest, y) \leq \widehat Q^\class_{1-\alpha}(y) \big] 
,
\end{align}
then \newCP~produces smaller expected prediction sets than CCP, i.e., 
\begin{align*}
\E_\Xtest [ | \widehat \calC^\newCP_{1-\widehat \alpha}(\Xtest) | ]
\leq 
\E_\Xtest [ | \widehat \calC^\CCP_{1-\alpha}(\Xtest) |  ] 
.
\end{align*}
\end{theorem}

\begin{proof}
(of Lemma \ref{lemma:RC3P_improved_efficiency})

The proof idea is to reduce the cardinality of the prediction set made by \newCP~to that made by CCP in expectation.
Let
$
\sigma_y 
= \frac{ \P_\Xtest \Big[ V(\Xtest, y) \leq \widehat Q^\class_{1-\widehat \alpha}(y), ~ r_f(\Xtest, y) \leq \widehat k(y) \Big] }{ \P_\Xtest \Big[ V(\Xtest, y) \leq \widehat Q^\class_{1-\alpha}(y) \Big] }
.
$
According to the assumption in (\ref{eq:appendix:smaller_ps_condition}), we know that $\sigma_y \leq 1$, which will be used later.

We start with the expected prediction set size of \newCP and then derive its upper bound.
\begin{align}\label{eq:kCCP_improved_efficiency}
&
\E_\Xtest [ | \widehat \calC^\newCP_{1-\hat  \alpha}(\Xtest) | ]
=
\E_\Xtest \Bigg[ \sum_{y \in \calY} \indicator \Big[ V(\Xtest, y) \leq \widehat Q^\class_{1-\widehat \alpha}(y),~ r_f(\Xtest, y) \leq \widehat k(y) \Big] \Bigg] 
\nonumber\\
= &
\sum_{y \in \calY} \E_\Xtest \Big[ \indicator [ V(\Xtest, y) \leq \widehat Q^\class_{1-\widehat \alpha}(y), ~ r_f(\Xtest, y) \leq \widehat k(y) ] \Big]
\nonumber\\
= &
\sum_{y \in \calY} \P_\Xtest \Big[ V(\Xtest, y) \leq \widehat Q^\class_{1-\widehat \alpha}(y), ~ r_f(\Xtest, y) \leq \widehat k(y) \Big]
\nonumber\\
\stackrel{ (a) }{ = } &
\sum_{y \in \calY} \sigma_y \cdot \P_\Xtest \Big[ V(\Xtest, y) \leq \widehat Q^\class_{1-\alpha}(y) \Big]
\\
\stackrel{ (b) }{ \leq } &
\sum_{y \in \calY} \E_\Xtest \Big[ \indicator [ V(\Xtest, y) \leq \widehat Q^\class_{1-\alpha}(y) ] \Big]
\nonumber\\
= &
\E_\Xtest \Bigg[ \sum_{y \in \calY} \indicator [ V(\Xtest, y) \leq \widehat Q^\class_{1-\alpha}(y) ] \Bigg]
=
\E_\Xtest [ | \widehat \calC^\CCP_{1-\alpha}(\Xtest) | ] ,
\end{align}
where the equality $(a)$ is due to the definitions of $\sigma_y$, and
inequality $(b)$ is due to the assumption 
\begin{align*}
\sum_{y \in \calY} \sigma_y \cdot \P_\Xtest \Big[ V(\Xtest, y) \leq \widehat Q^\class_{1-\alpha}(y) \Big]
\leq 
\sum_{y \in \calY} \P_\Xtest \Big[ V(\Xtest, y) \leq \widehat Q^\class_{1-\alpha}(y) \Big] .
\end{align*}
This shows that \newCP~requires smaller prediction sets to guarantee the class-conditional coverage compared to CCP.
\end{proof}

\subsection{ Proof of Theorem \ref{theorem:RC3P_efficiency_condition} }

\begin{theorem}
\label{theorem:appendix:RC3P_efficiency_condition}
(Theorem \ref{theorem:RC3P_efficiency_condition} restated, conditions of improved predictive efficiency for \newCP)
Define $D = \P[ r_f(X, y) \leq \widehat k(y) | Y \neq y ]$, and $\bar r_f(X, y) = \lfloor \frac{ r_f(X, y) + 1 }{ 2 } \rfloor$.
Denote $B = \P[ f(X)_{ ( \bar r_f(X, y) ) } \leq \widehat Q_{1-\alpha}^\class(y) | Y \neq y ]$ if $V$ is APS, 
or $B = \P[ f(X)_{ ( \bar r_f(X, y) ) } + \lambda \leq \widehat Q_{1-\alpha}^\class(y) | Y \neq y ]$ if $V$ is RAPS.
If $B - D \geq \frac{ p_y }{ 1 - p_y } ( \alpha - \epsilon_y^{\widehat k(y)} )$,
then $\sigma_y \leq 1$.
\end{theorem}

\begin{proof}
(of Theorem \ref{theorem:RC3P_efficiency_condition})

Based on the different choices of scoring function, we first divide two scenarios: 

(i): 
If $V(X, y)$ is the APS scoring function, since the APS score cumulatively sums the ordered prediction of $f(X)$: $V(X, y) = \sum_{l=1}^{r_f(X,y)} f(X)_{(l)}$, it is easy to verify that $V(X, y)$ is concave in terms of $l$.
As a result, we have
\begin{align*}
V(X, y)
= &
\frac{ r_f(X, y) }{ r_f(X, y) } \cdot \sum_{l=1}^{r_f(X, y)} f(X)_{(l)}
\leq 
r_f(X, y) \cdot f(X)_{ ( \lfloor \sum_{l=1}^{r_f(X, y)} l / r_f(X, y) \rfloor ) }
=
r_f(X, y) \cdot f(X)_{ ( \bar r_f(X, y) ) }
,
\end{align*}
where $\bar r_f(X, y) = \bigg\lfloor \frac{ \sum_{l=1}^{r_f(X, y)} l }{ r_f(X, y) } \bigg\rfloor = \lfloor ( r_f(X, y) + 1 ) / 2 \rfloor$.

Now we lower bound $\P_X [ V(X, y) \leq \widehat Q_{1-\alpha}^\class(y) ]$ as follows.
\begin{align}
\label{eq:aps_proof}
& \P_X [ V(X, y) \leq \widehat Q_{1-\alpha}^\class(y) ]
\nonumber \\
= &
\underbrace{ 
\P_{XY} [ Y = y ] 
}_{ = p_y } 
\cdot 
\underbrace{
\P_X [ V(X, y) \leq \widehat Q_{1-\alpha}^\class(y) | Y = y ]
}_{ \geq 1 - \alpha }
+ 
\underbrace{
\P_{XY} [ Y \neq y ] 
}_{ = 1 - p_y }
\cdot 
\underbrace{ 
\P_X [ V(X, y) \leq \widehat Q_{1-\alpha}^\class(y) | Y \neq y ]
}_{ \geq B }
\nonumber \\
\geq &
p_y (1-\alpha) + (1-p_y) B
+ p_y (1-\epsilon_y^{\widehat k(y)}) + (1-p_y) D
- p_y (1-\epsilon_y^{\widehat k(y)}) - (1-p_y) D
\nonumber \\
\geq &
\P_X [ r_f(X, y) \leq \widehat k(y) ]
- p_y ( \alpha - \epsilon_y^{\widehat k(y)})
+ (1-p_y) (B-D)
.
\end{align}
According to the assumption $B - D \geq \frac{ p_y }{ 1 - p_y } ( \alpha - \epsilon_y^{\widehat k(y)} )$, we have
\begin{align*}
\P_X [ r_f(X, y) \leq \widehat k(y) ]
\leq 
\P_X [ V(X, y) \leq \widehat Q_{1-\alpha}^\class(y) ] 
.
\end{align*}

(ii): 
If $V(X, y)$ is the RAPS scoring function and $r_f(X, y) \leq k_{reg}$, then the RAPS scoring function could be rewritten as: $V(X, y) = \sum_{l=1}^{r_f(X,y)} f(X)_{(l)}$. 
As a result, we have: 
\begin{align*}
V(X, y)
= &
\frac{ r_f(X, y) }{ r_f(X, y) } \cdot \sum_{l=1}^{r_f(X, y)} f(X)_{(l)}
\\
\leq &
r_f(X, y) \cdot f(X)_{ ( \lfloor \sum_{l=1}^{r_f(X, y)} l / r_f(X, y) \rfloor ) }
\\
= &
r_f(X, y) \cdot f(X)_{ ( \bar r_f(X, y) ) }
\\
\leq &
r_f(X, y) \cdot \Big ( f(X)_{ ( \bar r_f(X, y) ) } + \lambda \Big )
.
\end{align*}

If $r_f(X, y) > k_{reg}$, then the RAPS scoring function could be rewritten as: $V(X, y) = \sum_{l=1}^{r_f(X,y)} f(X)_{(l)} + \lambda (r_f(X, y) - k_{reg})$. 
As a result, we have
\begin{align*}
V(X, y)
& = 
\frac{ r_f(X, y) }{ r_f(X, y) } \cdot \Big ( \sum_{l=1}^{r_f(X, y)} f(X)_{(l)} + \lambda \big (r_f(X,y) - k_{reg} \big ) \Big )
\\
& \leq 
r_f(X, y) \cdot \Big ( f(X)_{ ( \bar r_f(X, y) ) } + \lambda \big (1 - \frac{ k_{reg} }{ r_f(X, y) } \big ) \Big )
\\
& \leq 
r_f(X, y) \cdot \Big ( f(X)_{ ( \bar r_f(X, y) ) } + \lambda \Big )
.
\end{align*}

Then, by applying the Inequality \ref{eq:aps_proof}, we have:
\begin{align*}
\P_X [ r_f(X, y) \leq \widehat k(y) ]
\leq 
\P_X [ V(X, y) \leq \widehat Q_{1-\alpha}^\class(y) ] 
.
\end{align*}
This completes the proof for Theorem \ref{theorem:RC3P_efficiency_condition}.
\end{proof}

\section{ Complete Experimental Results }
\label{section:appendix:experiment_result}

\subsection{Training Details}
\label{appendix:training}

For CIFAR-10 and CIFAR-100, we train ResNet20 using LDAM loss function given in \cite{cao2019learning} with standard mini-batch stochastic
gradient descent (SGD) using learning rate $0.1$, momentum $0.9$, and weight decay $2e-4$ for $200$ epochs and $50$ epochs. The batch size is $128$. For experiments on mini-ImageNet, we use the same setting. For Food-101, the batch size is $256$ and other parameters are kept the same.
We reported our main results when models were trained in $200$ epochs. Other results are reported in Appendix \ref{appendix:epochs} and Table \ref{tab:overall_comparison_four_datasets_50epochs}. 

\vspace{1.0ex}

We also evaluate the top-1 accuracy over the majority, medium, and minority groups of classes as the class-wise performance when $200$ epochs. 
To show the variation of class-wise performance, we divide some classes with the largest number of data samples into the majority group, and the number of these classes is a quarter ($25\%$) of the total number of classes. Similarly, we divide the classes with the smallest number of data into the minority group ($25\%$) and the remaining classes as the medium group ($50\%$). In the above table, we show the accuracy of three groups with three imbalance types and two imbalance ratios $\rho = 0.1, \rho = 0.5$ on four datasets.

\vspace{1.0ex}

The results are summarized in Table \ref{tab:acc_comparison_four_datasets}.
As can be seen, the group-wise performance can vary significantly from high to very low.
The class-imbalance setting is the case where the classifier does not perform very well in some classes.

\begin{table}[!ht]
\centering
\caption{
Top-1 accuracy of minority, medium, and majority groups with three imbalance types and two imbalance ratios $\rho = 0.1, \rho = 0.5$ on four datasets. We could observe that the class-wise performance varies significantly over different classes.
}
\resizebox{0.65\textwidth}{!}{
\begin{NiceTabular}{@{}c!{~}cc!{~}cc!{~}cc@{}} \toprule 
\Block{2-1}{Groups} & \Block{1-2}{\EXP} & & \Block{1-2}{\POLY} & & \Block{1-2}{\MAJ} \\ 
& $\rho$ = 0.5 & $\rho$ = 0.1 & $\rho$ = 0.5 & $\rho$ = 0.1 & $\rho$ = 0.5 & $\rho$ = 0.1\\ 
\midrule 
\Block{1-*}{CIFAR-10}
\\
\midrule
Minority & 0.913 & 0.961
& 0.932 & 0.901
& 0.940 & 0.927
\\
Medium & 0.872 & 0.822
& 0.867 & 0.847
& 0.848 & 0.75
\\ 
Majority & 0.949 & 0.832
& 0.933 & 0.948
& 0.914 & 0.795
\\ 
\midrule
\Block{1-*}{CIFAR-100}
\\
\midrule
Minority & 0.554 & 0.295
& 0.468 & 0.352
& 0.572 & 0.365
\\ 
Medium & 0.589 & 0.536
& 0.517 & 0.413
& 0.574 & 0.476
\\ 
Majority & 0.668 & 0.720
& 0.671 & 0.588 
& 0.616 & 0.562
\\ 
\midrule
\Block{1-*}{mini-ImageNet}
\\
\midrule 
Minority & 0.677 & 0.640
& 0.624 & 0.627
& 0.626 & 0.642
\\ 
Medium & 0.527 & 0.546
& 0.533 & 0.530
& 0.526 & 0.538
\\ 
Majority & 0.633 & 0.679
& 0.684 & 0.67
& 0.673 & 0.686
\\ 
\midrule
\Block{1-*}{Food-101}
\\
\midrule 
Minority & 0.453 & 0.231
& 0.379 & 0.289
& 0.505 & 0.333
\\ 
Medium & 0.579 & 0.474
& 0.496 & 0.398
& 0.579 & 0.467
\\ 
Majority & 0.582 & 0.660
& 0.596 & 0.563
& 0.532 & 0.490
\\ 
\bottomrule 
\end{NiceTabular}
}
\label{tab:acc_comparison_four_datasets}
\end{table}

\subsection{Calibration Details}
\label{subsection:appendix:calibration_details}

As mentioned in Section \ref{section: Experimental setup}, we balanced split the validation set of CIFAR-10 and CIFAR-100, the number of calibration data is $5000$. For mini-ImageNet, the number of calibration data is $15000$. For Food-101, the total number is $12625$.
To compute the mean and standard deviation for the overall performance, we repeat calibration experiments for $10$ times. In our main results, We set $\alpha = 0.1$. We also report other experiment results of different $\alpha$ values, $\alpha = 0.05$ and $\alpha = 0.01$, 
in Appendix \ref{appendix:alphas}, and Table \ref{tab:overall_comparison_four_datasets_0.05} and \ref{tab:overall_comparison_four_datasets_0.01}.

\vspace{1.0ex}

The regularization parameter for RAPS scoring function is from the set $k_{reg} \in \{3, 5, 7\}$ and $\lambda \in \{0.001, 0.01, 0.1\}$ based on the empirical setting in \texttt{cluster-CP}. We select the combination of $k_{reg}$ and $\lambda$ for each experiment with the same imbalanced type and imbalanced ratio on the same dataset, where most of the APSS values of all methods are minimum.

\vspace{1.0ex}

The hyper-parameter $g$ is selected from the set $\{0.25, 0.5, 0.75, 1.0\}$ to find the minimal $g$ that \texttt{CCP}, \texttt{Cluster-CP} 
\footnote{\url{https://github.com/tiffanyding/class-conditional-conformal/tree/main}}, and \texttt{\newCP} achieve the target class-conditional coverage. 
We clarify that for each dataset and each class-conditional CP method, we use fixed $g$ values. The detailed $g$ values are displayed in Table \ref{tab:appendix:g_detail}.
From Table \ref{tab:appendix:g_detail}, we could observe that the hyperparameter $g$ for \texttt{\newCP} is always smaller than other methods, which means that comparing other class-wise CP algorithms, our algorithm needs the smallest inflation on $1-\widehat \alpha$ to achieve the target class-conditional coverage. This could also match the result of histograms of class-conditional coverage.

\begin{table*}[!ht]
\centering
\caption{
Hyperparameter $g$ choices for each class-conditional CP methods \texttt{CCP}, \texttt{Cluster-CP}, and \texttt{\newCP} on four datasets CIFAR-10, CIFAR-100, mini-ImageNet, and Food101.
We could observe that all $g$ values are in constant order to make a fair comparison. Meanwhile, the hyperparameter $g$ for \texttt{\newCP} is always smaller than other methods.
}
\resizebox{0.65\textwidth}{!}{
\begin{NiceTabular}{@{}c!{~}cccc@{}} \toprule 
\Block{2-1}{Methods} & \Block{1-*}{Dataset} \\ 
& CIFAR-10 & CIFAR-100 & mini-ImageNet & FOOD-101 \\ 
\\
\midrule
CCP    
& 0.5 
& 0.5
& 0.75 
& 0.75 
\\ 
Cluster-CP      
& 1.0 
& 0.5 
& 0.75  
& 0.75
\\ 
\newCP  
& 0.5 
& 0.25
& 0.5 
& 0.5 
\\
\bottomrule 
\end{NiceTabular}
}
\label{tab:appendix:g_detail}
\end{table*}

\subsection{ Illustration of Imbalanced Data }
\label{subsection:appendix:illustration_imbalanced_data}

\begin{figure}[!ht]
    \centering
    \begin{minipage}{.33\textwidth}
        \centering
        {\small (a) \EXP}
    \end{minipage}%
    \begin{minipage}{.33\textwidth}
        \centering
        {\small (b) \POLY}
    \end{minipage}%
    \begin{minipage}{.33\textwidth}
        \centering
        {\small (c) \MAJ}
    \end{minipage}
    \begin{minipage}{.33\textwidth}
        \centering
            \includegraphics[width=\linewidth]{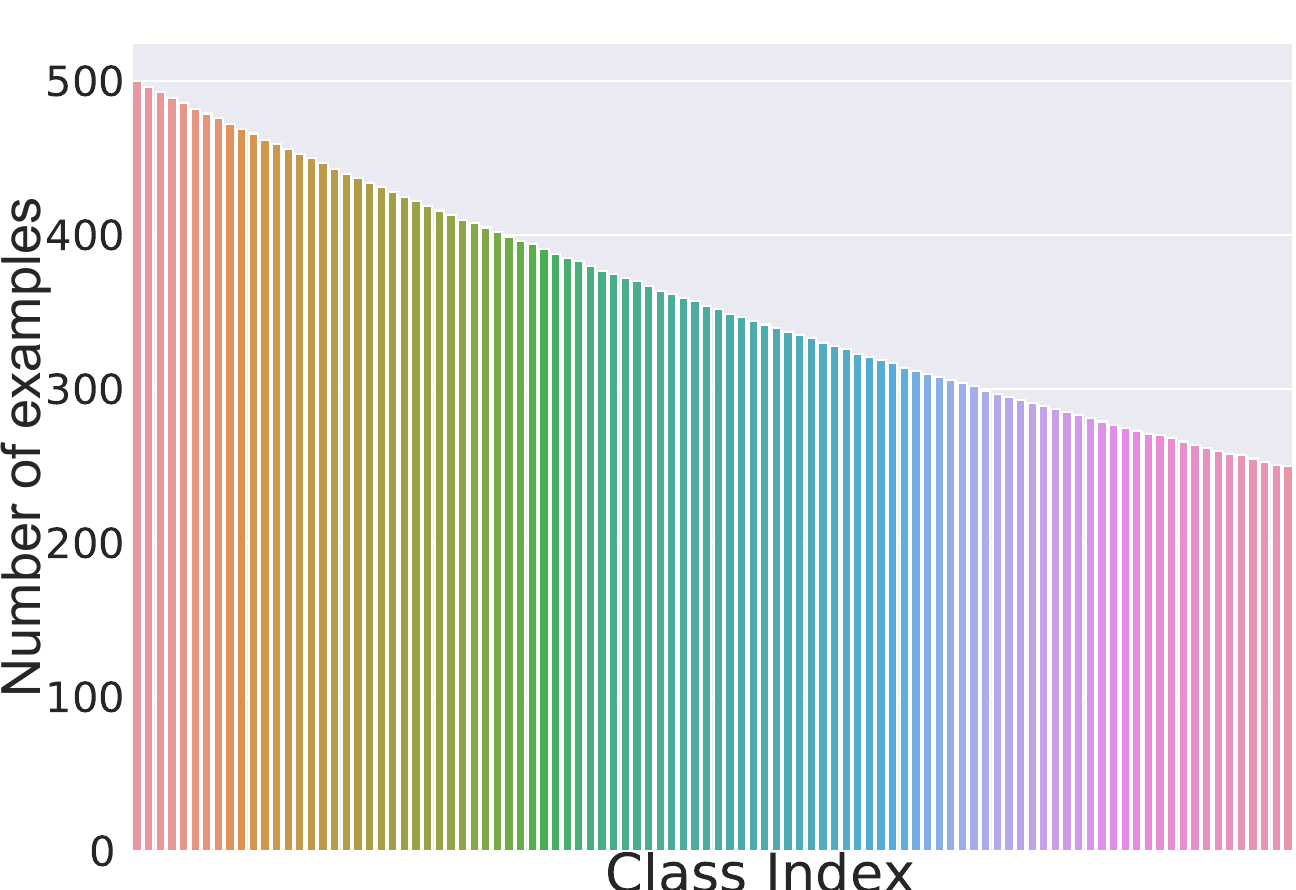}
    \end{minipage}%
    \begin{minipage}{.33\textwidth}
        \centering
            \includegraphics[width=\linewidth]{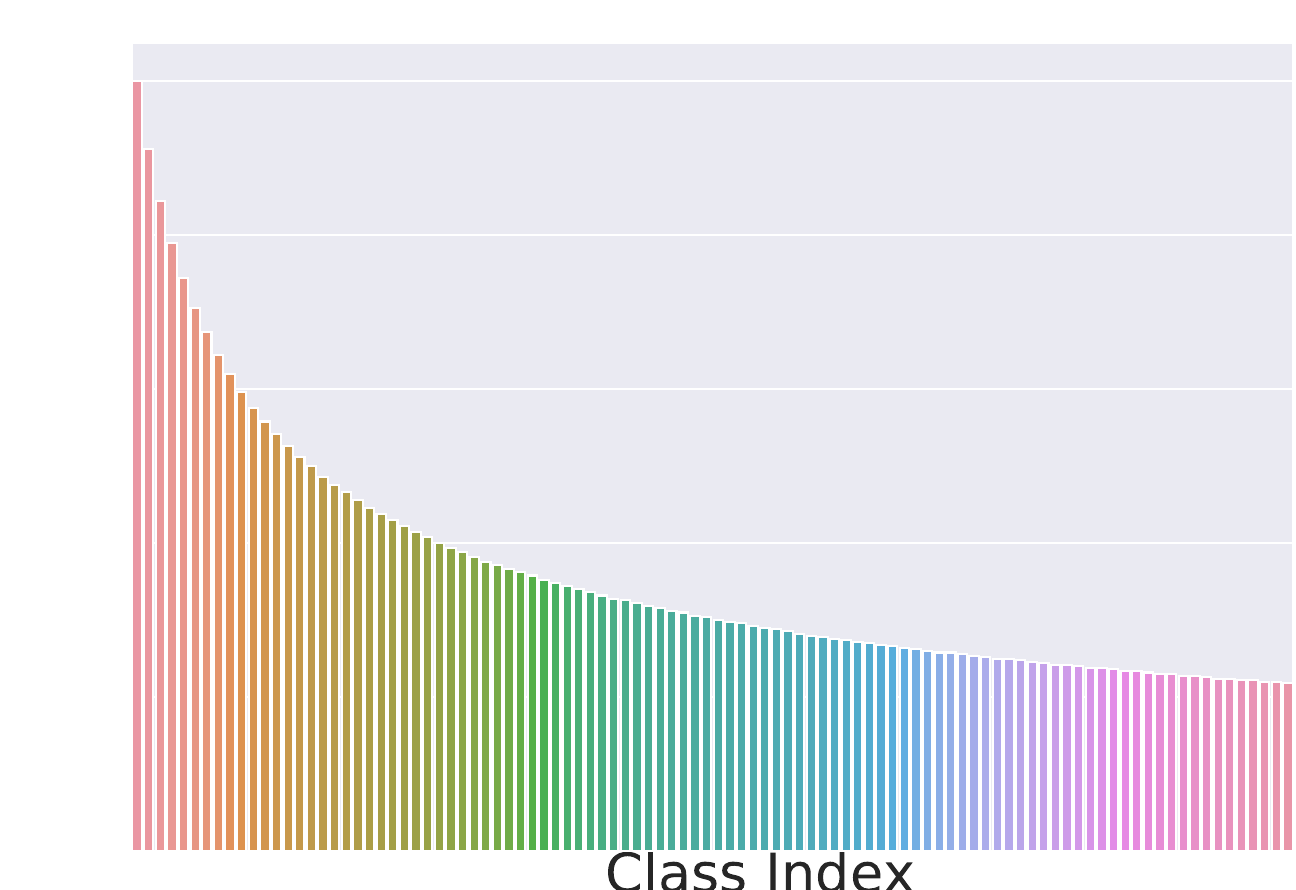}
    \end{minipage}%
    \begin{minipage}{.33\textwidth}
        \centering
            \includegraphics[width=\linewidth]{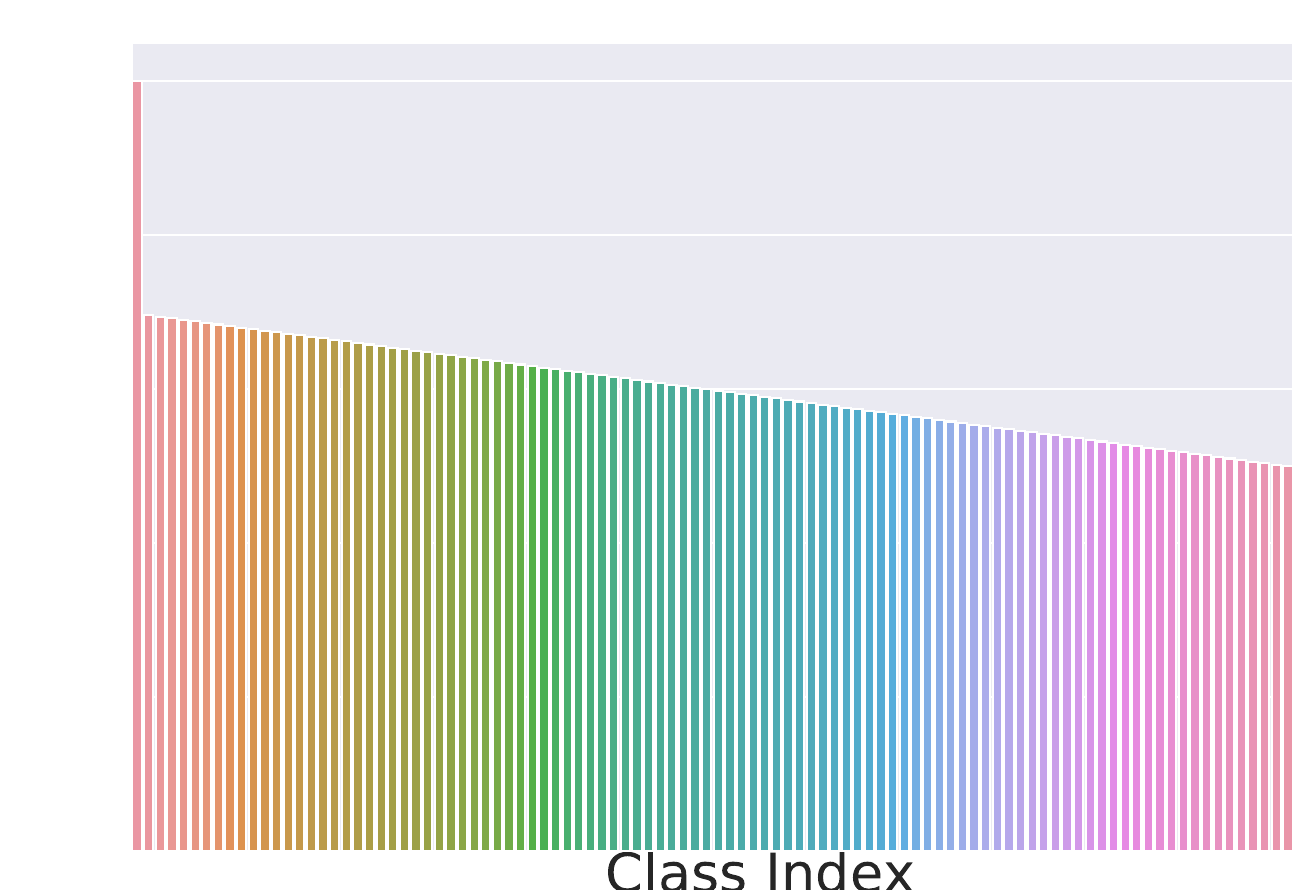}
    \end{minipage}
    
    \caption{Illustrative examples of the different imbalanced distributions of the number of training examples per class index $c$ on CIFAR-100}
    \label{fig:imbalanced_decay_examples}
\end{figure}

\subsection{Comparison Experiments Using APS Score Function}
\label{appendix:aps}

\vspace{1.0ex}

Based on the results in Table \ref{tab:overall_comparison_cluster_four_datasets}, we make the following observations: (i) \texttt{CCP}, \texttt{Cluster-CP}, and \texttt{\newCP} can guarantee the class-conditional coverage; and (ii) \texttt{\newCP} significantly outperforms \texttt{CCP} and \texttt{Cluster-CCP} on three datasets by producing smaller prediction sets.

\begin{table}[!ht]
\centering
\caption{
Results comparing \texttt{CCP}, \texttt{Cluster-CP}, and \texttt{\newCP}~with ResNet-20 model and APS scoring function under different imbalance ratios $\rho=0.5$ and $\rho=0.1$ when $\alpha = 0.1$.
We set UCR of \texttt{\newCP} the same as or better than that of \texttt{CCP} and \texttt{Cluster-CP} for a fair comparison of prediction set size.
The APSS results show that \texttt{\newCP} significantly outperforms \texttt{Cluster-CP} in terms of the average prediction set size over all settings on CIFAR-100, mini-ImageNet, and Food-101.
}
\resizebox{\textwidth}{!}{
\begin{NiceTabular}{@{}cc!{~}cc!{~}cc!{~}cc@{}} \toprule 
\Block{2-1}{Measure}  & \Block{2-1}{Methods} & \Block{1-2}{\EXP} & & \Block{1-2}{\POLY} & & \Block{1-2}{\MAJ} \\ 
 & & $\rho$ = 0.5 & $\rho$ = 0.1 & $\rho$ = 0.5 & $\rho$ = 0.1 & $\rho$ = 0.5 & $\rho$ = 0.1\\ 
\midrule 
\Block{1-*}{CIFAR-10}
\\
\midrule
\Block{3-1}{UCR} 
& CCP     & 0.050 $\pm$ 0.016 & 0.100 $\pm$ 0.020
& 0.100 $\pm$ 0.032 & \textbf{0.050 $\pm$ 0.021}
& 0.070 $\pm$ 0.014 & \textbf{0.040 $\pm$ 0.015}
\\
& Cluster-CP   & \textbf{0.010 $\pm$ 0.009} & \textbf{0.090 $\pm$ 0.009}
& \textbf{0.080 $\pm$ 0.019} & 0.060 $\pm$ 0.001
& \textbf{0.020 $\pm$ 0.012} & 0.070 $\pm$ 0.014
\\ 
& \textbf{\newCP} & 0.050 $\pm$ 0.016 & 0.100 $\pm$ 0.020
& 0.100 $\pm$ 0.032 & \textbf{0.050 $\pm$ 0.021}
& 0.070 $\pm$ 0.014 & \textbf{0.040 $\pm$ 0.015}
\\ 
\midrule
\Block{3-1}{APSS}
& CCP     & \textbf{1.555 $\pm$ 0.010} & \textbf{1.855 $\pm$ 0.014}
& \textbf{1.538 $\pm$ 0.010} & \textbf{1.776 $\pm$ 0.012}
& \textbf{1.840 $\pm$ 0.020} & \textbf{2.629 $\pm$ 0.013}
\\ 
& Cluster-CP     & 1.714 $\pm$ 0.018 & 2.162 $\pm$ 0.015
& 1.706 $\pm$ 0.014 & 1.928 $\pm$ 0.013
& 1.948 $\pm$ 0.023 & 3.220 $\pm$ 0.020
\\ 
& \textbf{\newCP} & \textbf{1.555 $\pm$ 0.010} & \textbf{1.855 $\pm$ 0.014}
& \textbf{1.538 $\pm$ 0.010} & \textbf{1.776 $\pm$ 0.012}
& \textbf{1.840 $\pm$ 0.020} & \textbf{2.629 $\pm$ 0.013}
\\
\midrule
\Block{1-*}{CIFAR-100}
\\
\midrule
\Block{3-1}{UCR} 
& CCP     & 0.007 $\pm$ 0.002 & \textbf{0.010 $\pm$ 0.002}
& 0.010 $\pm$ 0.002 & \textbf{0.014 $\pm$ 0.003}
& 0.016 $\pm$ 0.003 & \textbf{0.008 $\pm$ 0.004}
\\ 
& Cluster-CP     & 0.012 $\pm$ 0.002 & 0.016 $\pm$ 0.004
& 0.020 $\pm$ 0.003 & 0.004 $\pm$ 0.002
& 0.016 $\pm$ 0.003 & 0.019 $\pm$ 0.005
\\ 
& \textbf{\newCP} & \textbf{0.005 $\pm$ 0.002} & 0.011 $\pm$ 0.002
& \textbf{0.009 $\pm$ 0.003} & 0.015 $\pm$ 0.003 
& \textbf{0.008 $\pm$ 0.002} & \textbf{0.008 $\pm$ 0.004}
\\ 
\midrule
\Block{3-1}{APSS}
& CCP     & 44.224 $\pm$ 0.341 & 50.969 $\pm$ 0.345
& 49.889 $\pm$ 0.353 & 64.343 $\pm$ 0.237
& 44.194 $\pm$ 0.514 & 64.642 $\pm$ 0.535
\\ 
& Cluster-CP     & 29.238 $\pm$ 0.609 & 37.592 $\pm$ 0.857
& 38.252 $\pm$ 0.353 & 52.391 $\pm$ 0.595
& 31.518 $\pm$ 0.335 & 50.883 $\pm$ 0.673
\\ 
& \textbf{\newCP} & \textbf{17.705 $\pm$ 0.004} & \textbf{21.954 $\pm$ 0.005}
& \textbf{23.048 $\pm$ 0.008} & \textbf{33.185 $\pm$ 0.005}
& \textbf{18.581 $\pm$ 0.007} & \textbf{32.699 $\pm$ 0.005}
\\
\midrule
\Block{1-*}{mini-ImageNet}
\\
\midrule 
\Block{3-1}{UCR} 
& CCP     & 0.008 $\pm$ 0.004 & 0.008 $\pm$ 0.004
& 0.005 $\pm$ 0.002 & 0.004 $\pm$ 0.001
& 0.010 $\pm$ 0.004 & 0.005 $\pm$ 0.002
\\ 
& Cluster-CP     & 0.014 $\pm$ 0.004 & 0.012 $\pm$ 0.004
& 0.011 $\pm$ 0.003 & 0.014 $\pm$ 0.003
& 0.008 $\pm$ 0.002 & 0.010 $\pm$ 0.003
\\ 
& \textbf{\newCP} & \textbf{0.000 $\pm$ 0.000} & \textbf{0.001 $\pm$ 0.001}
& \textbf{0.000 $\pm$ 0.000} & \textbf{0.000 $\pm$ 0.000}
& \textbf{0.000 $\pm$ 0.000} & \textbf{0.000 $\pm$ 0.000}
\\ 
\midrule
\Block{3-1}{APSS} 
& CCP     & 26.676 $\pm$ 0.171 & 26.111 $\pm$ 0.194
& 26.626 $\pm$ 0.133 & 26.159 $\pm$ 0.208
& 27.313 $\pm$ 0.154 & 25.629 $\pm$ 0.207
\\ 
& Cluster-CP     & 25.889 $\pm$ 0.301 & 25.253 $\pm$ 0.346
& 26.150 $\pm$ 0.393 & 25.633 $\pm$ 0.268
& 26.918 $\pm$ 0.241 & 25.348 $\pm$ 0.334
\\ 
& \textbf{\newCP} & \textbf{18.129 $\pm$ 0.003} & \textbf{17.082 $\pm$ 0.002}
& \textbf{17.784 $\pm$ 0.003} & \textbf{17.465 $\pm$ 0.003}
& \textbf{18.111 $\pm$ 0.002} & \textbf{17.167 $\pm$ 0.004}
\\
\midrule
\Block{1-*}{Food-101}
\\
\midrule 
\Block{3-1}{UCR} 
& CCP     & 0.006 $\pm$ 0.002 & 0.006 $\pm$ 0.002
& 0.009 $\pm$ 0.003 & 0.008 $\pm$ 0.001
& 0.006 $\pm$ 0.001 & 0.008 $\pm$ 0.002
\\ 
& Cluster-CP     & 0.003 $\pm$ 0.002 & 0.009 $\pm$ 0.003
& 0.004 $\pm$ 0.001 & 0.009 $\pm$ 0.002
& 0.011 $\pm$ 0.003 & 0.011 $\pm$ 0.002
\\ 
&\textbf{\newCP} & \textbf{0.000 $\pm$ 0.000} & \textbf{0.000 $\pm$ 0.000}
& \textbf{0.000 $\pm$ 0.000} & \textbf{0.001 $\pm$ 0.001} 
& \textbf{0.000 $\pm$ 0.000} & \textbf{0.000 $\pm$ 0.000}
\\ 
\midrule
\Block{3-1}{APSS}
& CCP     & 27.022 $\pm$ 0.192 & 30.900 $\pm$ 0.170
& 30.943 $\pm$ 0.119 & 35.912 $\pm$ 0.105
& 27.415 $\pm$ 0.194 & 36.776 $\pm$ 0.132
\\ 
& Cluster-CP     & 28.953 $\pm$ 0.333 & 33.375 $\pm$ 0.377
& 33.079 $\pm$ 0.393 & 38.301 $\pm$ 0.232
& 30.071 $\pm$ 0.412 & 39.632 $\pm$ 0.342
\\ 
& \textbf{\newCP} & \textbf{18.369 $\pm$ 0.004} & \textbf{21.556 $\pm$ 0.006}
& \textbf{21.499 $\pm$ 0.003} & \textbf{25.853 $\pm$ 0.004}
& \textbf{19.398 $\pm$ 0.006} & \textbf{26.585 $\pm$ 0.004}
\\
\bottomrule 
\end{NiceTabular}
}
\label{tab:overall_comparison_cluster_four_datasets}
\end{table}

\clearpage

\subsection{Comparison Experiments Using RAPS Score Function}
\label{appendix:raps}

With the same model, evaluation metrics, and RAPS score function \cite{angelopoulos2020uncertainty}, 
we add the comparison experiments with \texttt{CCP}, and \texttt{Cluster-CP} on four datasets with different imbalanced types and imbalance ratio $\rho = 0.5$ and $\rho = 0.1$. 
The regularization parameter for RAPS scoring function is from the set $k_{reg} \in \{3, 5, 7\}$ and $\lambda \in \{0.001, 0.01, 0.1\}$. 
We select the combination of $k_{reg}$ and $\lambda$ for each experiment with the same imbalanced type and imbalanced ratio on the same dataset, where most of the $APSS$ values of all methods are minimum. 
The overall performance is summarized in Table \ref{tab:overall_comparison_RAPS_four_datasets}. 
We highlight that we also select the $g$ from the set $g \in \{0.25, 0.5, 0.75, 1.0\}$ to find the minimal $g$ that CCP, Cluster-CP, and \texttt{\newCP} approximately achieves the target class conditional coverage.

\vspace{1.0ex}

Based on the results in Table \ref{tab:overall_comparison_RAPS_four_datasets}, we make the following observations: (i) \texttt{CCP}, \texttt{Cluster-CP}, and \texttt{\newCP} can guarantee the class-conditional coverage; and (ii) \texttt{\newCP} significantly outperforms \text{CCP} and \text{Cluster-CP} on three datasets by producing smaller prediction sets.

\begin{table}[!ht]
\centering
\caption{
Results comparing \texttt{CCP}, \texttt{Cluster-CP}, and \texttt{\newCP}~with ResNet-20 model and the RAPS scoring function under different imbalance ratios $\rho=0.5$ and $\rho=0.1$ when $\alpha = 0.1$.
The regularization parameter for RAPS scoring function is selected from the set $[3, 5, 7]$ and $[0.001, 0.01, 0.1]$. We select the best results for each element in the table.
We set UCR of \texttt{\newCP} the same as or better than that of \texttt{CCP} and \texttt{Cluster-CP} for a fair comparison of prediction set size.
The APSS results show that \texttt{\newCP} significantly outperforms \texttt{CCP} and \texttt{Cluster-CP} in terms of average prediction set size over all settings on CIFAR-100, mini-ImageNet, and Food-101.
}
\resizebox{\textwidth}{!}{
\begin{NiceTabular}{@{}cc!{~}cc!{~}cc!{~}cc@{}} \toprule 
\Block{2-1}{Measure}  & \Block{2-1}{Methods} & \Block{1-2}{\EXP} & & \Block{1-2}{\POLY} & & \Block{1-2}{\MAJ} \\ 
 & & $\rho$ = 0.5 & $\rho$ = 0.1 & $\rho$ = 0.5 & $\rho$ = 0.1 & $\rho$ = 0.5 & $\rho$ = 0.1\\ 
\midrule 
\Block{1-*}{CIFAR-10}
\\
\midrule
\Block{3-1}{UCR} 
& CCP     & 0.050 $\pm$ 0.016 & \textbf{0.010 $\pm$ 0.020}
& 0.100 $\pm$ 0.028 & \textbf{0.050 $\pm$ 0.021}
& 0.070 $\pm$ 0.014 & \textbf{0.040 $\pm$ 0.015}
\\ 
& Cluster-CP     & \textbf{0.010 $\pm$ 0.009} & \textbf{0.010 $\pm$ 0.010}
& \textbf{0.080 $\pm$ 0.019} & 0.060 $\pm$ 0.015
& \textbf{0.020 $\pm$ 0.025} & 0.070 $\pm$ 0.014
\\ 
& \textbf{\newCP} & 0.050 $\pm$ 0.016 & \textbf{0.010 $\pm$ 0.020}
& 0.100 $\pm$ 0.028 & \textbf{0.050 $\pm$ 0.021}
& 0.070 $\pm$ 0.014 & \textbf{0.040 $\pm$ 0.015}
\\ 
\midrule
\Block{3-1}{APSS}
& CCP   & \textbf{1.555 $\pm$ 0.010} & \textbf{1.855 $\pm$ 0.014}
& \textbf{1.538 $\pm$ 0.010} & \textbf{1.776 $\pm$ 0.012}
& \textbf{1.840 $\pm$ 0.020} & \textbf{2.632 $\pm$ 0.012}
\\ 
& Cluster-CP     & 1.714 $\pm$ 0.018 & 2.162 $\pm$ 0.015
& 1.706 $\pm$ 0.014 & 1.929 $\pm$ 0.013
& 1.787 $\pm$ 0.019 & 2.968 $\pm$ 0.024
\\ 
& \textbf{\newCP} & \textbf{1.555 $\pm$ 0.010} & \textbf{1.855 $\pm$ 0.014}
& \textbf{1.538 $\pm$ 0.010} & \textbf{1.776 $\pm$ 0.012}
& \textbf{1.840 $\pm$ 0.020} & \textbf{2.632 $\pm$ 0.012}
\\
\midrule
\Block{1-*}{CIFAR-100}
\\
\midrule
\Block{3-1}{UCR} 
& CCP     & 0.007 $\pm$ 0.002 & \textbf{0.011 $\pm$ 0.002}
& 0.010 $\pm$ 0.002 & \textbf{0.015 $\pm$ 0.003}
& 0.015 $\pm$ 0.003 & 0.008 $\pm$ 0.004
\\ 
& Cluster-CP     & 0.012 $\pm$ 0.002 & 0.017 $\pm$ 0.004
& 0.019 $\pm$ 0.004 & 0.034 $\pm$ 0.005
& \textbf{0.008 $\pm$ 0.003} & 0.018 $\pm$ 0.006
\\ 
& \textbf{\newCP} & \textbf{0.005 $\pm$ 0.002} & \textbf{0.011 $\pm$ 0.002}
& \textbf{0.009 $\pm$ 0.003} & \textbf{0.015 $\pm$ 0.003}
& 0.015 $\pm$ 0.003 & \textbf{0.008 $\pm$ 0.004}
\\ 
\midrule
\Block{3-1}{APSS}
& CCP     & 44.250 $\pm$ 0.342 & 50.970 $\pm$ 0.345
& 49.886 $\pm$ 0.353 & 64.332 $\pm$ 0.236
& 48.343 $\pm$ 0.353 & 64.663 $\pm$ 0.535
\\ 
& Cluster-CP     & 29.267 $\pm$ 0.612 & 37.795 $\pm$ 0.862
& 38.258 $\pm$ 0.320 & 52.374 $\pm$ 0.592
& 31.513 $\pm$ 0.325 & 50.379 $\pm$ 0.684
\\ 
& \textbf{\newCP} & \textbf{17.705 $\pm$ 0.004} & \textbf{21.954 $\pm$ 0.005}
& \textbf{23.048 $\pm$ 0.008} & \textbf{33.185 $\pm$ 0.005}
& \textbf{18.581 $\pm$ 0.006} & \textbf{32.699 $\pm$ 0.006}
\\
\midrule
\Block{1-*}{mini-ImageNet}
\\
\midrule 
\Block{3-1}{UCR} 
& CCP     & 0.008 $\pm$ 0.003 & 0.009 $\pm$ 0.004
& 0.005 $\pm$ 0.002 & 0.004 $\pm$ 0.002
& 0.009 $\pm$ 0.003 & 0.005 $\pm$ 0.002
\\ 
& Cluster-CP     & 0.006 $\pm$ 0.002 & 0.013 $\pm$ 0.005
& 0.009 $\pm$ 0.003 & 0.016 $\pm$ 0.001
& 0.007 $\pm$ 0.002 & 0.009 $\pm$ 0.004
\\ 
& \textbf{\newCP} & \textbf{0.000 $\pm$ 0.000} & \textbf{0.001 $\pm$ 0.001}
& \textbf{0.000 $\pm$ 0.000} & \textbf{0.000 $\pm$ 0.000}
& \textbf{0.000 $\pm$ 0.000} & \textbf{0.000 $\pm$ 0.000}
\\ 
\midrule
\Block{3-1}{APSS}
& CCP     & 26.756 $\pm$ 0.178 & 26.212 $\pm$ 0.199
& 26.689 $\pm$ 0.142 & 26.248 $\pm$ 0.219
& 27.397 $\pm$ 0.162 & 25.725 $\pm$ 0.214
\\ 
& Cluster-CP     & 26.027 $\pm$ 0.325 & 25.415 $\pm$ 0.289
& 26.288 $\pm$ 0.407 & 25.712 $\pm$ 0.315
& 26.969 $\pm$ 0.305 & 25.532 $\pm$ 0.350
\\ 
& \textbf{\newCP} & \textbf{18.129 $\pm$ 0.003} & \textbf{17.082 $\pm$ 0.002}
& \textbf{17.784 $\pm$ 0.003} & \textbf{17.465 $\pm$ 0.003}
& \textbf{18.111 $\pm$ 0.002} & \textbf{17.167 $\pm$ 0.004}
\\
\midrule
\Block{1-*}{Food-101}
\\
\midrule 
\Block{3-1}{UCR} 
& CCP     & 0.006 $\pm$ 0.003 & 0.006 $\pm$ 0.002
& 0.009 $\pm$ 0.003 & 0.008 $\pm$ 0.001
& 0.006 $\pm$ 0.002 & 0.008 $\pm$ 0.002
\\ 
& Cluster-CP     & 0.004 $\pm$ 0.003 & 0.012 $\pm$ 0.004
& 0.006 $\pm$ 0.002 & 0.006 $\pm$ 0.003
& 0.011 $\pm$ 0.003 & 0.014 $\pm$ 0.004
\\ 
&\textbf{ \newCP} & \textbf{0.000 $\pm$ 0.000} & \textbf{0.000 $\pm$ 0.000}
& \textbf{0.000 $\pm$ 0.000} & \textbf{0.001 $\pm$ 0.001} 
& \textbf{0.000 $\pm$ 0.000} & \textbf{0.000 $\pm$ 0.000}
\\ 
\midrule
\Block{3-1}{APSS}
& CCP     & 27.022 $\pm$ 0.192 & 30.900 $\pm$ 0.170
& 30.966 $\pm$ 0.125 & 35.940 $\pm$ 0.111
& 27.439 $\pm$ 0.203 & 36.802 $\pm$ 0.138
\\ 
& Cluster-CP     & 28.953 $\pm$ 0.333 & 33.375 $\pm$ 0.377
& 33.337 $\pm$ 0.409 & 38.499 $\pm$ 0.216
& 29.946 $\pm$ 0.407 & 39.529 $\pm$ 0.306
\\ 
& \textbf{\newCP} & \textbf{18.369 $\pm$ 0.004} & \textbf{21.556$\pm$ 0.006}
& \textbf{21.499 $\pm$ 0.003} & \textbf{25.853 $\pm$ 0.004}
& \textbf{19.397 $\pm$ 0.006} & \textbf{26.585 $\pm$ 0.004}
\\
\bottomrule 
\end{NiceTabular}
}
\label{tab:overall_comparison_RAPS_four_datasets}
\end{table}

\subsection{Comparison Experiments Using HPS Score Function}
\label{appendix:hps}
With the same model, evaluation metrics, and HPS score function \cite{angelopoulos2020uncertainty}, 
we add the comparison experiments with \texttt{CCP}, and \texttt{Cluster-CP} on four datasets with different imbalanced types and imbalance ratio $\rho = 0.5$ and $\rho = 0.1$. 
The overall performance is summarized in Table \ref{tab:overall_comparison_HPS_four_datasets}. 
We highlight that we also select the $g$ from the set $g \in \{0.25, 0.5, 0.75, 1.0\}$ to find the minimal $g$ that CCP, Cluster-CP, and \texttt{\newCP} approximately achieves the target class conditional coverage.
\vspace{1.0ex}

Based on the results in Table \ref{tab:overall_comparison_HPS_four_datasets}, we make the following observations: (i) \texttt{CCP}, \texttt{Cluster-CP}, and \texttt{\newCP} can guarantee the class-conditional coverage; and (ii) \texttt{\newCP} significantly outperforms \text{CCP} and \text{Cluster-CP} on three datasets by producing smaller prediction sets.

\begin{table}[!ht]
\centering
\caption{
Results comparing \texttt{CCP}, \texttt{Cluster-CP}, and \texttt{\newCP}~with ResNet-20 model and the HPS scoring function under different imbalance ratios $\rho=0.5$ and $\rho=0.1$ when $\alpha = 0.1$.
We set UCR of \texttt{\newCP} the same as or better than that of \texttt{CCP} and \texttt{Cluster-CP} for a fair comparison of prediction set size.
\texttt{\newCP} significantly outperforms \texttt{CCP} and \texttt{Cluster-CP} with $20.91\%$ (four datasets) or $27.88\%$ (excluding CIFAR-10) reduction in APSS.
}
\resizebox{\textwidth}{!}{
\begin{NiceTabular}{@{}cc!{~}cc!{~}cc!{~}cc@{}} \toprule 
\Block{2-1}{Measure}  & \Block{2-1}{Methods} & \Block{1-2}{\EXP} & & \Block{1-2}{\POLY} & & \Block{1-2}{\MAJ} \\ 
 & & $\rho$ = 0.5 & $\rho$ = 0.1 & $\rho$ = 0.5 & $\rho$ = 0.1 & $\rho$ = 0.5 & $\rho$ = 0.1\\ 
\midrule 
\Block{1-*}{CIFAR-10}
\\
\midrule
\Block{3-1}{UCR} 
& CCP     & 0.050 $\pm$ 0.016 & \textbf{0.010 $\pm$ 0.020}
& 0.100 $\pm$ 0.028 & \textbf{0.050 $\pm$ 0.021}
& 0.070 $\pm$ 0.014 & \textbf{0.040 $\pm$ 0.015}
\\ 
& Cluster-CP     & \textbf{0.010 $\pm$ 0.009} & \textbf{0.010 $\pm$ 0.010}
& \textbf{0.080 $\pm$ 0.019} & 0.060 $\pm$ 0.015
& \textbf{0.020 $\pm$ 0.025} & 0.070 $\pm$ 0.014
\\ 
& \textbf{\newCP} & 0.050 $\pm$ 0.016 & \textbf{0.010 $\pm$ 0.020}
& 0.100 $\pm$ 0.028 & \textbf{0.050 $\pm$ 0.021}
& 0.070 $\pm$ 0.014 & \textbf{0.040 $\pm$ 0.015}
\\ 
\midrule
\Block{3-1}{APSS}
& CCP     & \textbf{1.144 $\pm$ 0.005} & \textbf{1.324 $\pm$ 0.007}
& \textbf{1.137 $\pm$ 0.003} & \textbf{1.243 $\pm$ 0.005}
& \textbf{1.272 $\pm$ 0.008} & \textbf{1.936 $\pm$ 0.010}
\\ 
& Cluster-CP    & 1.214 $\pm$ 0.008 & 1.508 $\pm$ 0.010
& 1.211 $\pm$ 0.004 & 1.354 $\pm$ 0.005
& 1.336 $\pm$ 0.009 & 2.312 $\pm$ 0.025
\\ 
& \textbf{\newCP} & \textbf{1.144 $\pm$ 0.005} & \textbf{1.324 $\pm$ 0.007}
& \textbf{1.137 $\pm$ 0.003} & \textbf{1.243 $\pm$ 0.005}
& \textbf{1.272 $\pm$ 0.008} & \textbf{1.936 $\pm$ 0.010}
\\
\midrule
\Block{1-*}{CIFAR-100}
\\
\midrule
\Block{3-1}{UCR} 
& CCP     & 0.007 $\pm$ 0.002 & \textbf{0.011 $\pm$ 0.002}
& 0.010 $\pm$ 0.002 & \textbf{0.015 $\pm$ 0.003}
& 0.015 $\pm$ 0.003 & 0.008 $\pm$ 0.004
\\ 
& Cluster-CP     & 0.012 $\pm$ 0.002 & 0.017 $\pm$ 0.004
& 0.019 $\pm$ 0.004 & 0.034 $\pm$ 0.005
& \textbf{0.008 $\pm$ 0.003} & 0.018 $\pm$ 0.006
\\ 
& \textbf{\newCP} & \textbf{0.005 $\pm$ 0.002} & \textbf{0.011 $\pm$ 0.002}
& \textbf{0.009 $\pm$ 0.003} & \textbf{0.015 $\pm$ 0.003}
& 0.015 $\pm$ 0.003 & \textbf{0.008 $\pm$ 0.004}
\\ 
\midrule
\Block{3-1}{APSS}
& CCP     & 41.351 $\pm$ 0.242 & 49.469 $\pm$ 0.344
& 48.063 $\pm$ 0.376 & 63.963 $\pm$ 0.277
& 46.125 $\pm$ 0.351 & 64.371 $\pm$ 0.564
\\ 
& Cluster-CP     & 27.566 $\pm$ 0.555 & 35.528 $\pm$ 0.979
& 36.101 $\pm$ 0.565 & 51.333 $\pm$ 0.776
& 29.323 $\pm$ 0.363 & 50.519 $\pm$ 0.679
\\ 
& \textbf{\newCP} & \textbf{20.363 $\pm$ 0.006} & \textbf{25.212 $\pm$ 0.010}
& \textbf{25.908 $\pm$ 0.007} & \textbf{36.951 $\pm$ 0.018}
& \textbf{21.149 $\pm$ 0.006} & \textbf{35.606 $\pm$ 0.005}
\\
\midrule
\Block{1-*}{mini-ImageNet}
\\
\midrule 
\Block{3-1}{UCR} 
& CCP     & 0.008 $\pm$ 0.003 & 0.009 $\pm$ 0.004
& 0.005 $\pm$ 0.002 & 0.004 $\pm$ 0.002
& 0.009 $\pm$ 0.003 & 0.005 $\pm$ 0.002
\\ 
& Cluster-CP     & 0.006 $\pm$ 0.002 & 0.013 $\pm$ 0.005
& 0.009 $\pm$ 0.003 & 0.016 $\pm$ 0.001
& 0.007 $\pm$ 0.002 & 0.009 $\pm$ 0.004
\\ 
& \textbf{\newCP} & \textbf{0.000 $\pm$ 0.000} & \textbf{0.001 $\pm$ 0.001}
& \textbf{0.000 $\pm$ 0.000} & \textbf{0.000 $\pm$ 0.000}
& \textbf{0.000 $\pm$ 0.000} & \textbf{0.000 $\pm$ 0.000}
\\ 
\midrule
\Block{3-1}{APSS}
& CCP     & 24.633 $\pm$ 0.212 & 24.467 $\pm$ 0.149
& 24.379 $\pm$ 0.152 & 24.472 $\pm$ 0.167
& 25.449 $\pm$ 0.196 & 23.885 $\pm$ 0.159
\\ 
& Cluster-CP     & 23.911 $\pm$ 0.322 & 24.023 $\pm$ 0.195
& 24.233 $\pm$ 0.428 & 23.263 $\pm$ 0.295
& 24.987 $\pm$ 0.319 & 23.323 $\pm$ 0.378
\\ 
& \textbf{\newCP} & \textbf{17.830 $\pm$ 0.104} & \textbf{17.036 $\pm$ 0.014}
& \textbf{17.684 $\pm$ 0.062} & \textbf{17.393 $\pm$ 0.013}
& \textbf{18.024 $\pm$ 0.049} & \textbf{17.086 $\pm$ 0.059}
\\
\midrule
\Block{1-*}{Food-101}
\\
\midrule 
\Block{3-1}{UCR} 
& CCP     & 0.006 $\pm$ 0.003 & 0.006 $\pm$ 0.002
& 0.009 $\pm$ 0.003 & 0.008 $\pm$ 0.001
& 0.006 $\pm$ 0.002 & 0.008 $\pm$ 0.002
\\ 
& Cluster-CP     & 0.004 $\pm$ 0.003 & 0.012 $\pm$ 0.004
& 0.006 $\pm$ 0.002 & 0.006 $\pm$ 0.003
& 0.011 $\pm$ 0.003 & 0.014 $\pm$ 0.004
\\ 
&\textbf{ \newCP} & \textbf{0.000 $\pm$ 0.000} & \textbf{0.000 $\pm$ 0.000}
& \textbf{0.000 $\pm$ 0.000} & \textbf{0.001 $\pm$ 0.001} 
& \textbf{0.000 $\pm$ 0.000} & \textbf{0.000 $\pm$ 0.000}
\\ 
\midrule
\Block{3-1}{APSS}
& CCP     & 26.481 $\pm$ 0.142 & 30.524 $\pm$ 0.152
& 30.787 $\pm$ 0.099 & 35.657 $\pm$ 0.107
& 26.826 $\pm$ 0.163 & 36.518 $\pm$ 0.122
\\ 
& Cluster-CP     & 29.347 $\pm$ 0.288 & 33.806 $\pm$ 0.513
& 33.407 $\pm$ 0.345 & 38.956 $\pm$ 0.242
& 29.606 $\pm$ 0.436 & 39.880 $\pm$ 0.318
\\ 
& \textbf{\newCP} & \textbf{18.337 $\pm$ 0.004} & \textbf{21.558$\pm$ 0.006}
& \textbf{21.477 $\pm$ 0.003} & \textbf{25.853 $\pm$ 0.005}
& \textbf{19.396 $\pm$ 0.008} & \textbf{26.584 $\pm$ 0.003}
\\
\bottomrule 
\end{NiceTabular}
}
\label{tab:overall_comparison_HPS_four_datasets}
\end{table}

\subsection{Comparison Experiments with different target coverages}
\label{appendix:alphas}

With the same model, evaluation metrics, and scoring functions, 
we add the comparison experiments with \texttt{CCP}, and \texttt{Cluster-CP} on four datasets with different imbalanced types and imbalance ratio $\rho = 0.5$ and $\rho = 0.1$ under the different $\alpha$ values. 
The overall performance is summarized in Table \ref{tab:overall_comparison_four_datasets_0.05} and \ref{tab:overall_comparison_four_datasets_0.01}, with $\alpha = 0.05$ and $\alpha = 0.01$, respectively. 
We highlight that we also select the $g$ from the set $g \in [0.15, 0.75]$ with $0.05$ range to find the minimal $g$ that CCP, Cluster-CP, and \texttt{\newCP} approximately achieves the target class conditional coverage.

\vspace{1.0ex}

Based on the results in Table \ref{tab:overall_comparison_RAPS_four_datasets}, we make the following observations: (i) \texttt{CCP}, \texttt{Cluster-CP}, and \texttt{\newCP} can guarantee the class-conditional coverage; and (ii) \texttt{\newCP} significantly outperforms \text{CCP} and \text{Cluster-CP} on three datasets by producing smaller prediction sets.

\begin{table*}[!ht]
\centering
\caption{
APSS results comparing \texttt{CCP}, \texttt{Cluster-CP}, and \texttt{\newCP}~with ResNet-20 model under different imbalance ratio $\rho=0.5$ and $\rho=0.1$ where $\alpha = 0.05$.
For a fair comparison of prediction set size, we set UCR of \newCP~the same as or smaller (more restrictive) than that of \texttt{CCP} and \texttt{Cluster-CP} under $0.16$ on CIFAR-10 and $0.03$ on other datasets. 
The APSS results show that \texttt{\newCP}~significantly outperforms \texttt{CCP} and \texttt{Cluster-CP} in terms of average prediction set size with $21.036\%$ (four datasets) or $28.048\%$ (excluding CIFAR-10) reduction in prediction size on average over $\min \{ \texttt{CCP}, \texttt{cluster-CP}\}$. 
}
\resizebox{\textwidth}{!}{
\begin{NiceTabular}{@{}cc!{~}cc!{~}cc!{~}cc@{}} \toprule 
\Block{2-1}{Conformity Score}  & \Block{2-1}{Methods} & \Block{1-2}{\EXP} & & \Block{1-2}{\POLY} & & \Block{1-2}{\MAJ} \\ 
 & & $\rho$ = 0.5 & $\rho$ = 0.1 & $\rho$ = 0.5 & $\rho$ = 0.1 & $\rho$ = 0.5 & $\rho$ = 0.1\\ 
\midrule 
\Block{1-*}{CIFAR-10}
\\
\midrule
\Block{3-1}{APS} 
& CCP     & \textbf{2.861 $\pm$ 0.027} & \textbf{3.496 $\pm$ 0.037}
& \textbf{2.744 $\pm$ 0.033} & \textbf{3.222 $\pm$ 0.018}
& \textbf{3.269 $\pm$ 0.037} & \textbf{4.836 $\pm$ 0.035}
\\ 
& Cluster-CP    & 3.443 $\pm$ 0.041 & 4.551 $\pm$ 0.049
& 3.309 $\pm$ 0.037 & 4.012 $\pm$ 0.039
& 4.075 $\pm$ 0.069 & 5.958 $\pm$ 0.070
\\ 
& \textbf{\newCP} & \textbf{2.861 $\pm$ 0.027} & \textbf{3.496 $\pm$ 0.037}
& \textbf{2.744 $\pm$ 0.033} & \textbf{3.222 $\pm$ 0.018}
& \textbf{3.269 $\pm$ 0.037} & \textbf{4.836 $\pm$ 0.035}
\\
\midrule
\Block{3-1}{RAPS}
& CCP     & \textbf{2.833 $\pm$ 0.018} & \textbf{3.448 $\pm$ 0.036}
& \textbf{2.774 $\pm$ 0.033} & \textbf{3.231 $\pm$ 0.021}
& \textbf{3.301 $\pm$ 0.024} & \textbf{4.842 $\pm$ 0.037}
\\ 
& Cluster-CP    & 3.430 $\pm$ 0.044 & 4.389 $\pm$ 0.062
& 3.352 $\pm$ 0.035 & 3.876 $\pm$ 0.034
& 4.044 $\pm$ 0.055 & 5.959 $\pm$ 0.083
\\ 
& \textbf{\newCP}  & \textbf{2.833 $\pm$ 0.018} & \textbf{3.448 $\pm$ 0.036}
& \textbf{2.774 $\pm$ 0.033} & \textbf{3.231 $\pm$ 0.021}
& \textbf{3.301 $\pm$ 0.024} & \textbf{4.842 $\pm$ 0.037}
\\
\midrule
\Block{1-*}{CIFAR-100}
\\
\midrule
\Block{3-1}{APS} 
& CCP     & 44.019 $\pm$ 0.295 & 51.004 $\pm$ 0.366
& 49.564 $\pm$ 0.315 & 64.314 $\pm$ 0.231
& 48.024 $\pm$ 0.386 & 64.941 $\pm$ 0.532
\\ 
& Cluster-CP     & 39.641 $\pm$ 0.567 & 46.746 $\pm$ 0.147
& 47.654 $\pm$ 0.371 & 62.340 $\pm$ 0.404
& 37.634 $\pm$ 0.537 & 60.841 $\pm$ 0.391
\\ 
& \textbf{\newCP} & \textbf{32.128 $\pm$ 0.011} & \textbf{38.769 $\pm$ 0.006}
& \textbf{39.930 $\pm$ 0.008} & \textbf{53.147 $\pm$ 0.010}
& \textbf{34.361 $\pm$ 0.007} & \textbf{51.498 $\pm$ 0.009}
\\
\midrule
\Block{3-1}{RAPS}
& CCP     & 44.234 $\pm$ 0.341 & 50.950 $\pm$ 0.344
& 49.889 $\pm$ 0.355 & 64.339 $\pm$ 0.236
& 48.310 $\pm$ 0.353 & 64.628 $\pm$ 0.535
\\ 
& Cluster-CP     & 39.212 $\pm$ 0.365 & 46.840 $\pm$ 0.186
& 49.094 $\pm$ 0.280 & 62.095 $\pm$ 0.278
& 41.596 $\pm$ 0.323 & 60.158 $\pm$ 0.536
\\ 
& \textbf{\newCP} & \textbf{32.135 $\pm$ 0.010} & \textbf{38.793 $\pm$ 0.007}
& \textbf{39.871 $\pm$ 0.010} & \textbf{53.169 $\pm$ 0.009}
& \textbf{34.380 $\pm$ 0.007} & \textbf{51.512 $\pm$ 0.008}
\\
\midrule
\Block{1-*}{mini-ImageNet}
\\
\midrule 
\Block{3-1}{APS} 
& CCP     & 58.527 $\pm$ 0.445 & 57.527 $\pm$ 0.408
& 60.327 $\pm$ 0.520 & 56.581 $\pm$ 0.438
& 59.360 $\pm$ 0.430 & 56.636 $\pm$ 0.469
\\ 
& Cluster-CP     & 47.613 $\pm$ 0.544 & 46.650 $\pm$ 0.699
& 47.117 $\pm$ 0.930 & 45.360 $\pm$ 0.582
& 59.002 $\pm$ 0.434 & 56.147 $\pm$ 0.456
\\ 
& \textbf{\newCP} & \textbf{32.046 $\pm$ 0.002} & \textbf{31.729 $\pm$ 0.003}
& \textbf{31.718 $\pm$ 0.004} & \textbf{32.048 $\pm$ 0.003}
& \textbf{32.909 $\pm$ 0.007} & \textbf{31.441 $\pm$ 0.004}
\\
\midrule
\Block{3-1}{RAPS}
& CCP     & 58.615 $\pm$ 0.428 & 57.626 $\pm$ 0.394
& 60.173 $\pm$ 0.527 & 56.702 $\pm$ 0.414
& 59.532 $\pm$ 0.430 & 56.903 $\pm$ 0.460
\\ 
& Cluster-CP     & 47.427 $\pm$ 0.588 & 46.767 $\pm$ 0.724
& 47.302 $\pm$ 1.126 & 45.603 $\pm$ 0.639
& 59.408 $\pm$ 0.482 & 56.779 $\pm$ 0.486
\\ 
& \textbf{\newCP} & \textbf{32.040 $\pm$ 0.003} & \textbf{31.741 $\pm$ 0.003}
& \textbf{31.752 $\pm$ 0.003} & \textbf{32.067 $\pm$ 0.002}
& \textbf{32.914 $\pm$ 0.005} & \textbf{31.417 $\pm$ 0.005}
\\
\midrule
\Block{1-*}{Food-101}
\\
\midrule 
\Block{3-1}{APS} 
& CCP     & 55.967 $\pm$ 0.464 & 60.374 $\pm$ 0.383
& 60.717 $\pm$ 0.596 & 65.698 $\pm$ 0.405
& 56.934 $\pm$ 0.446 & 66.654 $\pm$ 0.511
\\ 
& Cluster-CP     & 48.699 $\pm$ 0.512 & 55.288 $\pm$ 0.815
& 54.063 $\pm$ 0.885 & 60.104 $\pm$ 0.608
& 48.894 $\pm$ 0.919 & 59.432 $\pm$ 0.754
\\ 
& \textbf{\newCP} & \textbf{31.224 $\pm$ 0.004} & \textbf{35.273 $\pm$ 0.007}
& \textbf{35.364 $\pm$ 0.003} & \textbf{41.109 $\pm$ 0.005}
& \textbf{31.661 $\pm$ 0.005} & \textbf{39.135 $\pm$ 0.003}
\\
\midrule
\Block{3-1}{RAPS}
& CCP     & 55.872 $\pm$ 0.465 & 60.764 $\pm$ 0.394
& 60.618 $\pm$ 0.579 & 65.681 $\pm$ 0.401
& 56.982 $\pm$ 0.447 & 66.615 $\pm$ 0.504
\\ 
& Cluster-CP     & 48.371 $\pm$ 0.513 & 55.155 $\pm$ 0.775
& 53.813 $\pm$ 0.864 & 59.912 $\pm$ 0.530
& 49.259 $\pm$ 0.846 & 59.307 $\pm$ 0.648
\\ 
& \textbf{\newCP} & \textbf{31.229 $\pm$ 0.004} & \textbf{35.283 $\pm$ 0.006}
& \textbf{35.379 $\pm$ 0.003} & \textbf{41.113 $\pm$ 0.005}
& \textbf{31.631 $\pm$ 0.004} & \textbf{39.118 $\pm$ 0.003}
\\
\bottomrule 
\end{NiceTabular}
}
\label{tab:overall_comparison_four_datasets_0.05}
\end{table*}

\begin{table*}[!ht]
\centering
\caption{
APSS results comparing \texttt{CCP}, \texttt{Cluster-CP}, and \texttt{\newCP}~with ResNet-20 model under different imbalance ratio $\rho=0.5$ and $\rho=0.1$ where $\alpha = 0.01$.
For a fair comparison of prediction set size, we set UCR of \newCP~the same as or smaller (more restrictive) than that of \texttt{CCP} and \texttt{Cluster-CP} under $0.16$ on CIFAR-10 and $0.03$ on other datasets. 
The APSS results show that \texttt{\newCP}~significantly outperforms \texttt{CCP} and \texttt{Cluster-CP} in terms of average prediction set size with $16.911\%$ (four datasets) or $22.549\%$ (excluding CIFAR-10) reduction in prediction size on average over $\min \{ \texttt{CCP}, \texttt{cluster-CP}\}$. 
}
\resizebox{\textwidth}{!}{
\begin{NiceTabular}{@{}cc!{~}cc!{~}cc!{~}cc@{}} \toprule 
\Block{2-1}{Conformity Score}  & \Block{2-1}{Methods} & \Block{1-2}{\EXP} & & \Block{1-2}{\POLY} & & \Block{1-2}{\MAJ} \\ 
 & & $\rho$ = 0.5 & $\rho$ = 0.1 & $\rho$ = 0.5 & $\rho$ = 0.1 & $\rho$ = 0.5 & $\rho$ = 0.1\\ 
\midrule 
\Block{1-*}{CIFAR-10}
\\
\midrule
\Block{3-1}{APS} 
& CCP     & 7.250 $\pm$ 0.164 & \textbf{7.387 $\pm$ 0.116}
& 7.173 $\pm$ 0.079 &7.596 $\pm$ 0.109
& 7.392 $\pm$ 0.128 & \textbf{8.864 $\pm$ 0.108}
\\ 
& Cluster-CP    & \textbf{5.528 $\pm$ 0.103} & 8.332 $\pm$ 0.060
& 6.954 $\pm$ 0.084 & 7.762 $\pm$ 0.143
& 7.586 $\pm$ 0.113 & 9.308 $\pm$ 0.054
\\ 
& \textbf{\newCP} & 5.671 $\pm$ 0.046 & \textbf{7.387 $\pm$ 0.116}
& \textbf{6.309 $\pm$ 0.042} & \textbf{7.276 $\pm$ 0.010}
& \textbf{6.779 $\pm$ 0.013} & \textbf{8.864 $\pm$ 0.108}
\\
\midrule
\Block{3-1}{RAPS}
& CCP     & 7.294 $\pm$ 0.160 & \textbf{7.458 $\pm$ 0.101}
& 7.067 $\pm$ 0.106 & 7.597 $\pm$ 0.096
& 7.547 $\pm$ 0.134 & \textbf{8.884 $\pm$ 0.106}
\\ 
& Cluster-CP    & \textbf{5.568 $\pm$ 0.103} & 8.288 $\pm$ 0.118
& 6.867 $\pm$ 0.078 & 7.795 $\pm$ 0.136
& 7.813 $\pm$ 0.142 & 9.239 $\pm$ 0.055
\\ 
& \textbf{\newCP} & 5.673 $\pm$ 0.040 & \textbf{7.458 $\pm$ 0.101}
& \textbf{6.310 $\pm$ 0.046} & \textbf{7.253 $\pm$ 0.006}
& \textbf{6.780 $\pm$ 0.015} & \textbf{8.884 $\pm$ 0.106}
\\
\midrule
\Block{1-*}{CIFAR-100}
\\
\midrule
\Block{3-1}{APS} 
& CCP     & 100.0 $\pm$ 0.0 & 100.0 $\pm$ 0.0
& 100.0 $\pm$ 0.0 & 100.0 $\pm$ 0.0
& \textbf{100.0 $\pm$ 0.0} & \textbf{100.0 $\pm$ 0.0}
\\ 
& Cluster-CP     & 65.523 $\pm$ 0.495 & 69.063 $\pm$ 0.512
& 67.012 $\pm$ 0.739 & 81.997 $\pm$ 0.390
& \textbf{100.0 $\pm$ 0.0} & \textbf{100.0 $\pm$ 0.0}
\\ 
& \textbf{\newCP} & \textbf{55.621 $\pm$ 0.007} & \textbf{63.039 $\pm$ 0.007}
& \textbf{60.258 $\pm$ 0.005} & \textbf{74.927 $\pm$ 0.007}
& \textbf{100.0 $\pm$ 0.0} & \textbf{100.0 $\pm$ 0.0}
\\
\midrule
\Block{3-1}{RAPS}
& CCP     & 100.0 $\pm$ 0.0 & 100.0 $\pm$ 0.0
& 100.0 $\pm$ 0.0 & 100.0 $\pm$ 0.0
& \textbf{100.0 $\pm$ 0.0} & \textbf{100.0 $\pm$ 0.0}
\\ 
& Cluster-CP     & 65.584 $\pm$ 0.508 & 69.373 $\pm$ 0.466
& 66.313 $\pm$ 0.745 & 82.043$\pm$ 0.439
& \textbf{100.0 $\pm$ 0.0} & \textbf{100.0 $\pm$ 0.0}
\\ 
& \textbf{\newCP} & \textbf{55.632 $\pm$ 0.008} & \textbf{63.021 $\pm$ 0.006}
& \textbf{60.205 $\pm$ 0.006} & \textbf{74.885 $\pm$ 0.006}
& \textbf{100.0 $\pm$ 0.0} & \textbf{100.0 $\pm$ 0.0}
\\
\midrule
\Block{1-*}{mini-ImageNet}
\\
\midrule 
\Block{3-1}{APS} 
& CCP     & 100.0 $\pm$ 0.0 & 100.0 $\pm$ 0.0
& 100.0 $\pm$ 0.0 & 100.0 $\pm$ 0.0
& 100.0 $\pm$ 0.0 & 100.0 $\pm$ 0.0
\\ 
& Cluster-CP     & 74.019 $\pm$ 0.699 & 71.300 $\pm$ 0.674
& 75.546 $\pm$ 0.683 & 70.996 $\pm$ 0.702
& 74.508 $\pm$ 0.531 & 72.803 $\pm$ 0.536
\\ 
& \textbf{\newCP} & \textbf{55.321 $\pm$ 0.003} & \textbf{54.214 $\pm$ 0.004}
& \textbf{56.018 $\pm$ 0.006} & \textbf{53.732 $\pm$ 0.004}
& \textbf{54.483 $\pm$ 0.007} & \textbf{53.522 $\pm$ 0.005}
\\
\midrule
\Block{3-1}{RAPS}
& CCP     & 100.0 $\pm$ 0.0 & 100.0 $\pm$ 0.0
& 100.0 $\pm$ 0.0 & 100.0 $\pm$ 0.0
& 100.0 $\pm$ 0.0 & 100.0 $\pm$ 0.0
\\ 
& Cluster-CP     & 73.893 $\pm$ 0.734 & 70.638 $\pm$ 0.657
& 75.546 $\pm$ 0.683 & 71.098 $\pm$ 0.706
& 74.675 $\pm$ 0.578 & 73.345 $\pm$ 0.474
\\ 
& \textbf{\newCP} & \textbf{55.270 $\pm$ 0.003} & \textbf{54.184 $\pm$ 0.003}
& \textbf{56.733 $\pm$ 0.006} & \textbf{53.736 $\pm$ 0.004}
& \textbf{55.304 $\pm$ 0.004} & \textbf{53.532 $\pm$ 0.005}
\\
\midrule
\Block{1-*}{Food-101}
\\
\midrule 
\Block{3-1}{APS} 
& CCP     & 101.0 $\pm$ 0.0 & 101.0 $\pm$ 0.0
& 101.0 $\pm$ 0.0 & 101.0 $\pm$ 0.0
& 101.0 $\pm$ 0.0 & 101.0 $\pm$ 0.0
\\ 
& Cluster-CP     & 81.489 $\pm$ 0.957 & 87.092 $\pm$ 0.588
& 82.257 $\pm$ 0.514 & 86.539 $\pm$ 0.453
& 83.293 $\pm$ 0.583 & 88.603 $\pm$ 0.401
\\ 
& \textbf{\newCP} & \textbf{67.443 $\pm$ 0.004} & \textbf{57.055 $\pm$ 0.005}
& \textbf{57.722 $\pm$ 0.006} & \textbf{62.931 $\pm$ 0.005}
& \textbf{68.267 $\pm$ 0.005} & \textbf{65.413 $\pm$ 0.005}
\\
\midrule
\Block{3-1}{RAPS}
& CCP     & 101.0 $\pm$ 0.0 & 101.0 $\pm$ 0.0
& 101.0 $\pm$ 0.0 & 101.0 $\pm$ 0.0
& 101.0 $\pm$ 0.0 & 101.0 $\pm$ 0.0
\\ 
& Cluster-CP     & 81.505 $\pm$ 0.955 & 87.103 $\pm$ 0.587
& 82.272 $\pm$ 0.513 & 86.517 $\pm$ 0.455
& 83.367 $\pm$ 0.635 & 88.604 $\pm$ 0.404
\\ 
& \textbf{\newCP} & \textbf{67.444 $\pm$ 0.004} & \textbf{57.069 $\pm$ 0.005}
& \textbf{57.722 $\pm$ 0.006} & \textbf{62.938 $\pm$ 0.004}
& \textbf{68.266 $\pm$ 0.005} & \textbf{65.457 $\pm$ 0.006}
\\
\bottomrule 
\end{NiceTabular}
}
\label{tab:overall_comparison_four_datasets_0.01}
\end{table*}

\clearpage

\subsection{Comparison Experiments when models are trained in different epochs}
\label{appendix:epochs}

With the same loss function, training criteria, evaluation metrics, and two scoring functions, 
we add the comparison experiments with \texttt{CCP}, and \texttt{Cluster-CP} on four datasets with different imbalanced types and imbalance ratio $\rho = 0.5$ and $\rho = 0.1$ and $\alpha = 0.1$ when models are trained with $50$ epochs. 
The overall performance is summarized in Table \ref{tab:overall_comparison_four_datasets_50epochs}. 
We highlight that we also select the $g$ from the set $g \in \{0.25, 0.5, 0.75, 1.0\}$ to find the minimal $g$ that CCP, Cluster-CP, and \texttt{\newCP} approximately achieves the target class conditional coverage.

\vspace{1.0ex}

Based on the results in Table \ref{tab:overall_comparison_RAPS_four_datasets}, we make the following observations: (i) \texttt{CCP}, \texttt{Cluster-CP}, and \texttt{\newCP} can guarantee the class-conditional coverage; and (ii) \texttt{\newCP} significantly outperforms \text{CCP} and \text{Cluster-CP} on three datasets by producing smaller prediction sets.

\begin{table*}[!ht]
\centering
\caption{
APSS results comparing \texttt{CCP}, \texttt{Cluster-CP}, and \texttt{\newCP}~with ResNet-20 model under different imbalance ratio $\rho=0.5$ and $\rho=0.1$ where $\alpha = 0.1$ and models are trained with $50$ epochs.
For a fair comparison of prediction set size, we set UCR of \newCP~the same as or smaller (more restrictive) than that of \texttt{CCP} and \texttt{Cluster-CP} under $0.16$ on CIFAR-10 and $0.03$ on other datasets. 
The APSS results show that \texttt{\newCP}~significantly outperforms \texttt{CCP} and \texttt{Cluster-CP} in terms of average prediction set size with $21.441\%$ (four datasets) or $28.588\%$ (excluding CIFAR-10) reduction in prediction size on average over $\min \{ \texttt{CCP}, \texttt{cluster-CP}\}$. 
}
\resizebox{\textwidth}{!}{
\begin{NiceTabular}{@{}cc!{~}cc!{~}cc!{~}cc@{}} \toprule 
\Block{2-1}{Conformity Score}  & \Block{2-1}{Methods} & \Block{1-2}{\EXP} & & \Block{1-2}{\POLY} & & \Block{1-2}{\MAJ} \\ 
 & & $\rho$ = 0.5 & $\rho$ = 0.1 & $\rho$ = 0.5 & $\rho$ = 0.1 & $\rho$ = 0.5 & $\rho$ = 0.1\\ 
\midrule 
\Block{1-*}{CIFAR-10}
\\
\midrule
\Block{3-1}{APS} 
& CCP     & \textbf{2.420 $\pm$ 0.019} & \textbf{2.661 $\pm$ 0.015}
& \textbf{2.399 $\pm$ 0.013} & \textbf{2.519 $\pm$ 0.022}
& \textbf{2.651 $\pm$ 0.031} & \textbf{4.053 $\pm$ 0.021}
\\ 
& Cluster-CP    & 4.006 $\pm$ 0.019 & 3.574 $\pm$ 0.023
& 3.144 $\pm$ 0.020 & 2.994 $\pm$ 0.029
& 3.698 $\pm$ 0.044 & 5.290 $\pm$ 0.016
\\ 
& \textbf{\newCP} & \textbf{2.420 $\pm$ 0.019} & \textbf{2.661 $\pm$ 0.015}
& \textbf{2.399 $\pm$ 0.013} & \textbf{2.519 $\pm$ 0.022}
& \textbf{2.651 $\pm$ 0.031} & \textbf{4.053 $\pm$ 0.021}
\\
\midrule
\Block{3-1}{RAPS}
& CCP     & \textbf{2.096 $\pm$ 0.014} & \textbf{2.533 $\pm$ 0.019}
& \textbf{2.383 $\pm$ 0.026} & \textbf{2.247 $\pm$ 0.017}
& \textbf{2.232 $\pm$ 0.019} & \textbf{3.233 $\pm$ 0.021}
\\ 
& Cluster-CP    & 2.625 $\pm$ 0.017 & 3.099 $\pm$ 0.021
& 2.840 $\pm$ 0.043 & 2.843 $\pm$ 0.026
& 2.770 $\pm$ 0.025 & 3.961 $\pm$ 0.029
\\ 
& \textbf{\newCP} & \textbf{2.096 $\pm$ 0.014} & \textbf{2.533 $\pm$ 0.019}
& \textbf{2.383 $\pm$ 0.026} & \textbf{2.247 $\pm$ 0.017}
& \textbf{2.232 $\pm$ 0.019} & \textbf{3.233 $\pm$ 0.021}
\\
\midrule
\Block{1-*}{CIFAR-100}
\\
\midrule
\Block{3-1}{APS} 
& CCP     & 52.655 $\pm$ 0.473 & 52.832 $\pm$ 0.308
& 54.523 $\pm$ 0.441 & 61.768 $\pm$ 0.195
& 52.119 $\pm$ 0.197 & 58.333 $\pm$ 0.299
\\ 
& Cluster-CP     & 42.990 $\pm$ 0.655 & 43.275 $\pm$ 0.833
& 44.114 $\pm$ 0.458 & 58.226 $\pm$ 0.627
& 39.841 $\pm$ 0.836 & 53.409 $\pm$ 0.520
\\ 
& \textbf{\newCP} & \textbf{24.872 $\pm$ 0.008} & \textbf{25.107 $\pm$ 0.006}
& \textbf{27.757 $\pm$ 0.004} & \textbf{35.733 $\pm$ 0.010}
& \textbf{24.496 $\pm$ 0.010} & \textbf{32.172 $\pm$ 0.007}
\\
\midrule
\Block{3-1}{RAPS}
& CCP     & 52.662 $\pm$ 0.473 & 52.841 $\pm$ 0.307
& 54.528 $\pm$ 0.442 & 61.766 $\pm$ 0.195
& 52.129 $\pm$ 0.197 & 58.331 $\pm$ 0.299
\\ 
& Cluster-CP     & 43.024 $\pm$ 0.648 & 43.277 $\pm$ 0.839
& 44.120 $\pm$ 0.458 & 58.212 $\pm$ 0.629
& 39.864 $\pm$ 0.845 & 53.402 $\pm$ 0.518
\\ 
& \textbf{\newCP} & \textbf{24.872 $\pm$ 0.008} & \textbf{25.107 $\pm$ 0.006}
& \textbf{27.757 $\pm$ 0.004} & \textbf{35.733 $\pm$ 0.010}
& \textbf{24.496 $\pm$ 0.010} & \textbf{32.173 $\pm$ 0.007}
\\
\midrule
\Block{1-*}{mini-ImageNet}
\\
\midrule 
\Block{3-1}{APS} 
& CCP     & 42.404 $\pm$ 0.213 & 41.154 $\pm$ 0.191
& 38.433 $\pm$ 0.248 & 36.363 $\pm$ 0.228
& 36.047 $\pm$ 0.191 & 37.600 $\pm$ 0.208
\\ 
& Cluster-CP     & 42.006 $\pm$ 0.430 & 41.101 $\pm$ 0.224
& 39.016 $\pm$ 0.273 & 36.046 $\pm$ 0.467
& 35.721 $\pm$ 0.355 & 37.975 $\pm$ 0.559
\\ 
& \textbf{\newCP} & \textbf{32.022 $\pm$ 0.005} & \textbf{31.909 $\pm$ 0.004}
& \textbf{28.460 $\pm$ 0.003} & \textbf{26.383 $\pm$ 0.003}
& \textbf{26.128 $\pm$ 0.005} & \textbf{28.127 $\pm$ 0.005}
\\
\midrule
\Block{3-1}{RAPS}
& CCP     & 42.516 $\pm$ 0.215 & 37.552 $\pm$ 0.192
& 38.730 $\pm$ 0.218 & 37.800 $\pm$ 0.186
& 36.595 $\pm$ 0.244 & 36.057 $\pm$ 0.206
\\ 
& Cluster-CP     & 42.231 $\pm$ 0.386 & 37.448 $\pm$ 0.332
& 38.602 $\pm$ 0.327 & 37.939 $\pm$ 0.309
& 36.351 $\pm$ 0.308 & 35.724 $\pm$ 0.242
\\ 
& \textbf{\newCP} & \textbf{32.022 $\pm$ 0.005} & \textbf{29.114 $\pm$ 0.004}
& \textbf{28.197 $\pm$ 0.006} & \textbf{27.626 $\pm$ 0.004}
& \textbf{25.853 $\pm$ 0.003} & \textbf{25.948 $\pm$ 0.003}
\\
\midrule
\Block{1-*}{Food-101}
\\
\midrule 
\Block{3-1}{APS} 
& CCP     & 41.669 $\pm$ 0.118 & 51.395 $\pm$ 0.247
& 44.261 $\pm$ 0.165 & 58.816 $\pm$ 0.162
& 52.672 $\pm$ 0.169 & 57.312 $\pm$ 0.162
\\ 
& Cluster-CP     & 44.883 $\pm$ 0.336 & 54.684 $\pm$ 0.475
& 47.794 $\pm$ 0.420 & 60.727 $\pm$ 0.178
& 56.100 $\pm$ 0.257 & 60.200 $\pm$ 0.543
\\ 
& \textbf{\newCP} & \textbf{31.987 $\pm$ 0.005} & \textbf{36.118 $\pm$ 0.016}
& \textbf{34.576 $\pm$ 0.006} & \textbf{49.299 $\pm$ 0.005}
& \textbf{43.680 $\pm$ 0.005} & \textbf{47.649 $\pm$ 0.006}
\\
\midrule
\Block{3-1}{RAPS}
& CCP     & 41.803 $\pm$ 0.157 & 48.548 $\pm$ 0.107
& 44.288 $\pm$ 0.165 & 56.592 $\pm$ 0.165
& 47.264 $\pm$ 0.120 & 56.666 $\pm$ 0.160
\\ 
& Cluster-CP     & 44.810 $\pm$ 0.565 & 51.091 $\pm$ 0.375
& 47.861 $\pm$ 0.428 & 59.262 $\pm$ 0.306
& 50.211 $\pm$ 0.474 & 60.183 $\pm$ 0.507
\\ 
& \textbf{\newCP} & \textbf{34.240 $\pm$ 0.115} & \textbf{36.425$\pm$ 0.024}
& \textbf{34.576 $\pm$ 0.006} & \textbf{46.074 $\pm$ 0.004}
& \textbf{37.055 $\pm$ 0.006} & \textbf{48.012 $\pm$ 0.076}
\\
\bottomrule 
\end{NiceTabular}
}
\label{tab:overall_comparison_four_datasets_50epochs}
\end{table*}

\subsection{Comparison Experiments with UCG metrics}
\label{appendix:UCG}

We add the experiments without controlling coverage on imbalanced datasets under the same setting as the main paper.
We then use the total under coverage gap (UCG, $\downarrow$ better) between class conditional coverage and target coverage $1-\alpha$ of all under covered classes. 
We choose UCG as the fine-grained metric to differentiate the coverage performance in our experiment setting. Conditioned on similar APSS of all methods, RC3P significantly outperforms the best baselines with $35.18\%$(four datasets) or $46.91\%$ (excluding CIFAR-10)reduction in UCG on average.

\begin{table*}[!ht]
\centering
\caption{
UCG and APSS results comparing \texttt{CCP}, \texttt{Cluster-CP}, and \texttt{\newCP}~with ResNet-20 model trained with $200$ epochs under different imbalance types with imbalance ratio $\rho=0.1$, where the coverage of each method are not aligned.
The APSS results show that \texttt{\newCP}~outperforms \texttt{CCP} and \texttt{Cluster-CP} in terms of average prediction set size with $1.64\%$(four datasets) or $2.19\%$ (excluding CIFAR-10) reduction in prediction size on average over $\min \{ \texttt{CCP}, \texttt{cluster-CP}\}$.
The UCG results show that \texttt{\newCP}achieve the similar class conditional coverage as \texttt{CCP} and \texttt{Cluster-CP} in terms of with $35.18\%$(four datasets) or $46.91\%$ (excluding CIFAR-10) increment in the proportion of under coverage classes on average over $\min \{ \texttt{CCP}, \texttt{cluster-CP}\}$.
}
\label{tab:overall_comparison_four_datasets}
\resizebox{\textwidth}{!}{
\begin{NiceTabular}{@{}cc!{~}cc!{~}cc!{~}cc@{}} \toprule 
\Block{2-1}{Conformity Score}  & \Block{2-1}{Methods} & \Block{1-2}{\EXP} & & \Block{1-2}{\POLY} & & \Block{1-2}{\MAJ} \\ 
 & & UCG & APSS & UCG & APSS & UCG & APSS \\ 
\midrule 
\Block{1-*}{CIFAR-10}
\\
\midrule
\Block{3-1}{APS} 
& CCP     & \textbf{0.014 $\pm$ 0.000} & 1.573 $\pm$ 0.009
& \textbf{0.032 $\pm$ 0.000} & 1.494 $\pm$ 0.015
& \textbf{0.068 $\pm$ 0.000} & 2.175 $\pm$ 0.019
\\ 
& Cluster-CP    & 0.166 $\pm$ 0.000 & \textbf{1.438 $\pm$ 0.012}
& 0.124 $\pm$ 0.000 & \textbf{1.280 $\pm$ 0.007}
& 0.144 $\pm$ 0.000 & \textbf{2.079 $\pm$ 0.023}
\\ 
& \textbf{\newCP} & \textbf{0.014 $\pm$ 0.000} & 1.573 $\pm$ 0.009
& \textbf{0.032 $\pm$ 0.043} & 1.494 $\pm$ 0.015
& \textbf{0.068 $\pm$ 0.031} & 2.175 $\pm$ 0.019
\\
\midrule
\Block{3-1}{RAPS}
& CCP     & \textbf{0.014 $\pm$ 0.000} & 1.573 $\pm$ 0.009
& \textbf{0.032 $\pm$ 0.000} & 1.494 $\pm$ 0.015
& \textbf{0.070 $\pm$ 0.000} & 2.179 $\pm$ 0.019
\\ 
& Cluster-CP     & 0.166 $\pm$ 0.000 & \textbf{1.438 $\pm$ 0.012}
& 0.124 $\pm$ 0.000 & \textbf{1.280 $\pm$ 0.007}
& 0.144 $\pm$ 0.000 & \textbf{2.079 $\pm$ 0.023}
\\ 
& \textbf{\newCP} & \textbf{0.014 $\pm$ 0.050} & 1.573 $\pm$ 0.009
& \textbf{0.032 $\pm$ 0.000} & 1.494 $\pm$ 0.015
& \textbf{0.070 $\pm$ 0.000} & 2.179 $\pm$ 0.019
\\
\midrule
\Block{1-*}{CIFAR-100}
\\
\midrule
\Block{3-1}{APS} 
& CCP     & 1.920 $\pm$ 0.000 & 16.721 $\pm$ 0.174
& 2.000 $\pm$ 0.000 & 26.831 $\pm$ 0.150
& 2.400 $\pm$ 0.000 & 26.211 $\pm$ 0.216
\\ 
& Cluster-CP     & 1.500 $\pm$ 0.000 & 15.657 $\pm$ 0.417
& 2.580 $\pm$ 0.000 & 26.709 $\pm$ 0.422
& 2.660 $\pm$ 0.000 &  25.145 $\pm$ 0.385
\\ 
& \textbf{\newCP} & \textbf{0.840 $\pm$ 0.000} & \textbf{14.642 $\pm$ 0.005}
& \textbf{1.200 $\pm$ 0.000} & \textbf{24.480 $\pm$ 0.004}
& \textbf{1.460 $\pm$ 0.000} & \textbf{23.332 $\pm$ 0.006}
\\
\midrule
\Block{3-1}{RAPS}
& CCP     & 1.920 $\pm$ 0.000 & 16.724 $\pm$ 0.174
& 2.020 $\pm$ 0.000 & 26.817 $\pm$ 0.150
& 2.400 $\pm$ 0.007 & 26.199 $\pm$ 0.216
\\ 
& Cluster-CP     & 1.500 $\pm$ 0.000 & 15.767 $\pm$ 0.410
& 2.760 $\pm$ 0.000 & 26.712 $\pm$ 0.512
& 2.480 $\pm$ 0.000 & 25.153 $\pm$ 0.250
\\ 
& \textbf{\newCP} & \textbf{0.840 $\pm$ 0.000} & \textbf{14.642 $\pm$ 0.005}
& \textbf{1.200 $\pm$ 0.000} & \textbf{24.480 $\pm$ 0.004}
& \textbf{1.460 $\pm$ 0.000} & \textbf{23.332 $\pm$ 0.006}
\\
\midrule
\Block{1-*}{mini-ImageNet}
\\
\midrule 
\Block{3-1}{APS} 
& CCP     & 1.486 $\pm$ 0.000 & 10.525 $\pm$ 0.093
& 1.620 $\pm$ 0.000 & 11.188 $\pm$ 0.094
& 1.280 $\pm$ 0.000 & 10.642 $\pm$ 0.055
\\ 
& Cluster-CP     & 1.313 $\pm$ 0.000 & 11.133 $\pm$ 0.118
& 1.453 $\pm$ 0.000 & 11.547 $\pm$ 0.129
& 1.640 $\pm$ 0.000 & 11.186 $\pm$ 0.151
\\ 
& \textbf{\newCP} & \textbf{0.713 $\pm$ 0.000} & \textbf{10.360 $\pm$ 0.042}
& \textbf{0.653 $\pm$ 0.000} & \textbf{11.089 $\pm$ 0.052}
& \textbf{0.600 $\pm$ 0.000} & \textbf{10.545 $\pm$ 0.029}
\\
\midrule
\Block{3-1}{RAPS}
& CCP     & 1.526 $\pm$ 0.000 & 10.570 $\pm$ 0.093
& 1.620 $\pm$ 0.000 & 11.250 $\pm$ 0.095
& 1.293 $\pm$ 0.000 & 10.702 $\pm$ 0.055
\\ 
& Cluster-CP     & 1.480 $\pm$ 0.000 & 11.192 $\pm$ 0.123
& 1.513 $\pm$ 0.000 & 11.704 $\pm$ 0.124
& 1.586 $\pm$ 0.000 & 11.231 $\pm$ 0.156
\\ 
& \textbf{\newCP} & \textbf{0.713 $\pm$ 0.000} & \textbf{10.377 $\pm$ 0.035}
& \textbf{0.653 $\pm$ 0.000} & \textbf{11.126 $\pm$ 0.046}
& \textbf{0.600 $\pm$ 0.000} & \textbf{10.571 $\pm$ 0.021}
\\
\midrule
\Block{1-*}{Food-101}
\\
\midrule 
\Block{3-1}{APS} 
& CCP     & 1.176 $\pm$ 0.000 & 14.019 $\pm$ 0.064
& 1.208 $\pm$ 0.000 & 17.288 $\pm$ 0.075
& 1.748 $\pm$ 0.000 & 17.663 $\pm$ 0.076
\\ 
& Cluster-CP     & 1.296 $\pm$ 0.000 & 13.998 $\pm$ 0.107
& 1.704 $\pm$ 0.000 & 17.300 $\pm$ 0.183
& 2.148 $\pm$ 0.000 & 17.410 $\pm$ 0.130
\\ 
& \textbf{\newCP} & \textbf{0.556 $\pm$ 0.000} & \textbf{13.564 $\pm$ 0.003}
& \textbf{0.664 $\pm$ 0.000} & \textbf{16.608 $\pm$ 0.006}
& \textbf{0.924 $\pm$ 0.000} & \textbf{16.890 $\pm$ 0.005}
\\
\midrule
\Block{3-1}{RAPS}
& CCP    & 1.160 $\pm$ 0.000 & 14.019 $\pm$ 0.064
& 1.208 $\pm$ 0.000 & 17.301 $\pm$ 0.075
& 1.764 $\pm$ 0.000 & 17.679 $\pm$ 0.076
\\ 
& Cluster-CP      & 1.308 $\pm$ 0.000 & 14.080 $\pm$ 0.113
& 1.804 $\pm$ 0.000 & 17.370 $\pm$ 0.198
& 1.944 $\pm$ 0.000 & 17.488 $\pm$ 0.138
\\ 
& \textbf{\newCP} & \textbf{0.556 $\pm$ 0.000} & \textbf{13.564 $\pm$ 0.003}
& \textbf{0.664 $\pm$ 0.000} & \textbf{16.608 $\pm$ 0.006}
& \textbf{0.924 $\pm$ 0.000} & \textbf{16.890 $\pm$ 0.005}
\\
\bottomrule 
\end{NiceTabular}
}
\end{table*}

\clearpage
\newpage

\subsection{ Complete Experiment Results on Imbalanced Datasets}
\label{subsection:appendix:complete experiment_results}

In this subsection, we report complete experimental results over four imbalanced datasets, three decaying types, and five imbalance ratios when epoch $=200$ and $\alpha = 0.1$.
Specifically, Table \ref{tab:appendix:overall_comparison_cifar10_exp}, \ref{tab:appendix:overall_comparison_cifar10_poly}, \ref{tab:appendix:overall_comparison_cifar10_maj} report results on CIFAR-10 with three decaying types.
Table \ref{tab:appendix:overall_comparison_cifar100_exp}, \ref{tab:appendix:overall_comparison_cifar100_poly}, \ref{tab:appendix:overall_comparison_cifar100_maj} report results on CIFAR-100 with three decaying types.
Table \ref{tab:appendix:overall_comparison_miniimagenet_exp}, \ref{tab:appendix:overall_comparison_miniimagenet_poly}, \ref{tab:appendix:overall_comparison_miniimagenet_maj} report results on mini-ImageNet with three decaying types.
Table \ref{tab:appendix:overall_comparison_food101_exp}, \ref{tab:appendix:overall_comparison_food101_poly}, \ref{tab:appendix:overall_comparison_food101_maj} report results on Food-101 with three decaying types.

\vspace{1.0ex}

Figure \ref{fig:overall_comparison_four_datasets_exp_0.5},
Figure \ref{fig:overall_comparison_four_datasets_poly_0.1},
Figure \ref{fig:overall_comparison_four_datasets_poly_0.5},
Figure \ref{fig:overall_comparison_four_datasets_maj_0.1} and
Figure \ref{fig:overall_comparison_four_datasets_maj_0.5}
show the class-conditional coverage and the corresponding prediction set sizes on \EXP~$\rho=0.5$, \POLY~$\rho=0.1$, \POLY~$\rho=0.5$, \MAJ~$\rho=0.1$, \MAJ~$\rho=0.5$, respectively.
This result on \EXP~$\rho=0.1$ is in Figure \ref{fig:overall_comparison_four_datasets_exp_0.1}.

\vspace{1.0ex}

Figure \ref{fig:condition_number_rank_exp_0.5},
Figure \ref{fig:condition_number_rank_poly_0.1},
Figure \ref{fig:condition_number_rank_poly_0.5},
Figure \ref{fig:condition_number_rank_maj_0.1} and
Figure \ref{fig:condition_number_rank_maj_0.5}
illustrates the normalized frequency distribution of label ranks included in the prediction sets on EXP $\rho=0.5$, POLY $\rho=0.1$, POLY $\rho=0.5$, MAJ $\rho=0.1$, MAJ $\rho=0.5$, respectively. 
This result on \EXP~$\rho=0.1$ is in Figure \ref{fig:condition_number_rank_exp_0.1}.
It is evident that the distribution of label ranks in the prediction set generated by \texttt{\newCP}~tends to be lower compared to those produced by \texttt{CCP} and \texttt{Cluster-CP}. 
Furthermore, the probability density function tail for label ranks in the \texttt{\newCP}~prediction set is notably shorter than that of other methods. This indicates that \texttt{\newCP}~more effectively incorporates lower-ranked labels into prediction sets, as a result of its augmented rank calibration scheme.

\vspace{1.0ex}

Figure \ref{fig:condition_number_sigma_y_exp_0.5},
Figure \ref{fig:condition_number_sigma_y_poly_0.1},
Figure \ref{fig:condition_number_sigma_y_poly_0.5},
Figure \ref{fig:condition_number_sigma_y_maj_0.1} and
Figure \ref{fig:condition_number_sigma_y_maj_0.5}
verify the condition numbers $\sigma_y$ when models are fully trained (epoch $=200$) on EXP $\rho=0.5$, POLY $\rho=0.1$, POLY $\rho=0.5$, MAJ $\rho=0.1$, MAJ $\rho=0.5$, respectively.
This result on \EXP~$\rho=0.1$ is in Figure Figure \ref{fig:condition_number_sigma_y_exp_0.1}.
We also evaluate the condition numbers $\sigma_y$ when models are lessly trained (epoch $=50$) and $\alpha = 0.1$on EXP $\rho=0.5$, EXP $\rho=0.1$, POLY $\rho=0.1$, POLY $\rho=0.5$, MAJ $\rho=0.1$, MAJ $\rho=0.5$, respectively. These results are shown from Figure  \ref{fig:condition_number_sigma_y_exp_0.5_50} to Figure \ref{fig:condition_number_sigma_y_maj_0.5_50}. 
These results verify the validity of Lemma \ref{lemma:RC3P_improved_efficiency} and Equation \ref{eq:sigma_y_defination} and confirm that the optimized trade-off between the coverage with inflated quantile and the constraint with calibrated rank leads to smaller prediction sets.
They also show a stronger condition ($\sigma_y \leq 1$ for all $y$) than the weighted aggregation condition in (\ref{eq:smaller_ps_condition}).
They also confirm that the condition number $\{\sigma_y\}_{y=1}^C$ could be evaluated on calibration datasets without testing datasets and thus decreases the computation cost.
We notice that \texttt{\newCP} degenerates to CCP on CIFAR-10, so $\sigma_y = 1$ for all $y$ and there is no trade-off.
On the other three datasets, we observe significant conditions for the optimized trade-off in \texttt{\newCP}.

\begin{table*}[!b]
\centering
\caption{
Results comparing \texttt{CCP}, \texttt{cluster-CP}, and \texttt{\newCP} with ResNet-20 model under different imbalance ratio $\rho=0.5$, $\rho=0.4$, $\rho=0.2$, and $\rho=0.1$ with imbalance type \EXP~and two scoring functions, APS and RAPS, on dataset CIFAR-10.
We set UCR of \texttt{\newCP} the same as or better than that of \texttt{CCP} and \texttt{Cluster-CP} for a fair comparison of prediction set size.
}
\resizebox{\textwidth}{!}{
\begin{NiceTabular}{@{}ccc!{~}ccccc@{}} \toprule 
\Block{2-1}{Scoring\\function} & \Block{2-1}{Measure}  & \Block{2-1}{Methods} & \Block{1-*}{\EXP} \\ 
 & & & $\rho$ = 0.5 & $\rho$ = 0.4 & $\rho$ = 0.3 & $\rho$ = 0.2 & $\rho$ = 0.1  \\ 
\\
\midrule
\Block{6-1}{APS} & \Block{3-1}{UCR} 
& CCP     & 0.050 $\pm$ 0.016 & 0.06 $\pm$ 0.021
& 0.050 $\pm$ 0.016 & 0.050 $\pm$ 0.021 & 0.100 $\pm$ 0.020
\\ 
& & Cluster-CP     & \textbf{0.010 $\pm$ 0.009} & \textbf{0.050 $\pm$ 0.021} & \textbf{0.0 $\pm$ 0.0} & \textbf{0.030 $\pm$ 0.015} & \textbf{0.090 $\pm$ 0.009} 
\\ 
& & \textbf{\newCP}  & 0.050 $\pm$ 0.016 & 0.06 $\pm$ 0.021
& 0.050 $\pm$ 0.016 & 0.050 $\pm$ 0.021 & 0.100 $\pm$ 0.020
\\ 
\cline{2-8}
& \Block{3-1}{APSS}
& CCP     & \textbf{1.555 $\pm$ 0.010} & \textbf{1.595 $\pm$ 0.013}
& \textbf{1.643 $\pm$ 0.008} & \textbf{1.676 $\pm$ 0.014} & \textbf{1.855 $\pm$ 0.014} 
\\ 
& & Cluster-CP     
& 1.714 $\pm$ 0.018 
& 1.745 $\pm$ 0.018
& 1.825 $\pm$ 0.014
& 1.901 $\pm$ 0.022
& 2.162 $\pm$ 0.015
\\ 
& & \textbf{\newCP}  & \textbf{1.555 $\pm$ 0.010} & \textbf{1.595 $\pm$ 0.013}
& \textbf{1.643 $\pm$ 0.008} & \textbf{1.676 $\pm$ 0.014} & \textbf{1.855 $\pm$ 0.014} 
\\
\midrule
\Block{6-1}{RAPS} & \Block{3-1}{UCR} 
& CCP      & 0.050 $\pm$ 0.016 & 0.060 $\pm$ 0.021
& 0.050 $\pm$ 0.016 & 0.050 $\pm$ 0.021 & \textbf{0.010 $\pm$ 0.020} 
\\ 
& & Cluster-CP     
& \textbf{0.010 $\pm$ 0.010} 
& \textbf{0.050 $\pm$ 0.021} 
& \textbf{0.000 $\pm$ 0.000}
& \textbf{0.030 $\pm$ 0.014}
& \textbf{0.010 $\pm$ 0.010}
\\ 
& & \textbf{\newCP}   & 0.050 $\pm$ 0.016 & 0.060 $\pm$ 0.021
& 0.050 $\pm$ 0.016 & 0.050 $\pm$ 0.021 & \textbf{0.010 $\pm$ 0.020} 
\\ 
\cline{2-8}
& \Block{3-1}{APSS}
& CCP     & \textbf{1.555 $\pm$ 0.010} & \textbf{1.595 $\pm$ 0.013}
& \textbf{1.643 $\pm$ 0.008} & \textbf{1.676 $\pm$ 0.014} & \textbf{1.855 $\pm$ 0.014} 
\\ 
& & Cluster-CP     
& 1.714 $\pm$ 0.018 
& 1.745 $\pm$ 0.018 
& 1.825 $\pm$ 0.014  
& 1.901 $\pm$ 0.022
& 2.162 $\pm$ 0.015
\\ 
& & \textbf{\newCP}  & \textbf{1.555 $\pm$ 0.010} & \textbf{1.595 $\pm$ 0.013}
& \textbf{1.643 $\pm$ 0.008} & \textbf{1.676 $\pm$ 0.014} & \textbf{1.855 $\pm$ 0.014} 
\\
\bottomrule 
\end{NiceTabular}
}
\label{tab:appendix:overall_comparison_cifar10_exp}
\end{table*}

\begin{table*}[!b]
\centering
\caption{
Results comparing \texttt{CCP}, \texttt{cluster-CP}, and \texttt{\newCP} with ResNet-20 model under different imbalance ratio $\rho=0.5$, $\rho=0.4$, $\rho=0.2$, and $\rho=0.1$ with imbalance type \POLY~and two scoring functions, APS and RAPS, on dataset CIFAR-10.
We set UCR of \texttt{\newCP} the same as or better than that of \texttt{CCP} and \texttt{Cluster-CP} for a fair comparison of prediction set size.
}
\resizebox{\textwidth}{!}{
\begin{NiceTabular}{@{}ccc!{~}ccccc@{}} \toprule 
\Block{2-1}{Scoring\\function} & \Block{2-1}{Measure}  & \Block{2-1}{Methods} & \Block{1-*}{\POLY} \\ 
 & & & $\rho$ = 0.5 & $\rho$ = 0.4 & $\rho$ = 0.3 & $\rho$ = 0.2 & $\rho$ = 0.1  \\ 
\\
\midrule
\Block{6-1}{APS} & \Block{3-1}{UCR} 
& CCP     & 0.100 $\pm$ 0.028 & 0.060 $\pm$ 0.026
& 0.060 $\pm$ 0.015 & \textbf{0.050 $\pm$ 0.021} & \textbf{0.050 $\pm$ 0.021} 
\\ 
& & Cluster-CP     
& \textbf{0.080 $\pm$ 0.019} 
& \textbf{0.050 $\pm$ 0.021} 
& \textbf{0.050 $\pm$ 0.025}  
& \textbf{0.050 $\pm$ 0.016}
& 0.060 $\pm$ 0.015 
\\ 
& & \textbf{\newCP}  & 0.100 $\pm$ 0.028 & 0.060 $\pm$ 0.026
& 0.060 $\pm$ 0.015 & \textbf{0.050 $\pm$ 0.021} & \textbf{0.050 $\pm$ 0.021} 
\\ 
\cline{2-8}
& \Block{3-1}{APSS}
& CCP     & \textbf{1.538 $\pm$ 0.010} & \textbf{1.546 $\pm$ 0.011}
& \textbf{1.580 $\pm$ 0.014} & \textbf{1.627 $\pm$ 0.011} & \textbf{1.776 $\pm$ 0.012} 
\\ 
& & Cluster-CP     
& 1.706 $\pm$ 0.014 
& 1.718 $\pm$ 0.014  
& 1.758 $\pm$ 0.016 
& 1.783 $\pm$ 0.016  
& 1.928 $\pm$ 0.013
\\ 
& & \textbf{\newCP}  & \textbf{1.538 $\pm$ 0.010} & \textbf{1.546 $\pm$ 0.011}
& \textbf{1.580 $\pm$ 0.014} & \textbf{1.627 $\pm$ 0.011} & \textbf{1.776 $\pm$ 0.012} 
\\
\midrule
\Block{6-1}{RAPS} & \Block{3-1}{UCR} 
& CCP     & 0.100 $\pm$ 0.028 & 0.060 $\pm$ 0.025
& 0.060 $\pm$ 0.016 & \textbf{0.050 $\pm$ 0.021} & \textbf{0.050 $\pm$ 0.021}
\\ 
& & Cluster-CP     
& \textbf{0.080 $\pm$ 0.019} 
& \textbf{0.050 $\pm$ 0.021} 
& \textbf{0.050 $\pm$ 0.025}
& \textbf{0.050 $\pm$ 0.016}
& 0.060 $\pm$ 0.015
\\ 
& & \textbf{\newCP}  & 0.100 $\pm$ 0.028 & 0.060 $\pm$ 0.025
& 0.060 $\pm$ 0.016 & \textbf{0.050 $\pm$ 0.021} & \textbf{0.050 $\pm$ 0.021}
\\ 
\cline{2-8}
& \Block{3-1}{APSS}
& CCP     & \textbf{1.538 $\pm$ 0.010} & \textbf{1.546 $\pm$ 0.011}
& \textbf{1.581 $\pm$ 0.014} & \textbf{1.627 $\pm$ 0.011} & \textbf{1.776 $\pm$ 0.012} 
\\ 
& & Cluster-CP     
& 1.706 $\pm$ 0.014 
& 1.719 $\pm$ 0.014 
& 1.759 $\pm$ 0.016  
& 1.783 $\pm$ 0.016
& 1.929 $\pm$ 0.013
\\ 
& & \textbf{\newCP}  & \textbf{1.538 $\pm$ 0.010} & \textbf{1.546 $\pm$ 0.011}
& \textbf{1.581 $\pm$ 0.014} & \textbf{1.627 $\pm$ 0.011} & \textbf{1.776 $\pm$ 0.012} 
\\
\bottomrule 
\end{NiceTabular}
}
\label{tab:appendix:overall_comparison_cifar10_poly}
\end{table*}

\begin{table*}[!b]
\centering
\caption{
Results comparing \texttt{CCP}, \texttt{cluster-CP}, and \texttt{\newCP} with ResNet-20 model under different imbalance ratio $\rho=0.5$, $\rho=0.4$, $\rho=0.2$, and $\rho=0.1$ with imbalance type \MAJ~and two scoring functions, APS and RAPS, on dataset CIFAR-10.
We set UCR of \texttt{\newCP} the same as or better than that of \texttt{CCP} and \texttt{Cluster-CP} for a fair comparison of prediction set size.
}
\resizebox{\textwidth}{!}{
\begin{NiceTabular}{@{}ccc!{~}ccccc@{}} \toprule 
\Block{2-1}{Scoring\\function} & \Block{2-1}{Measure}  & \Block{2-1}{Methods} & \Block{1-*}{\MAJ} \\ 
 & & & $\rho$ = 0.5 & $\rho$ = 0.4 & $\rho$ = 0.3 & $\rho$ = 0.2 & $\rho$ = 0.1  \\ 
\\
\midrule
\Block{6-1}{APS} & \Block{3-1}{UCR} 
& CCP     & 0.070 $\pm$ 0.014 & 0.050 $\pm$ 0.016
& 0.080 $\pm$ 0.019 & 0.070 $\pm$ 0.025 & \textbf{0.040 $\pm$ 0.015}
\\ 
& & Cluster-CP   & \textbf{0.020 $\pm$ 0.012} & \textbf{0.040 $\pm$ 0.015} & \textbf{0.020 $\pm$ 0.013} & \textbf{0.010 $\pm$ 0.010} & 0.070 $\pm$ 0.014 
\\ 
& & \textbf{\newCP}  & 0.070 $\pm$ 0.014 & 0.050 $\pm$ 0.016
& 0.080 $\pm$ 0.019 & 0.070 $\pm$ 0.025 & \textbf{0.040 $\pm$ 0.015}
\\ 
\cline{2-8}
& \Block{3-1}{APSS}
& CCP    & \textbf{1.84 $\pm$ 0.020} & \textbf{1.825 $\pm$ 0.014}
& \textbf{1.939 $\pm$ 0.016} & \textbf{2.054 $\pm$ 0.013} & \textbf{2.629 $\pm$ 0.013} 
\\ 
& & Cluster-CP    
& 1.948 $\pm$ 0.023 
& 1.999 $\pm$ 0.027 
& 2.167 $\pm$ 0.030 
& 2.457 $\pm$ 0.021 
& 3.220 $\pm$ 0.020
\\ 
& & \textbf{\newCP}  & \textbf{1.84 $\pm$ 0.020} & \textbf{1.825 $\pm$ 0.014}
& \textbf{1.939 $\pm$ 0.016} & \textbf{2.054 $\pm$ 0.013} & \textbf{2.629 $\pm$ 0.013} 
\\
\midrule
\Block{6-1}{RAPS} & \Block{3-1}{UCR} 
& CCP     & 0.070 $\pm$ 0.014 & 0.050 $\pm$ 0.016
& 0.080 $\pm$ 0.019 & 0.070 $\pm$ 0.025 & \textbf{0.040 $\pm$ 0.015} 
\\ 
& & Cluster-CP     
& \textbf{0.020 $\pm$ 0.013} 
& \textbf{0.040 $\pm$ 0.015} 
& \textbf{0.020 $\pm$ 0.012} 
& \textbf{0.010 $\pm$ 0.010} 
& 0.070 $\pm$ 0.014 
\\ 
& & \textbf{\newCP}  & 0.070 $\pm$ 0.014 & 0.050 $\pm$ 0.016
& 0.080 $\pm$ 0.019 & 0.070 $\pm$ 0.025 & \textbf{0.040 $\pm$ 0.015} 
\\ 
\cline{2-8}
& \Block{3-1}{APSS}
& CCP     & \textbf{1.840 $\pm$ 0.020} & \textbf{1.825 $\pm$ 0.014}
& \textbf{1.940 $\pm$ 0.016} & \textbf{2.055 $\pm$ 0.013} & \textbf{2.632 $\pm$ 0.012} 
\\ 
& & Cluster-CP     
& 1.948 $\pm$ 0.023 
& 1.999 $\pm$ 0.028
& 2.168 $\pm$ 0.030
& 2.458 $\pm$ 0.021
& 3.219 $\pm$ 0.030
\\ 
& & \textbf{\newCP}  & \textbf{1.840 $\pm$ 0.020} & \textbf{1.825 $\pm$ 0.014}
& \textbf{1.940 $\pm$ 0.016} & \textbf{2.055 $\pm$ 0.013} & \textbf{2.632 $\pm$ 0.012} 
\\
\bottomrule 
\end{NiceTabular}
}
\label{tab:appendix:overall_comparison_cifar10_maj}
\end{table*}

\begin{table}[!b]
\centering
\caption{
Results comparing \texttt{CCP}, \texttt{cluster-CP}, and \texttt{\newCP} with ResNet-20 model under different imbalance ratio $\rho=0.5$, $\rho=0.4$, $\rho=0.2$, and $\rho=0.1$ with imbalance type \EXP~and two scoring functions, APS and RAPS, on dataset CIFAR-100.
We set UCR of \texttt{\newCP} the same as or better than that of \texttt{CCP} and \texttt{Cluster-CP} for a fair comparison of prediction set size.
}
\resizebox{\textwidth}{!}{
\begin{NiceTabular}{@{}ccc!{~}ccccc@{}} \toprule 
\Block{2-1}{Scoring\\function} & \Block{2-1}{Measure}  & \Block{2-1}{Methods} & \Block{1-*}{\EXP} \\ 
 & & & $\rho$ = 0.5 & $\rho$ = 0.4 & $\rho$ = 0.3 & $\rho$ = 0.2 & $\rho$ = 0.1  \\ 
\\
\midrule
\Block{6-1}{APS} & \Block{3-1}{UCR} 
& CCP     
& 0.007$\pm$ 0.002  
& 0.017$\pm$ 0.004
& 0.012$\pm$ 0.004
& 0.015$\pm$ 0.003
& \textbf{0.010$\pm$ 0.002} 
\\ 
& & Cluster-CP     
& 0.012$\pm$ 0.002  
& 0.012$\pm$ 0.003
& \textbf{0.006$\pm$ 0.002 }
& 0.035$\pm$ 0.008
& 0.016$\pm$ 0.004 
\\ 
& & \textbf{\newCP} 
& \textbf{0.005$\pm$ 0.002 }
& \textbf{0.009$\pm$ 0.001 }
& 0.011$\pm$ 0.003
& \textbf{0.013$\pm$ 0.003}
& 0.011$\pm$ 0.002
\\ 
\cline{2-8}
& \Block{3-1}{APSS}
& CCP     
& 44.224$\pm$ 0.341  
& 44.486$\pm$ 0.420
& 47.672$\pm$ 0.463
& 46.955$\pm$ 0.402
& 50.969$\pm$ 0.345 
\\ 
& & Cluster-CP     
& 29.238$\pm$ 0.609  
& 30.602$\pm$ 0.553
& 32.126$\pm$ 0.563
& 33.714$\pm$ 0.863
& 37.592$\pm$ 0.857 
\\ 
& & \textbf{\newCP}  
& \textbf{17.705$\pm$ 0.004} 
& \textbf{18.311$\pm$ 0.005}
& \textbf{19.608$\pm$ 0.007}
& \textbf{20.675$\pm$ 0.005}
& \textbf{21.954$\pm$ 0.005}
\\
\midrule
\Block{6-1}{RAPS} & \Block{3-1}{UCR} 
& CCP     
& 0.007$\pm$ 0.002  
& 0.017$\pm$ 0.004
& 0.012$\pm$ 0.003
& 0.015$\pm$ 0.003
& \textbf{0.011$\pm$ 0.002 }
\\ 
& & Cluster-CP     
& 0.011$\pm$ 0.003 
& \textbf{0.009$\pm$ 0.002}
& \textbf{0.006$\pm$ 0.002}
& 0.034$\pm$ 0.007 
& 0.017$\pm$ 0.004 
\\ 
& & \textbf{\newCP} 
& \textbf{0.005$\pm$ 0.002 }
& 0.012$\pm$ 0.003 
& 0.011$\pm$ 0.003 
& \textbf{0.013$\pm$ 0.003 }
& \textbf{0.011$\pm$ 0.002}
\\ 
\cline{2-8}
& \Block{3-1}{APSS}
& CCP     
& 44.250$\pm$ 0.342  
& 44.499$\pm$ 0.420
& 47.688$\pm$ 0.569
& 46.960$\pm$ 0.404
& 50.970$\pm$ 0.345 
\\ 
& & Cluster-CP     
& 29.267$\pm$ 0.612  
& 30.595$\pm$ 0.549 
& 32.161$\pm$ 0.564
& 33.713$\pm$ 0.864
& 37.595$\pm$ 0.862 
\\ 
& & \textbf{\newCP}  
& \textbf{17.705$\pm$ 0.004} 
& \textbf{18.311$\pm$ 0.005}
& \textbf{19.609$\pm$ 0.007}
& \textbf{20.675$\pm$ 0.005}
& \textbf{21.954$\pm$ 0.005}
\\
\bottomrule 
\end{NiceTabular}
}
\label{tab:appendix:overall_comparison_cifar100_exp}
\end{table}

\begin{table}[!b]
\centering
\caption{
Results comparing \texttt{CCP}, \texttt{cluster-CP}, and \texttt{\newCP} with ResNet-20 model under different imbalance ratio $\rho=0.5$, $\rho=0.4$, $\rho=0.2$, and $\rho=0.1$ with imbalance type \POLY~and two scoring functions, APS and RAPS, on dataset CIFAR-100.
We set UCR of \texttt{\newCP} the same as or better than that of \texttt{CCP} and \texttt{Cluster-CP} for a fair comparison of prediction set size.
}
\resizebox{\textwidth}{!}{
\begin{NiceTabular}{@{}ccc!{~}ccccc@{}} \toprule 
\Block{2-1}{Scoring\\function} & \Block{2-1}{Measure}  & \Block{2-1}{Methods} & \Block{1-*}{\POLY} \\ 
 & & & $\rho$ = 0.5 & $\rho$ = 0.4 & $\rho$ = 0.3 & $\rho$ = 0.2 & $\rho$ = 0.1  \\ 
\\
\midrule
\Block{6-1}{APS} & \Block{3-1}{UCR} 
& CCP     
& 0.010$\pm$ 0.002  
& 0.008$\pm$ 0.002
& 0.016$\pm$ 0.003
& 0.012$\pm$ 0.004
& \textbf{0.014$\pm$ 0.003} 
\\ 
& & Cluster-CP     
& 0.020$\pm$ 0.003   
& 0.020$\pm$ 0.002
& 0.026$\pm$ 0.004
& \textbf{0.009$\pm$ 0.003 }
& 0.034$\pm$ 0.005
\\ 
& & \textbf{\newCP} 
& \textbf{0.009$\pm$ 0.003 }
& \textbf{0.005$\pm$ 0.002 }
& \textbf{0.013$\pm$ 0.004 }
& \textbf{0.011$\pm$ 0.004 }
& 0.015$\pm$ 0.003
\\ 
\cline{2-8}
& \Block{3-1}{APSS}
& CCP     
& 49.889$\pm$ 0.353  
& 54.011$\pm$ 0.466
& 56.031$\pm$ 0.406
& 59.888$\pm$ 0.255
& 64.343$\pm$ 0.237 
\\ 
& & Cluster-CP     
& 38.252$\pm$ 0.316   
& 39.585$\pm$ 0.545   
& 43.310$\pm$ 0.824
& 47.461$\pm$ 0.979
& 52.391$\pm$ 0.595 
\\ 
& & \textbf{\newCP}  
& \textbf{23.048$\pm$ 0.008 } 
& \textbf{24.335$\pm$ 0.005 }
& \textbf{26.366$\pm$ 0.010 }
& \textbf{28.887$\pm$ 0.006 }
& \textbf{33.829$\pm$ 0.005 }
\\
\midrule
\Block{6-1}{RAPS} & \Block{3-1}{UCR} 
& CCP     
& 0.010$\pm$ 0.002  
& 0.008$\pm$ 0.002
& 0.016$\pm$ 0.003
& 0.012$\pm$ 0.004
& \textbf{0.015$\pm$ 0.003} 
\\ 
& & Cluster-CP     
& 0.019$\pm$ 0.004  
& 0.020$\pm$ 0.002  
& 0.026$\pm$ 0.005 
& \textbf{0.009$\pm$ 0.003} 
& 0.034$\pm$ 0.005 
\\ 
& & \textbf{\newCP} 
& \textbf{0.009$\pm$ 0.003 }
& \textbf{0.005$\pm$ 0.002 }
& \textbf{0.013$\pm$ 0.004 }
& 0.011$\pm$ 0.004 
& \textbf{0.015$\pm$ 0.003}
\\ 
\cline{2-8}
& \Block{3-1}{APSS}
& CCP     
& 49.886$\pm$ 0.353 
& 53.994$\pm$ 0.467
& 56.020$\pm$ 0.406
& 59.870$\pm$ 0.253
& 64.332$\pm$ 0.236 
\\ 
& & Cluster-CP     
& 38.258$\pm$ 0.320 
& 39.566$\pm$ 0.549
& 43.304$\pm$ 0.549
& 47.450$\pm$ 0.969
& 52.374$\pm$ 0.592 
\\ 
& & \textbf{\newCP}  
& \textbf{23.048$\pm$ 0.008 } 
& \textbf{24.335$\pm$ 0.005 }
& \textbf{26.366$\pm$ 0.010 }
& \textbf{28.886$\pm$ 0.006 }
& \textbf{33.185$\pm$ 0.005 }
\\
\bottomrule 
\end{NiceTabular}
}
\label{tab:appendix:overall_comparison_cifar100_poly}
\end{table}

\begin{table}[!b]
\centering
\caption{
Results comparing \texttt{CCP}, \texttt{cluster-CP}, and \texttt{\newCP} with ResNet-20 model under different imbalance ratio $\rho=0.5$, $\rho=0.4$, $\rho=0.2$, and $\rho=0.1$ with imbalance type \MAJ~and two scoring functions, APS and RAPS, on dataset CIFAR-100.
We set UCR of \texttt{\newCP} the same as or better than that of \texttt{CCP} and \texttt{Cluster-CP} for a fair comparison of prediction set size.
}
\label{tab:appendix:overall_comparison_cifar100_maj}
\resizebox{\textwidth}{!}{
\begin{NiceTabular}{@{}ccc!{~}ccccc@{}} \toprule 
\Block{2-1}{Scoring\\function} & \Block{2-1}{Measure}  & \Block{2-1}{Methods} & \Block{1-*}{\MAJ} \\ 
 & & & $\rho$ = 0.5 & $\rho$ = 0.4 & $\rho$ = 0.3 & $\rho$ = 0.2 & $\rho$ = 0.1  \\ 
\\
\midrule
\Block{6-1}{APS} & \Block{3-1}{UCR} 
& CCP     
& 0.016$\pm$ 0.003  
& \textbf{0.007$\pm$ 0.002}
& 0.017$\pm$ 0.004
& \textbf{0.010$\pm$ 0.002}
& \textbf{0.008$\pm$ 0.004}
\\ 
& & Cluster-CP     
& \textbf{0.008$\pm$ 0.002}  
& 0.012$\pm$ 0.003
& 0.021$\pm$ 0.004
& 0.021$\pm$ 0.005
& 0.019$\pm$ 0.005 
\\ 
& & \textbf{\newCP} 
& 0.016$\pm$ 0.003 
& 0.010$\pm$ 0.003 
& \textbf{0.015$\pm$ 0.004} 
& \textbf{0.010$\pm$ 0.002 }
& \textbf{0.008$\pm$ 0.004}
\\ 
\cline{2-8}
& \Block{3-1}{APSS}
& CCP     
& 44.194$\pm$ 0.514  
& 49.231$\pm$ 0.129
& 53.676$\pm$ 0.372
& 55.024$\pm$ 0.254
& 64.642$\pm$ 0.535 
\\ 
& & Cluster-CP     
& 31.518$\pm$ 0.335  
& 35.355$\pm$ 0.563  
& 37.514$\pm$ 0.538
& 43.619$\pm$ 0.600
& 50.883$\pm$ 0.673 
\\ 
& & \textbf{\newCP}  
& \textbf{18.581$\pm$ 0.007} 
& \textbf{21.080$\pm$ 0.010}
& \textbf{22.606$\pm$ 0.007}
& \textbf{26.785$\pm$ 0.007}
& \textbf{32.699$\pm$ 0.005}
\\
\midrule
\Block{6-1}{RAPS} & \Block{3-1}{UCR} 
& CCP     
& 0.015$\pm$ 0.003  
& \textbf{0.007$\pm$ 0.002}
& 0.011$\pm$ 0.004
& \textbf{0.010$\pm$ 0.003}
& \textbf{0.008$\pm$ 0.004 }
\\ 
& & Cluster-CP     
& \textbf{0.008$\pm$ 0.003}  
& 0.011$\pm$ 0.003
& 0.021$\pm$ 0.004
& 0.021$\pm$ 0.002
& 0.018$\pm$ 0.005 
\\ 
& & \textbf{\newCP} 
& 0.015$\pm$ 0.003 
& 0.010$\pm$ 0.003 
& \textbf{0.015$\pm$ 0.004 }
& \textbf{0.010$\pm$ 0.002}
& \textbf{0.008$\pm$ 0.004 }
\\ 
\cline{2-8}
& \Block{3-1}{APSS}
& CCP     
& 48.343$\pm$ 0.353  
& 49.252$\pm$ 0.128
& 53.666$\pm$ 0.371
& 55.016$\pm$ 0.254
& 64.633$\pm$ 0.535 
\\ 
& & Cluster-CP     
& 31.513$\pm$ 0.325  
& 35.352$\pm$ 0.547 
& 37.503$\pm$ 0.535
& 43.615$\pm$ 0.608
& 50.379$\pm$ 0.684 
\\ 
& & \textbf{\newCP}  
& \textbf{18.581$\pm$ 0.006 } 
& \textbf{21.080$\pm$ 0.010 }
& \textbf{22.605$\pm$ 0.007 }
& \textbf{26.786$\pm$ 0.007 }
& \textbf{32.699$\pm$ 0.006 }
\\
\bottomrule 
\end{NiceTabular}
}
\end{table}

\begin{table}[!b]
\centering
\caption{
Results comparing \texttt{CCP}, \texttt{cluster-CP}, and \texttt{\newCP} with ResNet-20 model under different imbalance ratio $\rho=0.5$, $\rho=0.4$, $\rho=0.2$, and $\rho=0.1$ with imbalance type \EXP~and two scoring function, APS and RAPS, on dataset mini-ImageNet.
We set UCR of \texttt{\newCP} the same as or better than that of \texttt{CCP} and \texttt{Cluster-CP} for a fair comparison of prediction set size.
}
\resizebox{\textwidth}{!}{
\begin{NiceTabular}{@{}ccc!{~}ccccc@{}} \toprule 
\Block{2-1}{Scoring\\function} & \Block{2-1}{Measure}  & \Block{2-1}{Methods} & \Block{1-*}{\EXP} \\ 
 & & & $\rho$ = 0.5 & $\rho$ = 0.4 & $\rho$ = 0.3 & $\rho$ = 0.2 & $\rho$ = 0.1  \\ 
\\
\midrule
\Block{6-1}{APS} & \Block{3-1}{UCR} 
& CCP     
& 0.008$\pm$ 0.004  
& 0.003$\pm$ 0.002
& 0.003$\pm$ 0.001
& 0.003$\pm$ 0.003
& 0.008$\pm$ 0.004 
\\ 
& & Cluster-CP     
& 0.014$\pm$ 0.004  
& 0.005$\pm$ 0.002
& 0.010$\pm$ 0.002
& 0.010$\pm$ 0.003
& 0.012$\pm$ 0.004 
\\ 
& & \textbf{\newCP} 
& \textbf{0.0$\pm$ 0.0 }
& \textbf{0.0$\pm$ 0.0 }
& \textbf{0.0$\pm$ 0.0 }
& \textbf{0.0$\pm$ 0.0 }
& \textbf{0.001$\pm$ 0.001}
\\ 
\cline{2-8}
& \Block{3-1}{APSS}
& CCP     
& 26.676$\pm$ 0.171  
& 25.663$\pm$ 0.182
& 25.941$\pm$ 0.180
& 26.127$\pm$ 0.187
& 26.111$\pm$ 0.194 
\\ 
& & Cluster-CP     
& 25.889$\pm$ 0.301  
& 25.878$\pm$ 0.258
& 25.680$\pm$ 0.294
& 25.522$\pm$ 0.311
& 25.253$\pm$ 0.346 
\\ 
& & \textbf{\newCP}  
& \textbf{18.129$\pm$ 0.003 } 
& \textbf{17.546$\pm$ 0.002 }
& \textbf{17.352$\pm$ 0.003 }
& \textbf{17.006$\pm$ 0.003 }
& \textbf{17.082$\pm$ 0.002 }
\\
\midrule
\Block{6-1}{RAPS} & \Block{3-1}{UCR} 
& CCP     
& 0.008$\pm$ 0.004  
& 0.004$\pm$ 0.003
& 0.003$\pm$ 0.001
& 0.003$\pm$ 0.003
& 0.009$\pm$ 0.004 
\\ 
& & Cluster-CP     
& 0.006$\pm$ 0.002  
& 0.003$\pm$ 0.001
& 0.009$\pm$ 0.002
& 0.008$\pm$ 0.003
& 0.013$\pm$ 0.005  
\\ 
& & \textbf{\newCP} 
& \textbf{0.0$\pm$ 0.0 }
& \textbf{0.0$\pm$ 0.0 }
& \textbf{0.0$\pm$ 0.0 }
& \textbf{0.0$\pm$ 0.0 }
& \textbf{0.001$\pm$ 0.001 }
\\ 
\cline{2-8}
& \Block{3-1}{APSS}
& CCP     
& 26.756$\pm$ 0.178  
& 26.621$\pm$ 0.182
& 25.021$\pm$ 0.182
& 26.216$\pm$ 0.188
& 26.212$\pm$ 0.199
\\ 
& & Cluster-CP     
& 26.027$\pm$ 0.325  
& 26.000$\pm$ 0.283
& 25.922$\pm$ 0.253
& 25.564$\pm$ 0.358
& 25.415$\pm$ 0.289  
\\ 
& & \textbf{\newCP}  
& \textbf{18.129$\pm$ 0.003 } 
& \textbf{17.546$\pm$ 0.002 }
& \textbf{17.352$\pm$ 0.003 }
& \textbf{17.006$\pm$ 0.003 }
& \textbf{17.082$\pm$ 0.002 }
\\
\bottomrule 
\end{NiceTabular}
}
\label{tab:appendix:overall_comparison_miniimagenet_exp}
\end{table}

\begin{table}[!b]
\centering
\caption{
Results comparing \texttt{CCP}, \texttt{cluster-CP}, and \texttt{\newCP} with ResNet-20 model under different imbalance ratio $\rho=0.5$, $\rho=0.4$, $\rho=0.2$, and $\rho=0.1$ with imbalance type \POLY~and two scoring function, APS and RAPS, on dataset mini-ImageNet.
We set UCR of \texttt{\newCP} the same as or better than that of \texttt{CCP} and \texttt{Cluster-CP} for a fair comparison of prediction set size.
}\label{tab:appendix:overall_comparison_miniimagenet_poly}
\resizebox{\textwidth}{!}{
\begin{NiceTabular}{@{}ccc!{~}ccccc@{}} \toprule 
\Block{2-1}{Scoring\\function} & \Block{2-1}{Measure}  & \Block{2-1}{Methods} & \Block{1-*}{\POLY} \\ 
 & & & $\rho$ = 0.5 & $\rho$ = 0.4 & $\rho$ = 0.3 & $\rho$ = 0.2 & $\rho$ = 0.1  \\ 
\\
\midrule
\Block{6-1}{APS} & \Block{3-1}{UCR} 
& CCP     
& 0.005$\pm$ 0.002  
& 0.004$\pm$ 0.002
& 0.005$\pm$ 0.002
& 0.002$\pm$ 0.001
& 0.004$\pm$ 0.001 
\\ 
& & Cluster-CP     
& 0.011$\pm$ 0.003  
& 0.013$\pm$ 0.003
& 0.015$\pm$ 0.004
& 0.012$\pm$ 0.003
& 0.014$\pm$ 0.003 
\\ 
& & \textbf{\newCP} 
& \textbf{0.0$\pm$ 0.0 }
& \textbf{0.0$\pm$ 0.0 }
& \textbf{0.0$\pm$ 0.0 }
& \textbf{0.0$\pm$ 0.0 }
& \textbf{0.0$\pm$ 0.0}
\\ 
\cline{2-8}
& \Block{3-1}{APSS}
& CCP     
& 26.626$\pm$ 0.133  
& 26.343$\pm$ 0.214  
& 27.168$\pm$ 0.203
& 27.363$\pm$ 0.252
& 26.159$\pm$ 0.208
\\ 
& & Cluster-CP     
& 26.150$\pm$ 0.393  
& 25.348$\pm$ 0.231  
& 26.132$\pm$ 0.415
& 26.390$\pm$ 0.270
& 25.633$\pm$ 0.268
\\ 
& & \textbf{\newCP}  
& \textbf{17.784$\pm$ 0.003 } 
& \textbf{17.752$\pm$ 0.003 }
& \textbf{17.652$\pm$ 0.003 }
& \textbf{17.629$\pm$ 0.003 }
& \textbf{17.465$\pm$ 0.003 }
\\
\midrule
\Block{6-1}{RAPS} & \Block{3-1}{UCR} 
& CCP     
& 0.005$\pm$ 0.002  
& 0.004$\pm$ 0.002
& 0.005$\pm$ 0.002
& 0.002$\pm$ 0.001
& 0.004$\pm$ 0.002  
\\ 
& & Cluster-CP     
& 0.009$\pm$ 0.003  
& 0.016$\pm$ 0.004
& 0.017$\pm$ 0.004
& 0.009$\pm$ 0.003
& 0.016$\pm$ 0.003
\\ 
& & \textbf{\newCP} 
& \textbf{0.0$\pm$ 0.0 }
& \textbf{0.0$\pm$ 0.0 }
& \textbf{0.0$\pm$ 0.0 }
& \textbf{0.0$\pm$ 0.0 }
& \textbf{0.0$\pm$ 0.0}
\\ 
\cline{2-8}
& \Block{3-1}{APSS}
& CCP     
& 26.689$\pm$ 0.142  
& 26.437$\pm$ 0.213  
& 27.254$\pm$ 0.201
& 27.450$\pm$ 0.249
& 26.248$\pm$ 0.219 
\\ 
& & Cluster-CP     
& 26.288$\pm$ 0.407  
& 25.627$\pm$ 0.318  
& 26.220$\pm$ 0.432
& 26.559$\pm$ 0.242
& 25.712$\pm$ 0.315
\\ 
& & \textbf{\newCP}  
& \textbf{17.784$\pm$ 0.003 } 
& \textbf{17.752$\pm$ 0.003 }
& \textbf{17.652$\pm$ 0.003 }
& \textbf{17.629$\pm$ 0.003 }
& \textbf{17.465$\pm$ 0.003 }
\\
\bottomrule 
\end{NiceTabular}
}
\end{table}

\begin{table}[!b]
\centering
\caption{
Results comparing \texttt{CCP}, \texttt{cluster-CP}, and \texttt{\newCP} with ResNet-20 model under different imbalance ratio $\rho=0.5$, $\rho=0.4$, $\rho=0.2$, and $\rho=0.1$ with imbalance type \MAJ~and two scoring function, APS and RAPS, on dataset mini-ImageNet.
We set UCR of \texttt{\newCP} the same as or better than that of \texttt{CCP} and \texttt{Cluster-CP} for a fair comparison of prediction set size.
}\label{tab:appendix:overall_comparison_miniimagenet_maj}
\resizebox{\textwidth}{!}{
\begin{NiceTabular}{@{}ccc!{~}ccccc@{}} \toprule 
\Block{2-1}{Scoring\\function} & \Block{2-1}{Measure}  & \Block{2-1}{Methods} & \Block{1-*}{\MAJ} \\ 
 & & & $\rho$ = 0.5 & $\rho$ = 0.4 & $\rho$ = 0.3 & $\rho$ = 0.2 & $\rho$ = 0.1  \\ 
\\
\midrule
\Block{6-1}{APS} & \Block{3-1}{UCR} 
& CCP     
& 0.010$\pm$ 0.004  
& 0.009$\pm$ 0.003
& \textbf{0.0$\pm$ 0.0 }
& 0.005$\pm$ 0.002
& 0.005$\pm$ 0.002 
\\ 
& & Cluster-CP     
& 0.008$\pm$ 0.002  
& 0.010$\pm$ 0.000
& 0.010$\pm$ 0.003
& 0.012$\pm$ 0.004
& 0.010$\pm$ 0.003 
\\ 
& & \textbf{\newCP} 
& \textbf{0.0$\pm$ 0.0 }
& \textbf{0.0$\pm$ 0.0 }
& \textbf{0.0$\pm$ 0.0 }
& \textbf{0.0$\pm$ 0.0 }
& \textbf{0.0$\pm$ 0.0 }
\\ 
\cline{2-8}
& \Block{3-1}{APSS}
& CCP     
& 27.313$\pm$ 0.154  
& 27.233$\pm$ 0.246
& 26.939$\pm$ 0.177
& 26.676$\pm$ 0.267
& 25.629$\pm$ 0.207
\\ 
& & Cluster-CP     
& 26.918$\pm$ 0.241  
& 26.156$\pm$ 0.255
& 25.786$\pm$ 0.356
& 25.632$\pm$ 0.383
& 25.348$\pm$ 0.334
\\ 
& & \textbf{\newCP}  
& \textbf{18.111$\pm$ 0.002 } 
& \textbf{17.874$\pm$ 0.002 }
& \textbf{18.081$\pm$ 0.003 }
& \textbf{17.800$\pm$ 0.002 }
& \textbf{17.167$\pm$ 0.004 }
\\
\midrule
\Block{6-1}{RAPS} & \Block{3-1}{UCR} 
& CCP     
& 0.009$\pm$ 0.003  
& 0.009$\pm$ 0.003
& \textbf{0.0$\pm$ 0.0 }
& 0.005$\pm$ 0.002
& 0.005$\pm$ 0.002  
\\ 
& & Cluster-CP     
& 0.007$\pm$ 0.002  
& 0.011$\pm$ 0.002
& 0.013$\pm$ 0.004
& 0.014$\pm$ 0.004
& 0.009$\pm$ 0.002
\\ 
& & \textbf{\newCP} 
& \textbf{0.0$\pm$(0.0) }
& \textbf{0.0$\pm$(0.0) }
& \textbf{0.0$\pm$(0.0) }
& \textbf{0.0$\pm$(0.0)}
& \textbf{0.0$\pm$(0.0)}
\\ 
\cline{2-8}
& \Block{3-1}{APSS}
& CCP     
& 27.397$\pm$ 0.162  
& 27.320$\pm$ 0.244
& 27.013$\pm$ 0.177
& 26.782$\pm$ 0.269
& 25.725$\pm$ 0.214
\\ 
& & Cluster-CP     
& 26.969$\pm$ 0.305  
& 26.293$\pm$ 0.245
& 25.956$\pm$ 0.308
& 25.803$\pm$ 0.440
& 25.532$\pm$ 0.350 
\\ 
& & \textbf{\newCP}  
& \textbf{18.111$\pm$ 0.002 } 
& \textbf{17.874$\pm$ 0.002 }
& \textbf{18.081$\pm$ 0.003 }
& \textbf{17.800$\pm$ 0.002 }
& \textbf{17.167$\pm$ 0.004 }
\\
\bottomrule 
\end{NiceTabular}
}
\end{table}

\begin{table}[!b]
\centering
\caption{
Results comparing \texttt{CCP}, \texttt{cluster-CP}, and \texttt{\newCP} with ResNet-20 model under different imbalance ratio $\rho=0.5$, $\rho=0.4$, $\rho=0.2$, and $\rho=0.1$ with imbalance type \EXP~and two scoring function, APS and RAPS, on dataset Food-101.
We set UCR of \texttt{\newCP} the same as or better than that of \texttt{CCP} and \texttt{Cluster-CP} for a fair comparison of prediction set size.
}
\resizebox{\textwidth}{!}{
\begin{NiceTabular}{@{}ccc!{~}ccccc@{}} \toprule 
\Block{2-1}{Scoring\\function} & \Block{2-1}{Measure}  & \Block{2-1}{Methods} & \Block{1-*}{\EXP} \\ 
 & & & $\rho$ = 0.5 & $\rho$ = 0.4 & $\rho$ = 0.3 & $\rho$ = 0.2 & $\rho$ = 0.1  \\ 
\\
\midrule
\Block{6-1}{APS} & \Block{3-1}{UCR} 
& CCP     
& 0.006$\pm$ 0.002  
& 0.010$\pm$ 0.002
& 0.008$\pm$ 0.002
& 0.014$\pm$ 0.004
& 0.006$\pm$ 0.002 
\\ 
& & Cluster-CP     
& 0.003$\pm$ 0.002  
& 0.009$\pm$ 0.003
& 0.006$\pm$ 0.003
& 0.008$\pm$ 0.003
& 0.009$\pm$ 0.003 
\\ 
& & \textbf{\newCP} 
& \textbf{0.0$\pm$ 0.0 }
& \textbf{0.0$\pm$ 0.0 }
& \textbf{0.0$\pm$ 0.0 }
& \textbf{0.0$\pm$ 0.0 }
& \textbf{0.0$\pm$ 0.0}
\\ 
\cline{2-8}
& \Block{3-1}{APSS}
& CCP     
& 27.003$\pm$ 0.183  
& 27.024$\pm$ 0.162
& 28.074$\pm$ 0.199
& 28.512$\pm$ 0.154
& 30.875$\pm$ 0.163 
\\ 
& & Cluster-CP     
& 29.020$\pm$ 0.281  
& 30.120$\pm$ 0.440
& 30.529$\pm$ 0.381
& 31.096$\pm$ 0.350
& 33.327$\pm$ 0.440 
\\ 
& & \textbf{\newCP}  
& \textbf{18.369$\pm$ 0.003} 
& \textbf{18.339$\pm$ 0.004}
& \textbf{18.803$\pm$ 0.003}
& \textbf{19.612$\pm$ 0.005}
& \textbf{21.556$\pm$ 0.006}
\\
\midrule
\Block{6-1}{RAPS} & \Block{3-1}{UCR} 
& CCP     
& 0.006$\pm$ 0.003  
& 0.010$\pm$ 0.002
& 0.008$\pm$ 0.002
& 0.014$\pm$ 0.004
& 0.006$\pm$ 0.002 
\\ 
& & Cluster-CP     
& 0.004$\pm$ 0.003  
& 0.010$\pm$ 0.003
& 0.006$\pm$ 0.003
& 0.010$\pm$ 0.002
& 0.012$\pm$ 0.004  
\\ 
& & \textbf{\newCP} 
& \textbf{0.0$\pm$(0.0) }
& \textbf{0.0$\pm$(0.0) }
& \textbf{0.0$\pm$(0.0) }
& \textbf{0.0$\pm$(0.0)}
& \textbf{0.0$\pm$(0.0)}
\\ 
\cline{2-8}
& \Block{3-1}{APSS}
& CCP     
& 27.022$\pm$ 0.192  
& 27.043$\pm$ 0.163
& 28.098$\pm$ 0.199
& 28.535$\pm$ 0.155
& 30.900$\pm$ 0.170  
\\ 
& & Cluster-CP     
& 28.953$\pm$ 0.333  
& 30.242$\pm$ 0.466
& 30.587$\pm$ 0.377
& 30.924$\pm$ 0.317
& 33.375$\pm$ 0.377  
\\ 
& & \textbf{\newCP}  
& \textbf{18.369$\pm$ 0.004} 
& \textbf{18.339$\pm$ 0.004}
& \textbf{18.803$\pm$ 0.003}
& \textbf{19.612$\pm$ 0.005}
& \textbf{21.556$\pm$ 0.006}
\\
\bottomrule 
\end{NiceTabular}
}
\label{tab:appendix:overall_comparison_food101_exp}
\end{table}

\begin{table}[!b]
\centering
\caption{
Results comparing \texttt{CCP}, \texttt{cluster-CP}, and \texttt{\newCP} with ResNet-20 model under different imbalance ratio $\rho=0.5$, $\rho=0.4$, $\rho=0.2$, and $\rho=0.1$ with imbalance type \POLY~and two scoring function, APS and RAPS, on dataset Food-101.
We set UCR of \texttt{\newCP} the same as or better than that of \texttt{CCP} and \texttt{Cluster-CP} for a fair comparison of prediction set size.
}
\label{tab:appendix:overall_comparison_food101_poly}
\resizebox{\textwidth}{!}{
\begin{NiceTabular}{@{}ccc!{~}ccccc@{}} \toprule 
\Block{2-1}{Scoring\\function} & \Block{2-1}{Measure}  & \Block{2-1}{Methods} & \Block{1-*}{\POLY} \\ 
 & & & $\rho$ = 0.5 & $\rho$ = 0.4 & $\rho$ = 0.3 & $\rho$ = 0.2 & $\rho$ = 0.1  \\ 
\\
\midrule
\Block{6-1}{APS} & \Block{3-1}{UCR} 
& CCP     
& 0.009$\pm$ 0.003  
& 0.005$\pm$ 0.003
& 0.009$\pm$ 0.003
& 0.011$\pm$ 0.003
& 0.008$\pm$ 0.001 
\\ 
& & Cluster-CP     
& 0.004$\pm$ 0.001  
& 0.012$\pm$ 0.002
& 0.012$\pm$ 0.004
& 0.011$\pm$ 0.002
& 0.009$\pm$ 0.002 
\\ 
& & \textbf{\newCP} 
& \textbf{0.0$\pm$ 0.0 }
& \textbf{0.0$\pm$ 0.0 }
& \textbf{0.0$\pm$ 0.0 }
& \textbf{0.0$\pm$ 0.0 }
& \textbf{0.001$\pm$ 0.001 }
\\ 
\cline{2-8}
& \Block{3-1}{APSS}
& CCP     
& 30.943$\pm$ 0.119  
& 31.239$\pm$ 0.198
& 32.283$\pm$ 0.169
& 33.570$\pm$ 0.163
& 35.912$\pm$ 0.105 
\\ 
& & Cluster-CP     
& 33.079$\pm$ 0.393  
& 33.951$\pm$ 0.531
& 34.626$\pm$ 0.352
& 36.546$\pm$ 0.490
& 38.301$\pm$ 0.232 
\\ 
& & \textbf{\newCP}  
& \textbf{21.499$\pm$ 0.003 } 
& \textbf{21.460$\pm$ 0.005 }
& \textbf{22.882$\pm$ 0.005 }
& \textbf{23.708$\pm$ 0.004 }
& \textbf{25.853$\pm$ 0.004 }
\\
\midrule
\Block{6-1}{RAPS} & \Block{3-1}{UCR} 
& CCP     
& 0.009$\pm$ 0.003  
& 0.006$\pm$ 0.003
& 0.009$\pm$ 0.003
& 0.011$\pm$ 0.003
& 0.008$\pm$ 0.001
\\ 
& & Cluster-CP     
& 0.006$\pm$ 0.002  
& 0.013$\pm$ 0.002
& 0.012$\pm$ 0.004
& 0.016$\pm$ 0.002
& 0.006$\pm$ 0.003 
\\ 
& & \textbf{\newCP} 
& \textbf{0.0$\pm$ 0.0 }
& \textbf{0.0$\pm$ 0.0 }
& \textbf{0.0$\pm$ 0.0 }
& \textbf{0.0$\pm$ 0.0 }
& \textbf{0.001$\pm$ 0.001 }
\\ 
\cline{2-8}
& \Block{3-1}{APSS}
& CCP     
& 30.966$\pm$ 0.125  
& 31.257$\pm$ 0.197
& 32.302$\pm$ 0.169
& 33.595$\pm$ 0.164
& 35.940$\pm$ 0.111  
\\ 
& & Cluster-CP     
& 33.337$\pm$ 0.409  
& 33.936$\pm$ 0.448
& 34.878$\pm$ 0.282
& 36.505$\pm$ 0.520
& 38.499$\pm$ 0.216
\\ 
& & \textbf{\newCP}  
& \textbf{21.499$\pm$ 0.003 } 
& \textbf{21.460$\pm$ 0.005 }
& \textbf{22.882$\pm$ 0.005 }
& \textbf{23.708$\pm$ 0.004 }
& \textbf{25.853$\pm$ 0.004 }
\\
\bottomrule 
\end{NiceTabular}
}
\end{table}

\begin{table}[!b]
\centering
\caption{
Results comparing \texttt{CCP}, \texttt{cluster-CP}, and \texttt{\newCP} with ResNet-20 model under different imbalance ratio $\rho=0.5$, $\rho=0.4$, $\rho=0.2$, and $\rho=0.1$ with imbalance type \MAJ~and two scoring function, APS and RAPS, on dataset Food-101.
We set UCR of \texttt{\newCP} the same as or better than that of \texttt{CCP} and \texttt{Cluster-CP} for a fair comparison of prediction set size.
}
\label{tab:appendix:overall_comparison_food101_maj}
\resizebox{\textwidth}{!}{
\begin{NiceTabular}{@{}ccc!{~}ccccc@{}} \toprule 
\Block{2-1}{Scoring\\function} & \Block{2-1}{Measure}  & \Block{2-1}{Methods} & \Block{1-*}{\MAJ} \\ 
 & & & $\rho$ = 0.5 & $\rho$ = 0.4 & $\rho$ = 0.3 & $\rho$ = 0.2 & $\rho$ = 0.1  \\ 
\\
\midrule
\Block{6-1}{APS} & \Block{3-1}{UCR} 
& CCP     
& 0.006$\pm$ 0.001  
& 0.005$\pm$ 0.002
& 0.008$\pm$ 0.003
& 0.010$\pm$ 0.002
& 0.008$\pm$ 0.002 
\\ 
& & Cluster-CP     
& 0.011$\pm$ 0.003  
& 0.005$\pm$ 0.002
& 0.014$\pm$ 0.004 
& 0.016$\pm$ 0.004
& 0.011$\pm$ 0.002 
\\ 
& & \textbf{\newCP} 
& \textbf{0.0$\pm$ 0.0 }
& \textbf{0.0$\pm$ 0.0 }
& \textbf{0.0$\pm$ 0.0 }
& \textbf{0.0$\pm$ 0.0 }
& \textbf{0.0$\pm$ 0.0}
\\ 
\cline{2-8}
& \Block{3-1}{APSS}
& CCP     
& 27.415$\pm$ 0.194  
& 29.369$\pm$ 0.120
& 30.672$\pm$ 0.182
& 31.966$\pm$ 0.165
& 36.776$\pm$ 0.132 
\\ 
& & Cluster-CP     
& 30.071$\pm$ 0.412  
& 31.656$\pm$ 0.261
& 32.857$\pm$ 0.469
& 33.774$\pm$ 0.494
& 39.632$\pm$ 0.342
\\ 
& & \textbf{\newCP}  
& \textbf{19.398$\pm$ 0.006 } 
& \textbf{20.046$\pm$ 0.004 }
& \textbf{21.425$\pm$ 0.003 }
& \textbf{22.175$\pm$ 0.004 }
& \textbf{26.585$\pm$ 0.004 }
\\
\midrule
\Block{6-1}{RAPS} & \Block{3-1}{UCR} 
& CCP     
& 0.006$\pm$ 0.002  
& 0.005$\pm$ 0.002
& 0.008$\pm$ 0.003
& 0.010$\pm$ 0.002
& 0.008$\pm$ 0.002  
\\ 
& & Cluster-CP     
& 0.011$\pm$ 0.003  
& 0.005$\pm$ 0.002
& 0.013$\pm$ 0.004 
& 0.014$\pm$ 0.004
& 0.014$\pm$ 0.004 
\\ 
& & \textbf{\newCP} 
& \textbf{0.0$\pm$(0.0) }
& \textbf{0.0$\pm$(0.0) }
& \textbf{0.0$\pm$(0.0) }
& \textbf{0.0$\pm$(0.0)}
& \textbf{0.0$\pm$(0.0)}
\\ 
\cline{2-8}
& \Block{3-1}{APSS}
& CCP     
& 27.439$\pm$ 0.203  
& 29.393$\pm$ 0.120
& 30.691$\pm$ 0.182
& 31.987$\pm$ 0.165
& 36.802$\pm$ 0.138 
\\ 
& & Cluster-CP     
& 29.946$\pm$ 0.407  
& 31.409$\pm$ 0.303
& 32.724$\pm$ 0.551
& 33.686$\pm$ 0.501
& 39.529$\pm$ 0.306 
\\ 
& & \textbf{\newCP}  
& \textbf{19.397$\pm$ 0.006 } 
& \textbf{20.046$\pm$ 0.004 }
& \textbf{21.425$\pm$ 0.003 }
& \textbf{22.175$\pm$ 0.004 }
& \textbf{26.585$\pm$ 0.004 }
\\
\bottomrule 
\end{NiceTabular}
}
\end{table}

\clearpage
\newpage

\begin{figure*}[!ht]
    \centering
    \begin{minipage}{.24\textwidth}
        \centering
        (a) CIFAR-10
    \end{minipage}%
    \begin{minipage}{.24\textwidth}
        \centering
        (b) CIFAR-100
    \end{minipage}%
    \begin{minipage}{.24\textwidth}
        \centering
        (c) mini-ImageNet
    \end{minipage}%
    \begin{minipage}{.24\textwidth}
        \centering
        (d) Food-101
    \end{minipage}

    \subfigure{
        \begin{minipage}{0.23\linewidth}
            \includegraphics[width=\linewidth]{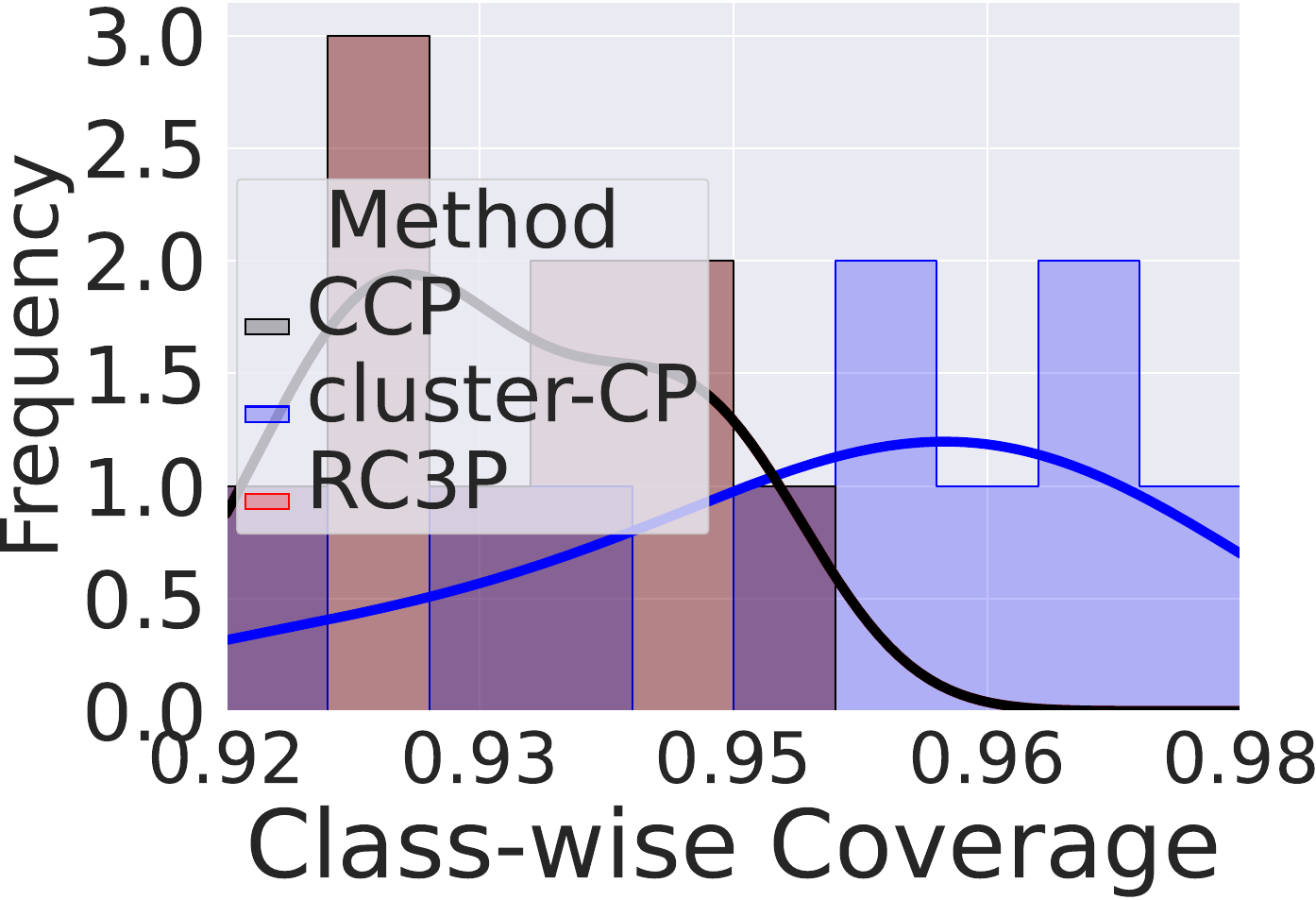}
        \end{minipage}
    }
    \subfigure{
        \begin{minipage}{0.23\linewidth}
            \includegraphics[width=\linewidth]{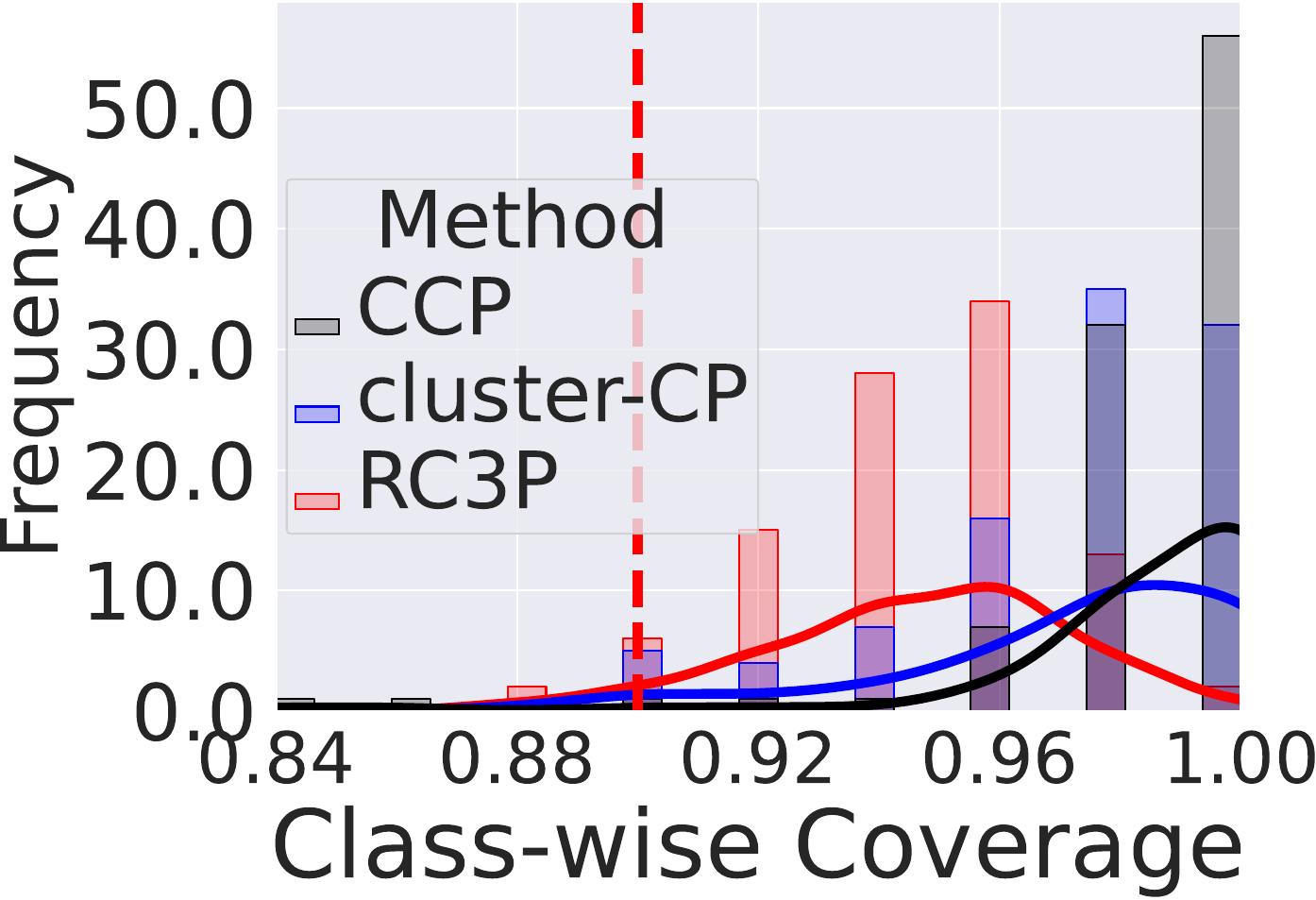}
        \end{minipage}
    }
    \subfigure{
        \begin{minipage}{0.23\linewidth}
            \includegraphics[width=\linewidth]{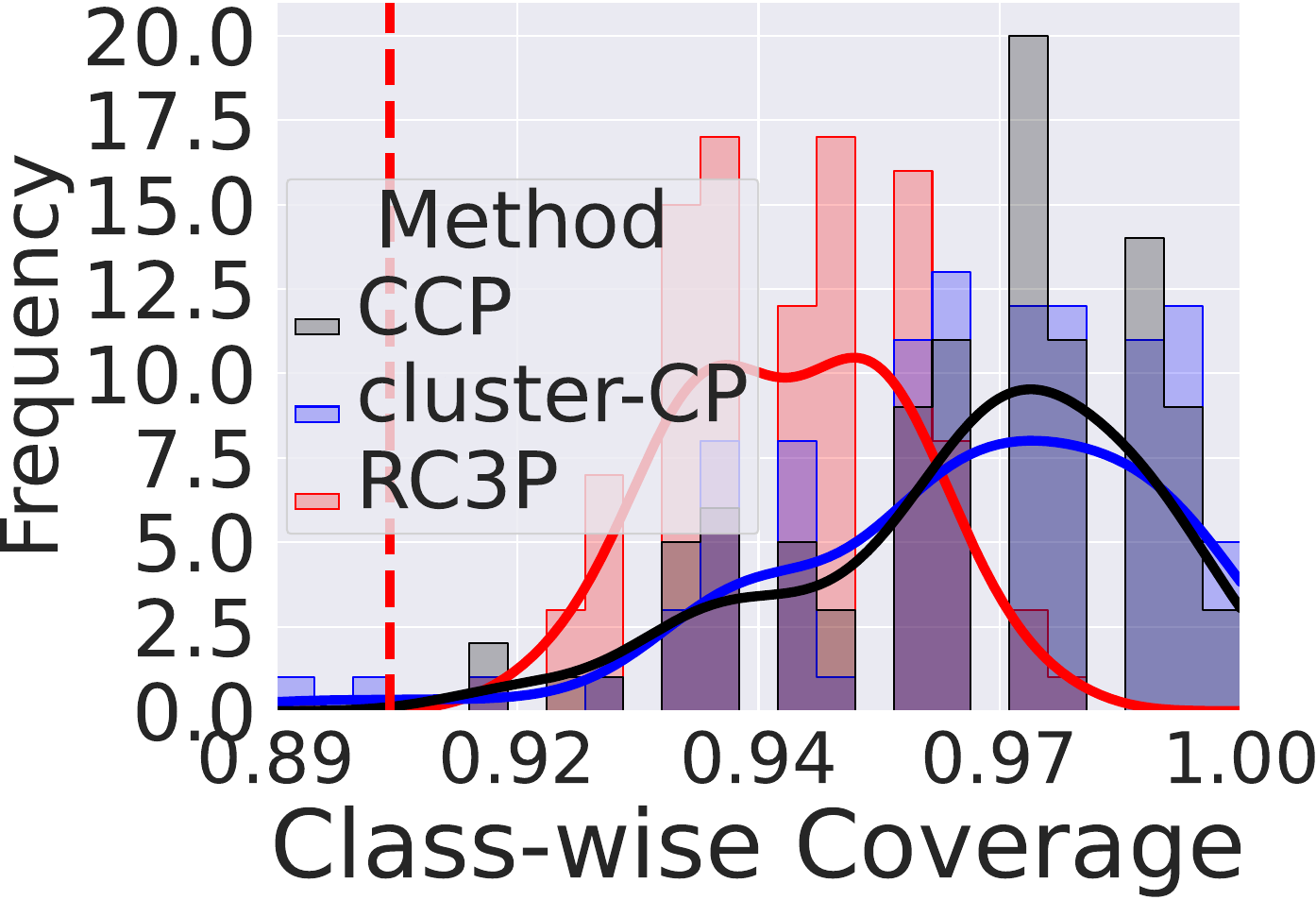}
        \end{minipage}
    }
    \subfigure{
        \begin{minipage}{0.23\linewidth}
            \includegraphics[width=\linewidth]{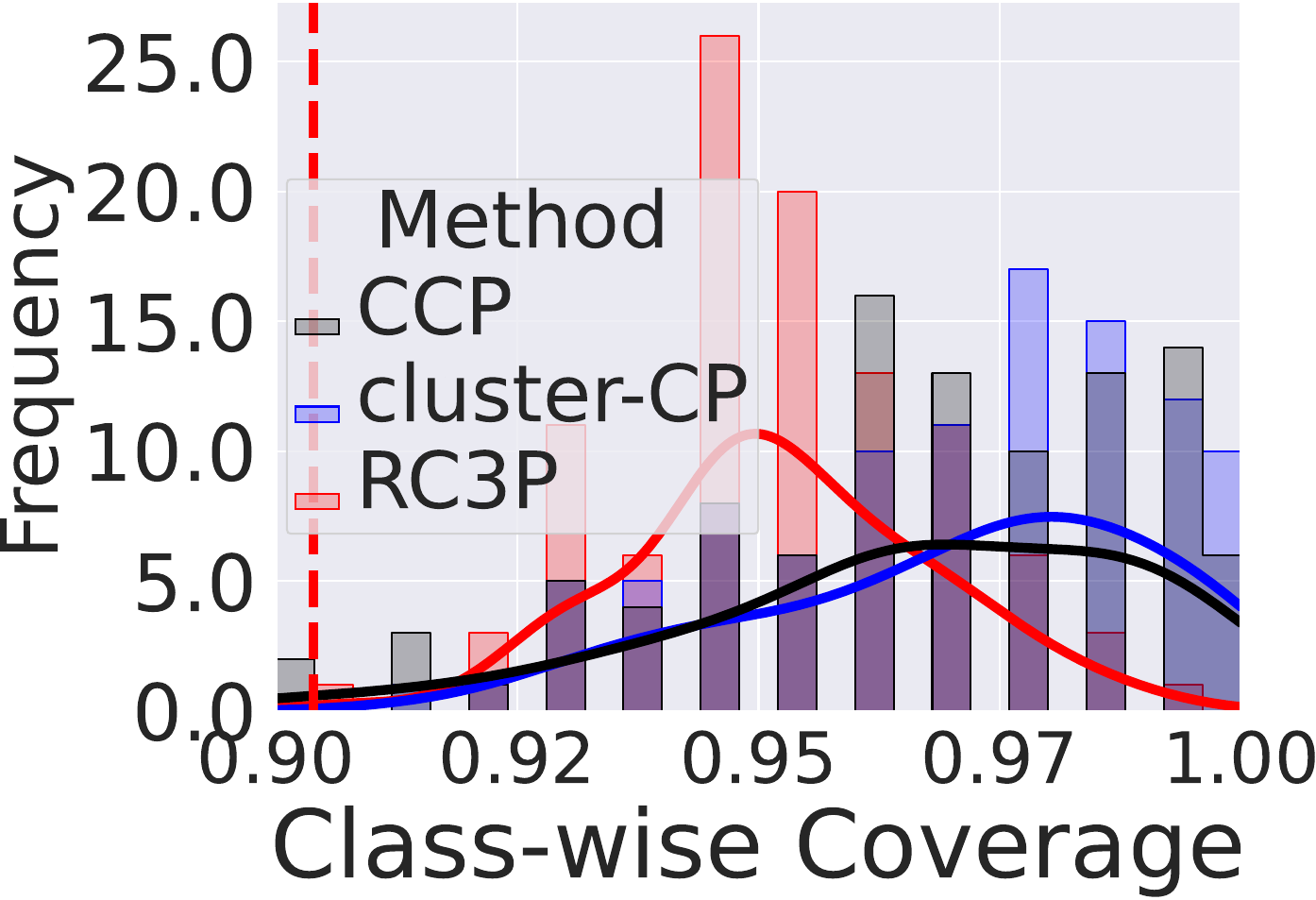}
        \end{minipage}
    }
    \subfigure{
        \begin{minipage}{0.23\linewidth}
            \includegraphics[width=\linewidth]{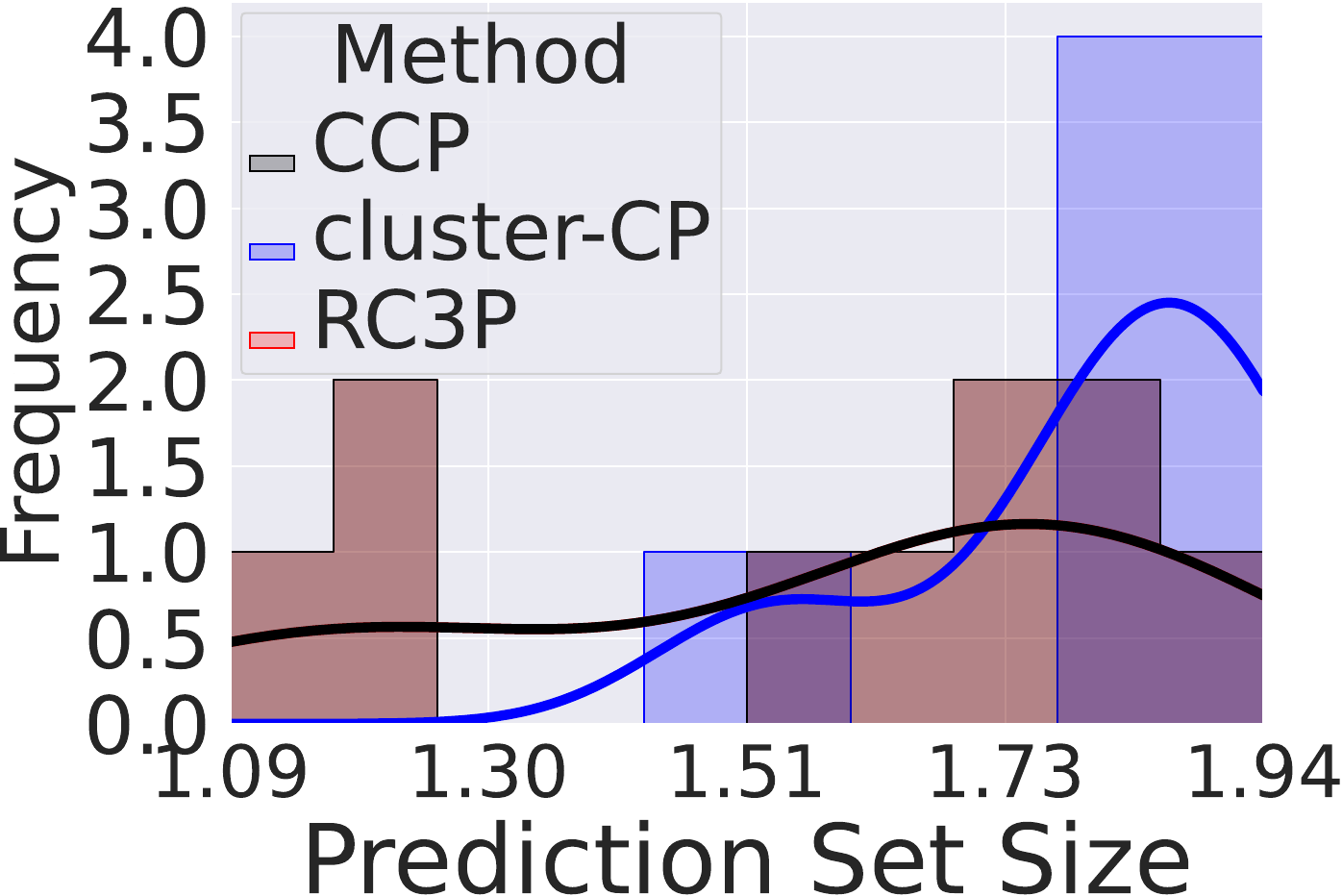}
        \end{minipage}
    }
    \subfigure{
        \begin{minipage}{0.23\linewidth}
            \includegraphics[width=\linewidth]{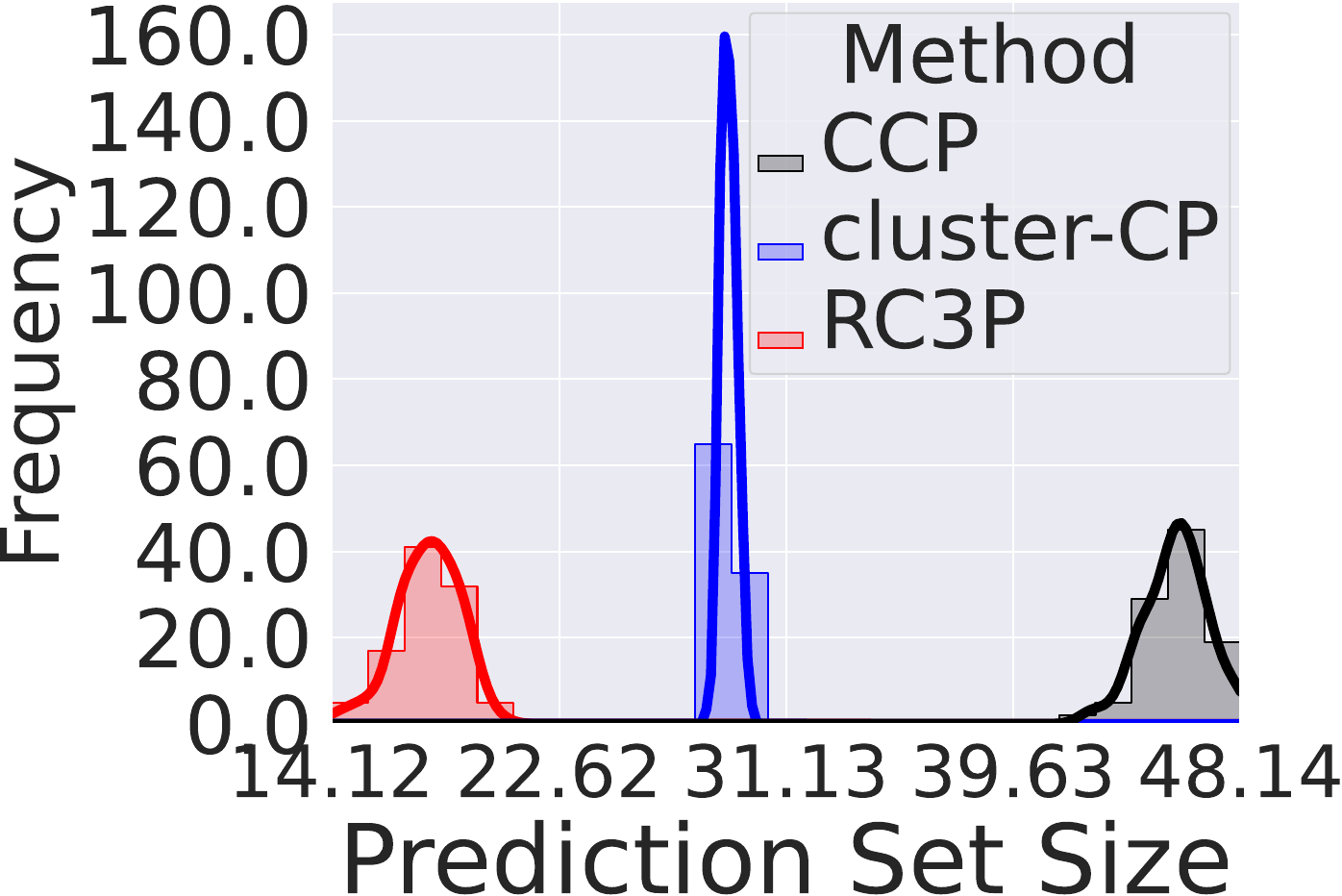}
        \end{minipage}
    }
    \subfigure{
        \begin{minipage}{0.23\linewidth}
            \includegraphics[width=\linewidth]{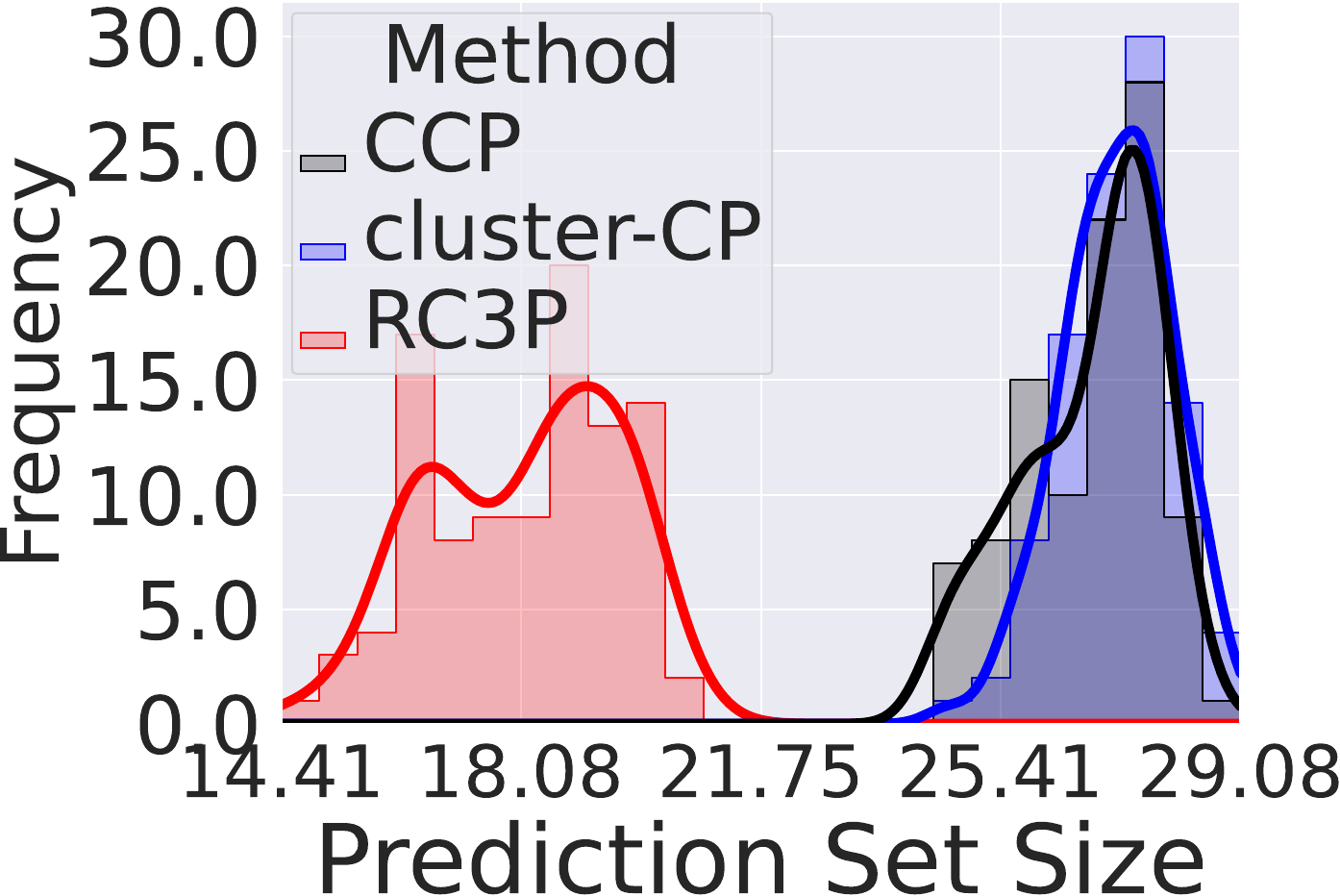}
        \end{minipage}
    }
    \subfigure{
        \begin{minipage}{0.23\linewidth}
            \includegraphics[width=\linewidth]{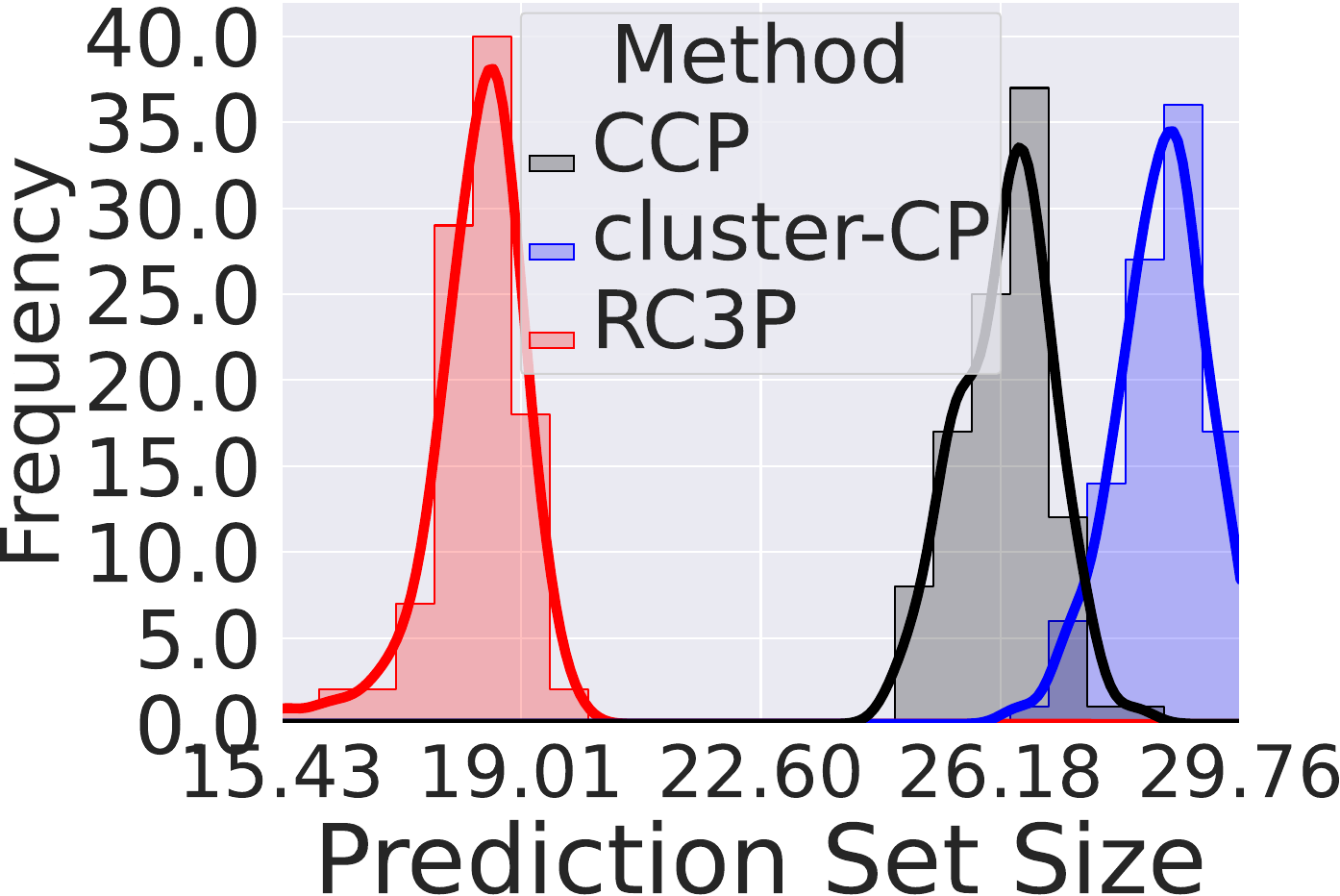}
        \end{minipage}
    }
    \caption{
    Class-conditional coverage (Top row) and prediction set size (Bottom row) achieved by \texttt{CCP}, \texttt{Cluster-CP}, and \texttt{\newCP} methods when $\alpha = 0.1$ on CIFAR-10, CIFAR-100, mini-ImageNet, and Food-101 datasets with imbalance type \EXP~for imbalance ratio $\rho=0.5$.
    We clarify that \texttt{\newCP} overlaps with \texttt{CCP} on CIFAR-10.
    It is clear that \texttt{\newCP} has more densely distributed class-conditional coverage above $0.9$ (the target $1-\alpha$ class-conditional coverage) than \texttt{CCP} and \texttt{Cluster-CP} with significantly smaller prediction sets on CIFAR-100, mini-ImageNet and Food-101.
    }
    \label{fig:overall_comparison_four_datasets_exp_0.5}
\end{figure*}

\begin{figure*}[!ht]
    \centering
    \begin{minipage}{.24\textwidth}
        \centering
        (a) CIFAR-10
    \end{minipage}%
    \begin{minipage}{.24\textwidth}
        \centering
        (b) CIFAR-100
    \end{minipage}%
    \begin{minipage}{.24\textwidth}
        \centering
        (c) mini-ImageNet
    \end{minipage}%
    \begin{minipage}{.24\textwidth}
        \centering
        (d) Food-101
    \end{minipage}

    \subfigure{
        \begin{minipage}{0.23\linewidth}
            \includegraphics[width=\linewidth]{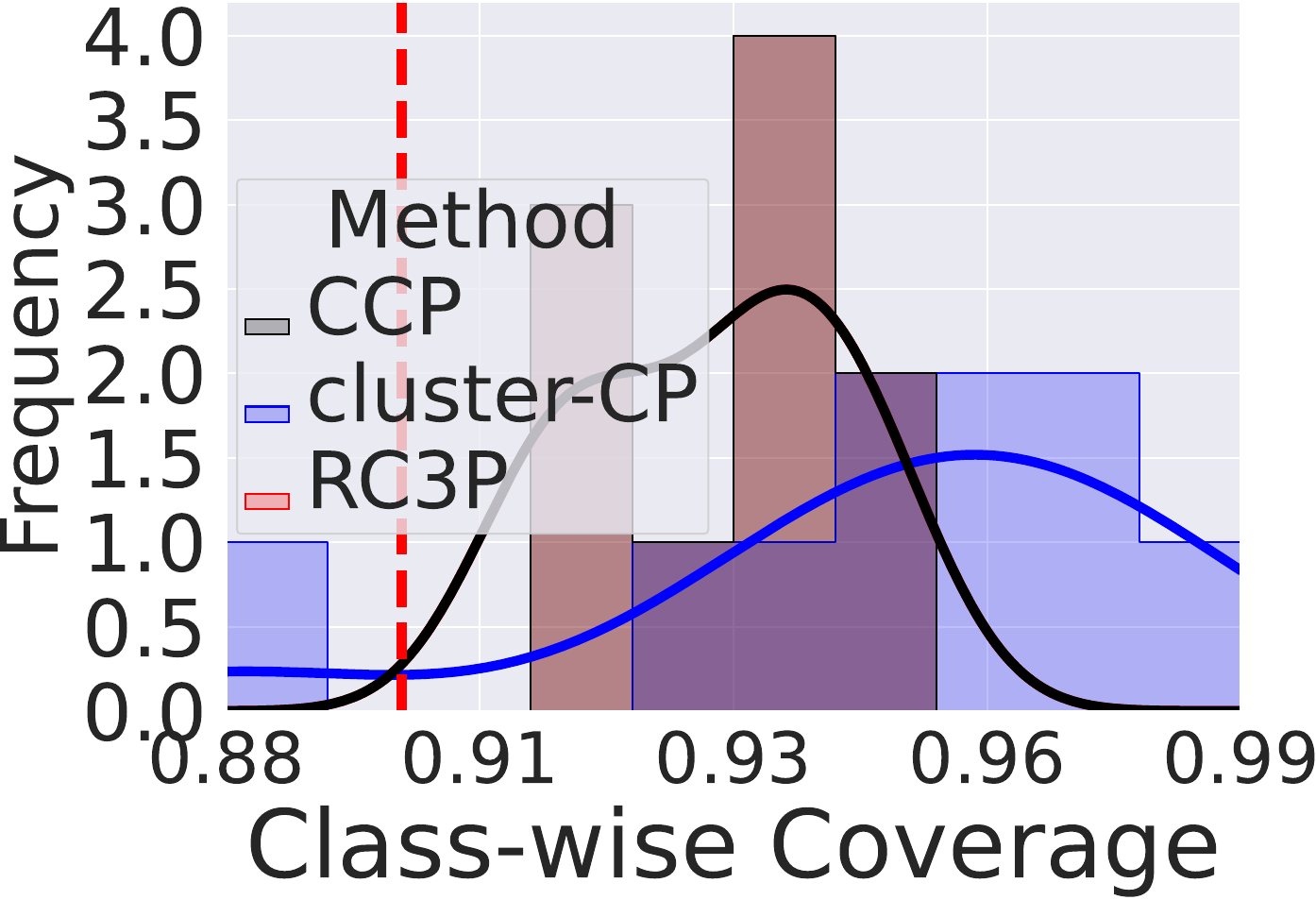}
        \end{minipage}
    }
    \subfigure{
        \begin{minipage}{0.23\linewidth}
            \includegraphics[width=\linewidth]{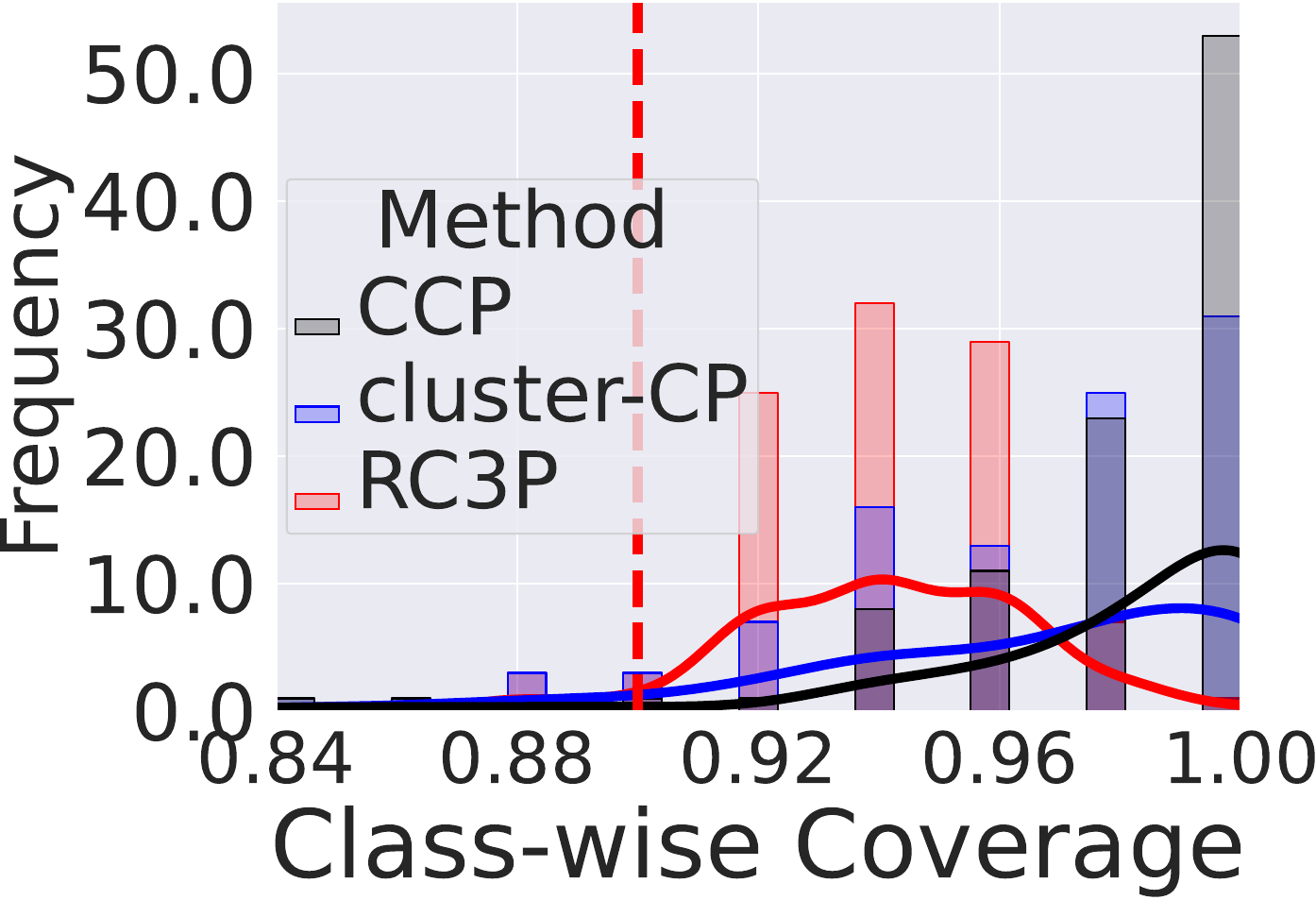}
        \end{minipage}
    }
    \subfigure{
        \begin{minipage}{0.23\linewidth}
            \includegraphics[width=\linewidth]{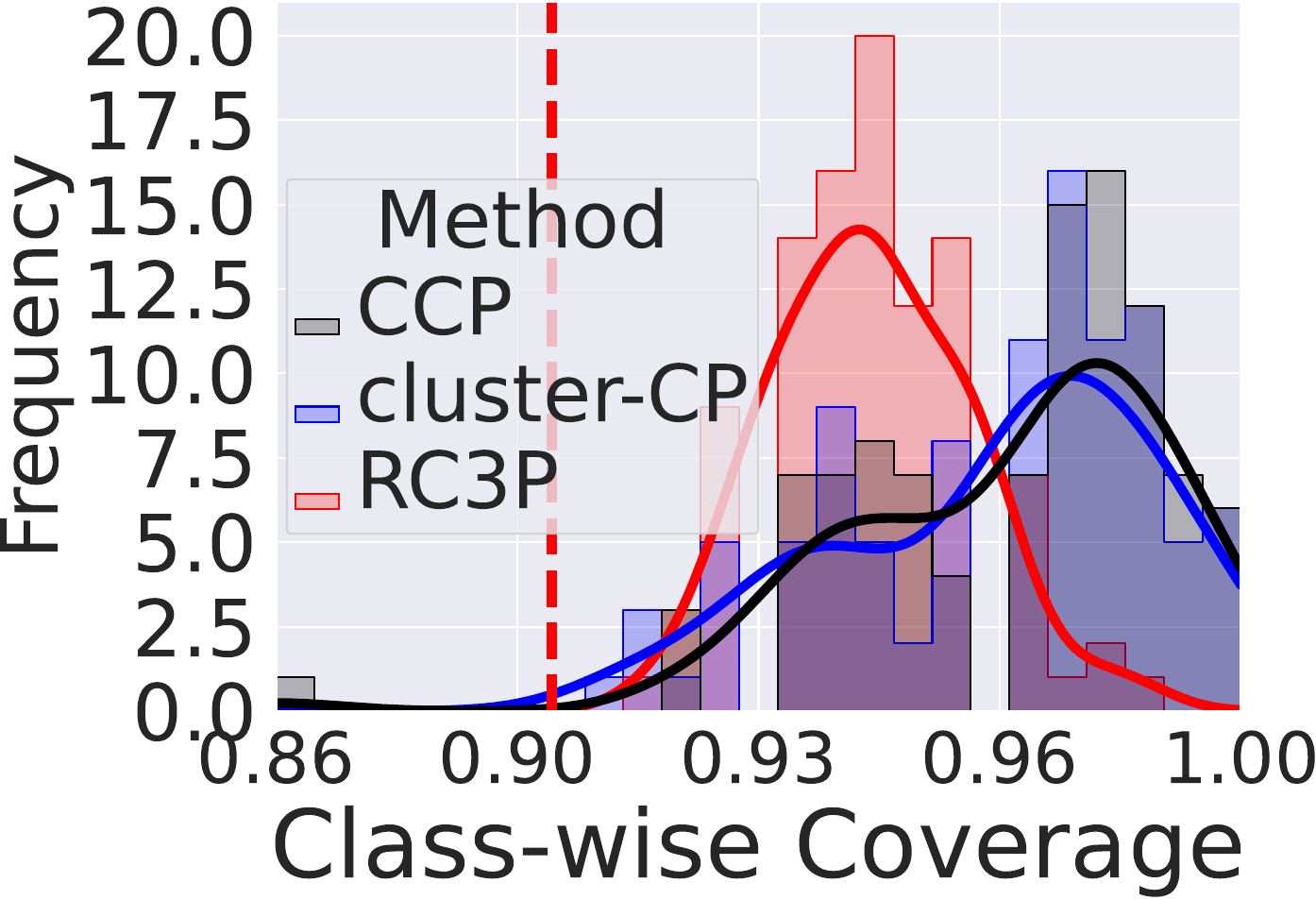}
        \end{minipage}
    }
    \subfigure{
        \begin{minipage}{0.23\linewidth}
            \includegraphics[width=\linewidth]{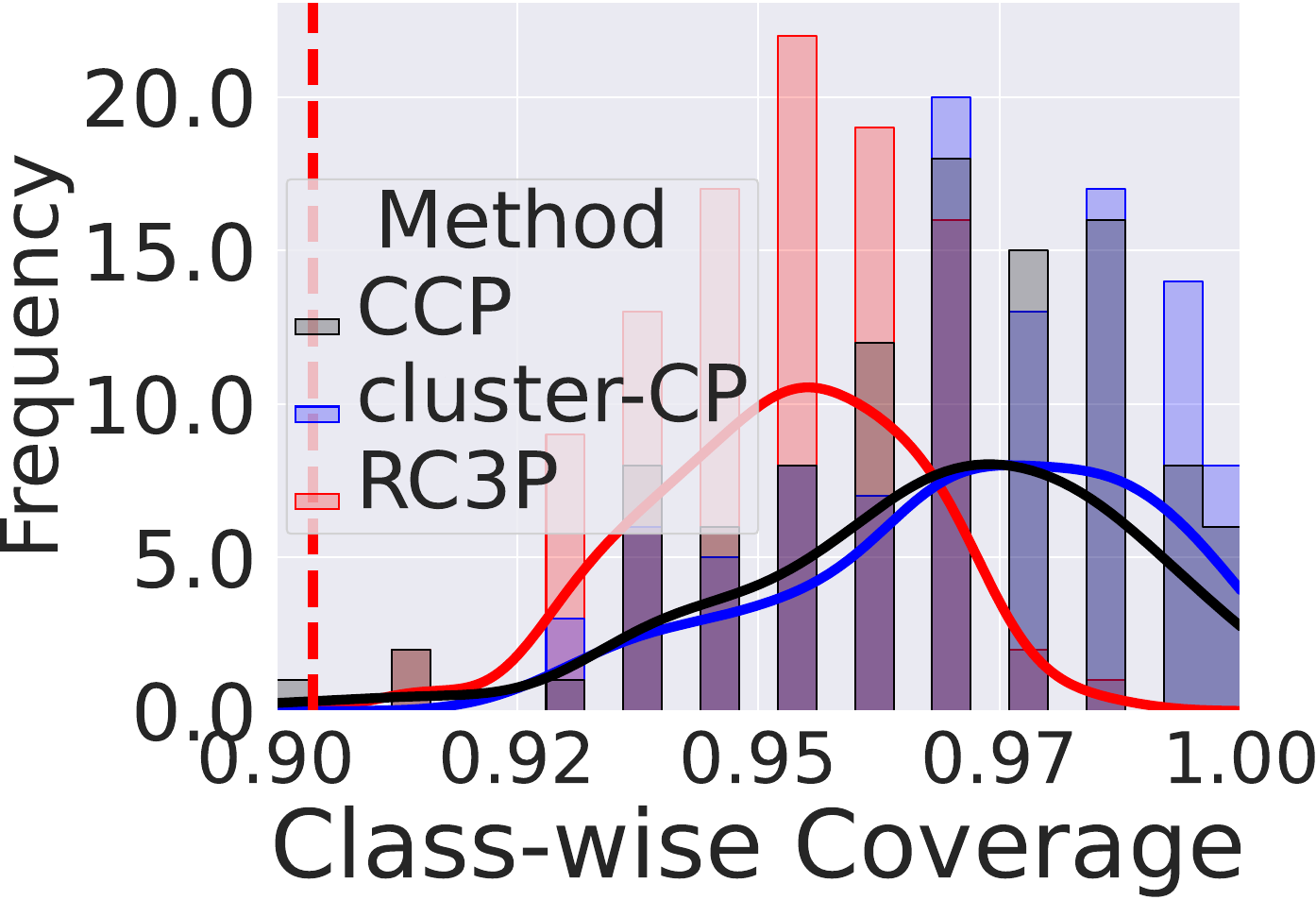}
        \end{minipage}
    }
    \subfigure{
        \begin{minipage}{0.23\linewidth}
            \includegraphics[width=\linewidth]{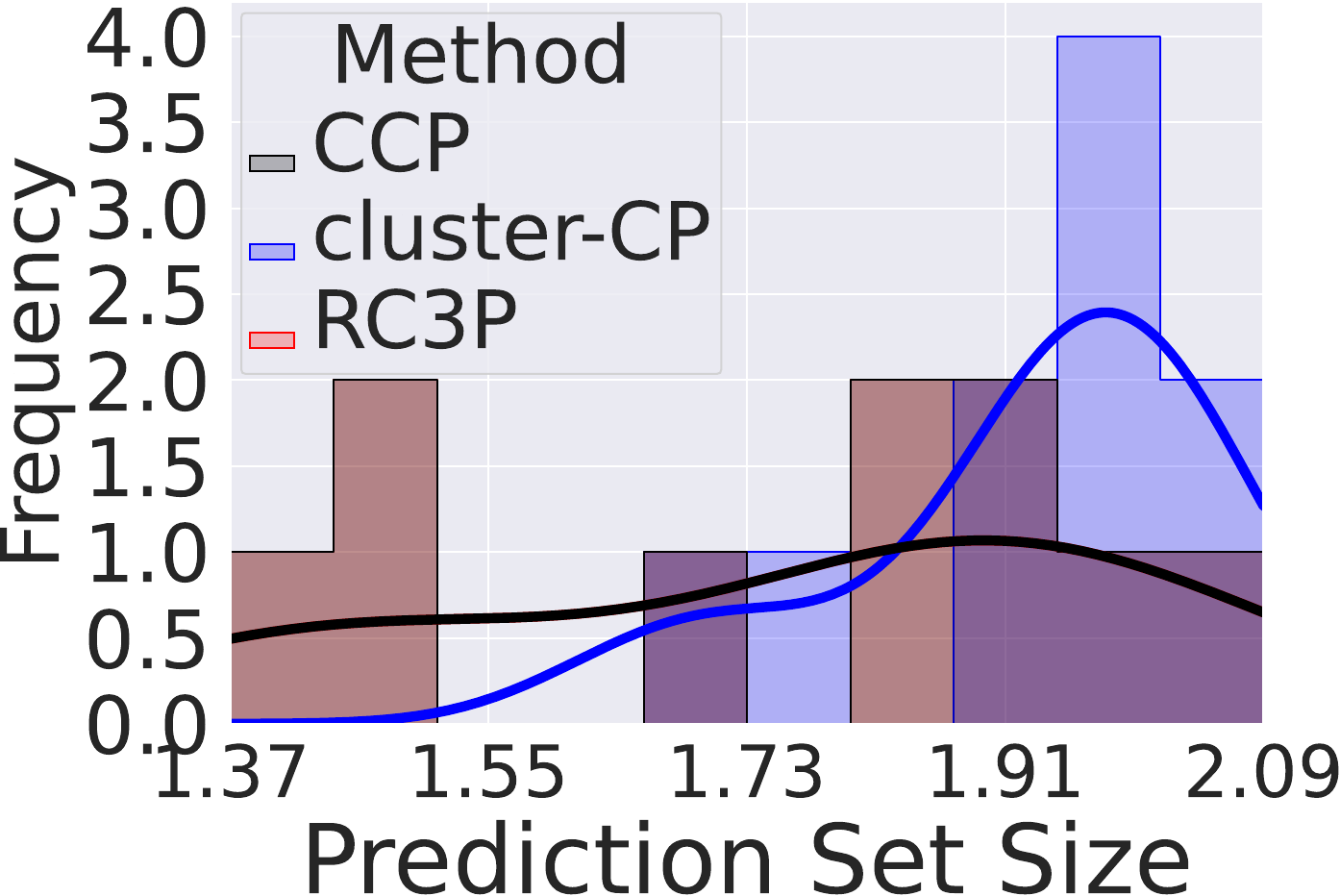}
        \end{minipage}
    }
    \subfigure{
        \begin{minipage}{0.23\linewidth}
            \includegraphics[width=\linewidth]{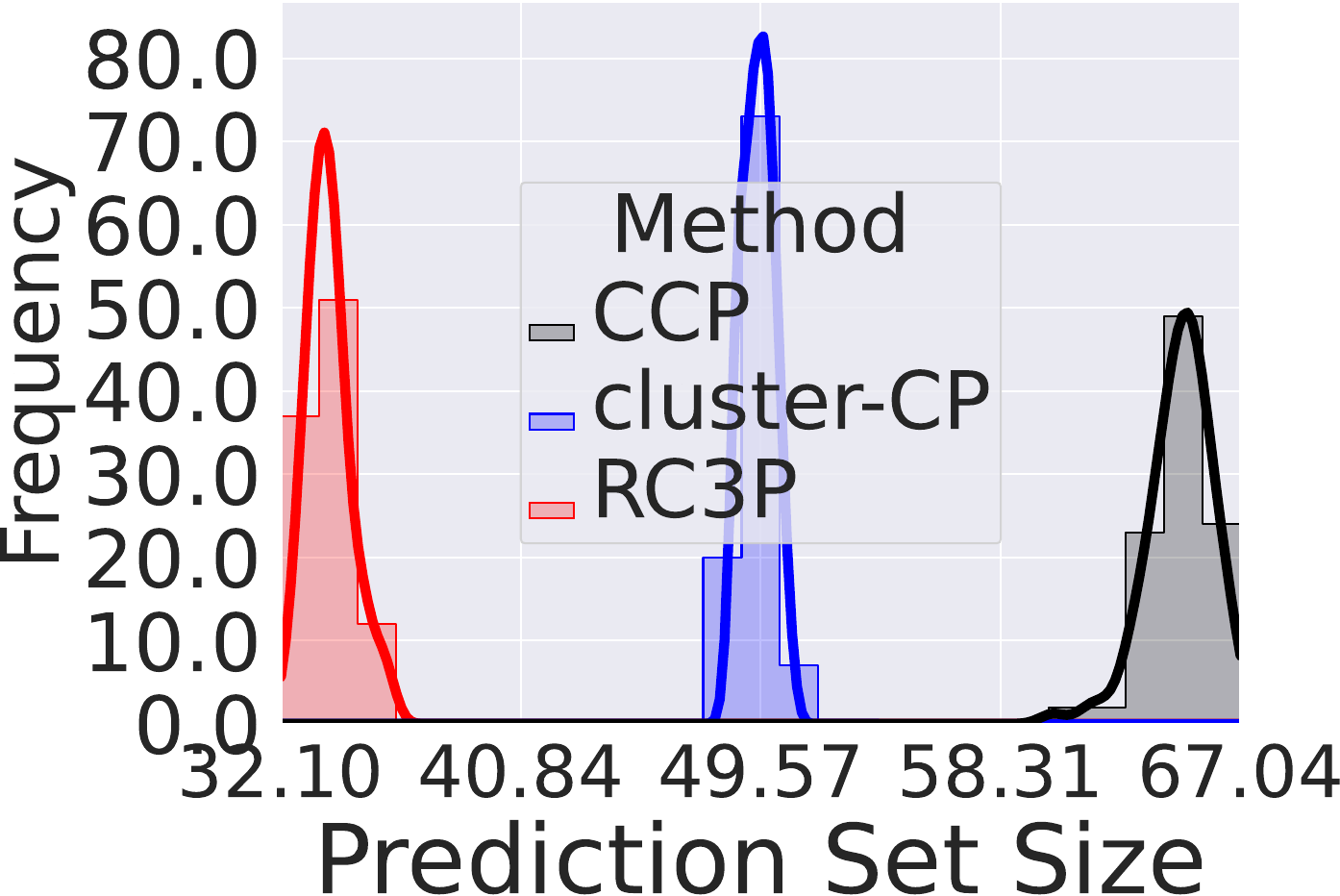}
        \end{minipage}
    }
    \subfigure{
        \begin{minipage}{0.23\linewidth}
            \includegraphics[width=\linewidth]{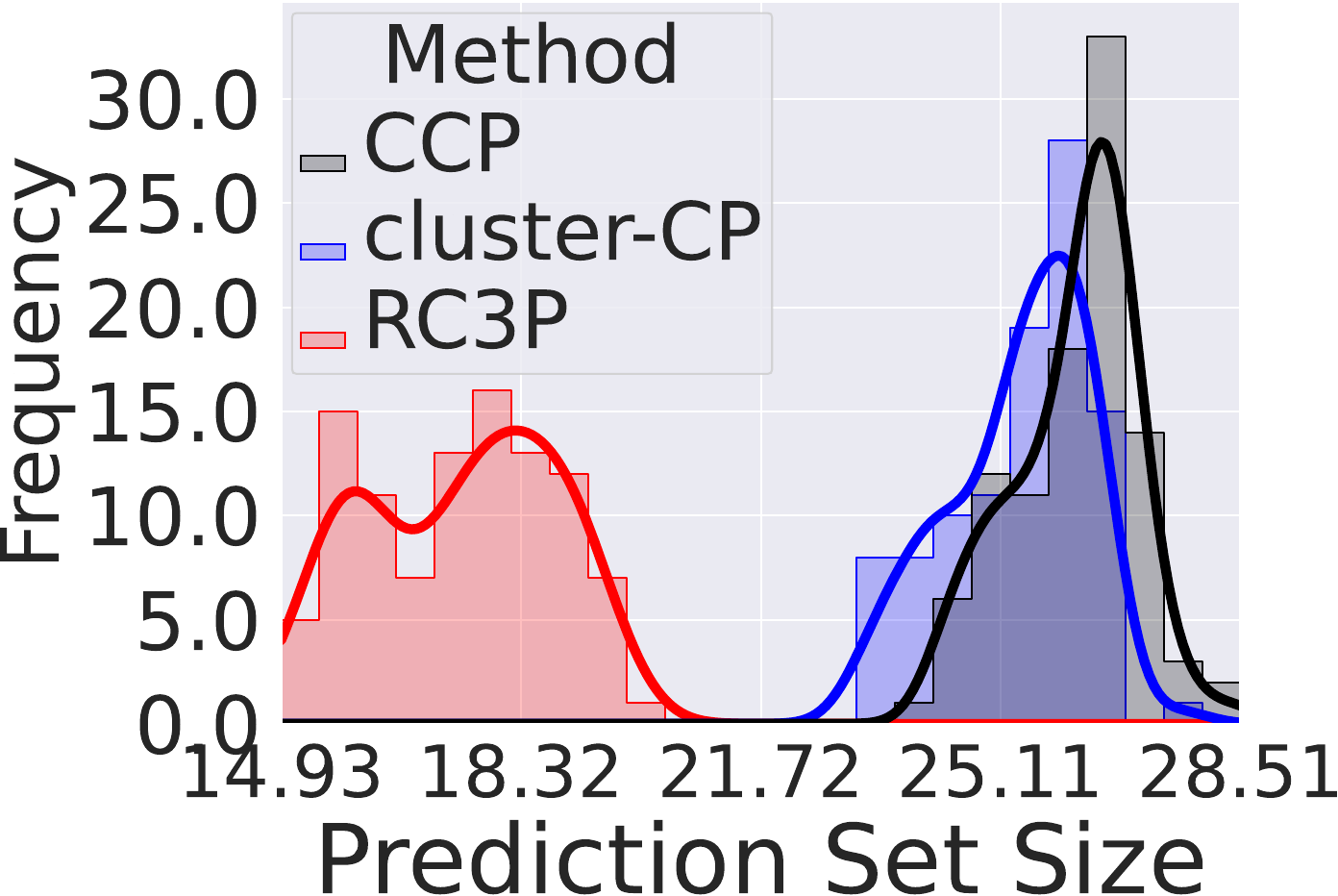}
        \end{minipage}
    }
    \subfigure{
        \begin{minipage}{0.23\linewidth}
            \includegraphics[width=\linewidth]{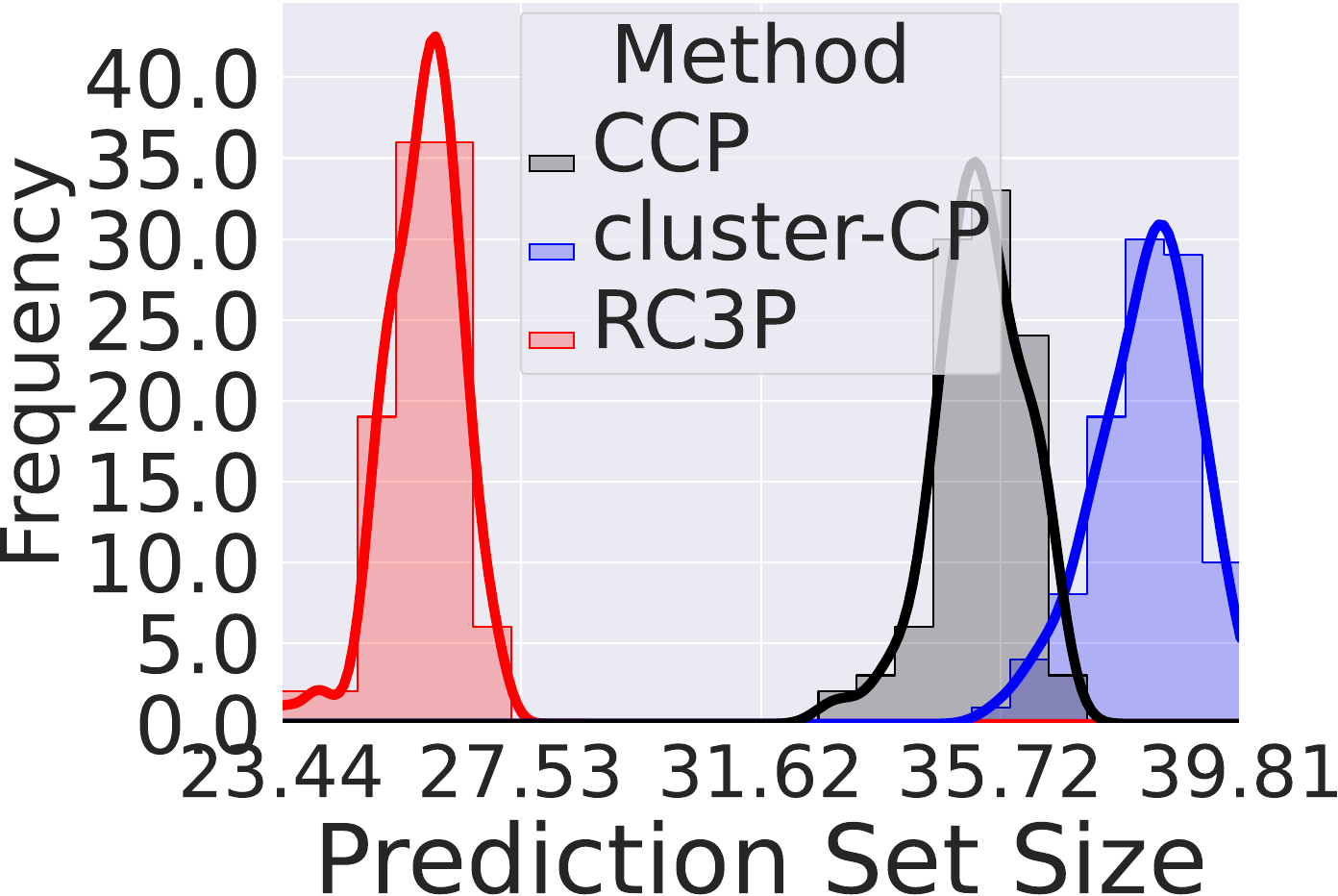}
        \end{minipage}
    }
    \caption{
    Class-conditional coverage (Top row) and prediction set size (Bottom row) achieved by \texttt{CCP}, \texttt{Cluster-CP}, and \texttt{\newCP} methods when $\alpha = 0.1$ on CIFAR-10, CIFAR-100, mini-ImageNet, and Food-101 datasets with imbalance type \POLY~for imbalance ratio $\rho=0.1$.
    We clarify that \texttt{\newCP} overlaps with \texttt{CCP} on CIFAR-10.
    It is clear that \texttt{\newCP} has more densely distributed class-conditional coverage above $0.9$ (the target $1-\alpha$ class-conditional coverage) than \texttt{CCP} and \texttt{Cluster-CP} with significantly smaller prediction sets on CIFAR-100, mini-ImageNet and Food-101.
    }
    \label{fig:overall_comparison_four_datasets_poly_0.1}
\end{figure*}

\begin{figure*}[!ht]
    \centering
    \begin{minipage}{.24\textwidth}
        \centering
        (a) CIFAR-10
    \end{minipage}%
    \begin{minipage}{.24\textwidth}
        \centering
        (b) CIFAR-100
    \end{minipage}%
    \begin{minipage}{.24\textwidth}
        \centering
        (c) mini-ImageNet
    \end{minipage}%
    \begin{minipage}{.24\textwidth}
        \centering
        (d) Food-101
    \end{minipage}

    \subfigure{
        \begin{minipage}{0.23\linewidth}
            \includegraphics[width=\linewidth]{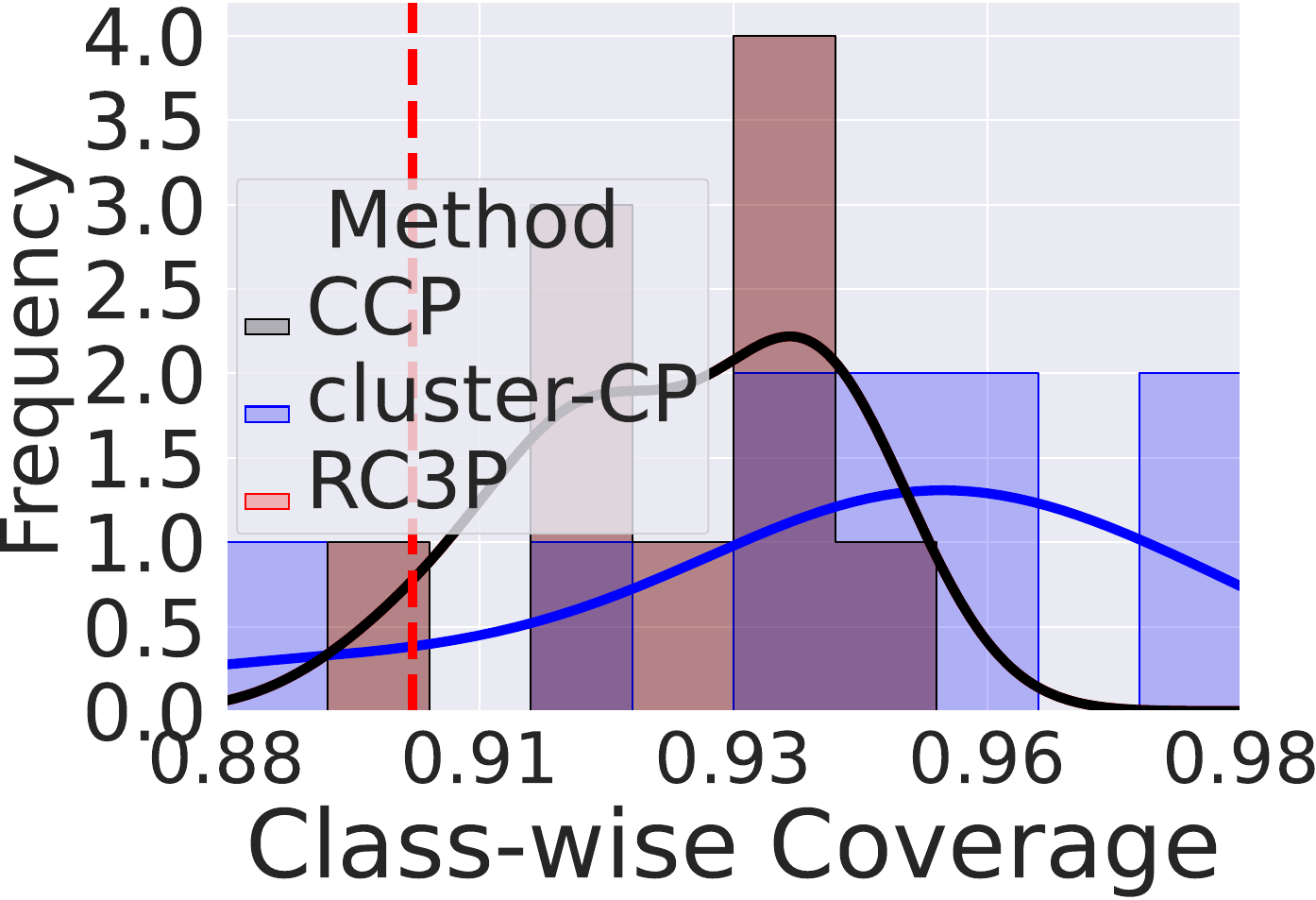}
        \end{minipage}
    }
    \subfigure{
        \begin{minipage}{0.23\linewidth}
            \includegraphics[width=\linewidth]{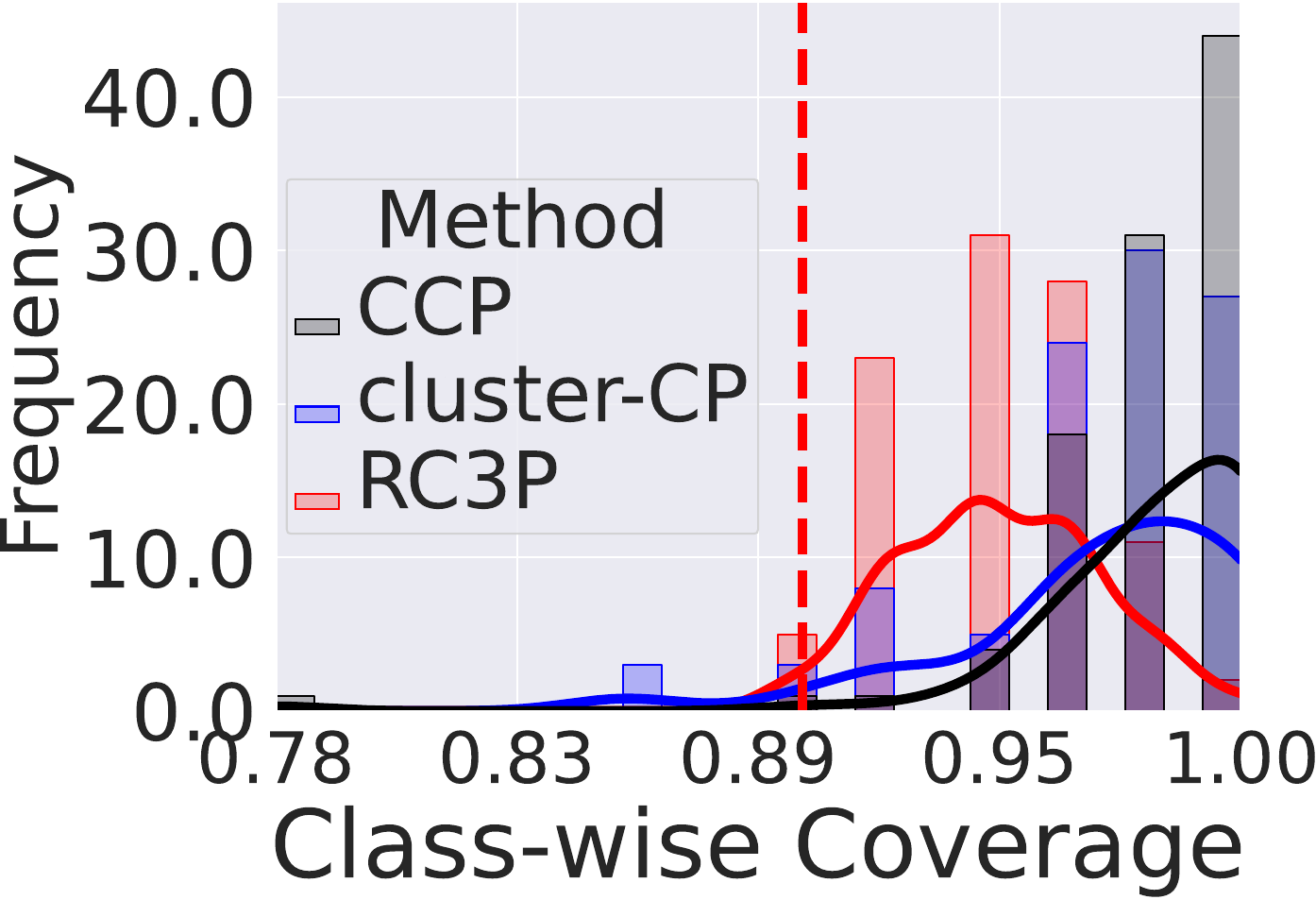}
        \end{minipage}
    }
    \subfigure{
        \begin{minipage}{0.23\linewidth}
            \includegraphics[width=\linewidth]{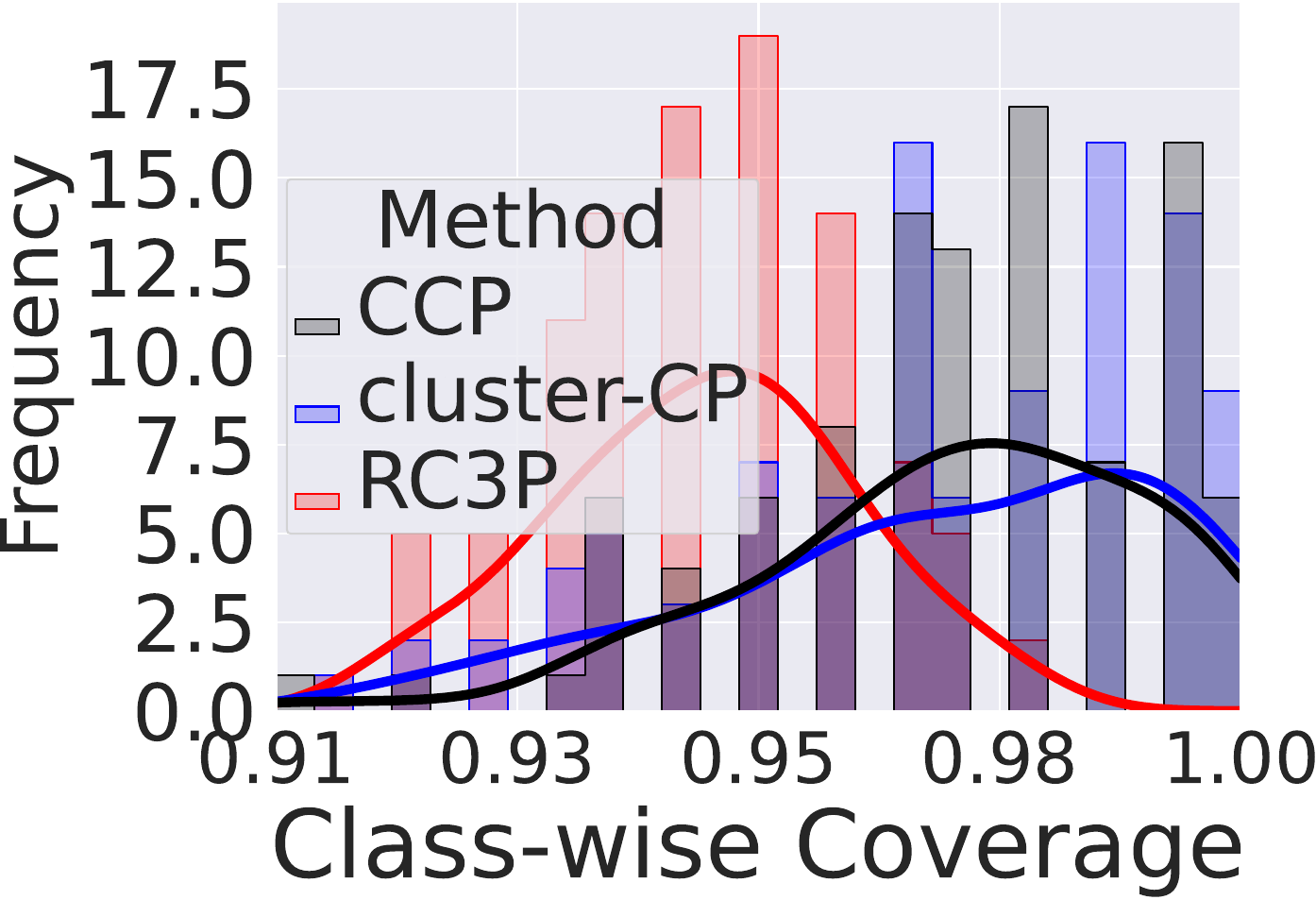}
        \end{minipage}
    }
    \subfigure{
        \begin{minipage}{0.23\linewidth}
            \includegraphics[width=\linewidth]{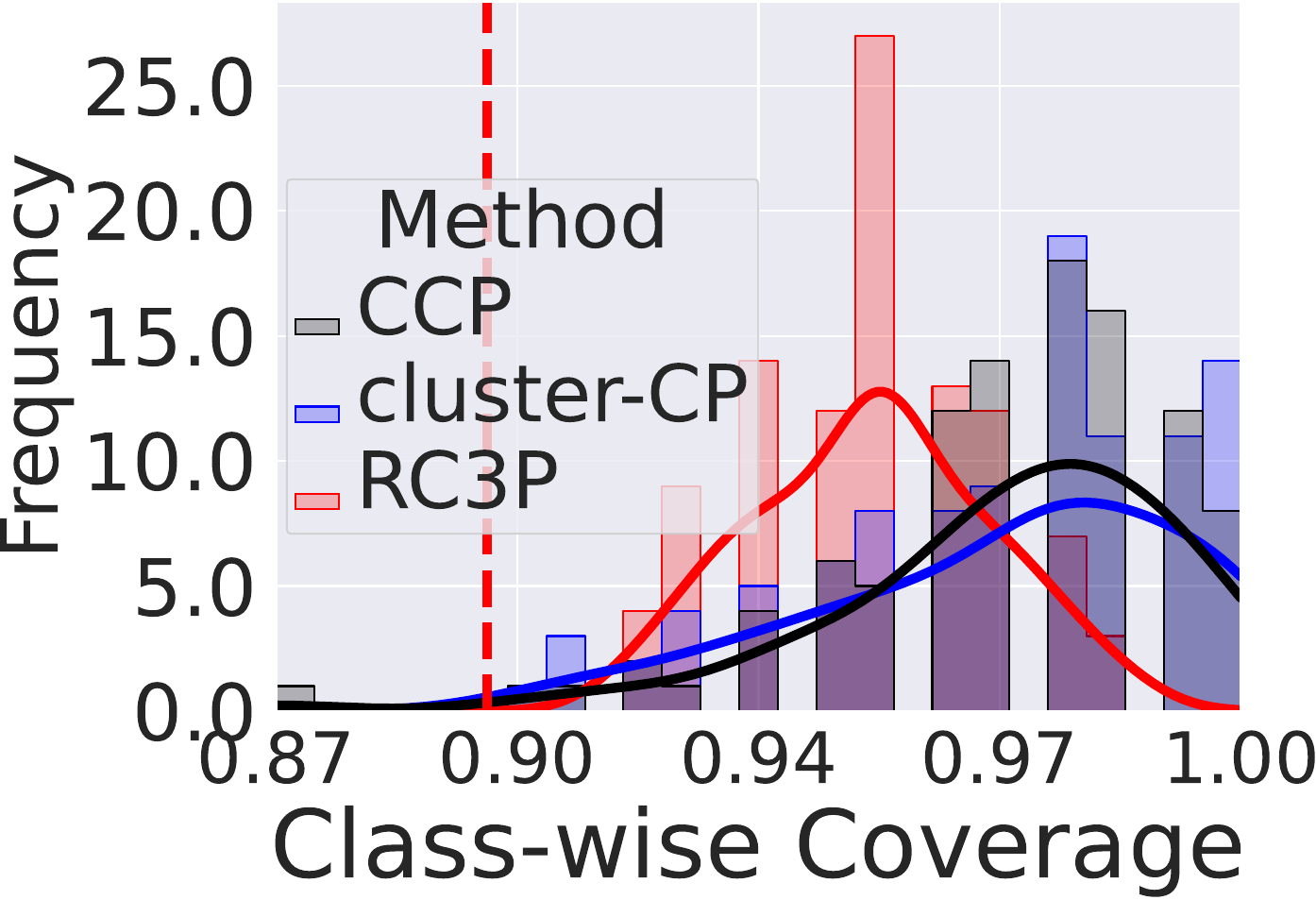}
        \end{minipage}
    }
    \subfigure{
        \begin{minipage}{0.23\linewidth}
            \includegraphics[width=\linewidth]{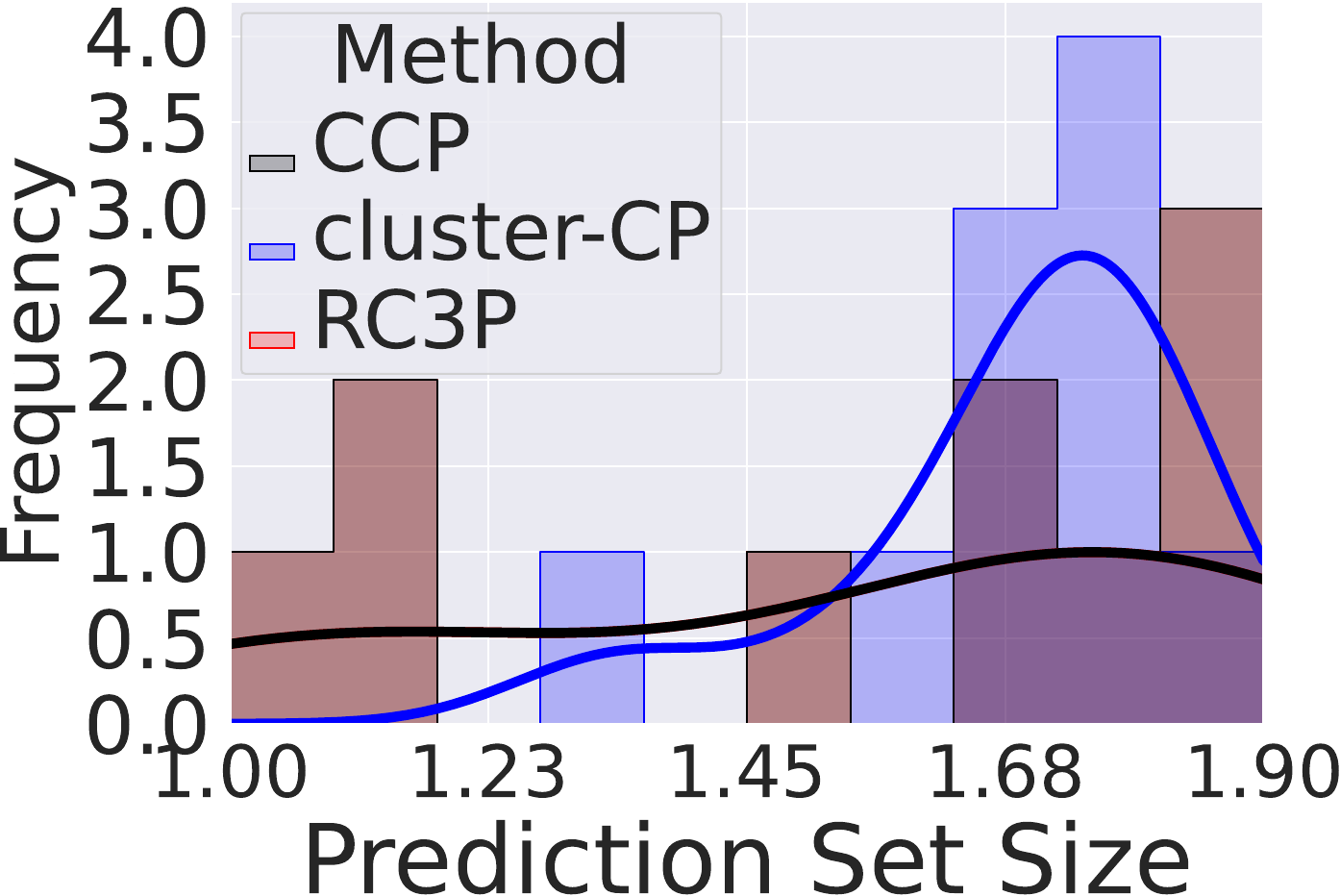}
        \end{minipage}
    }
    \subfigure{
        \begin{minipage}{0.23\linewidth}
            \includegraphics[width=\linewidth]{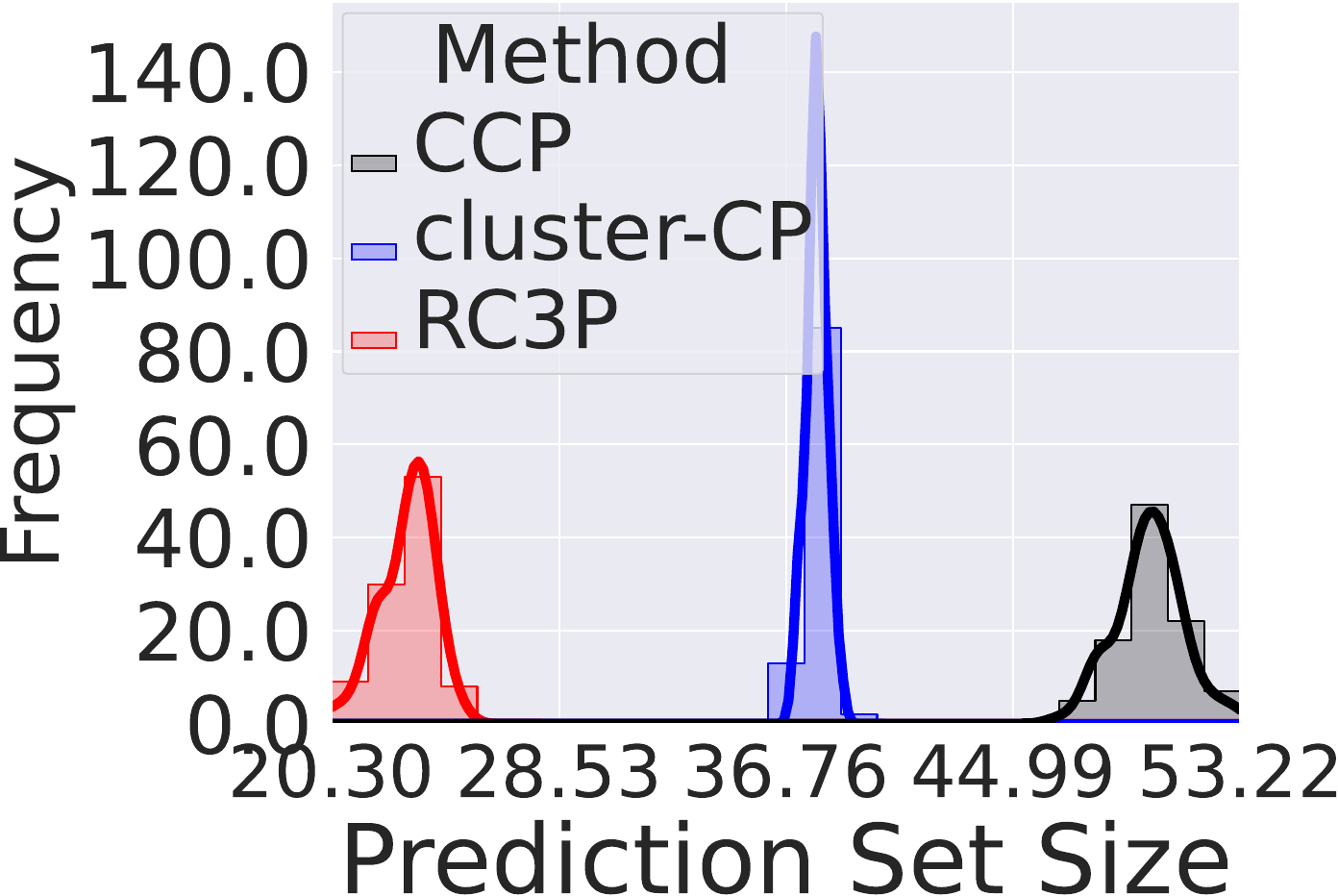}
        \end{minipage}
    }
    \subfigure{
        \begin{minipage}{0.23\linewidth}
            \includegraphics[width=\linewidth]{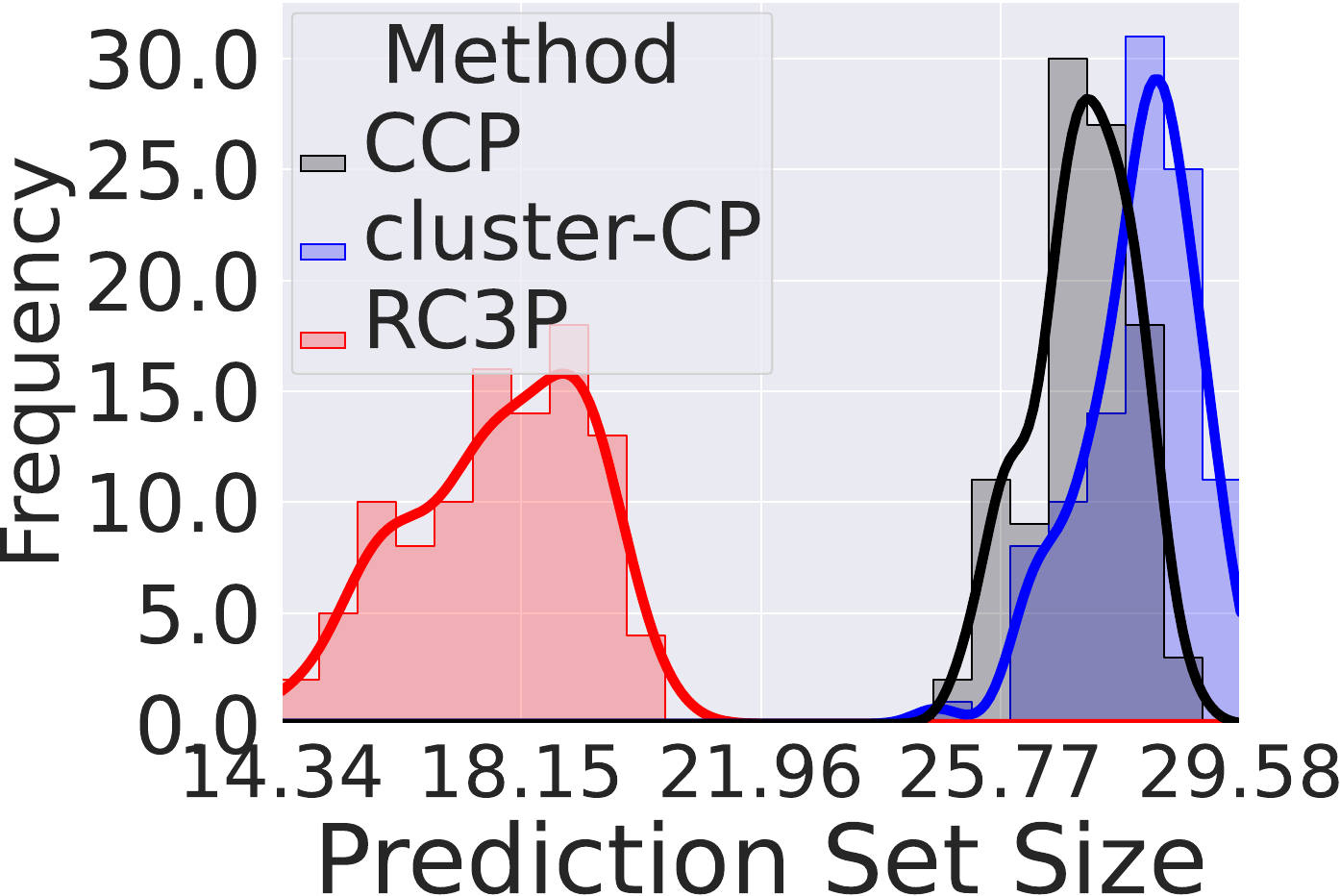}
        \end{minipage}
    }
    \subfigure{
        \begin{minipage}{0.23\linewidth}
            \includegraphics[width=\linewidth]{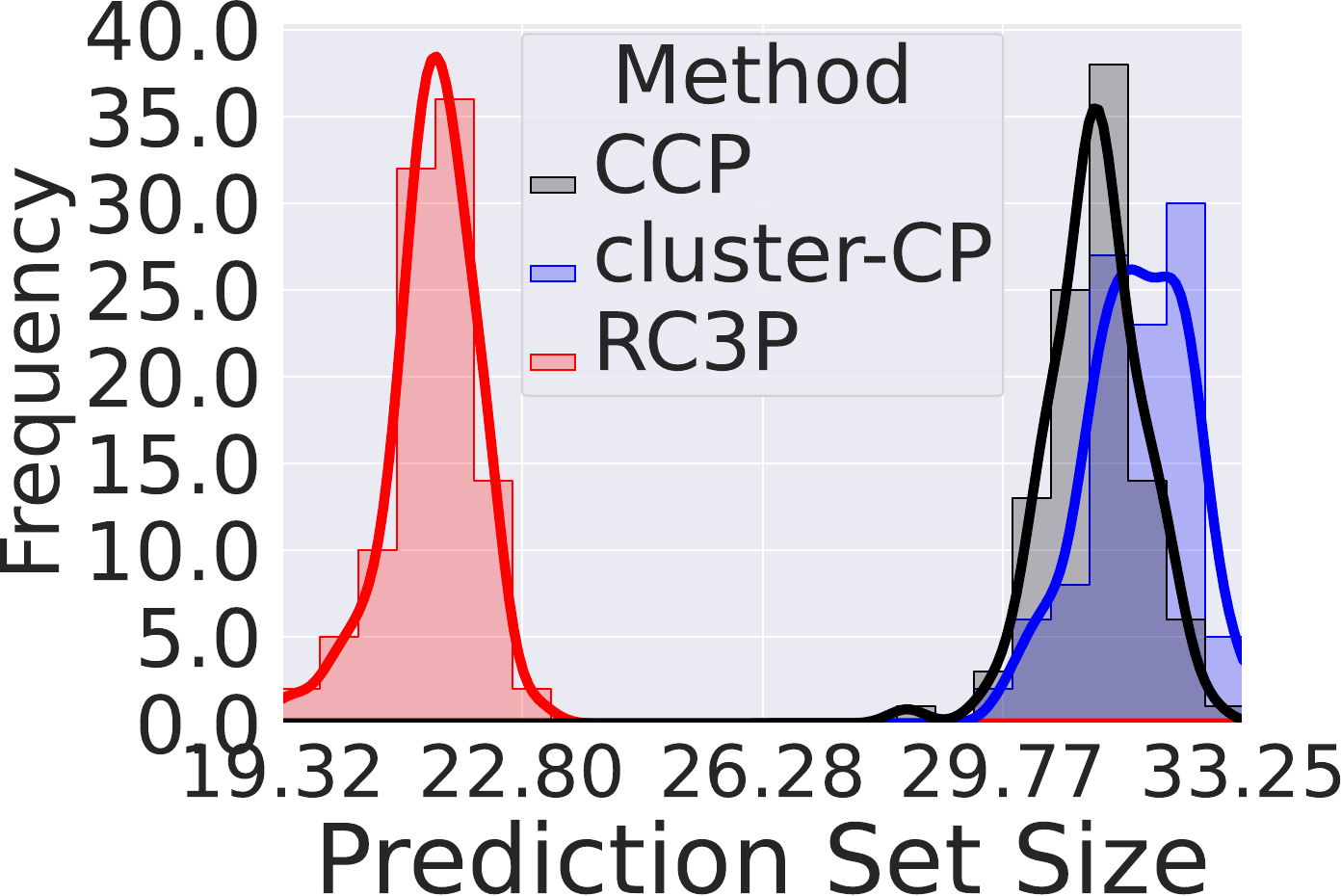}
        \end{minipage}
    }
    \caption{
    Class-conditional coverage (Top row) and prediction set size (Bottom row) achieved by \texttt{CCP}, \texttt{Cluster-CP}, and \texttt{\newCP} methods when $\alpha = 0.1$ on CIFAR-10, CIFAR-100, mini-ImageNet, and Food-101 datasets with imbalance type \POLY~for imbalance ratio $\rho=0.5$.
    We clarify that \texttt{\newCP} overlaps with \texttt{CCP} on CIFAR-10.
    It is clear that \texttt{\newCP} has more densely distributed class-conditional coverage above $0.9$ (the target $1-\alpha$ class-conditional coverage) than \texttt{CCP} and \texttt{Cluster-CP} with significantly smaller prediction sets on CIFAR-100, mini-ImageNet and Food-101.
    }
    \label{fig:overall_comparison_four_datasets_poly_0.5}
\end{figure*}

\begin{figure*}[!ht]
    \centering
    \begin{minipage}{.24\textwidth}
        \centering
        (a) CIFAR-10
    \end{minipage}%
    \begin{minipage}{.24\textwidth}
        \centering
        (b) CIFAR-100
    \end{minipage}%
    \begin{minipage}{.24\textwidth}
        \centering
        (c) mini-ImageNet
    \end{minipage}%
    \begin{minipage}{.24\textwidth}
        \centering
        (d) Food-101
    \end{minipage}

    \subfigure{
        \begin{minipage}{0.23\linewidth}
            \includegraphics[width=\linewidth]{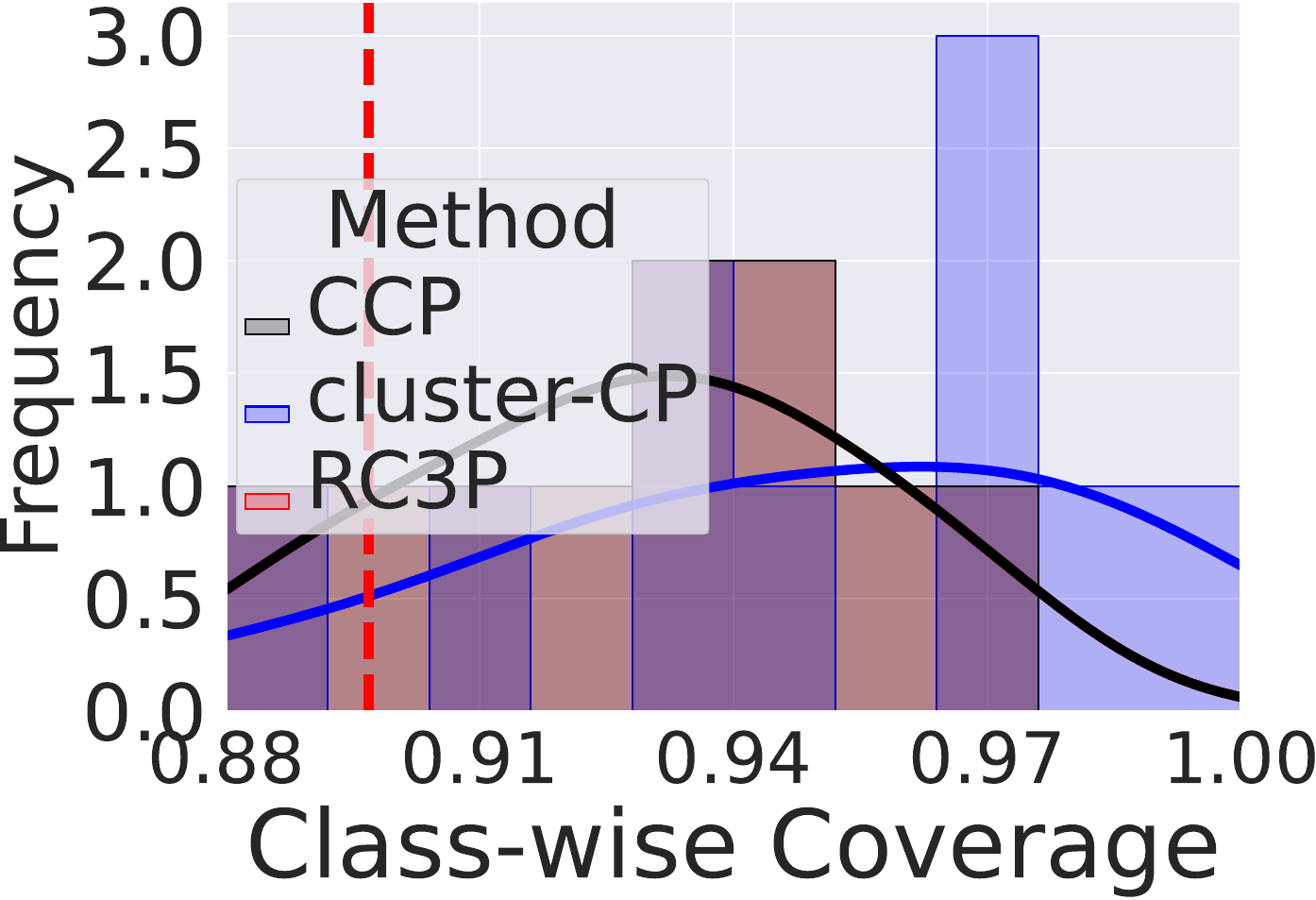}
        \end{minipage}
    }
    \subfigure{
        \begin{minipage}{0.23\linewidth}
            \includegraphics[width=\linewidth]{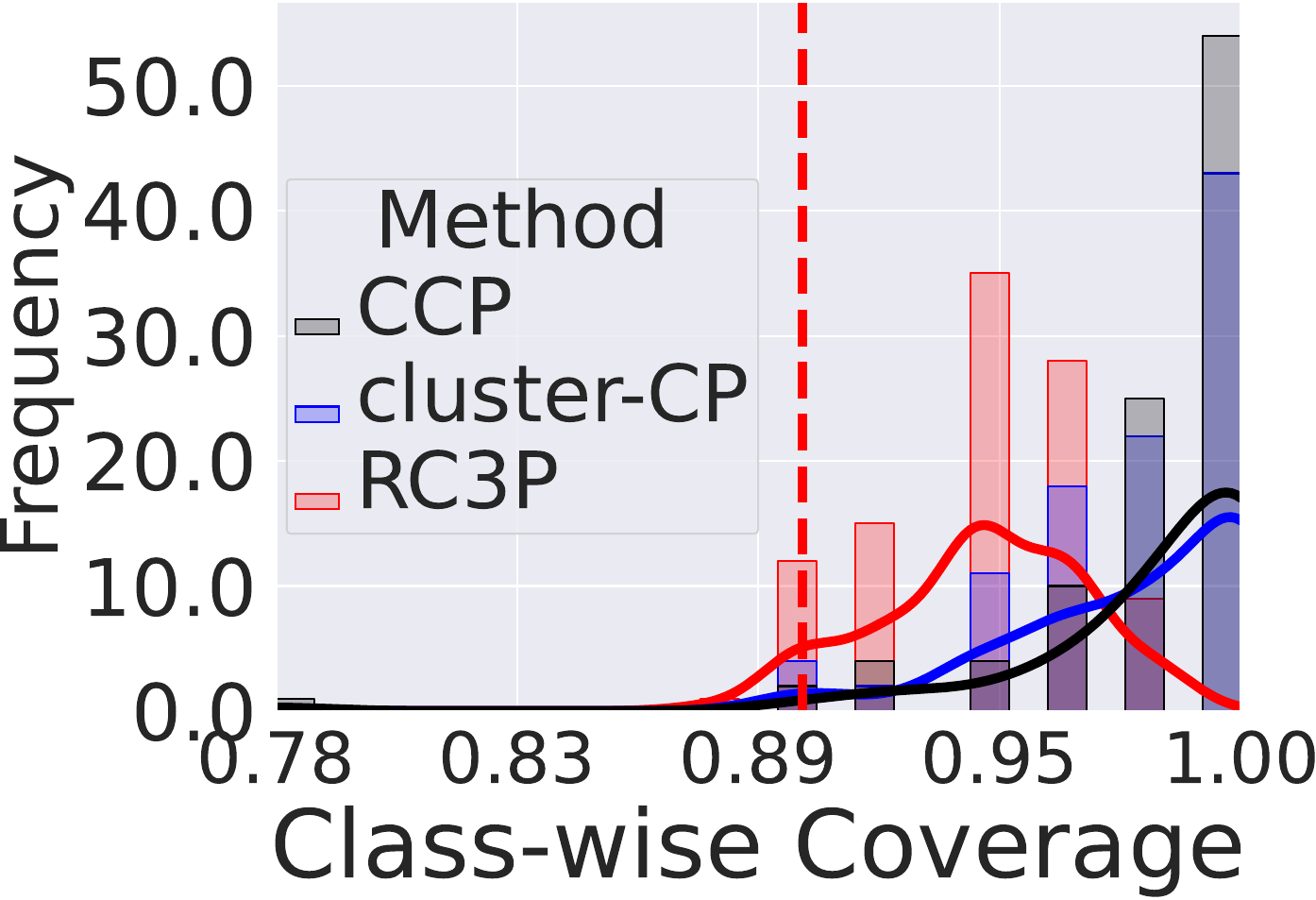}
        \end{minipage}
    }
    \subfigure{
        \begin{minipage}{0.23\linewidth}
            \includegraphics[width=\linewidth]{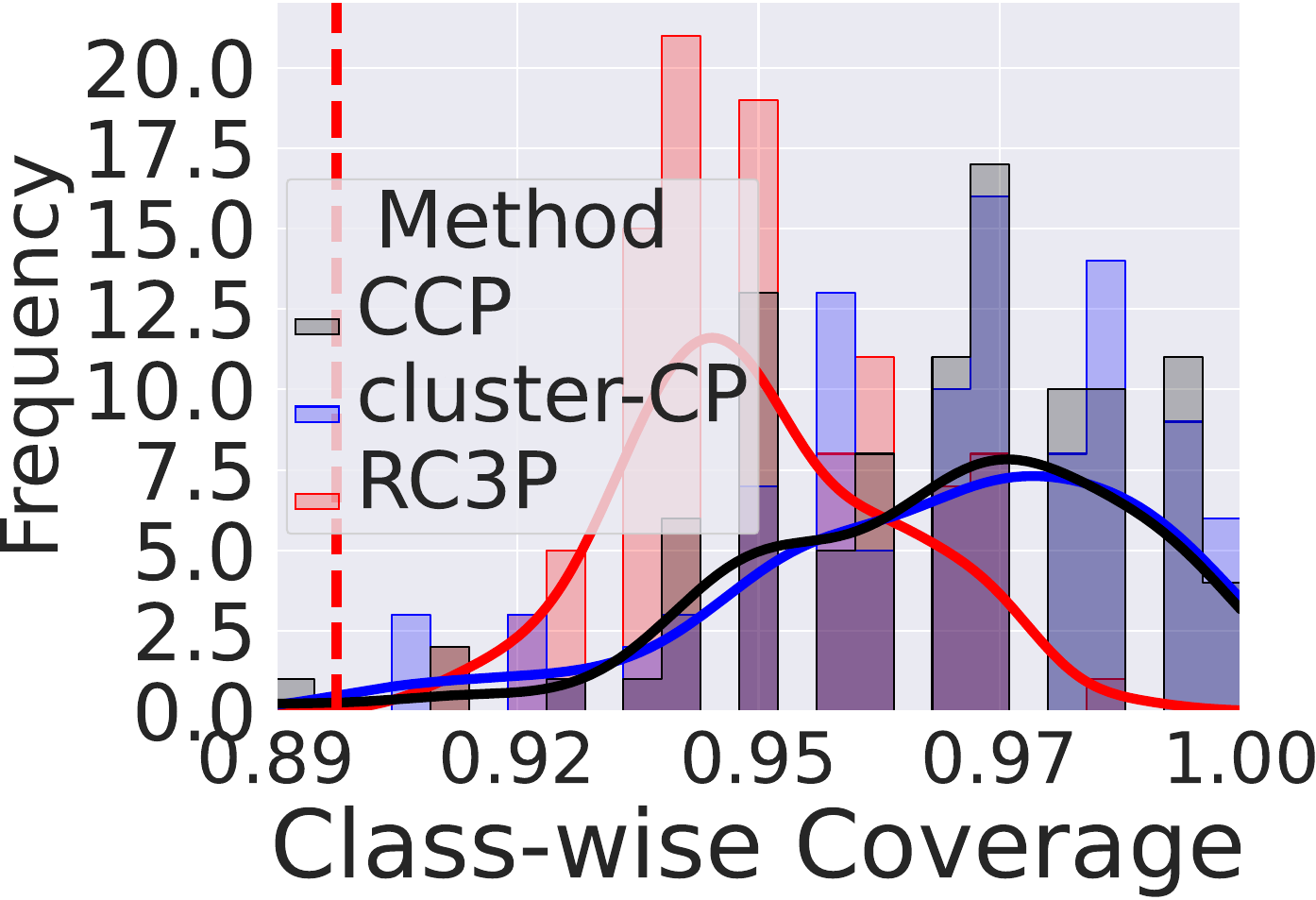}
        \end{minipage}
    }
    \subfigure{
        \begin{minipage}{0.23\linewidth}
            \includegraphics[width=\linewidth]{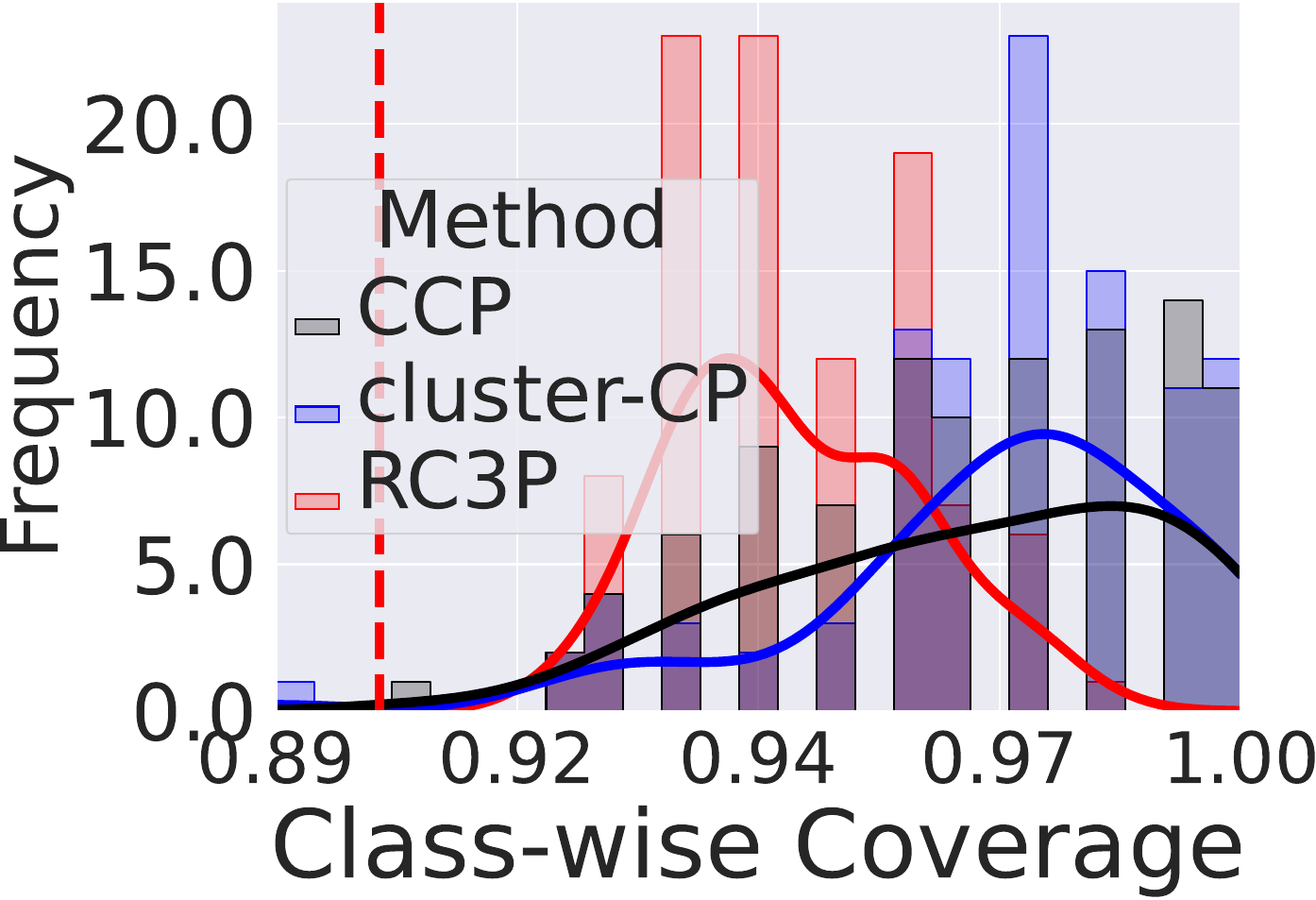}
        \end{minipage}
    }
    \subfigure{
        \begin{minipage}{0.23\linewidth}
            \includegraphics[width=\linewidth]{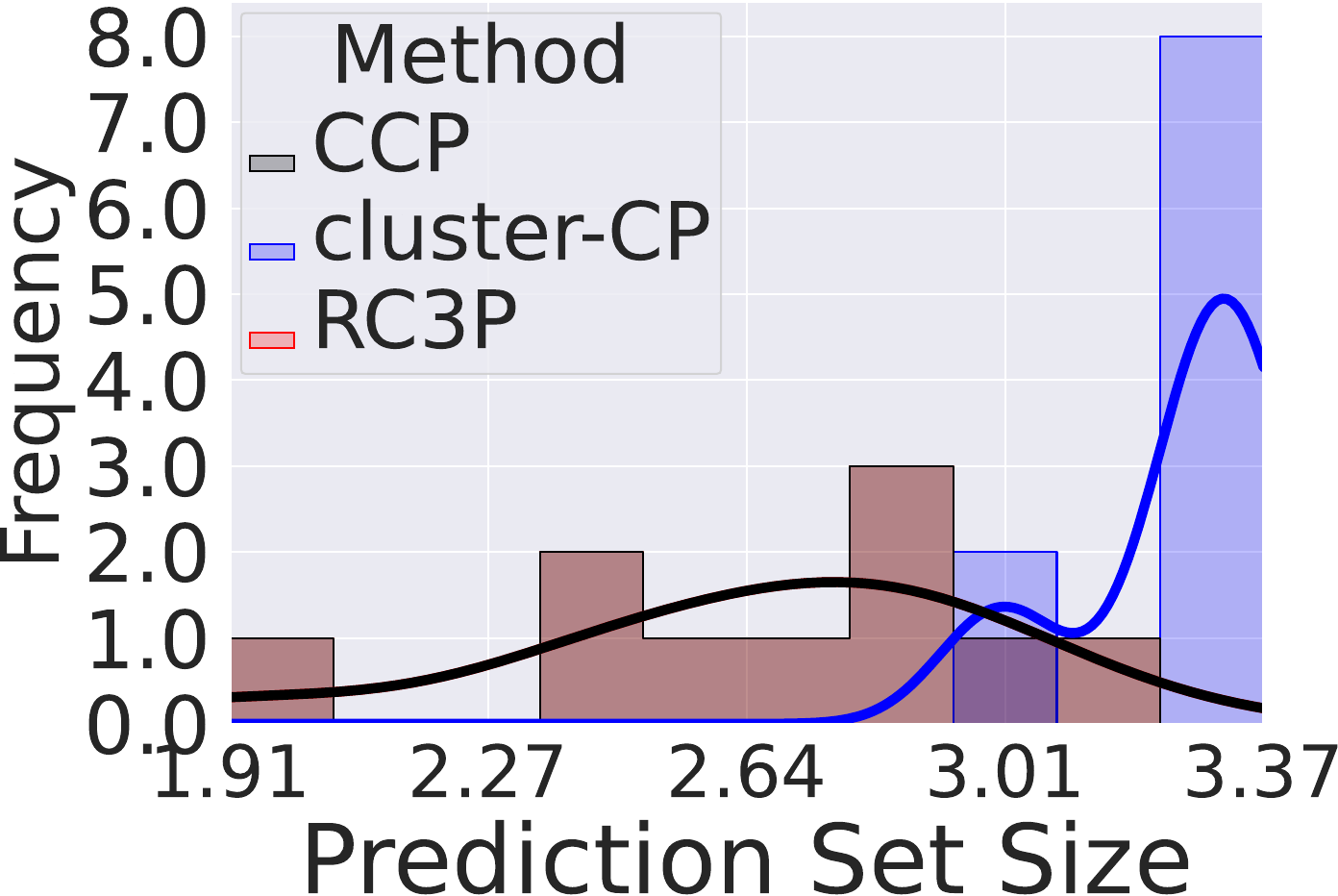}
        \end{minipage}
    }
    \subfigure{
        \begin{minipage}{0.23\linewidth}
            \includegraphics[width=\linewidth]{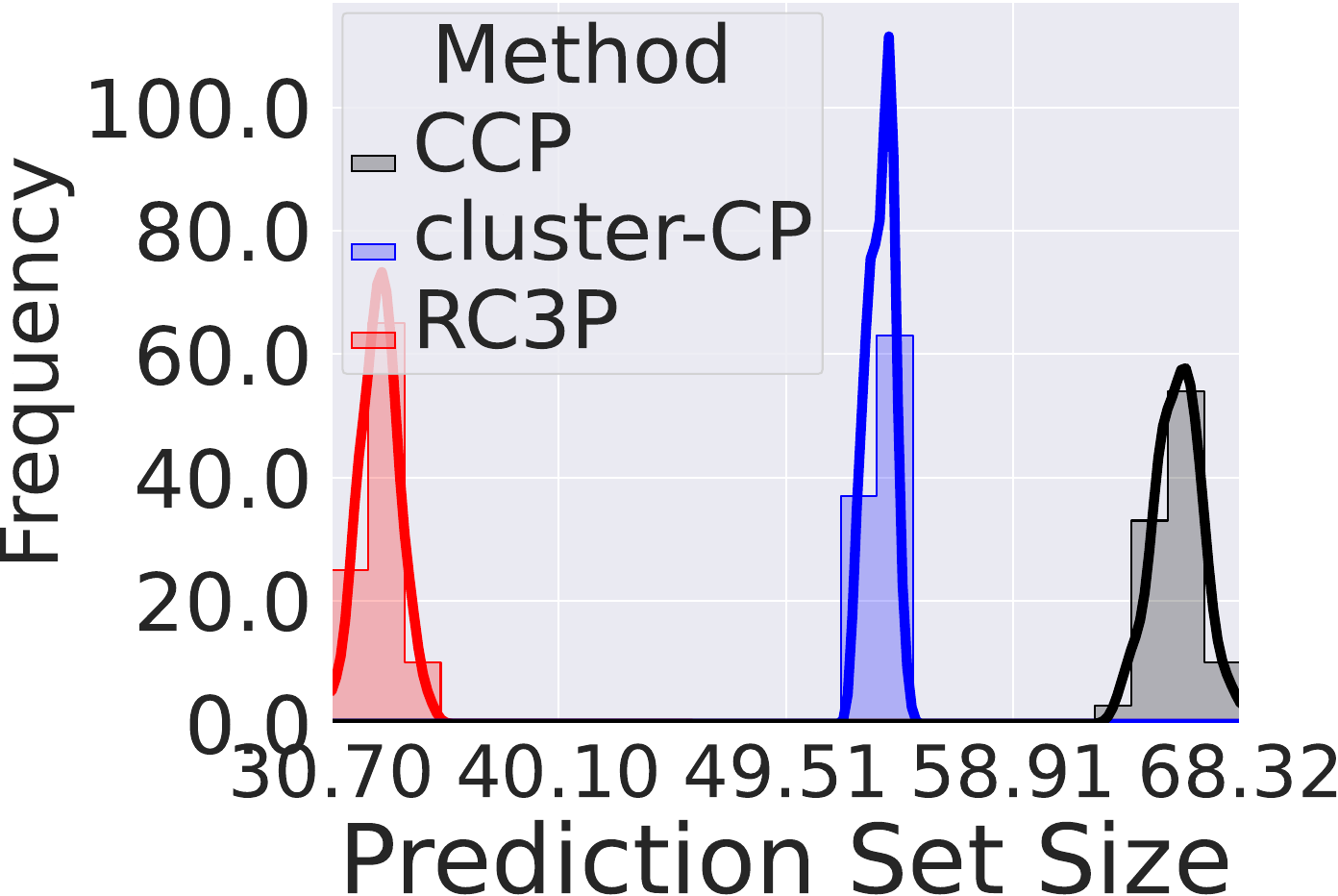}
        \end{minipage}
    }
    \subfigure{
        \begin{minipage}{0.23\linewidth}
            \includegraphics[width=\linewidth]{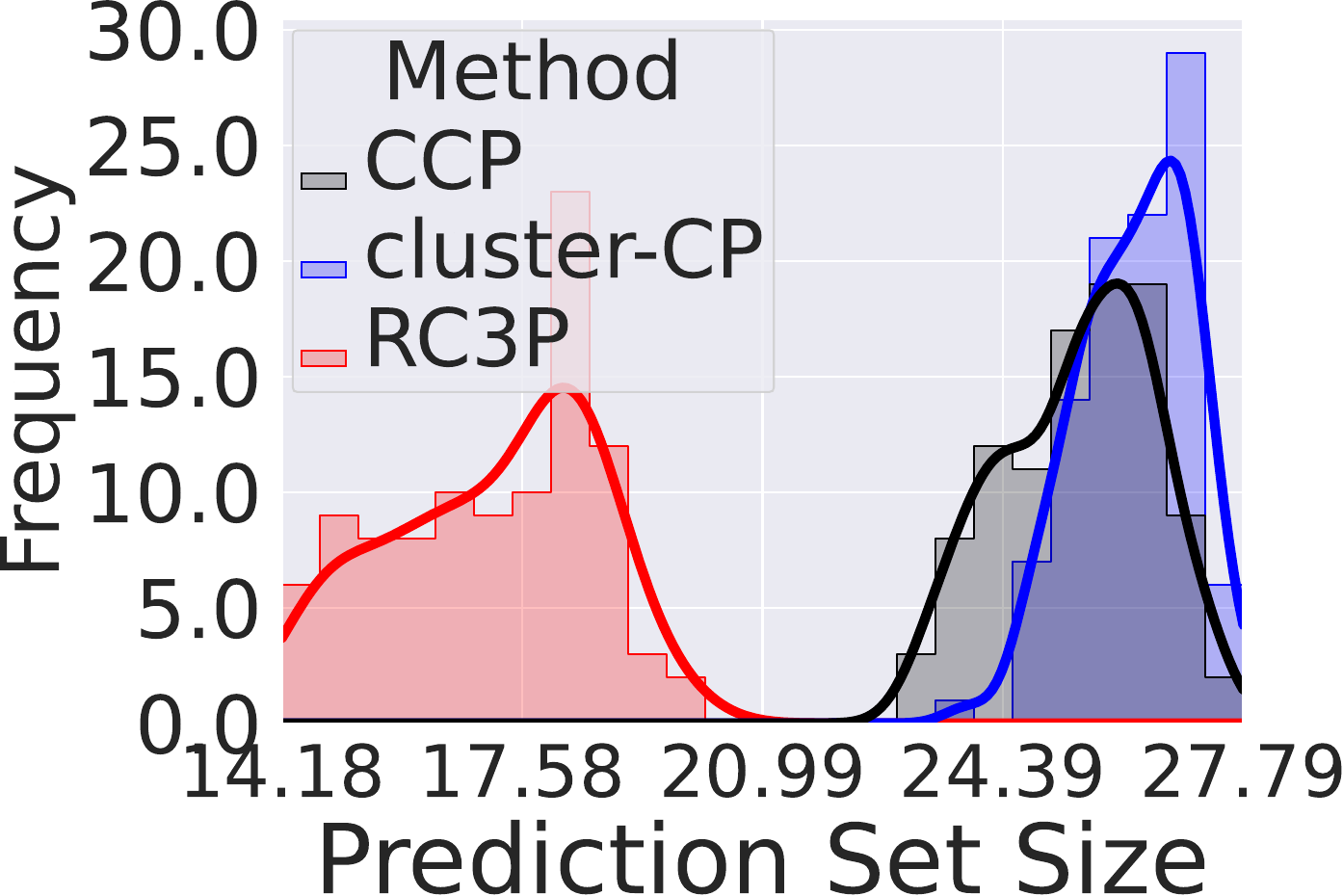}
        \end{minipage}
    }
    \subfigure{
        \begin{minipage}{0.23\linewidth}
            \includegraphics[width=\linewidth]{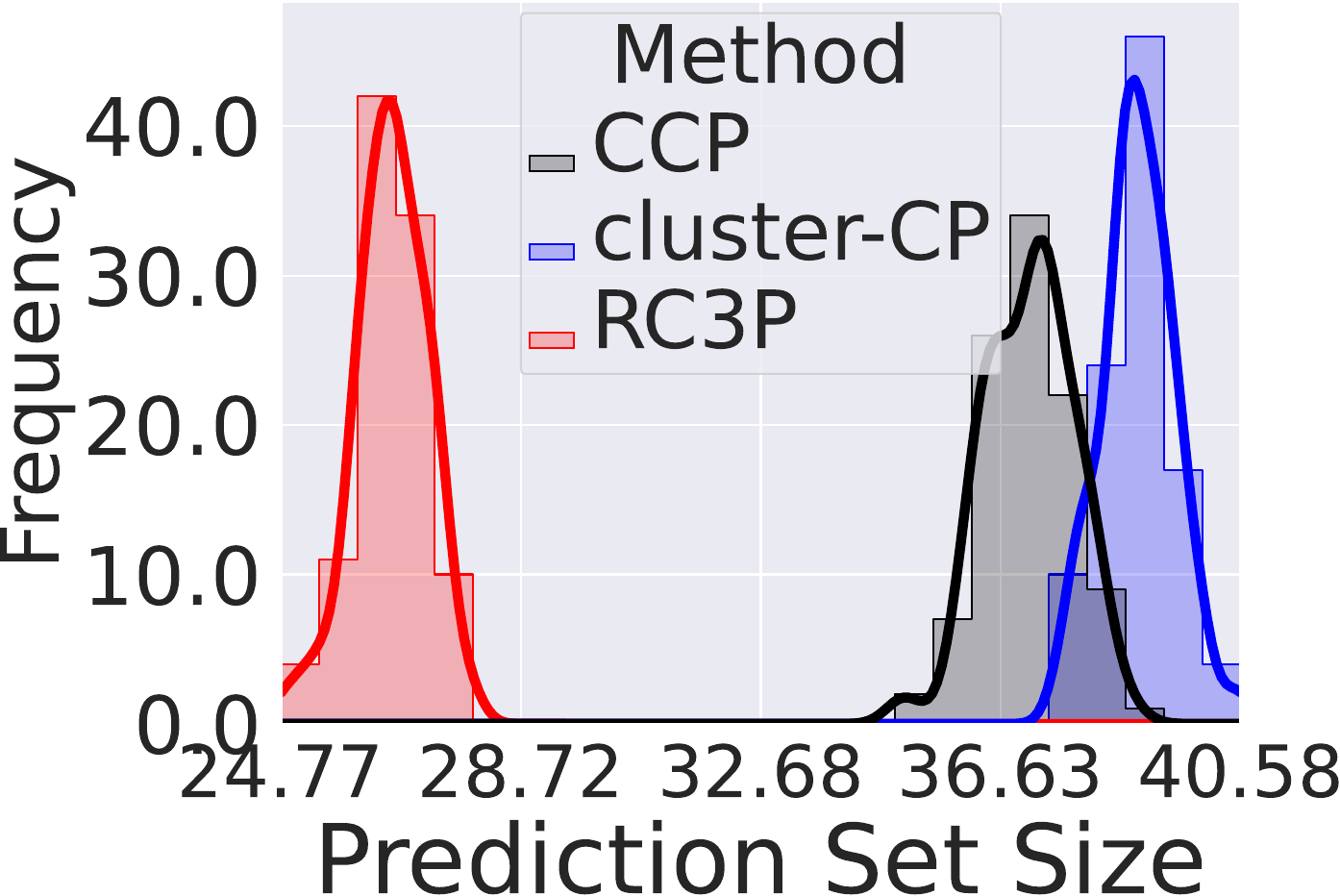}
        \end{minipage}
    }
    \caption{
    Class-conditional coverage (Top row) and prediction set size (Bottom row) achieved by \texttt{CCP}, \texttt{Cluster-CP}, and \texttt{\newCP} methods when $\alpha = 0.1$ on CIFAR-10, CIFAR-100, mini-ImageNet, and Food-101 datasets with imbalance type \MAJ~for imbalance ratio $\rho=0.1$.
    We clarify that \texttt{\newCP} overlaps with \texttt{CCP} on CIFAR-10.
    It is clear that \texttt{\newCP} has more densely distributed class-conditional coverage above $0.9$ (the target $1-\alpha$ class-conditional coverage) than \texttt{CCP} and \texttt{Cluster-CP} with significantly smaller prediction sets on CIFAR-100, mini-ImageNet and Food-101.
    }
    \label{fig:overall_comparison_four_datasets_maj_0.1}
\end{figure*}

\begin{figure*}[!ht]
    \centering
    \begin{minipage}{.24\textwidth}
        \centering
        (a) CIFAR-10
    \end{minipage}%
    \begin{minipage}{.24\textwidth}
        \centering
        (b) CIFAR-100
    \end{minipage}%
    \begin{minipage}{.24\textwidth}
        \centering
        (c) mini-ImageNet
    \end{minipage}%
    \begin{minipage}{.24\textwidth}
        \centering
        (d) Food-101
    \end{minipage}

    \subfigure{
        \begin{minipage}{0.23\linewidth}
            \includegraphics[width=\linewidth]{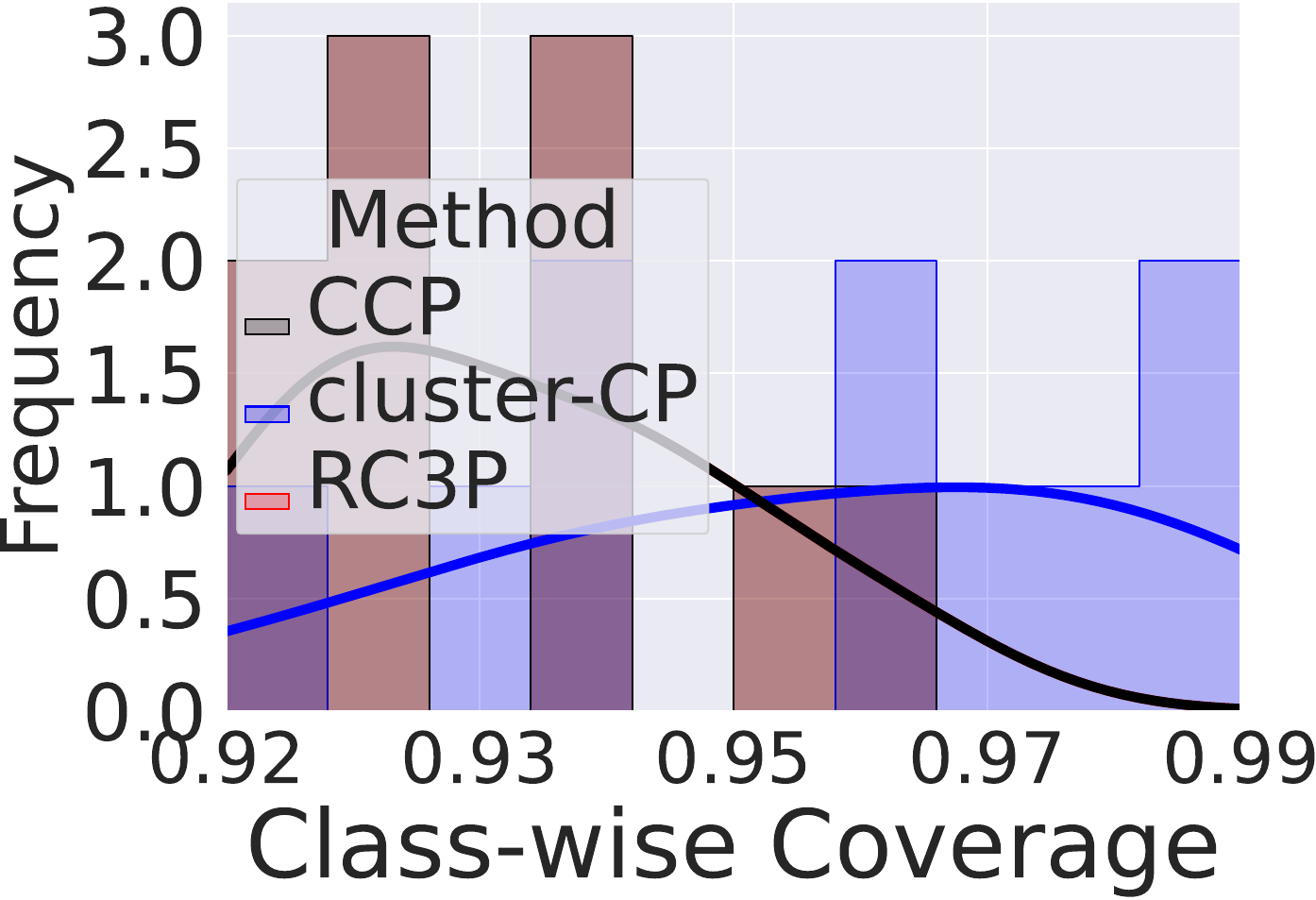}
        \end{minipage}
    }
    \subfigure{
        \begin{minipage}{0.23\linewidth}
            \includegraphics[width=\linewidth]{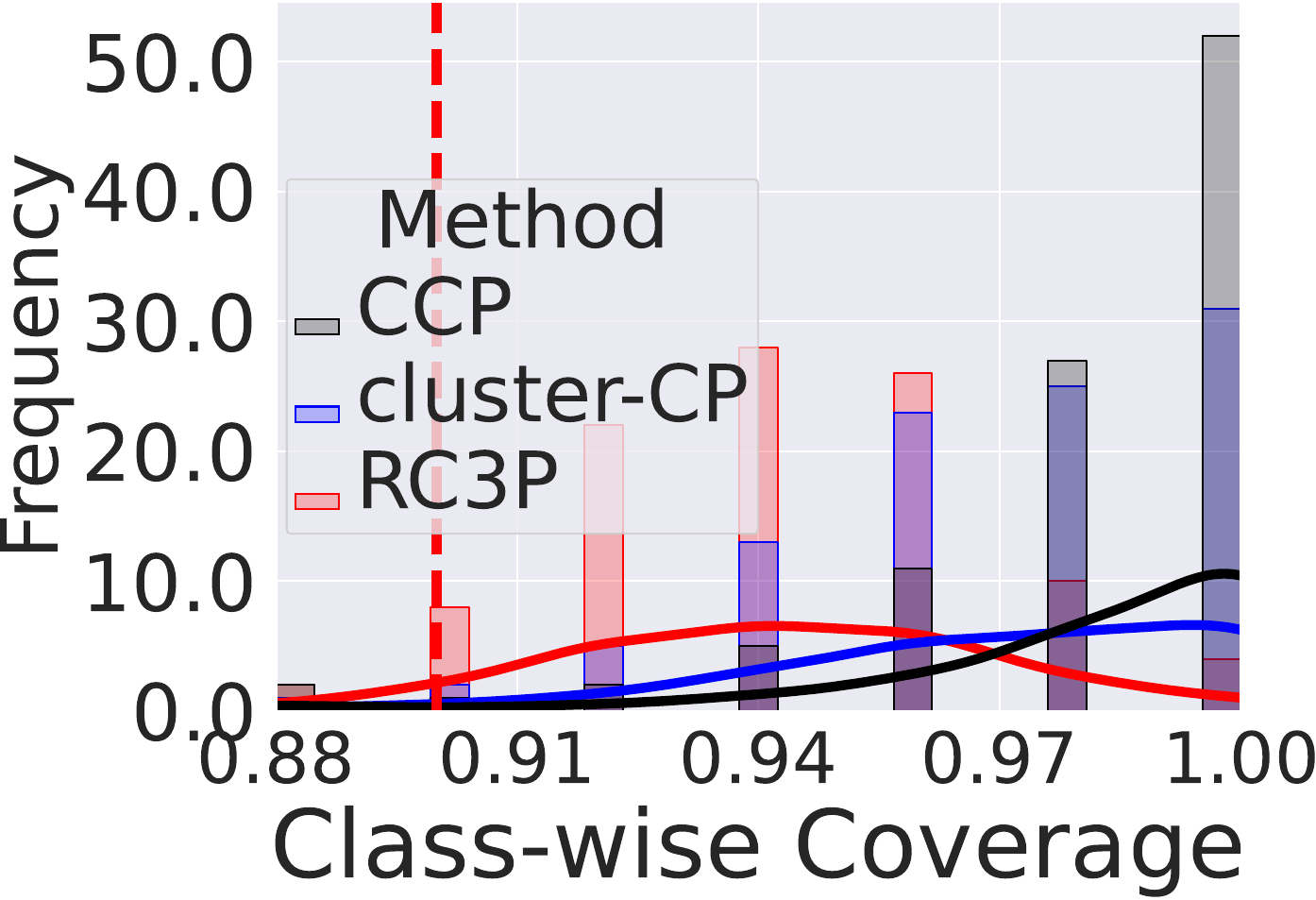}
        \end{minipage}
    }
    \subfigure{
        \begin{minipage}{0.23\linewidth}
            \includegraphics[width=\linewidth]{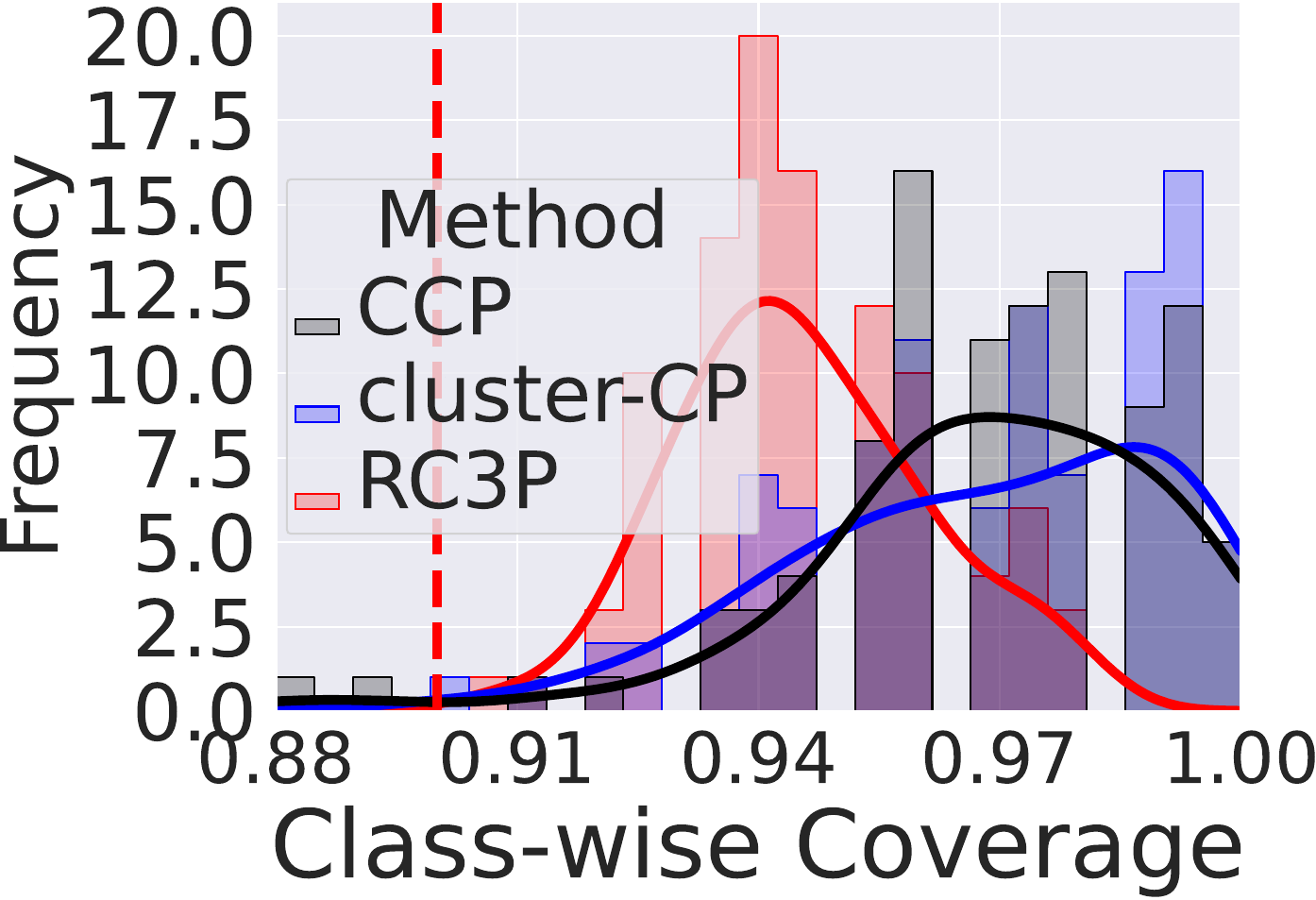}
        \end{minipage}
    }
    \subfigure{
        \begin{minipage}{0.23\linewidth}
            \includegraphics[width=\linewidth]{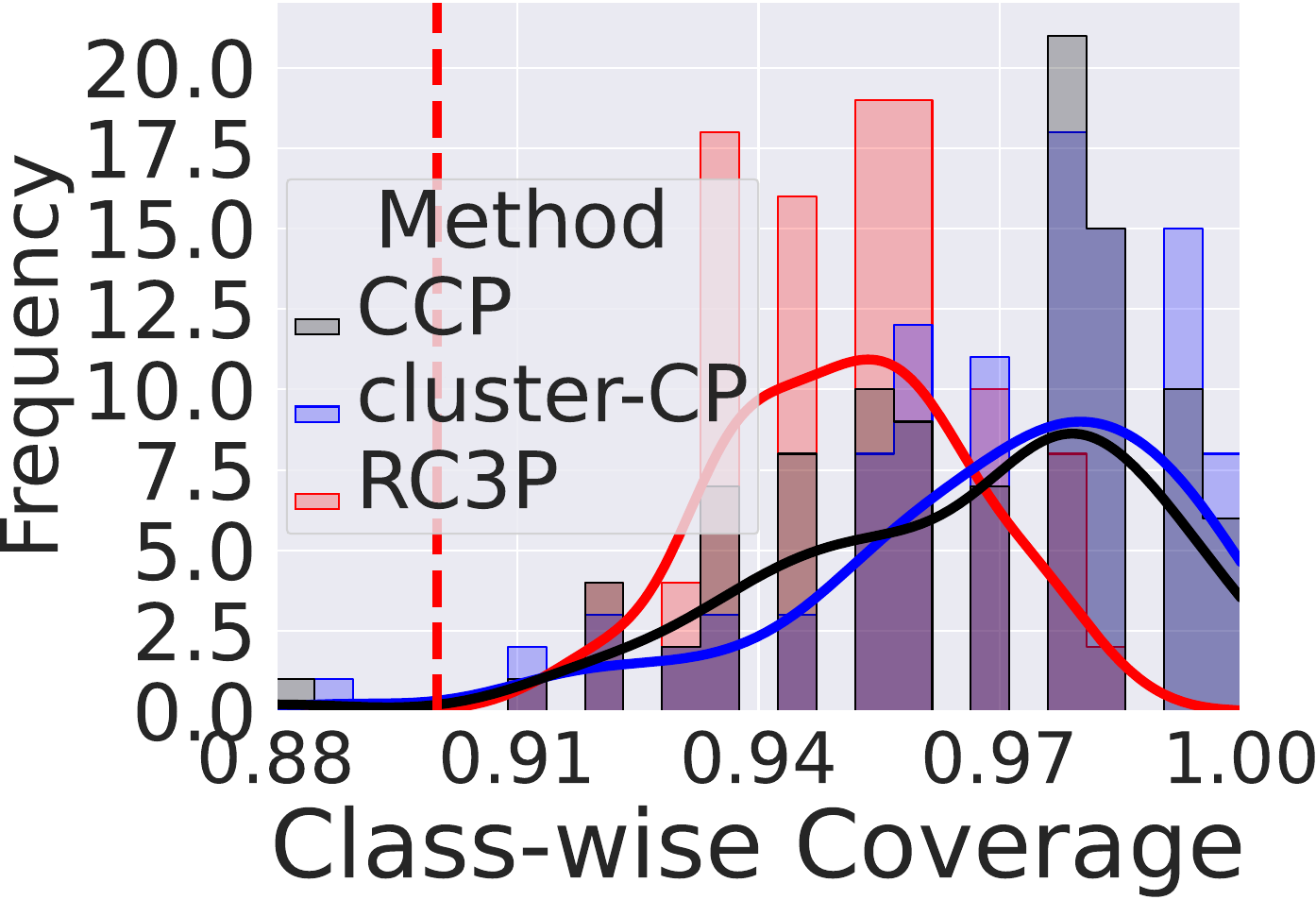}
        \end{minipage}
    }
    \subfigure{
        \begin{minipage}{0.23\linewidth}
            \includegraphics[width=\linewidth]{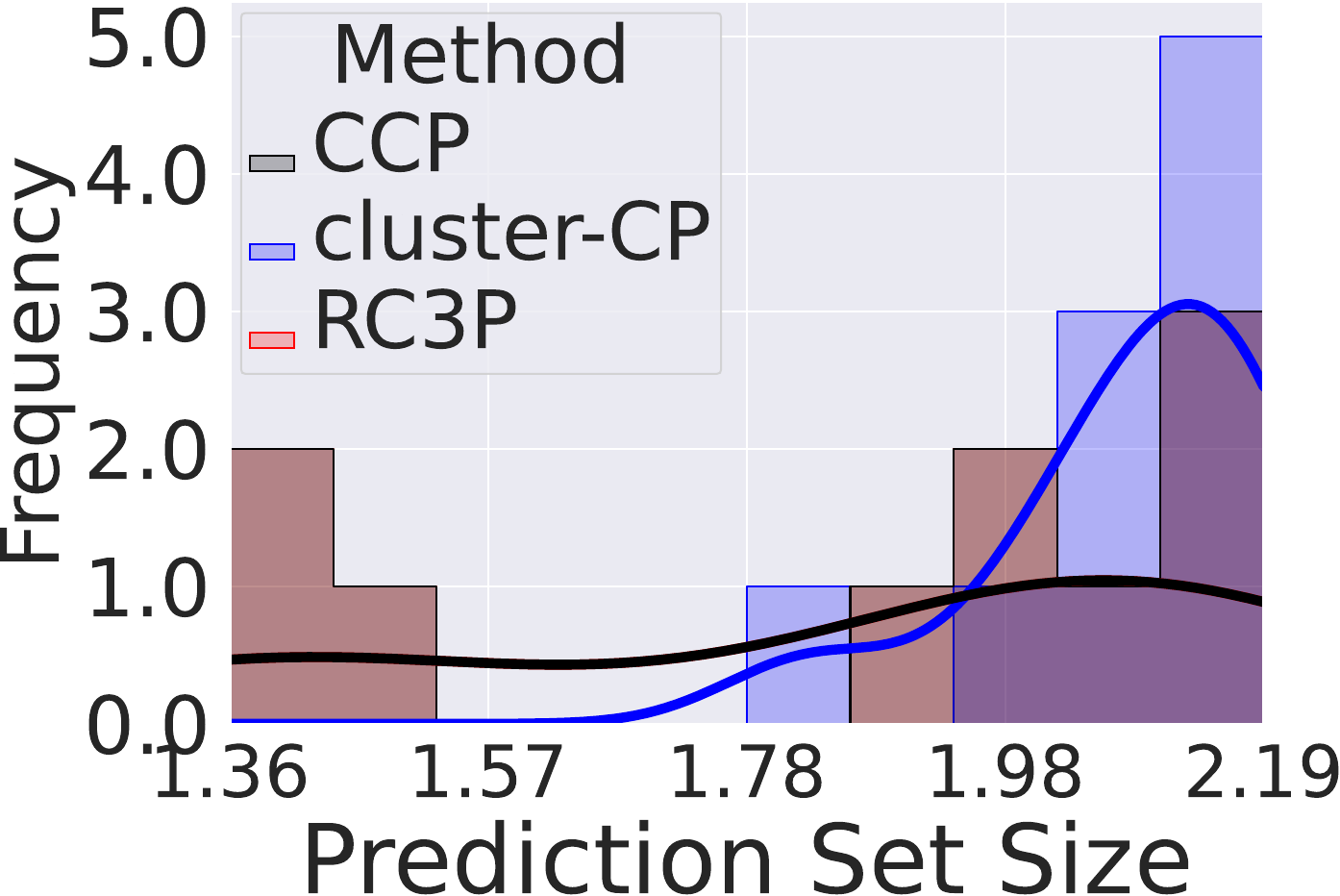}
        \end{minipage}
    }
    \subfigure{
        \begin{minipage}{0.23\linewidth}
            \includegraphics[width=\linewidth]{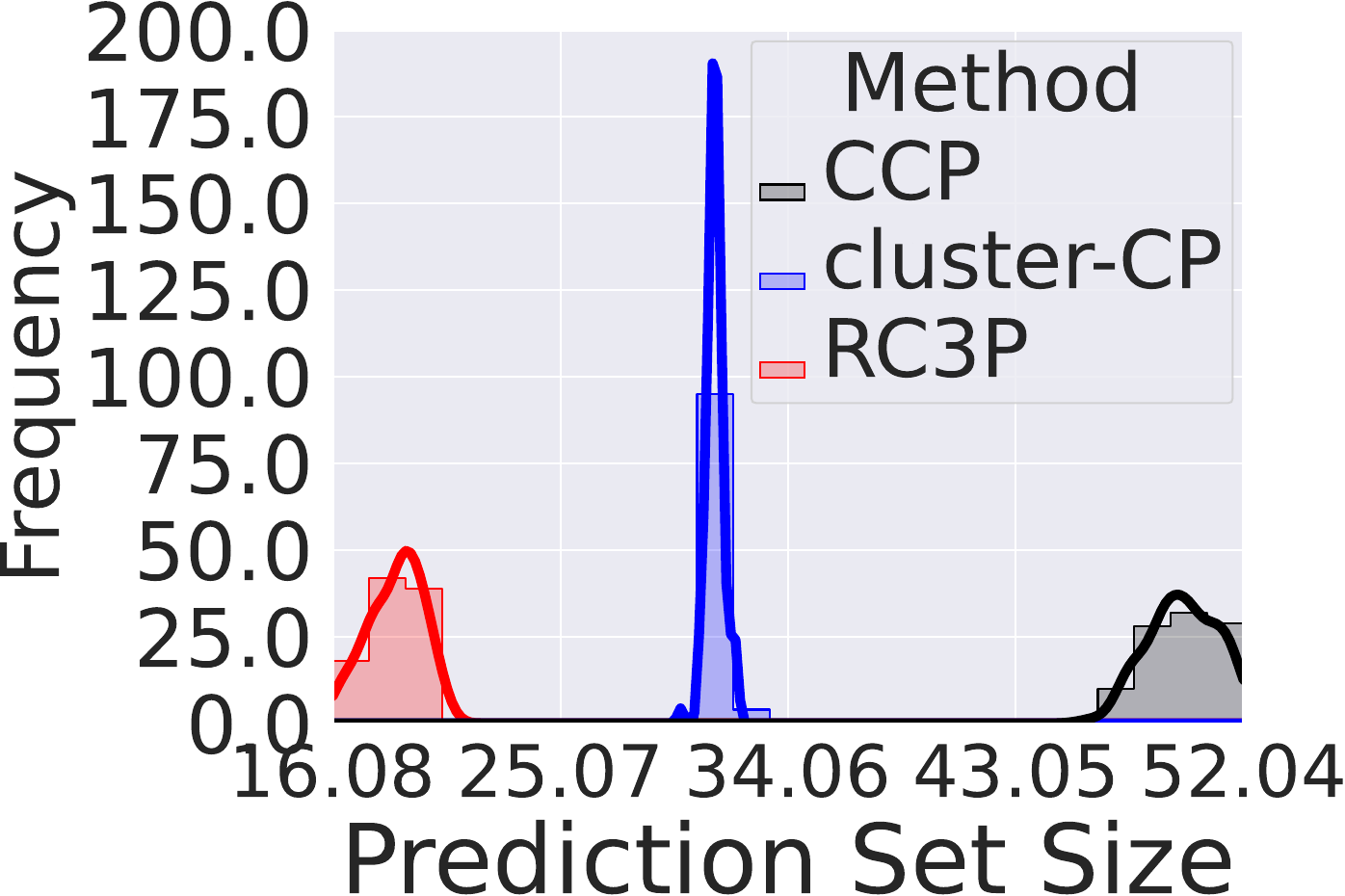}
        \end{minipage}
    }
    \subfigure{
        \begin{minipage}{0.23\linewidth}
            \includegraphics[width=\linewidth]{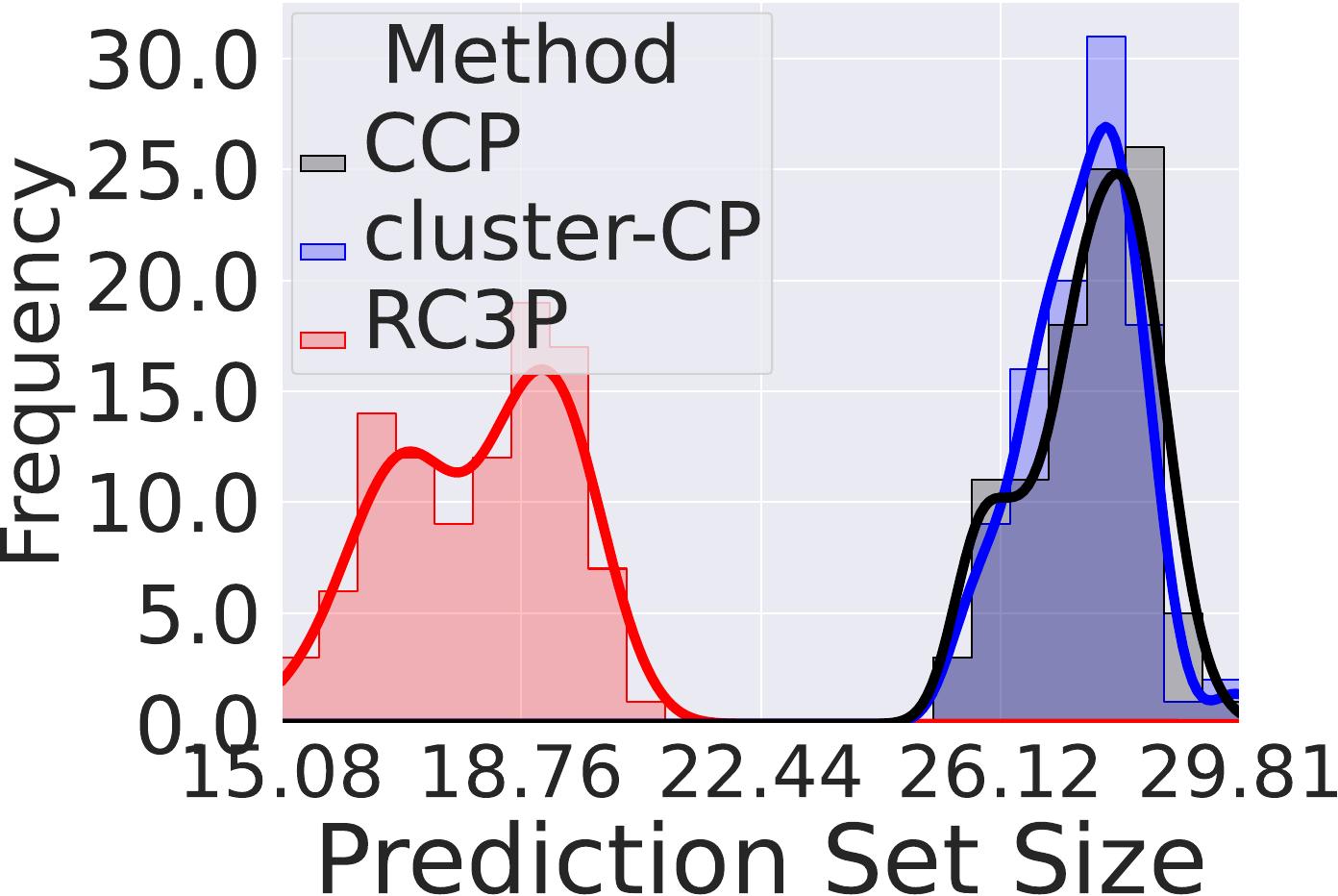}
        \end{minipage}
    }
    \subfigure{
        \begin{minipage}{0.23\linewidth}
            \includegraphics[width=\linewidth]{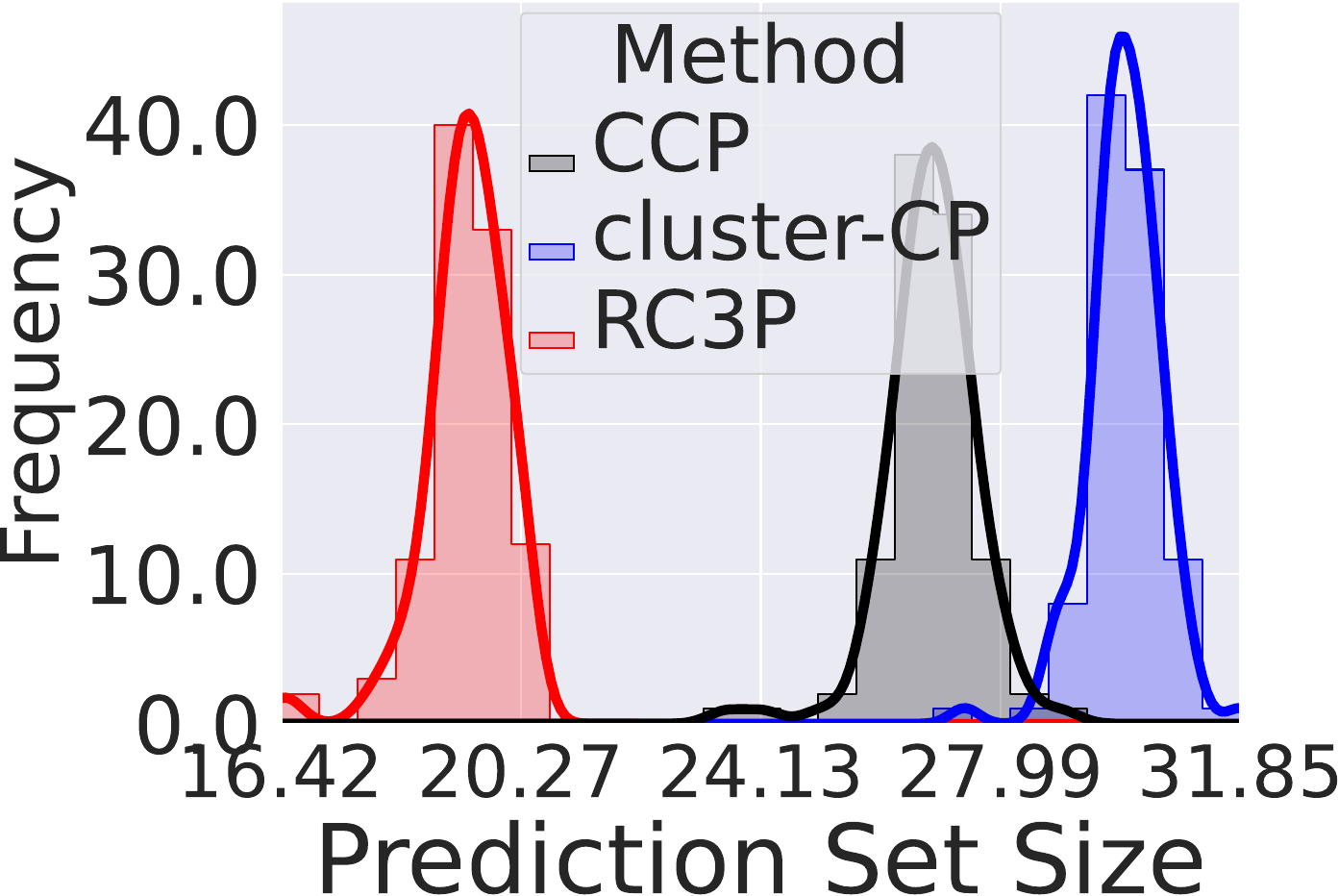}
        \end{minipage}
    }
    \caption{
    Class-conditional coverage (Top row) and prediction set size (Bottom row) achieved by \texttt{CCP}, \texttt{Cluster-CP}, and \texttt{\newCP} methods when $\alpha = 0.1$ on CIFAR-10, CIFAR-100, mini-ImageNet, and Food-101 datasets with imbalance type \MAJ~for imbalance ratio $\rho=0.5$.
    We clarify that \texttt{\newCP} overlaps with \texttt{CCP} on CIFAR-10.
    It is clear that \texttt{\newCP} has more densely distributed class-conditional coverage above $0.9$ (the target $1-\alpha$ class-conditional coverage) than \texttt{CCP} and \texttt{Cluster-CP} with significantly smaller prediction sets on CIFAR-100, mini-ImageNet and Food-101.
    }
    \label{fig:overall_comparison_four_datasets_maj_0.5}
\end{figure*}

\begin{figure}[!ht]
    \centering
    \begin{minipage}{.24\textwidth}
        \centering
        (a) CIFAR-10
    \end{minipage}%
    \begin{minipage}{.24\textwidth}
        \centering
        (b) CIFAR-100
    \end{minipage}%
    \begin{minipage}{.24\textwidth}
        \centering
        (c) mini-ImageNet
    \end{minipage}%
    \begin{minipage}{.24\textwidth}
        \centering
        (d) Food-101
    \end{minipage}
    \subfigure{
        \begin{minipage}{0.23\linewidth}
            \includegraphics[width=\linewidth]{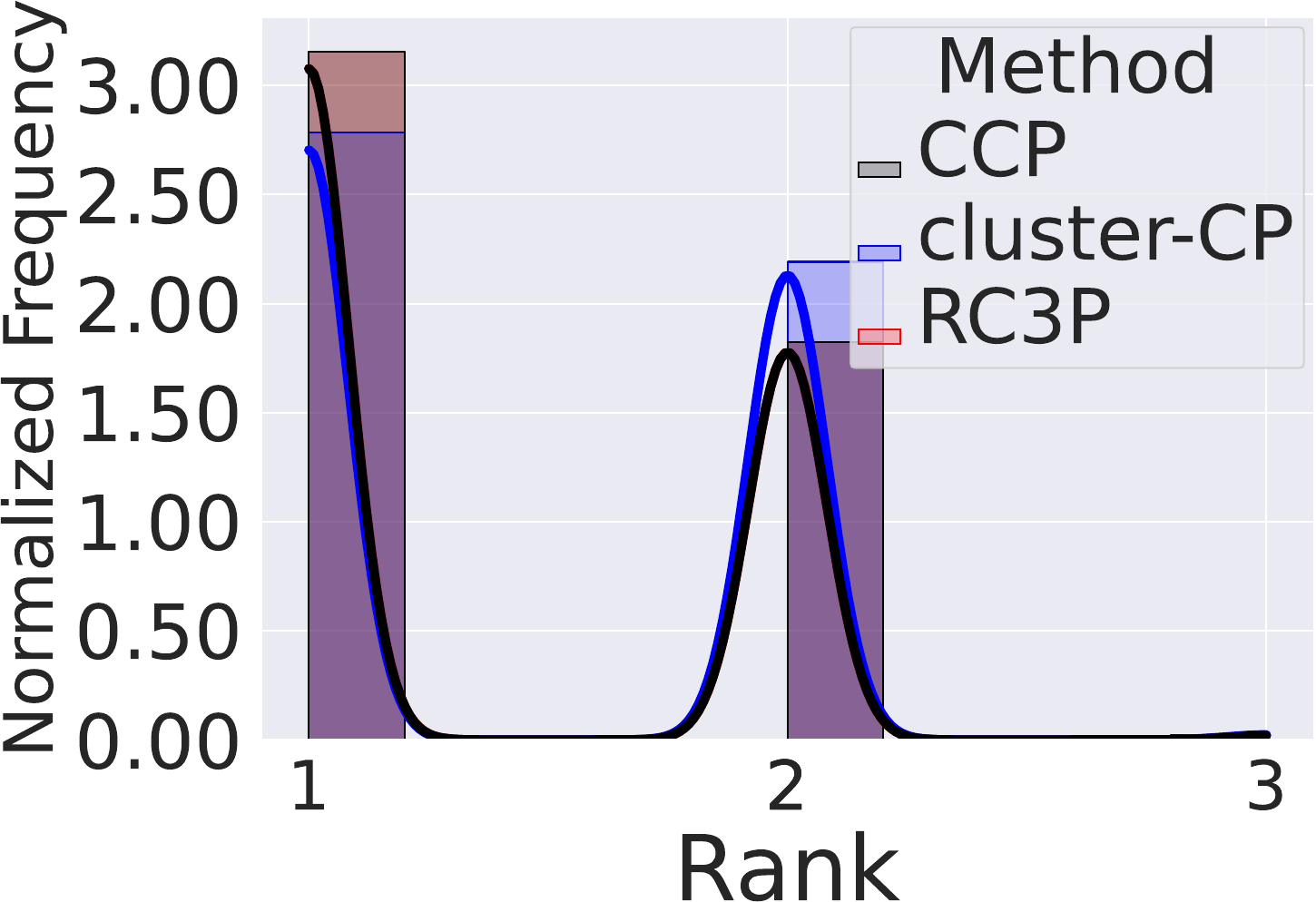}
        \end{minipage}
    }
    \subfigure{
        \begin{minipage}{0.23\linewidth}
            \includegraphics[width=\linewidth]{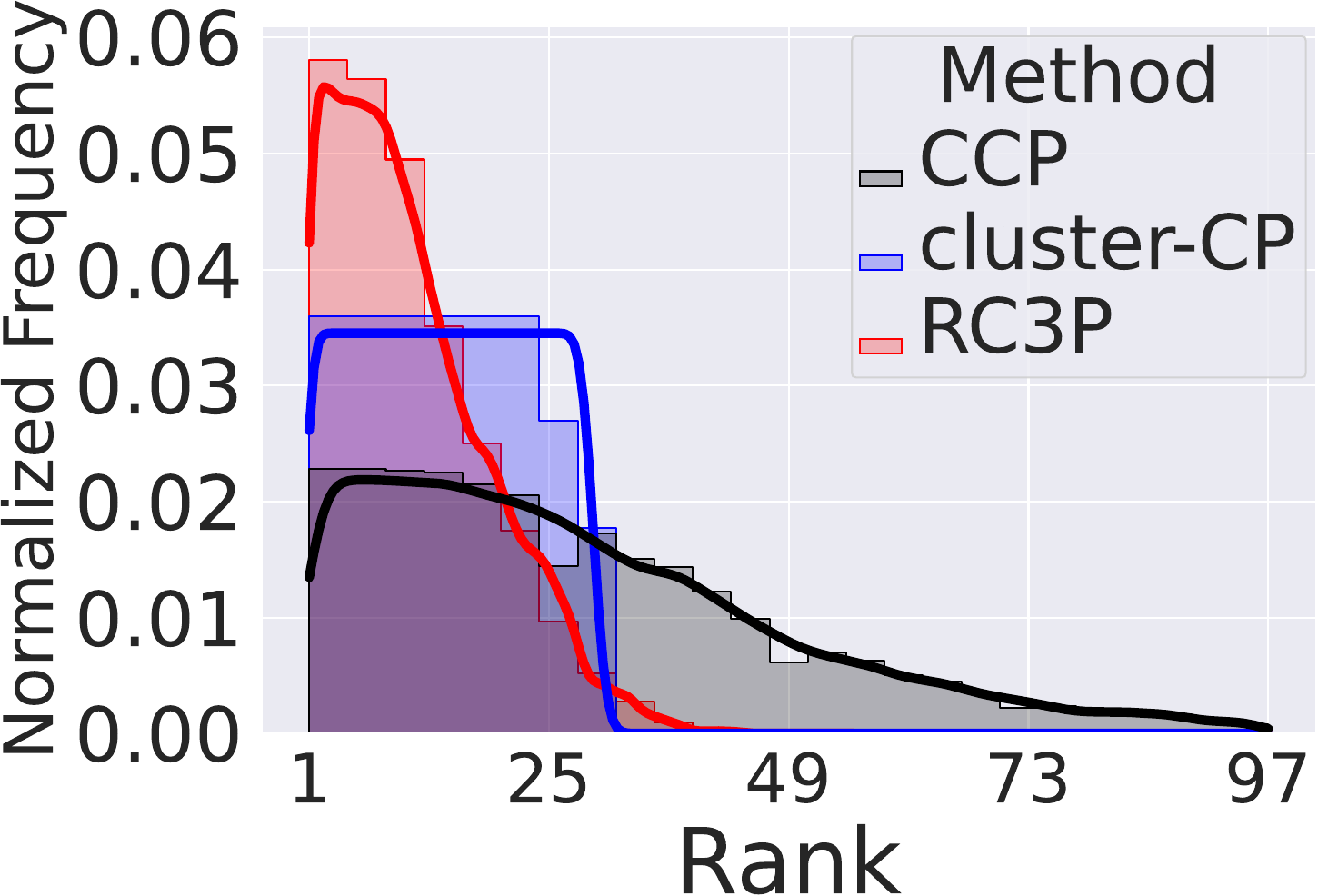}
        \end{minipage}
    }
    \subfigure{
        \begin{minipage}{0.23\linewidth}
            \includegraphics[width=\linewidth]{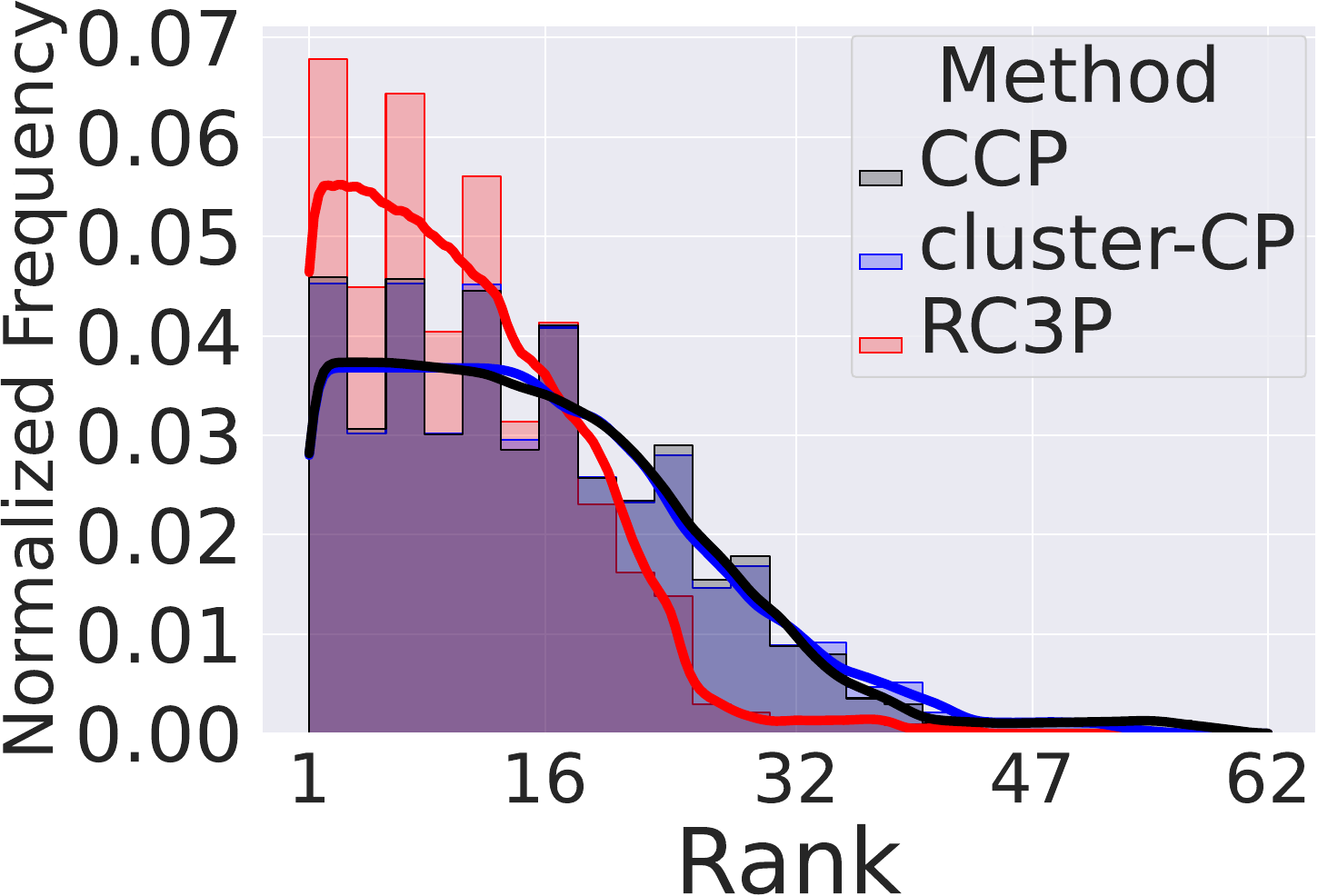}
        \end{minipage}
    }
    \subfigure{
        \begin{minipage}{0.23\linewidth}
            \includegraphics[width=\linewidth]{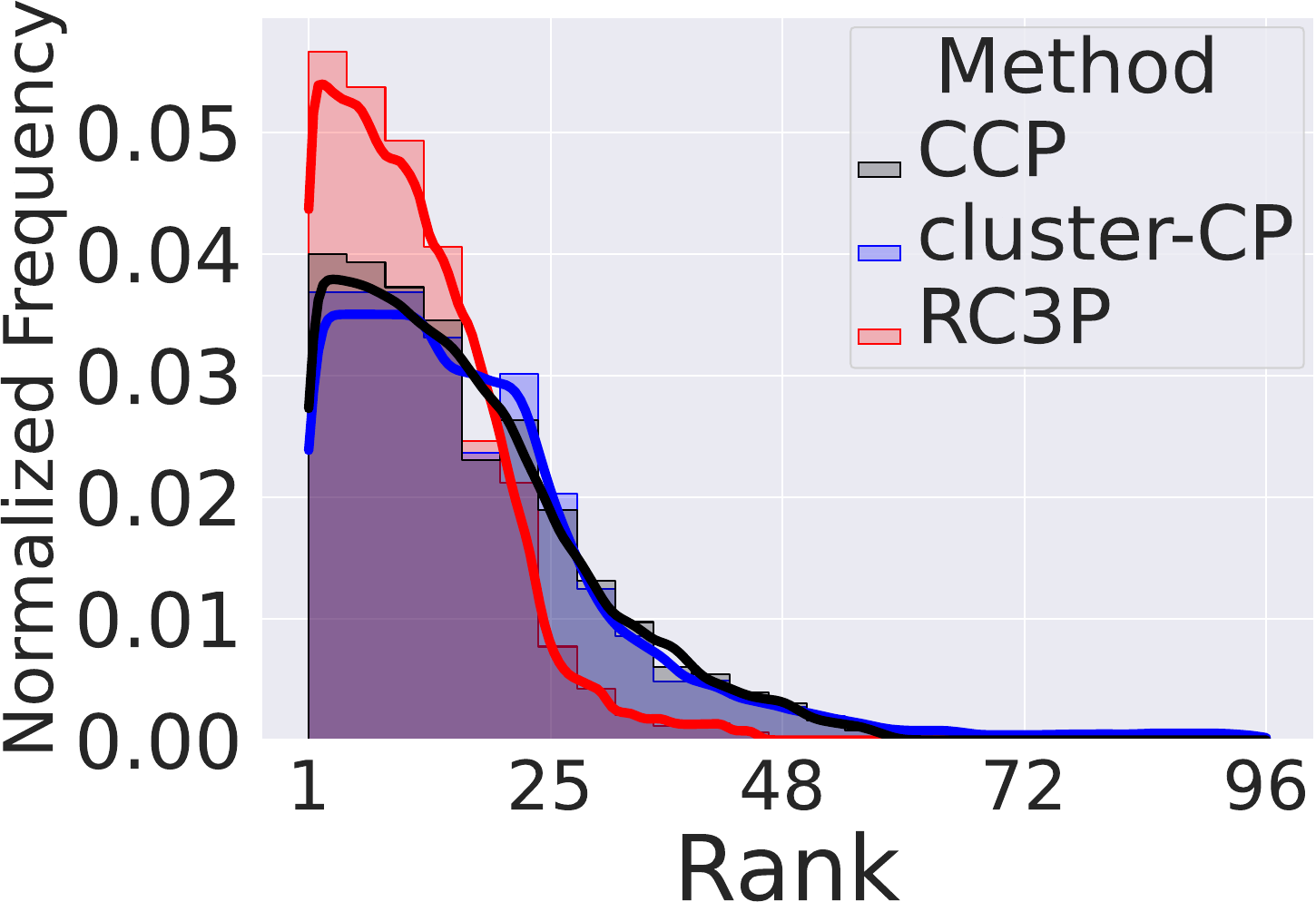}
        \end{minipage}
    }
    \caption{
    Visualization for the normalized frequency distribution of label ranks included in the prediction set of \texttt{CCP}, \texttt{Cluster-CP}, and \texttt{\newCP} with $\rho=0.5$ \EXP~when $\alpha = 0.1$.
    It is clear that the distribution of normalized frequency generated by \texttt{\newCP} tends to be lower compared to those produced by \texttt{CCP} and \texttt{Cluster-CP}.  
    Furthermore, the probability density function tail for label ranks in the \texttt{\newCP} prediction set is notably shorter than that of other methods.
    }
    \label{fig:condition_number_rank_exp_0.5}
\end{figure}

\begin{figure}[!ht]
    \centering
    \begin{minipage}{.24\textwidth}
        \centering
        (a) CIFAR-10
    \end{minipage}%
    \begin{minipage}{.24\textwidth}
        \centering
        (b) CIFAR-100
    \end{minipage}%
    \begin{minipage}{.24\textwidth}
        \centering
        (c) mini-ImageNet
    \end{minipage}%
    \begin{minipage}{.24\textwidth}
        \centering
        (d) Food-101
    \end{minipage}
    \subfigure{
        \begin{minipage}{0.23\linewidth}
            \includegraphics[width=\linewidth]{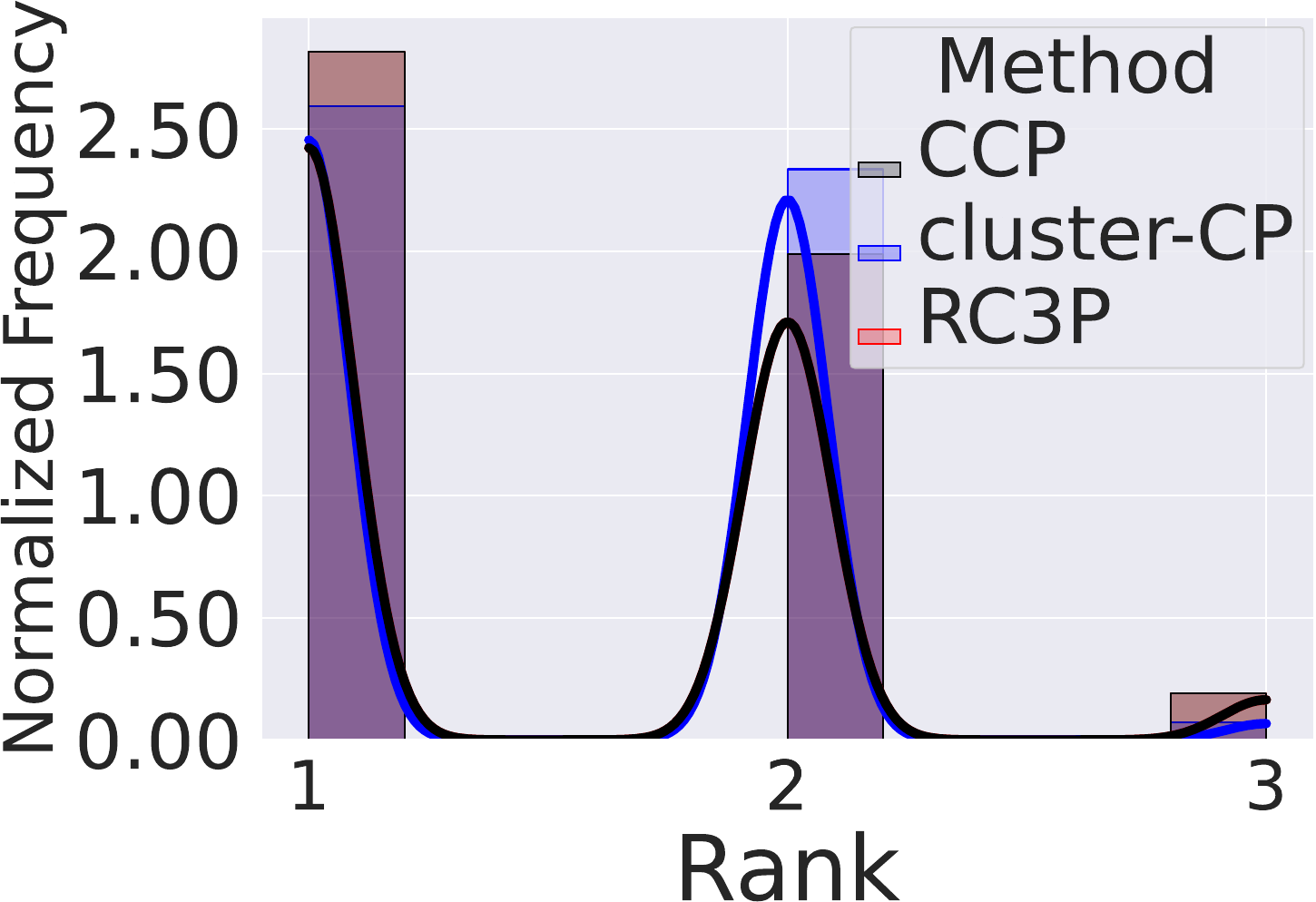}
        \end{minipage}
    }
    \subfigure{
        \begin{minipage}{0.23\linewidth}
            \includegraphics[width=\linewidth]{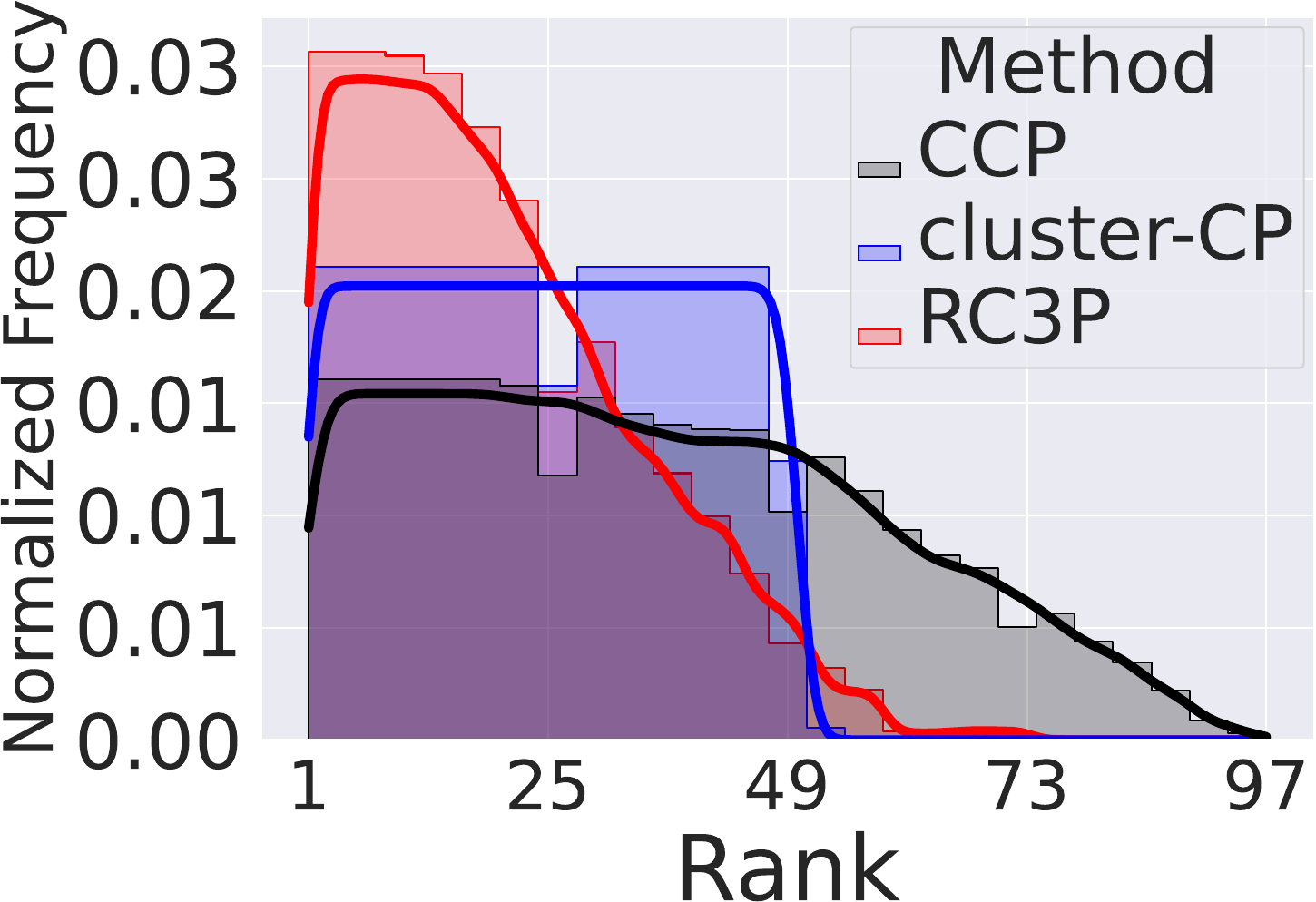}
        \end{minipage}
    }
    \subfigure{
        \begin{minipage}{0.23\linewidth}
            \includegraphics[width=\linewidth]{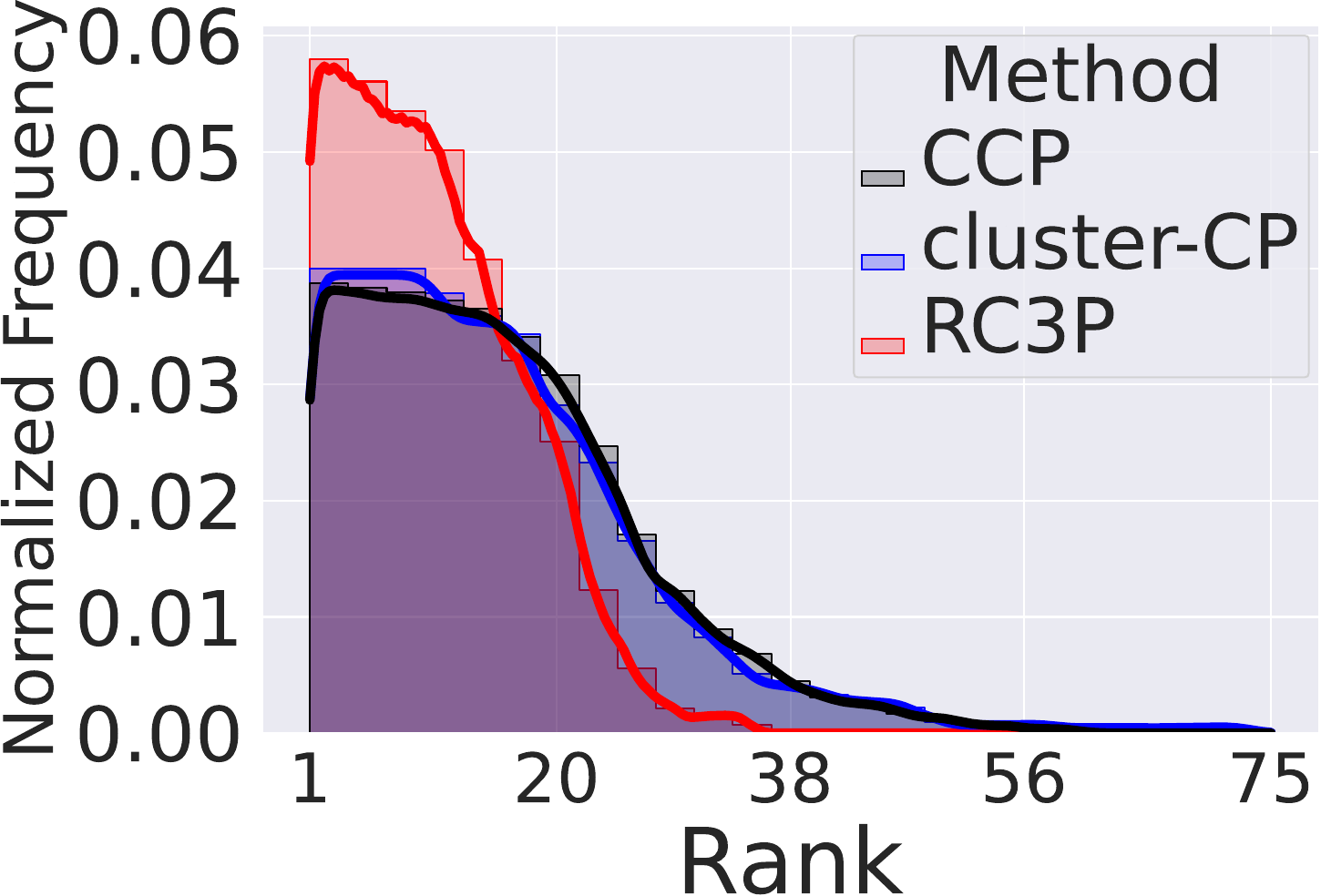}
        \end{minipage}
    }
    \subfigure{
        \begin{minipage}{0.23\linewidth}
            \includegraphics[width=\linewidth]{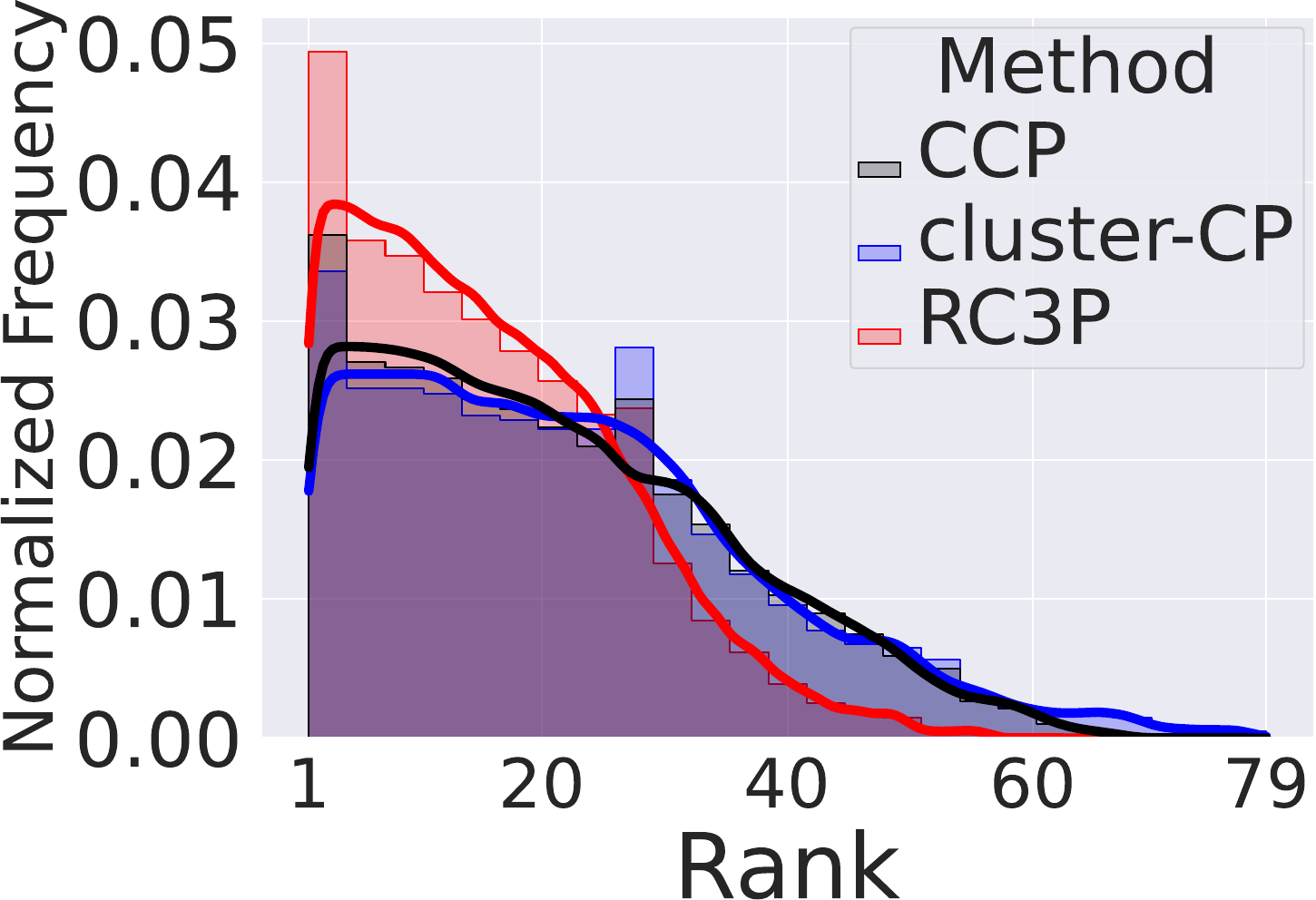}
        \end{minipage}
    }
    \caption{
    Visualization for the normalized frequency distribution of label ranks included in the prediction set of \texttt{CCP}, \texttt{Cluster-CP}, and \texttt{\newCP} with $\rho=0.1$ \POLY~when $\alpha = 0.1$.
    It is clear that the distribution of normalized frequency generated by \texttt{\newCP} tends to be lower compared to those produced by \texttt{CCP} and \texttt{Cluster-CP}.  
    Furthermore, the probability density function tail for label ranks in the \texttt{\newCP} prediction set is notably shorter than that of other methods.
    }
    \label{fig:condition_number_rank_poly_0.1}
\end{figure}

\begin{figure}[!ht]
    \centering
    \begin{minipage}{.24\textwidth}
        \centering
        (a) CIFAR-10
    \end{minipage}%
    \begin{minipage}{.24\textwidth}
        \centering
        (b) CIFAR-100
    \end{minipage}%
    \begin{minipage}{.24\textwidth}
        \centering
        (c) mini-ImageNet
    \end{minipage}%
    \begin{minipage}{.24\textwidth}
        \centering
        (d) Food-101
    \end{minipage}
    \subfigure{
        \begin{minipage}{0.23\linewidth}
            \includegraphics[width=\linewidth]{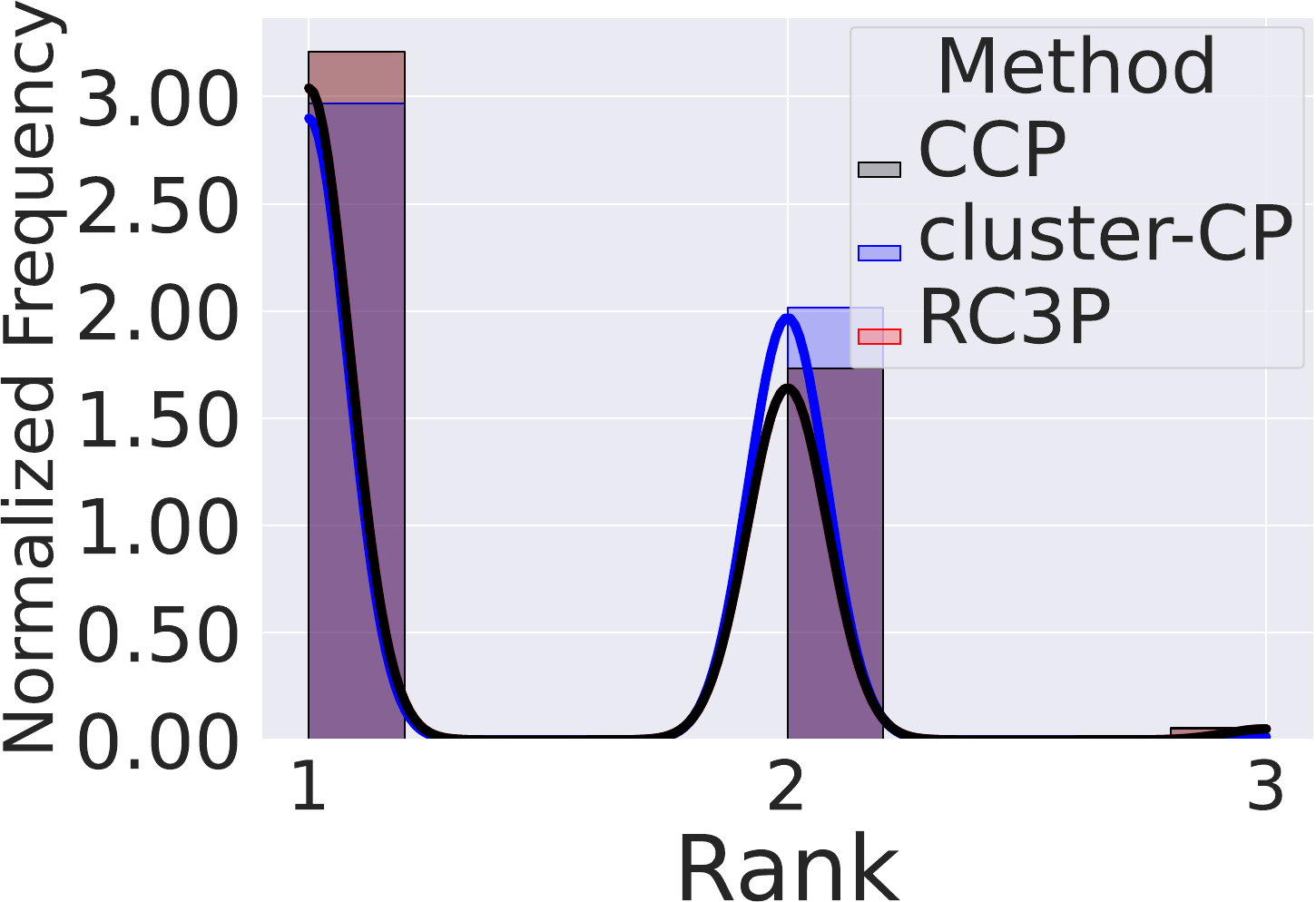}
        \end{minipage}
    }
    \subfigure{
        \begin{minipage}{0.23\linewidth}
            \includegraphics[width=\linewidth]{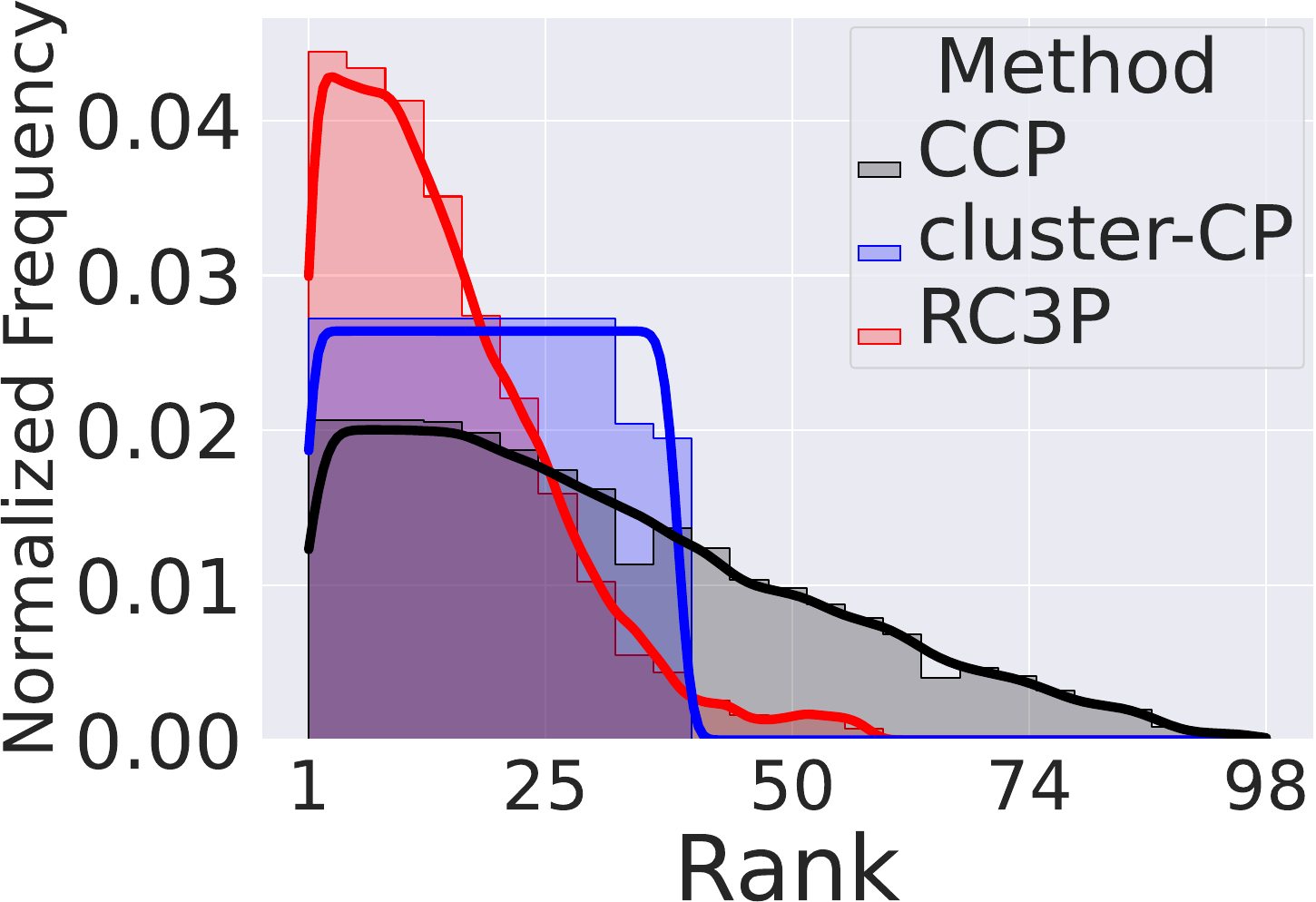}
        \end{minipage}
    }
    \subfigure{
        \begin{minipage}{0.23\linewidth}
            \includegraphics[width=\linewidth]{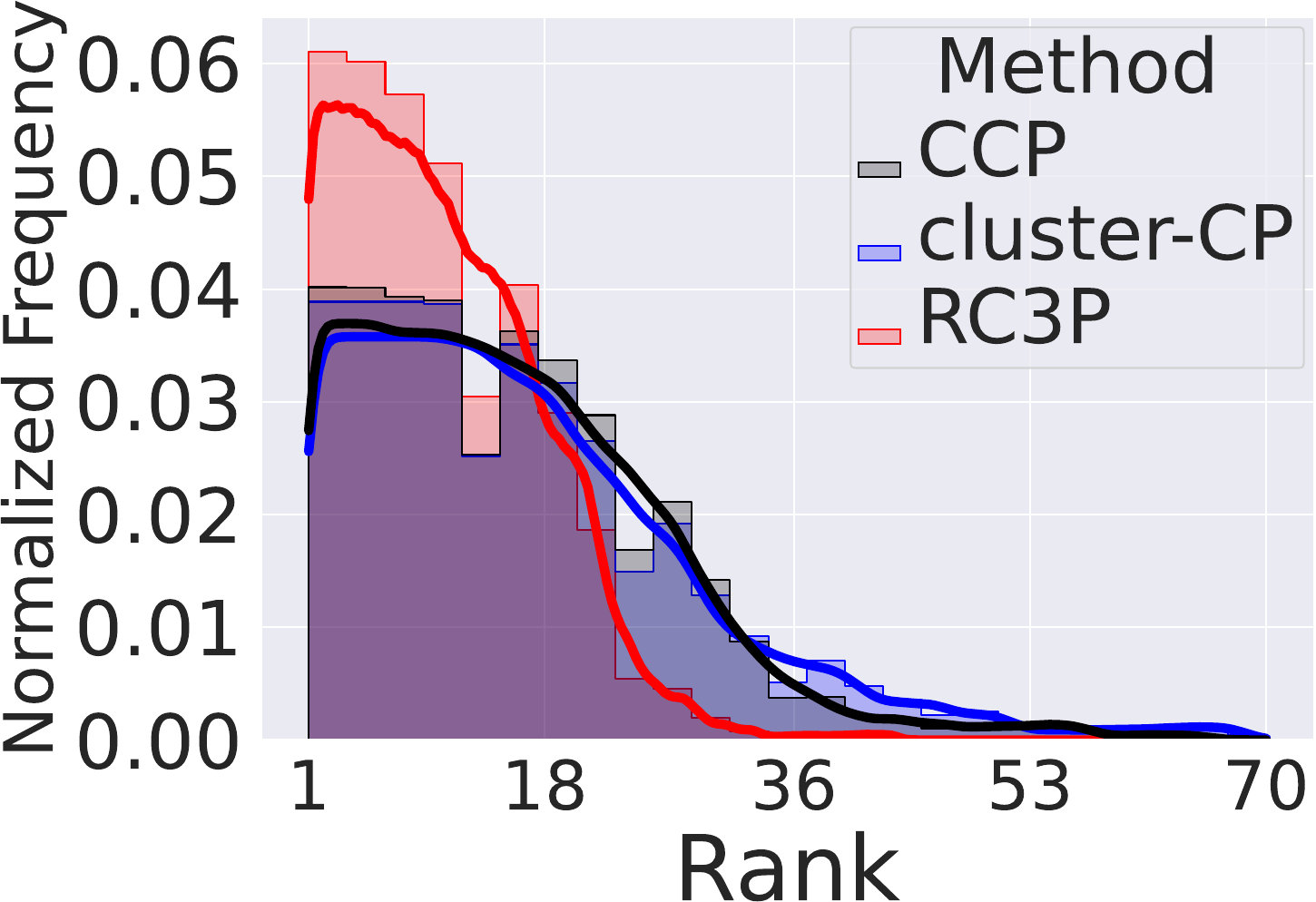}
        \end{minipage}
    }
    \subfigure{
        \begin{minipage}{0.23\linewidth}
            \includegraphics[width=\linewidth]{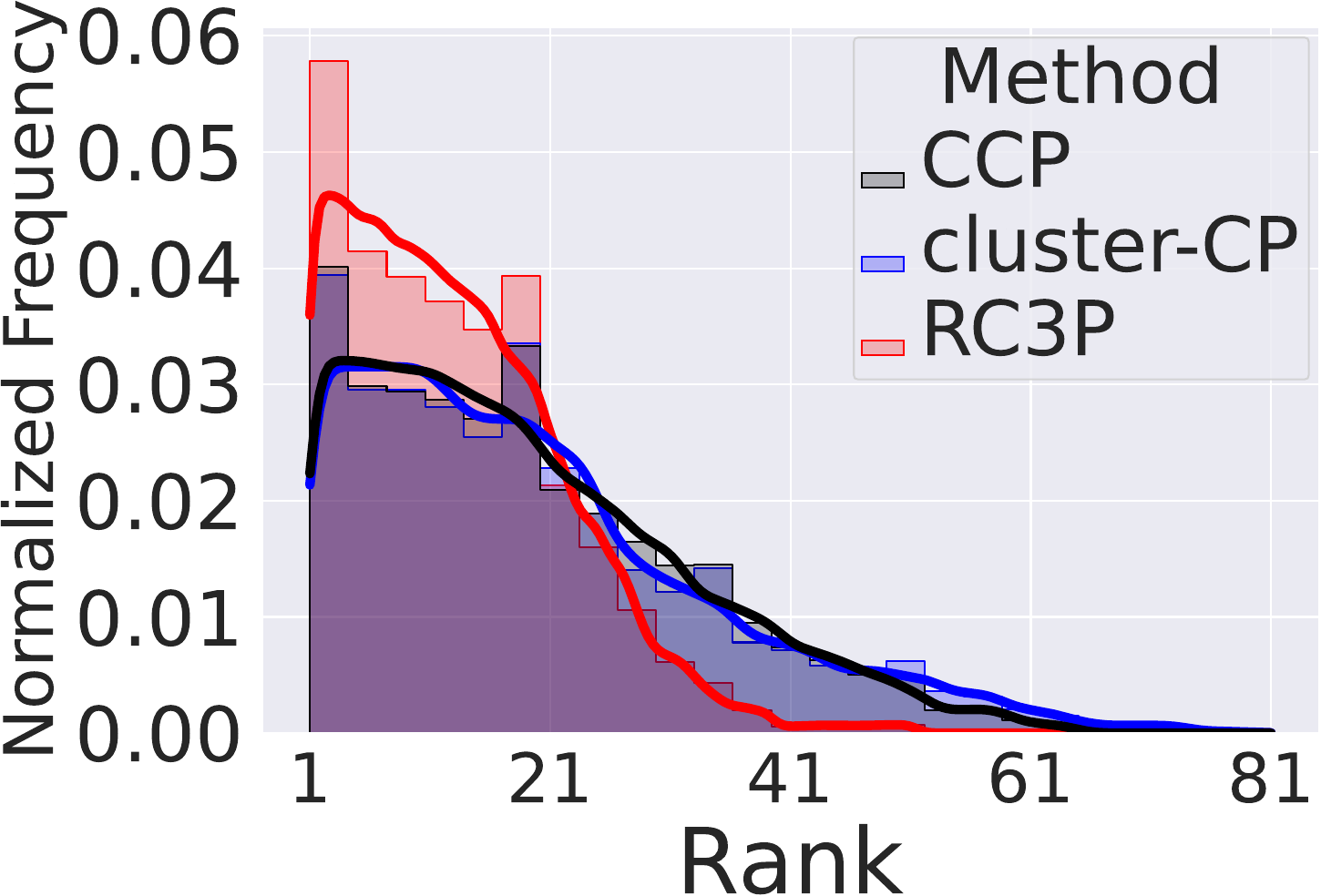}
        \end{minipage}
    }
    \caption{
    Visualization for the normalized frequency distribution of label ranks included in the prediction set of \texttt{CCP}, \texttt{Cluster-CP}, and \texttt{\newCP} with $\rho=0.5$ \POLY~when $\alpha = 0.1$.
    It is clear that the distribution of normalized frequency generated by \texttt{\newCP} tends to be lower compared to those produced by \texttt{CCP} and \texttt{Cluster-CP}.  
    Furthermore, the probability density function tail for label ranks in the \texttt{\newCP} prediction set is notably shorter than that of other methods.
    }
    \label{fig:condition_number_rank_poly_0.5}
\end{figure}

\begin{figure}[!ht]
    \centering
    \begin{minipage}{.24\textwidth}
        \centering
        (a) CIFAR-10
    \end{minipage}%
    \begin{minipage}{.24\textwidth}
        \centering
        (b) CIFAR-100
    \end{minipage}%
    \begin{minipage}{.24\textwidth}
        \centering
        (c) mini-ImageNet
    \end{minipage}%
    \begin{minipage}{.24\textwidth}
        \centering
        (d) Food-101
    \end{minipage}
    \subfigure{
        \begin{minipage}{0.23\linewidth}
            \includegraphics[width=\linewidth]{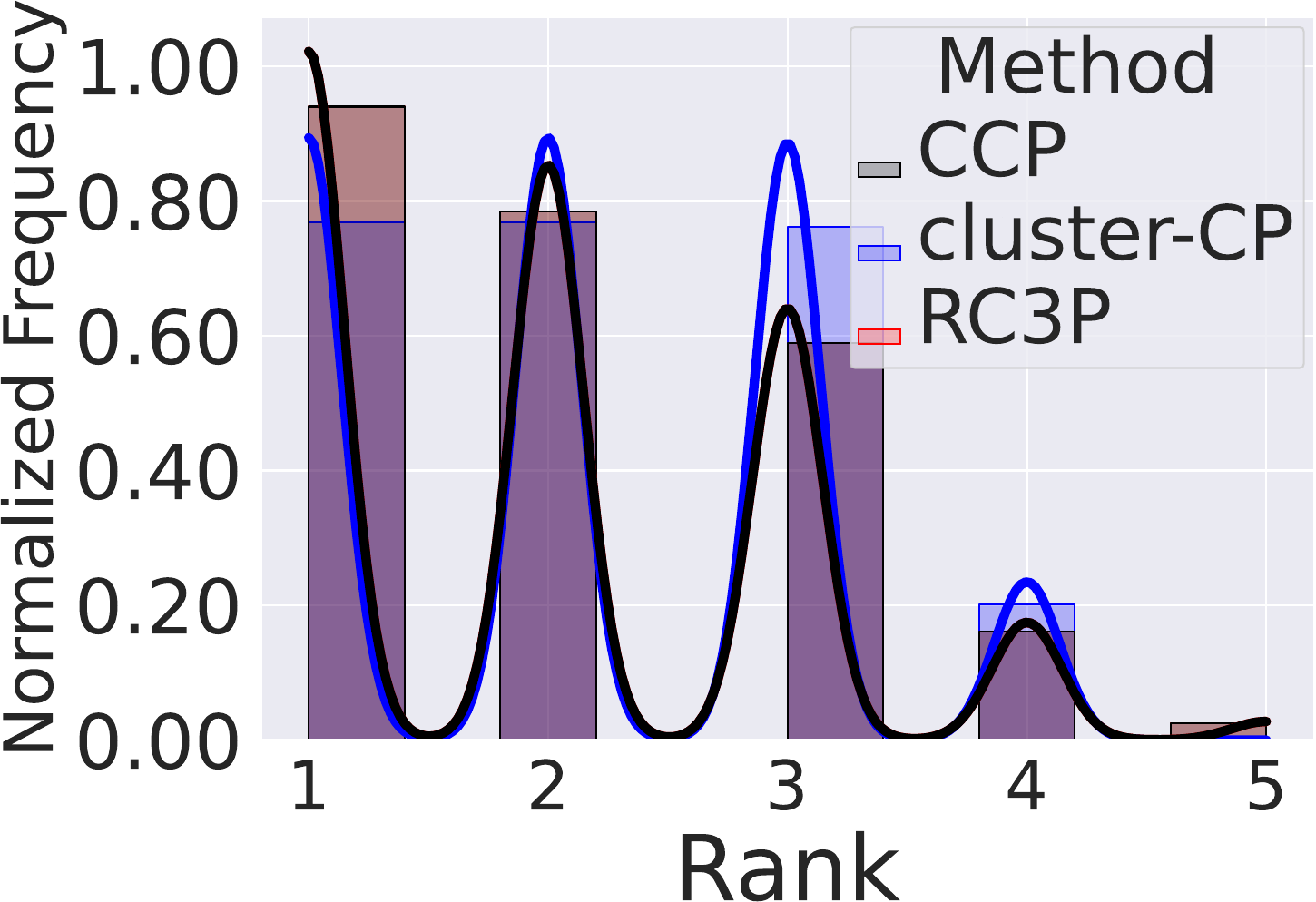}
        \end{minipage}
    }
    \subfigure{
        \begin{minipage}{0.23\linewidth}
            \includegraphics[width=\linewidth]{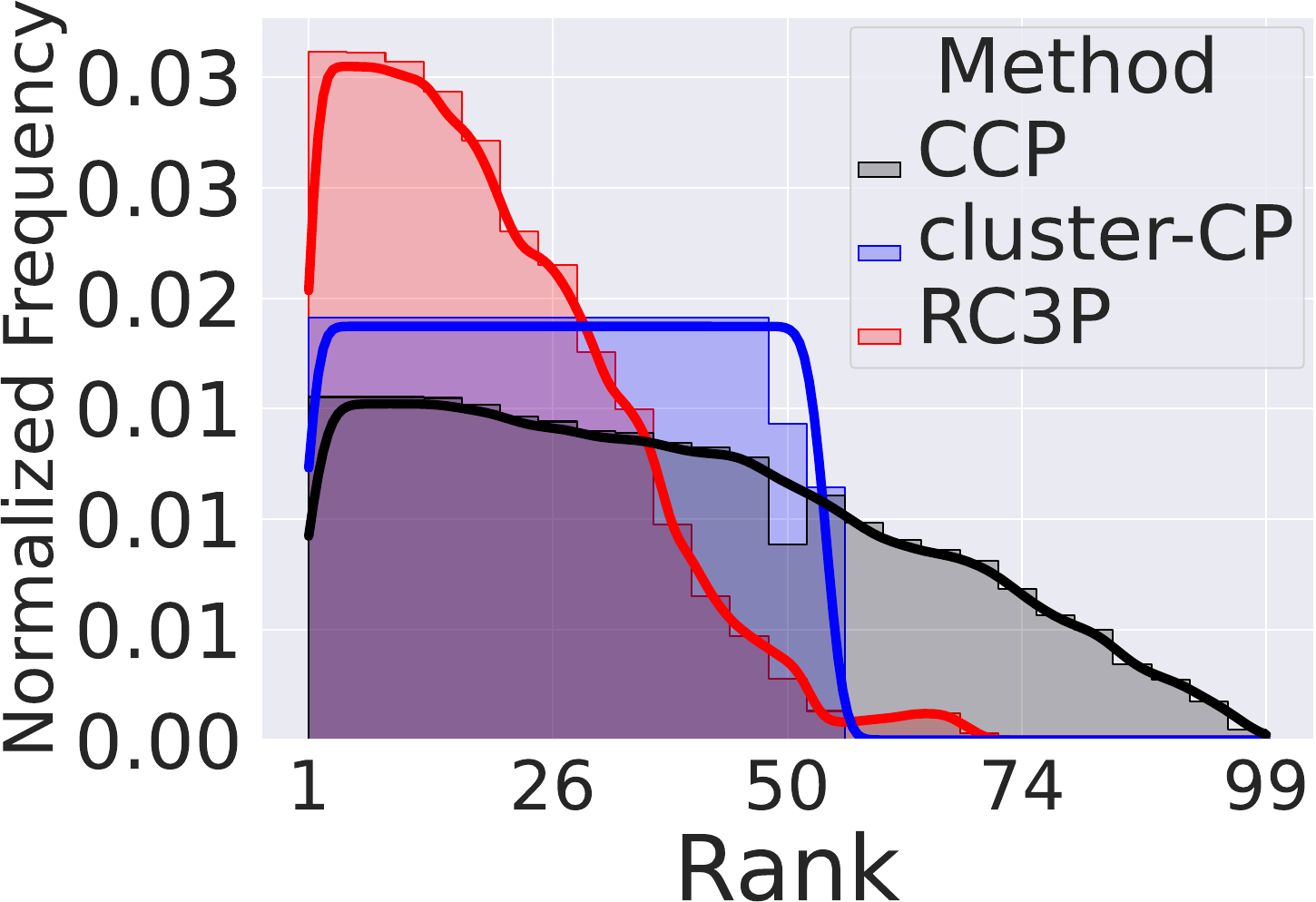}
        \end{minipage}
    }
    \subfigure{
        \begin{minipage}{0.23\linewidth}
            \includegraphics[width=\linewidth]{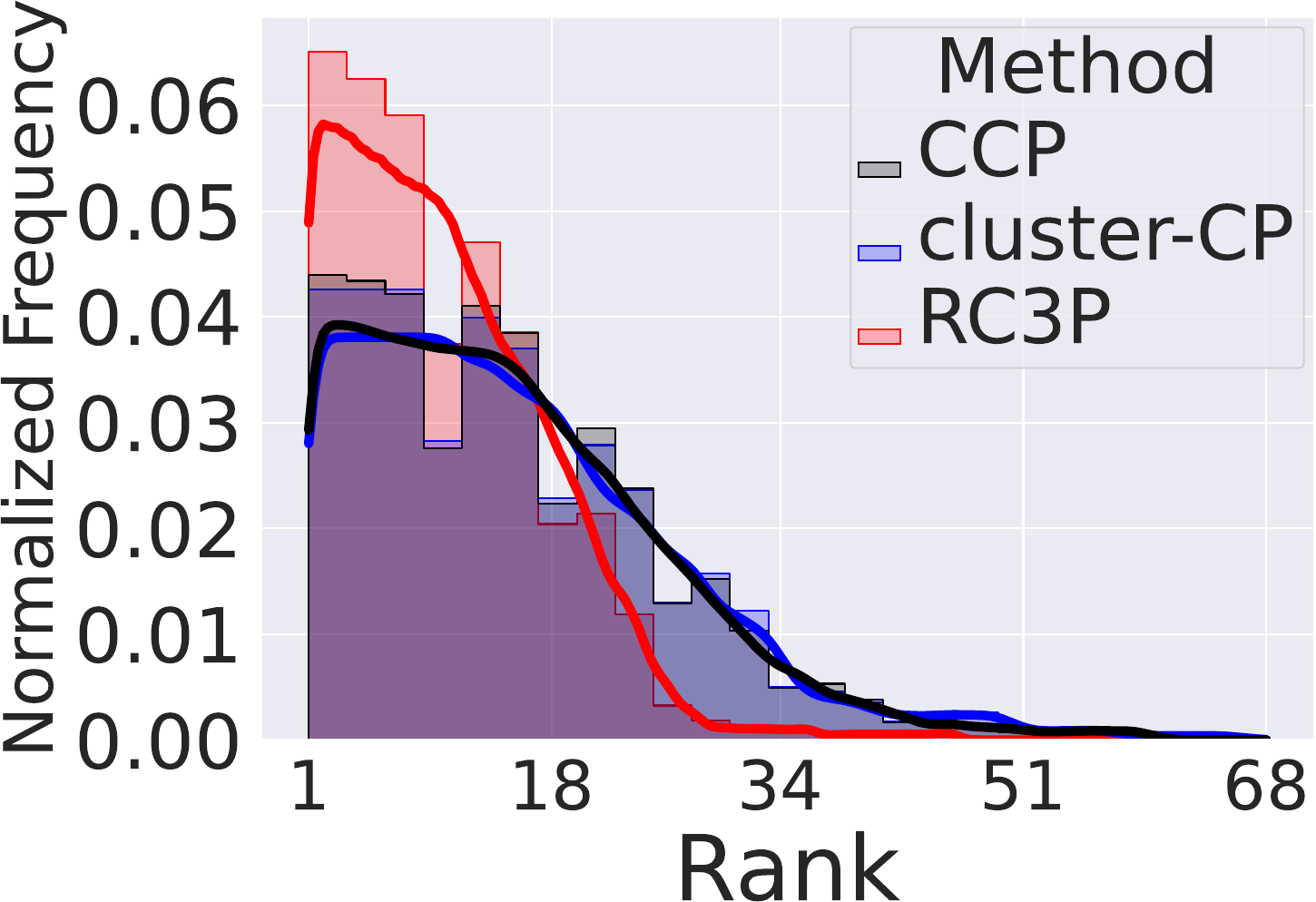}
        \end{minipage}
    }
    \subfigure{
        \begin{minipage}{0.23\linewidth}
            \includegraphics[width=\linewidth]{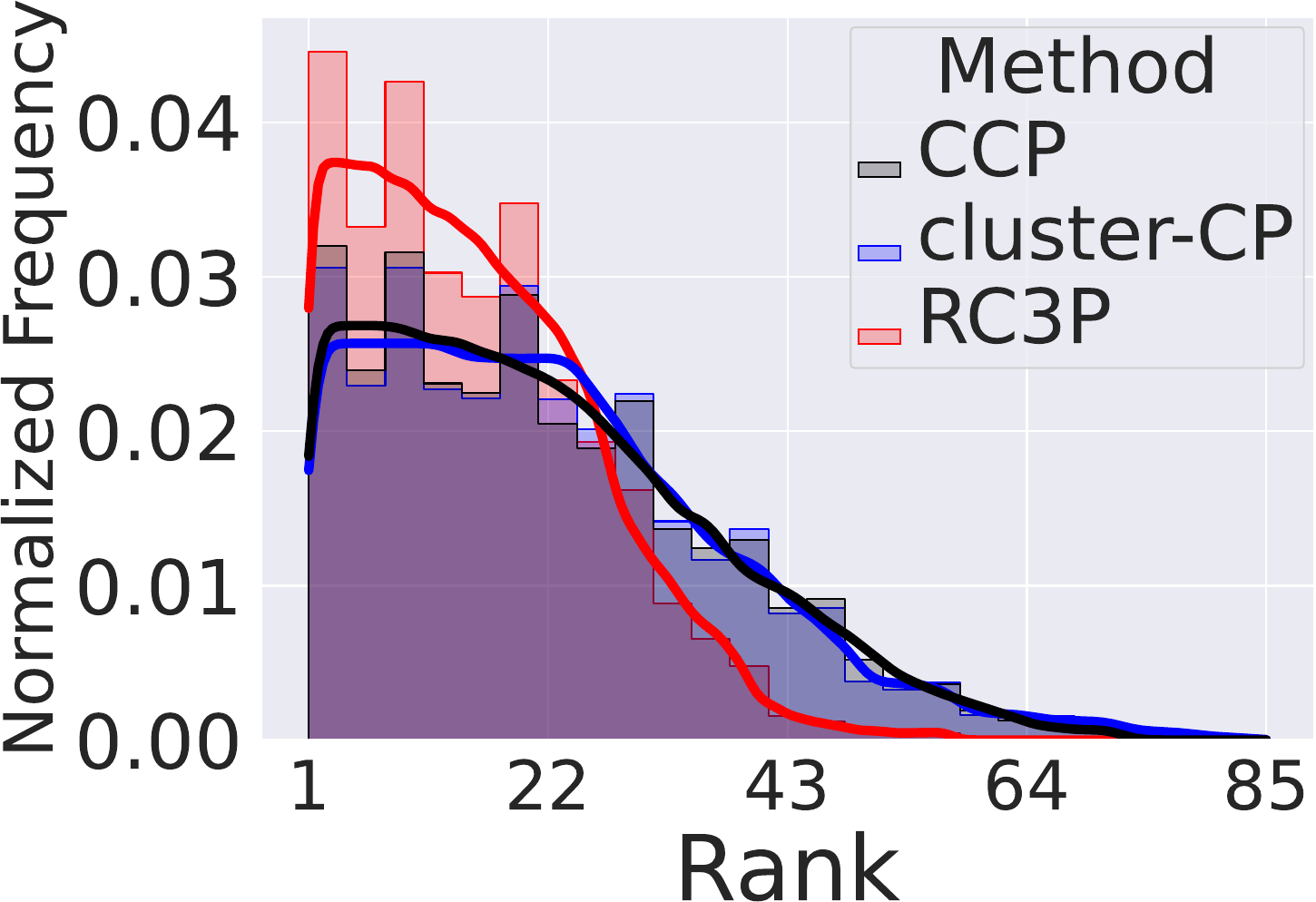}
        \end{minipage}
    }
    \caption{
    Visualization for the normalized frequency distribution of label ranks included in the prediction set of \texttt{CCP}, \texttt{Cluster-CP}, and \texttt{\newCP} with $\rho=0.1$ \MAJ~when $\alpha = 0.1$.
    It is clear that the distribution of normalized frequency generated by \texttt{\newCP} tends to be lower compared to those produced by \texttt{CCP} and \texttt{Cluster-CP}.  
    Furthermore, the probability density function tail for label ranks in the \texttt{\newCP} prediction set is notably shorter than that of other methods.
    }
    \label{fig:condition_number_rank_maj_0.1}
\end{figure}

\begin{figure}[!ht]
    \centering
    \begin{minipage}{.24\textwidth}
        \centering
        (a) CIFAR-10
    \end{minipage}%
    \begin{minipage}{.24\textwidth}
        \centering
        (b) CIFAR-100
    \end{minipage}%
    \begin{minipage}{.24\textwidth}
        \centering
        (c) mini-ImageNet
    \end{minipage}%
    \begin{minipage}{.24\textwidth}
        \centering
        (d) Food-101
    \end{minipage}
    \subfigure{
        \begin{minipage}{0.23\linewidth}
            \includegraphics[width=\linewidth]{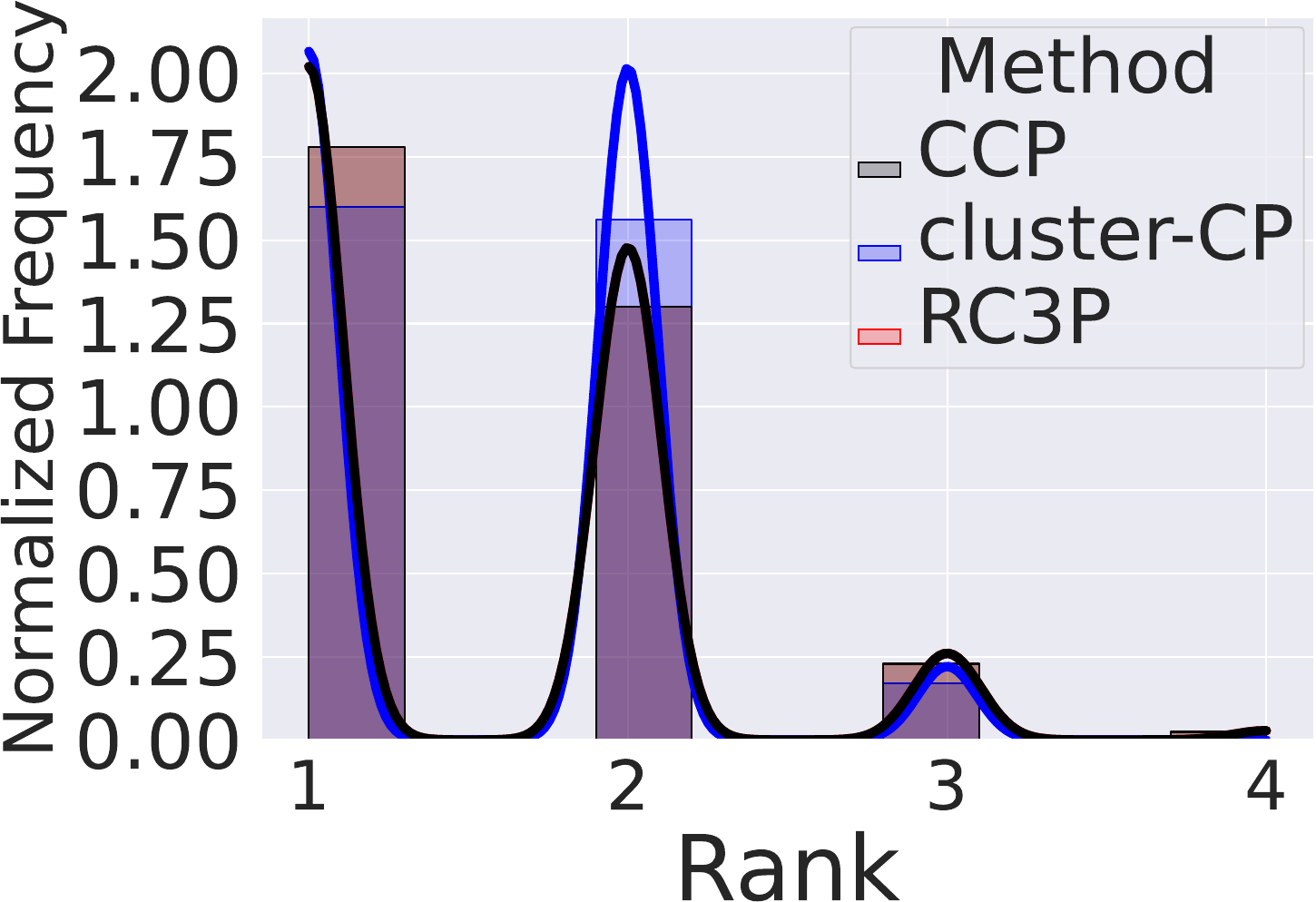}
        \end{minipage}
    }
    \subfigure{
        \begin{minipage}{0.23\linewidth}
            \includegraphics[width=\linewidth]{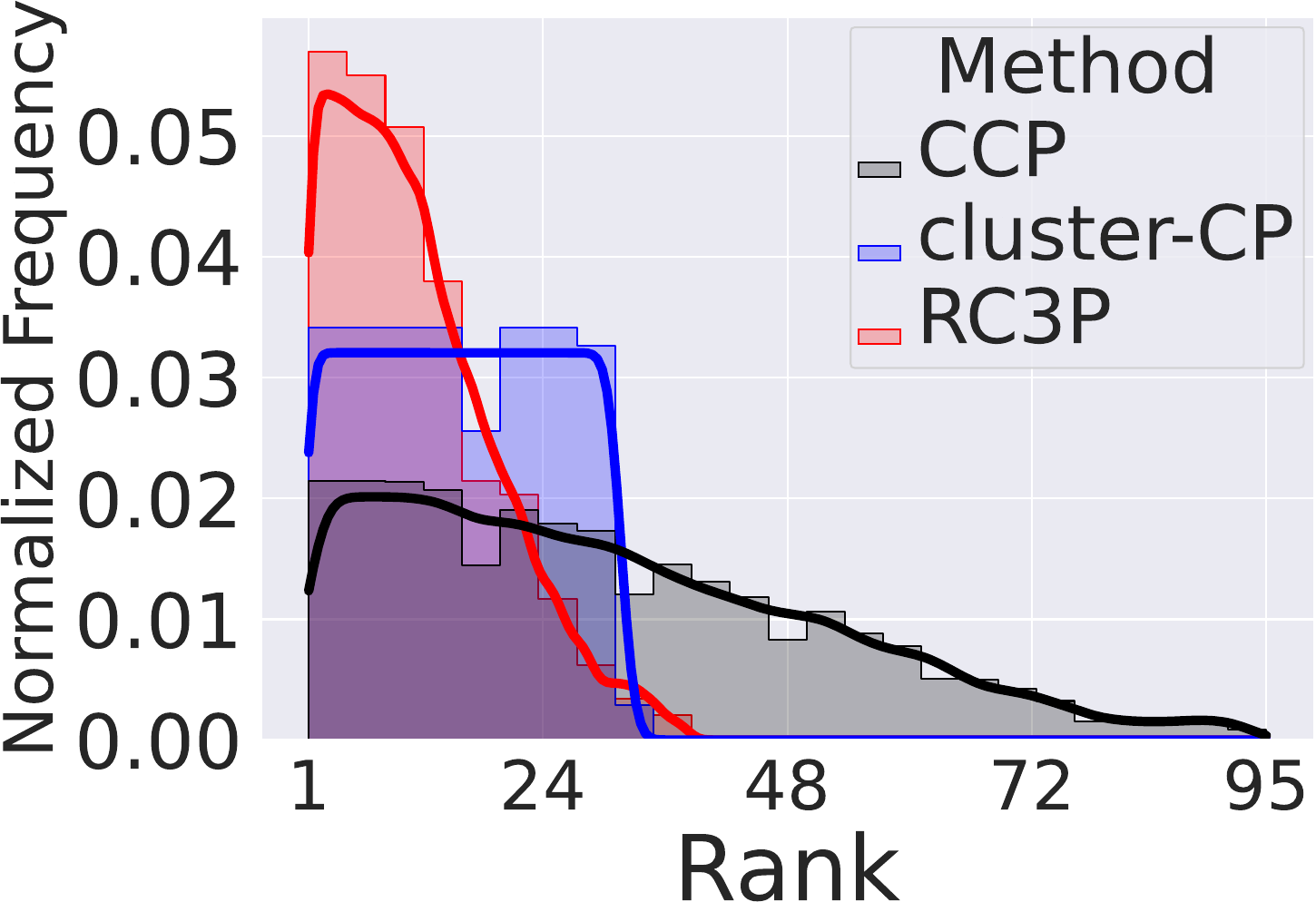}
        \end{minipage}
    }
    \subfigure{
        \begin{minipage}{0.23\linewidth}
            \includegraphics[width=\linewidth]{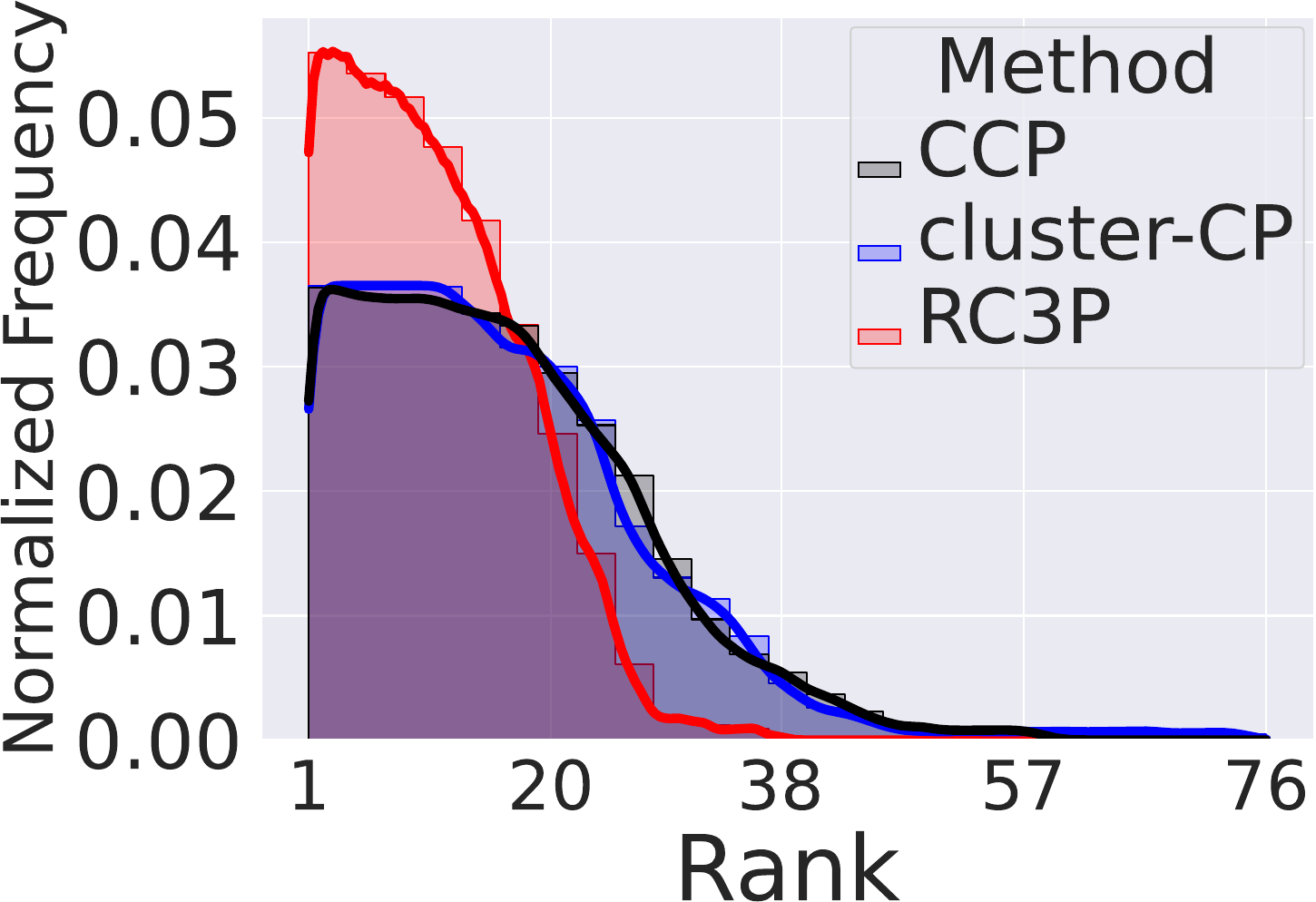}
        \end{minipage}
    }
    \subfigure{
        \begin{minipage}{0.23\linewidth}
            \includegraphics[width=\linewidth]{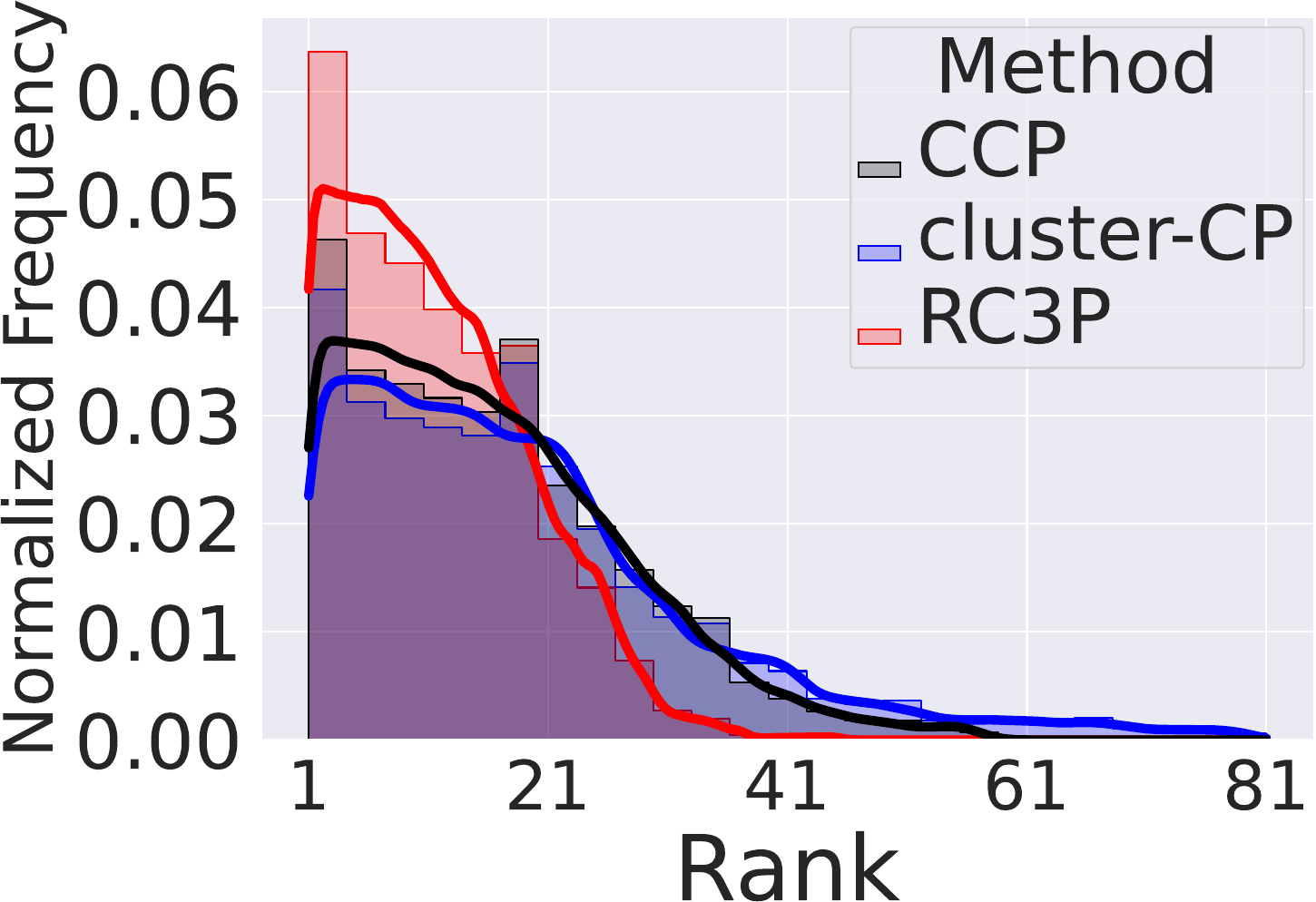}
        \end{minipage}
    }
    \caption{
    Visualization for the normalized frequency distribution of label ranks included in the prediction set of \texttt{CCP}, \texttt{Cluster-CP}, and \texttt{\newCP} with $\rho=0.5$ \MAJ~when $\alpha = 0.1$.
    It is clear that the distribution of normalized frequency generated by \texttt{\newCP} tends to be lower compared to those produced by \texttt{CCP} and \texttt{Cluster-CP}.  
    Furthermore, the probability density function tail for label ranks in the \texttt{\newCP} prediction set is notably shorter than that of other methods.
    }
    \label{fig:condition_number_rank_maj_0.5}
\end{figure}

\begin{figure}[!ht]
    \centering
    \begin{minipage}{.24\textwidth}
        \centering
        (a) CIFAR-10
    \end{minipage}%
    \begin{minipage}{.24\textwidth}
        \centering
        (b) CIFAR-100
    \end{minipage}%
    \begin{minipage}{.24\textwidth}
        \centering
        (c) mini-ImageNet
    \end{minipage}%
    \begin{minipage}{.24\textwidth}
        \centering
        (d) Food-101
    \end{minipage}
    \subfigure{
        \begin{minipage}{0.23\linewidth}
            \includegraphics[width=\linewidth]{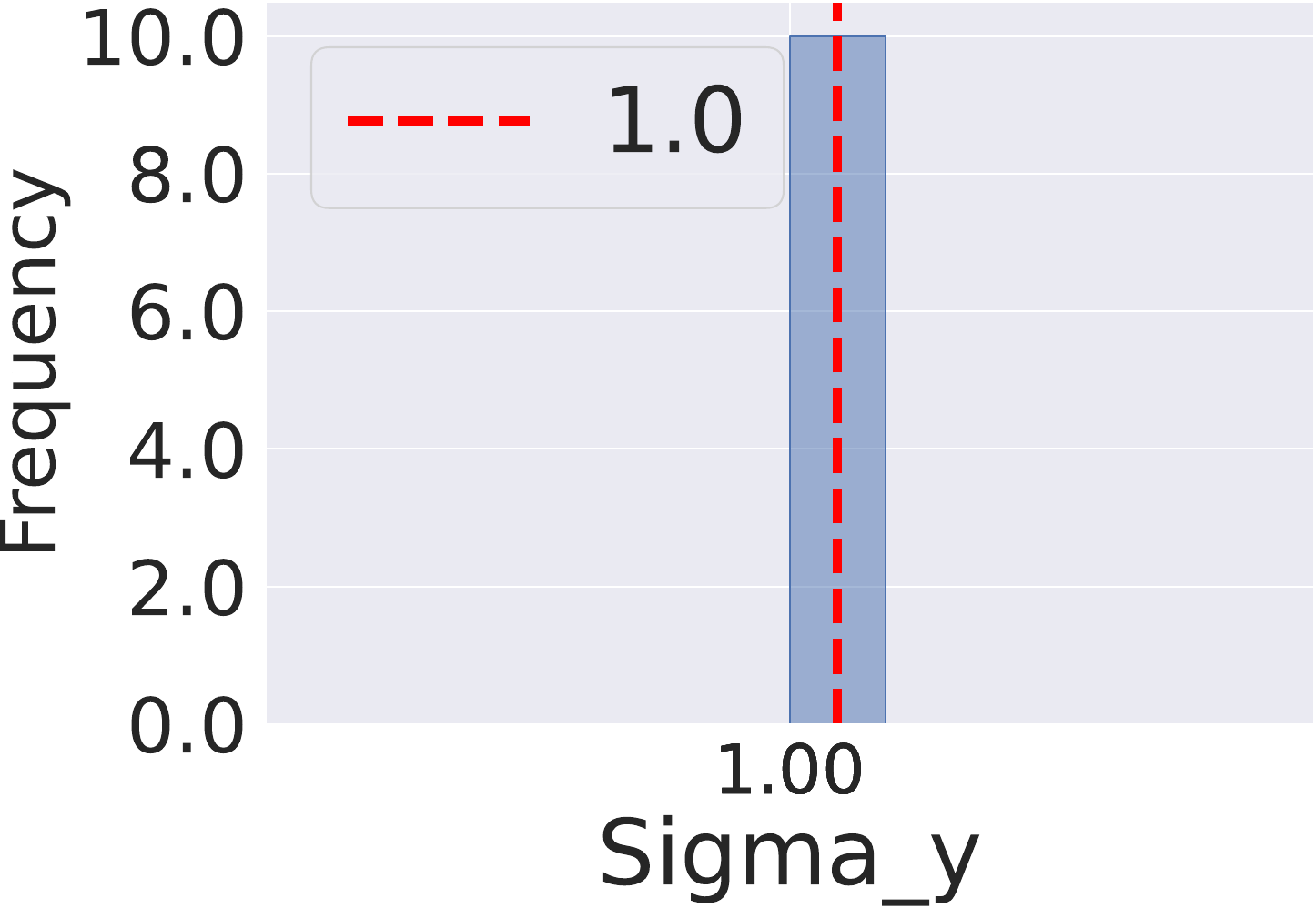}
        \end{minipage}
    }
    \subfigure{
        \begin{minipage}{0.23\linewidth}
            \includegraphics[width=\linewidth]{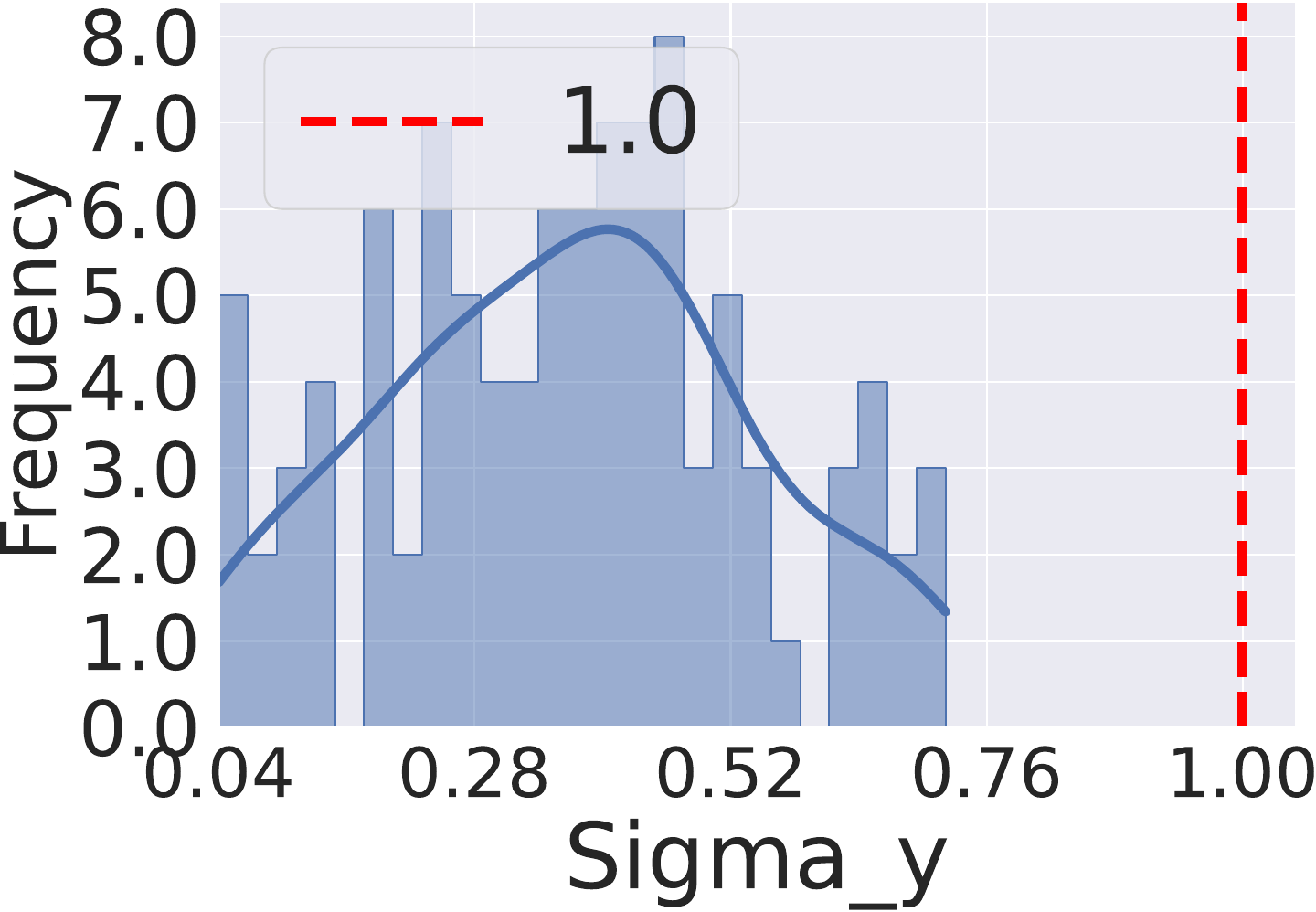}
        \end{minipage}
    }
    \subfigure{
        \begin{minipage}{0.23\linewidth}
            \includegraphics[width=\linewidth]{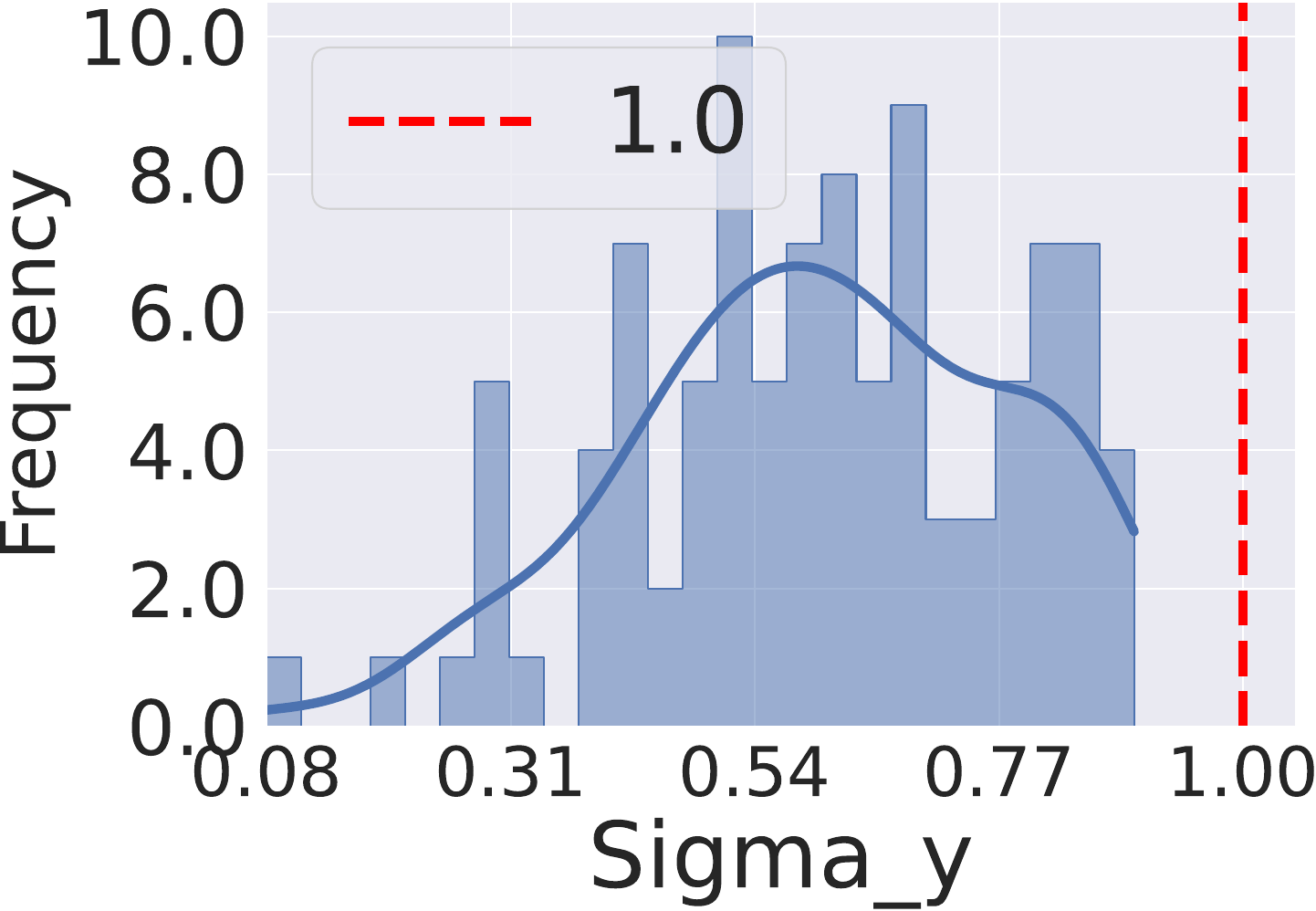}
        \end{minipage}
    }
    \subfigure{
        \begin{minipage}{0.23\linewidth}
            \includegraphics[width=\linewidth]{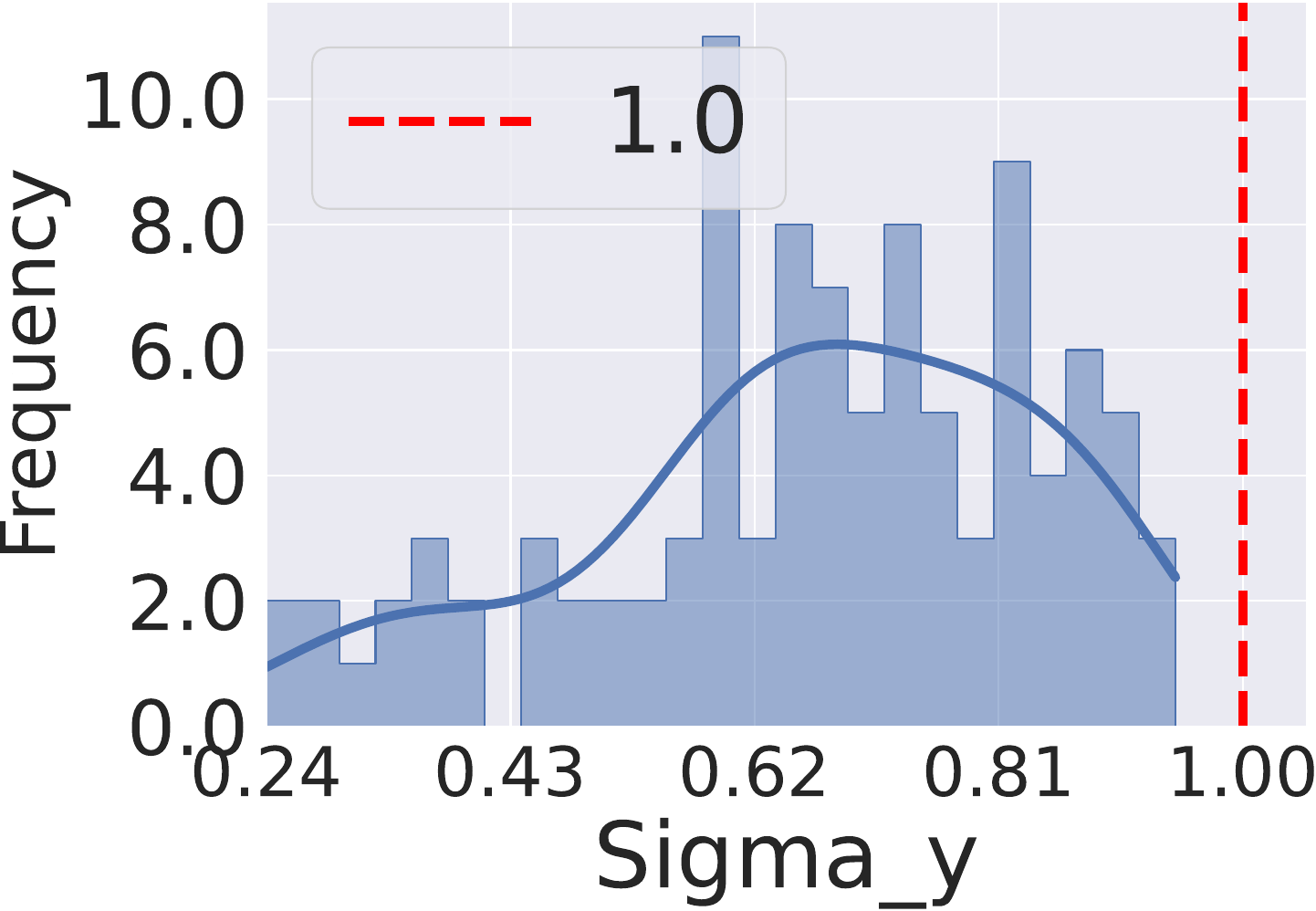}
        \end{minipage}
    }
    \caption{
    Verification of condition numbers $\{\sigma_y\}_{y=1}^C$ of Equation \ref{eq:sigma_y_defination} when epoch $=200$ and $\alpha = 0.1$ with $\rho=0.5$ \EXP.
    Vertical dashed lines represent the value $1$, and we observe that all the condition numbers are smaller than $1$.
    This verifies the validity of the condition for Lemma \ref{lemma:RC3P_improved_efficiency}, and thus confirms that \texttt{\newCP}produces smaller prediction sets than \texttt{CCP} by the optimized trade-off between calibration on non-conformity scores and calibrated label ranks.
    }
    \label{fig:condition_number_sigma_y_exp_0.5}
\end{figure}

\begin{figure}[!ht]
    \centering
    \begin{minipage}{.24\textwidth}
        \centering
        (a) CIFAR-10
    \end{minipage}%
    \begin{minipage}{.24\textwidth}
        \centering
        (b) CIFAR-100
    \end{minipage}%
    \begin{minipage}{.24\textwidth}
        \centering
        (c) mini-ImageNet
    \end{minipage}%
    \begin{minipage}{.24\textwidth}
        \centering
        (d) Food-101
    \end{minipage}
    \subfigure{
        \begin{minipage}{0.23\linewidth}
            \includegraphics[width=\linewidth]{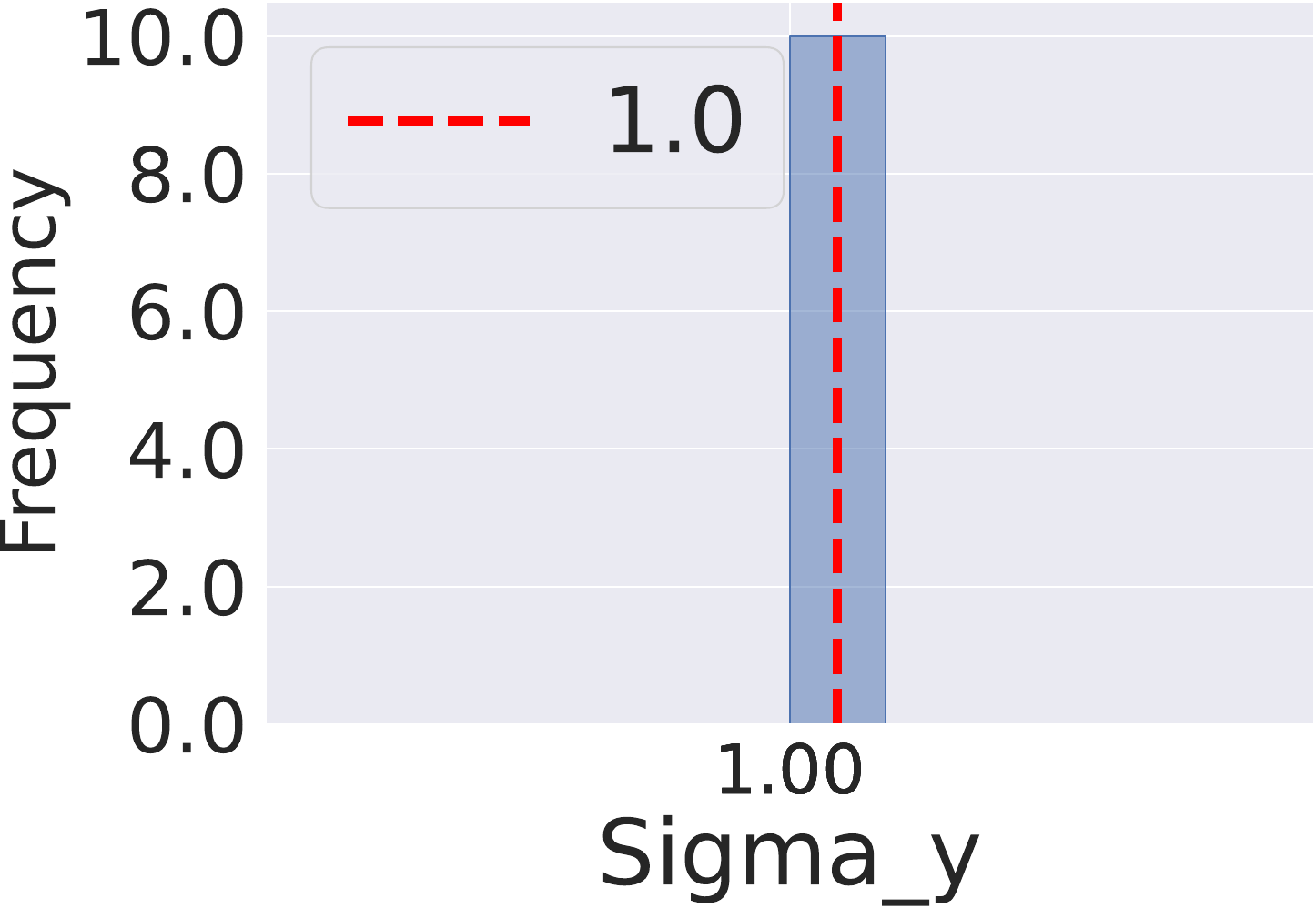}
        \end{minipage}
    }
    \subfigure{
        \begin{minipage}{0.23\linewidth}
            \includegraphics[width=\linewidth]{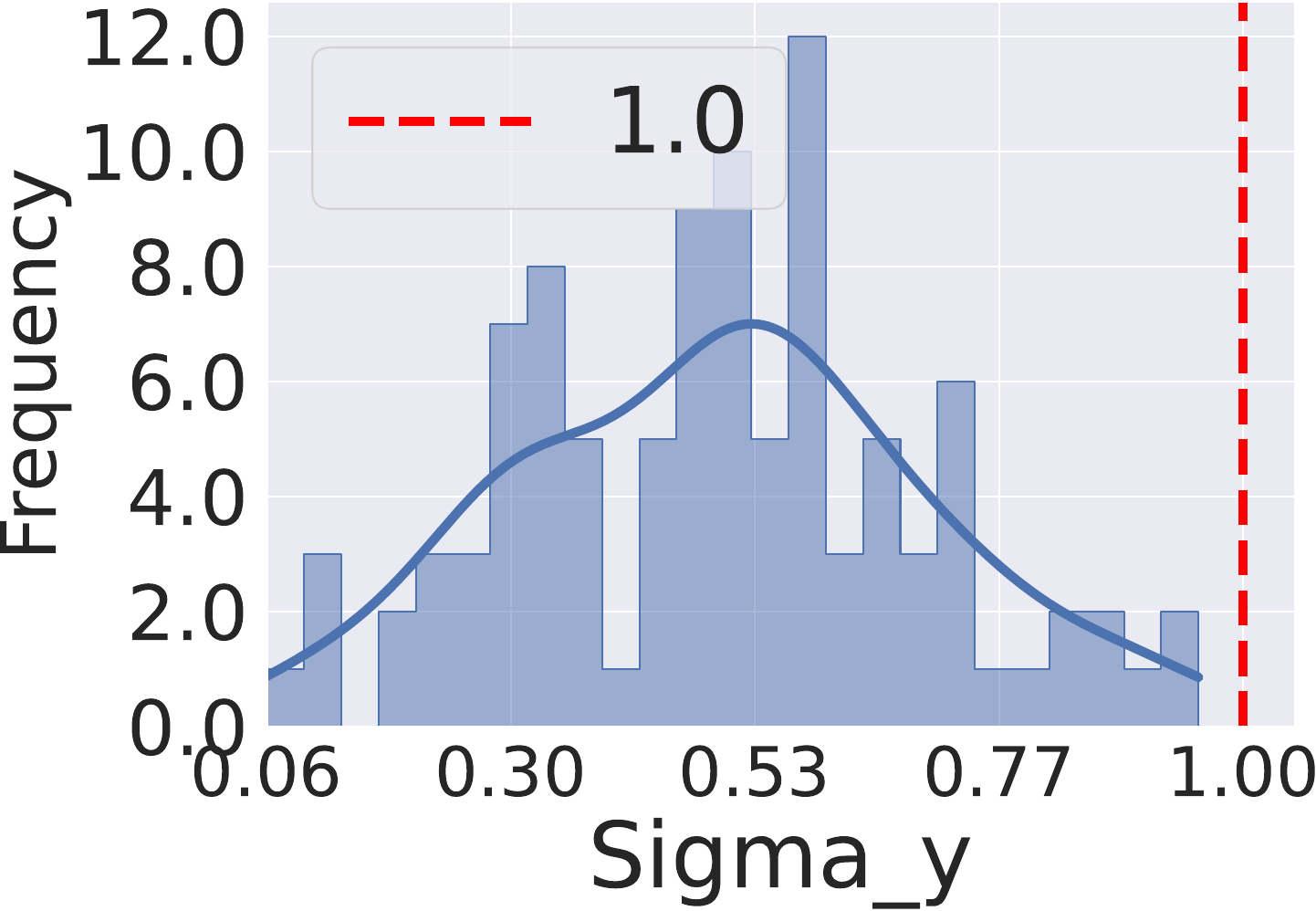}
        \end{minipage}
    }
    \subfigure{
        \begin{minipage}{0.23\linewidth}
            \includegraphics[width=\linewidth]{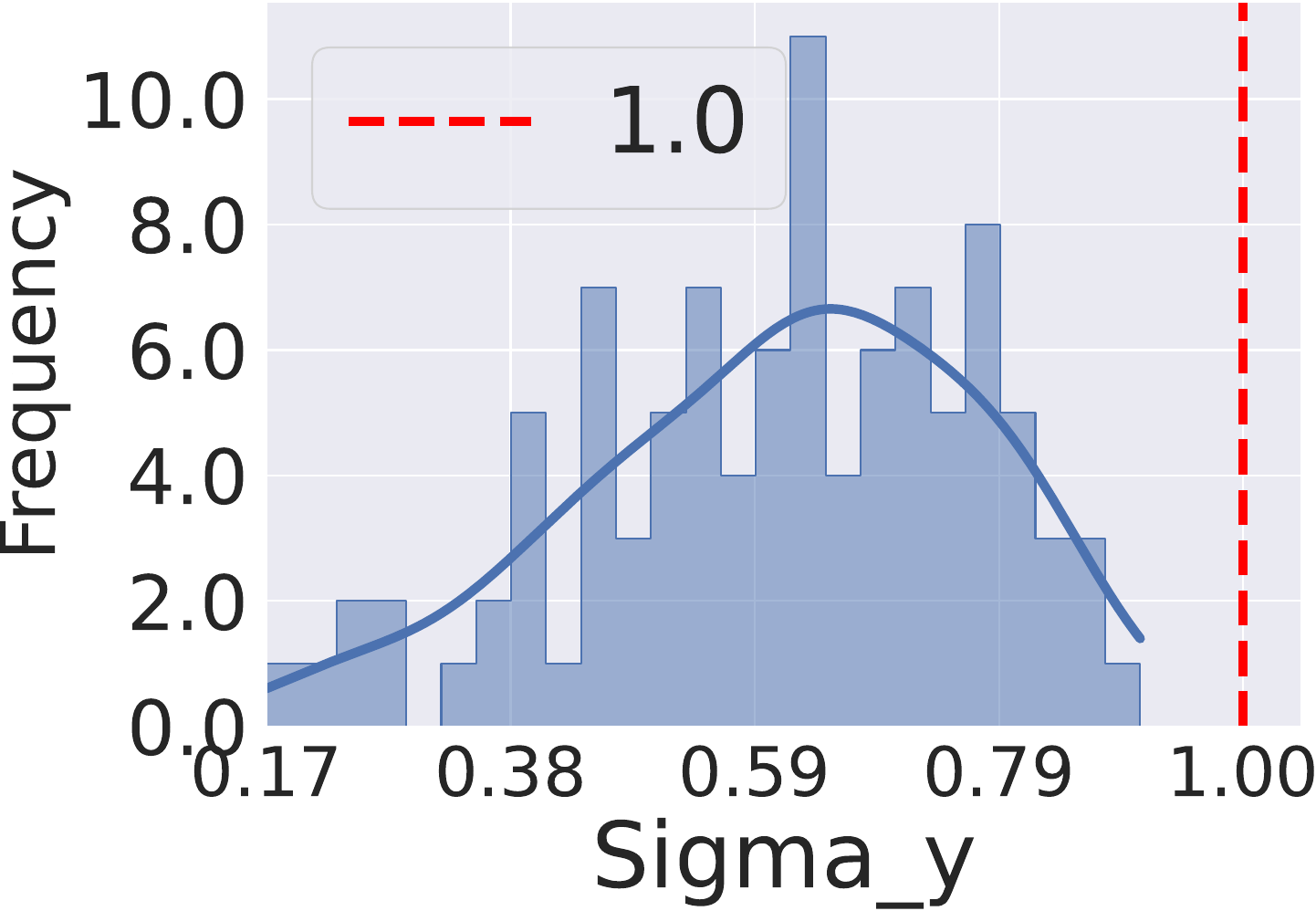}
        \end{minipage}
    }
    \subfigure{
        \begin{minipage}{0.23\linewidth}
            \includegraphics[width=\linewidth]{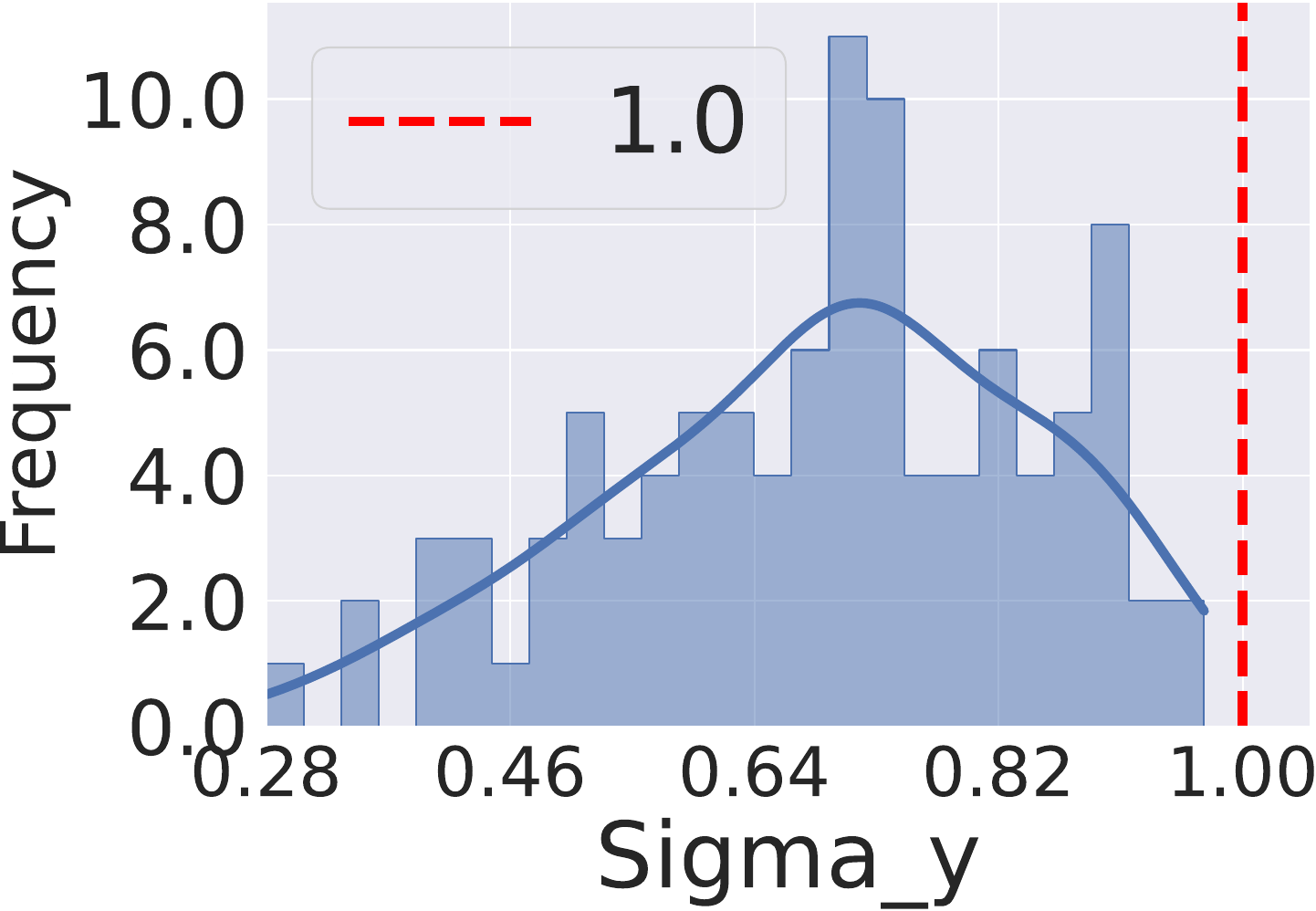}
        \end{minipage}
    }
    \caption{
    Verification of condition numbers $\{\sigma_y\}_{y=1}^C$ of Equation \ref{eq:sigma_y_defination} when epoch $=200$ and $\alpha = 0.1$ with $\rho=0.1$ \POLY.
    Vertical dashed lines represent the value $1$, and we observe that all the condition numbers are smaller than $1$.
    This verifies the validity of the condition for Lemma \ref{lemma:RC3P_improved_efficiency}, and thus confirms that \texttt{\newCP}produces smaller prediction sets than \texttt{CCP} by the optimized trade-off between calibration on non-conformity scores and calibrated label ranks.
    }
    \label{fig:condition_number_sigma_y_poly_0.1}
\end{figure}

\begin{figure}[!ht]
    \centering
    \begin{minipage}{.24\textwidth}
        \centering
        (a) CIFAR-10
    \end{minipage}%
    \begin{minipage}{.24\textwidth}
        \centering
        (b) CIFAR-100
    \end{minipage}%
    \begin{minipage}{.24\textwidth}
        \centering
        (c) mini-ImageNet
    \end{minipage}%
    \begin{minipage}{.24\textwidth}
        \centering
        (d) Food-101
    \end{minipage}
    \subfigure{
        \begin{minipage}{0.23\linewidth}
            \includegraphics[width=\linewidth]{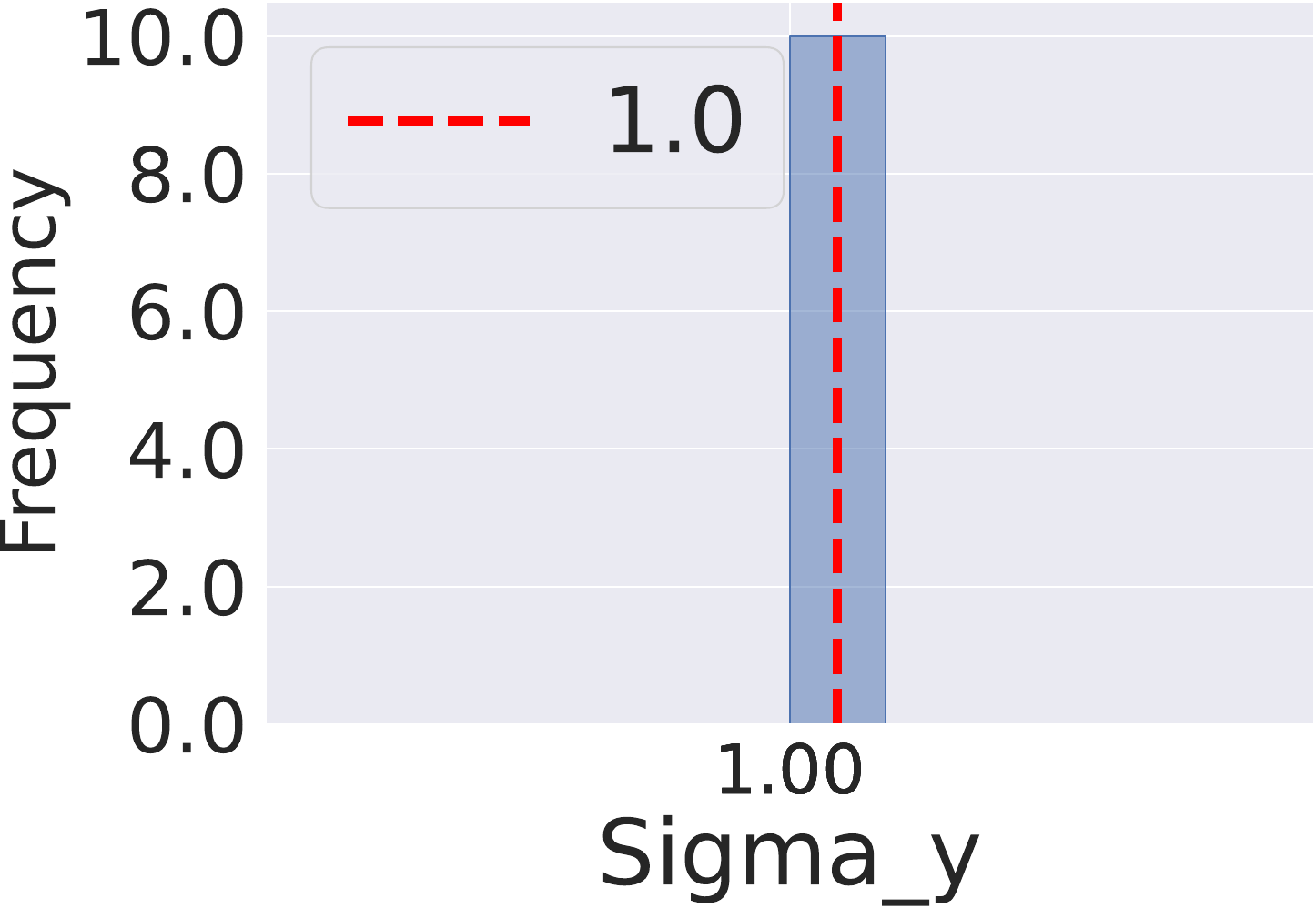}
        \end{minipage}
    }
    \subfigure{
        \begin{minipage}{0.23\linewidth}
            \includegraphics[width=\linewidth]{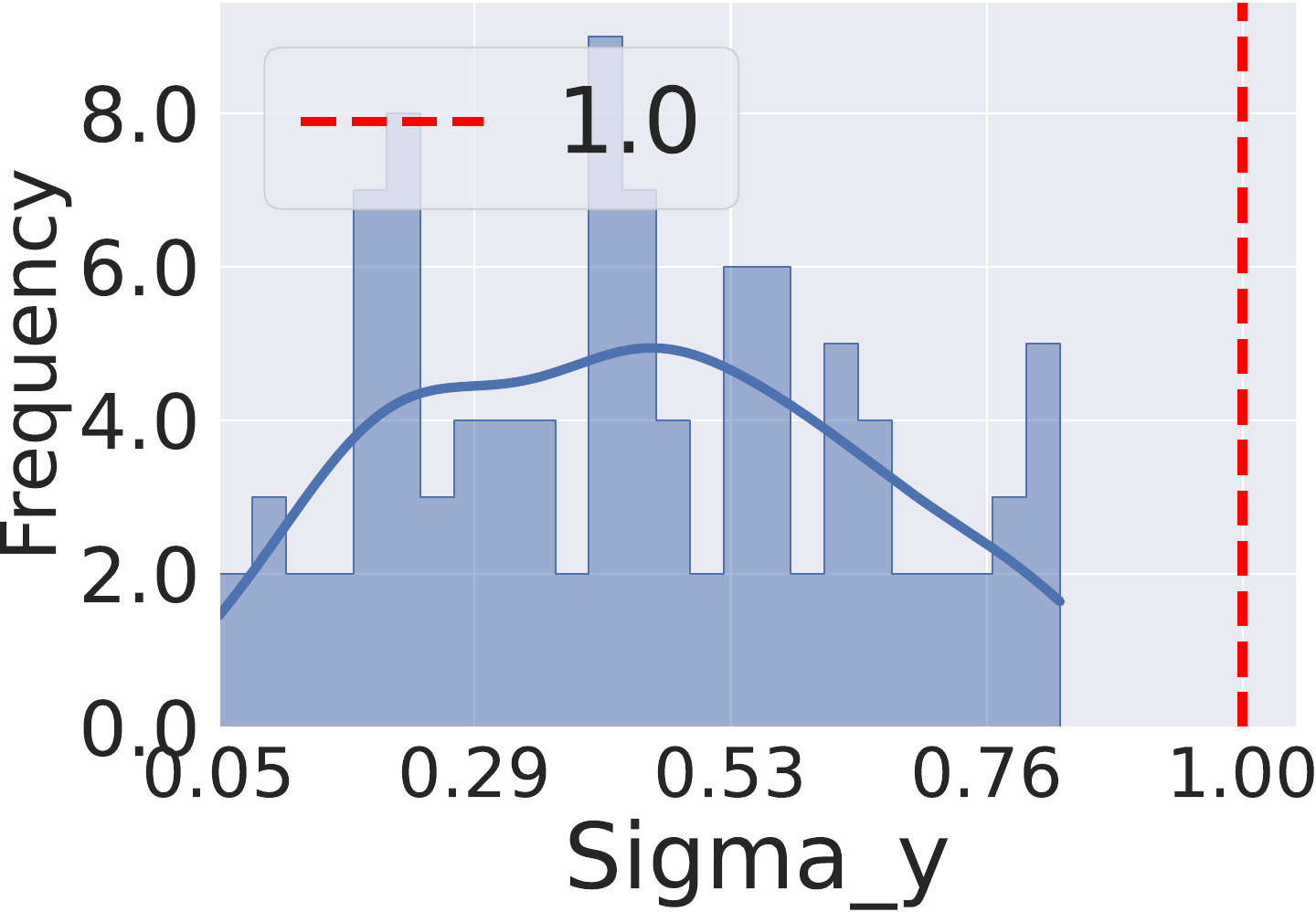}
        \end{minipage}
    }
    \subfigure{
        \begin{minipage}{0.23\linewidth}
            \includegraphics[width=\linewidth]{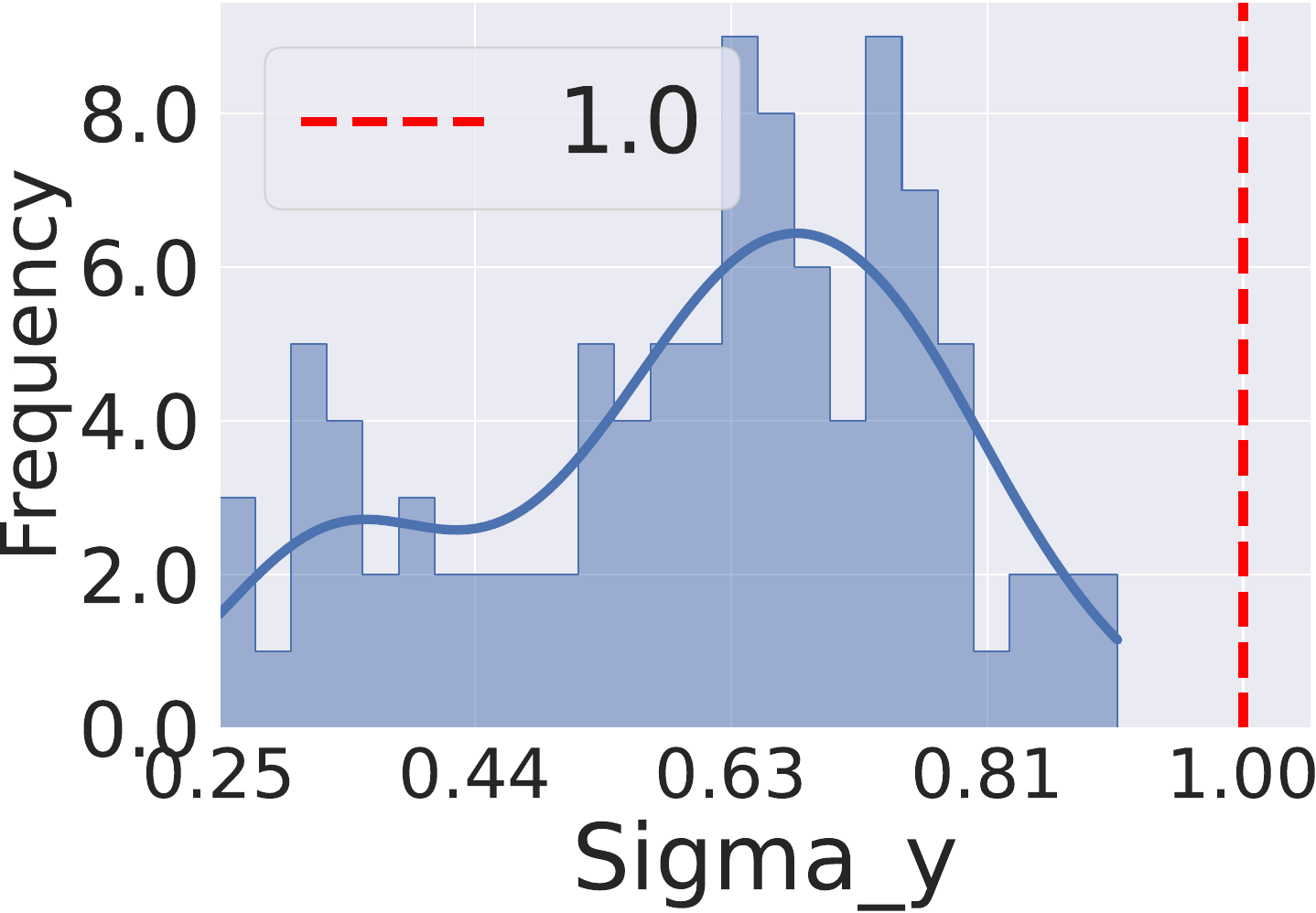}
        \end{minipage}
    }
    \subfigure{
        \begin{minipage}{0.23\linewidth}
            \includegraphics[width=\linewidth]{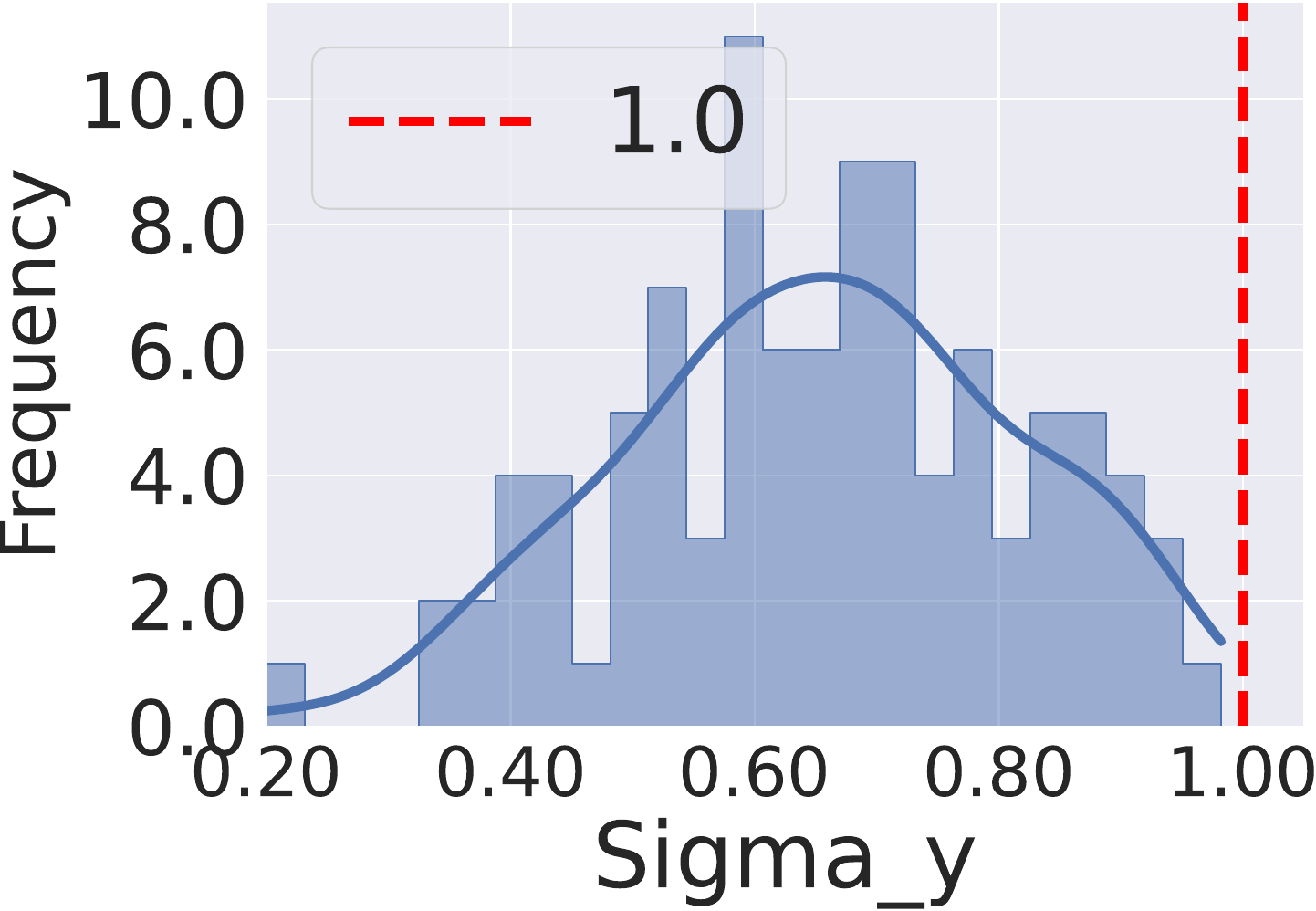}
        \end{minipage}
    }
    \caption{
    Verification of condition numbers $\{\sigma_y\}_{y=1}^C$ of Equation \ref{eq:sigma_y_defination} when epoch $=200$ and $\alpha = 0.1$ with $\rho=0.5$ \POLY.
    Vertical dashed lines represent the value $1$, and we observe that all the condition numbers are smaller than $1$.
    This verifies the validity of the condition for Lemma \ref{lemma:RC3P_improved_efficiency}, and thus confirms that \texttt{\newCP}produces smaller prediction sets than \texttt{CCP} by the optimized trade-off between calibration on non-conformity scores and calibrated label ranks.
    }
    \label{fig:condition_number_sigma_y_poly_0.5}
\end{figure}

\begin{figure}[!ht]
    \centering
    \begin{minipage}{.24\textwidth}
        \centering
        (a) CIFAR-10
    \end{minipage}%
    \begin{minipage}{.24\textwidth}
        \centering
        (b) CIFAR-100
    \end{minipage}%
    \begin{minipage}{.24\textwidth}
        \centering
        (c) mini-ImageNet
    \end{minipage}%
    \begin{minipage}{.24\textwidth}
        \centering
        (d) Food-101
    \end{minipage}
    \subfigure{
        \begin{minipage}{0.23\linewidth}
            \includegraphics[width=\linewidth]{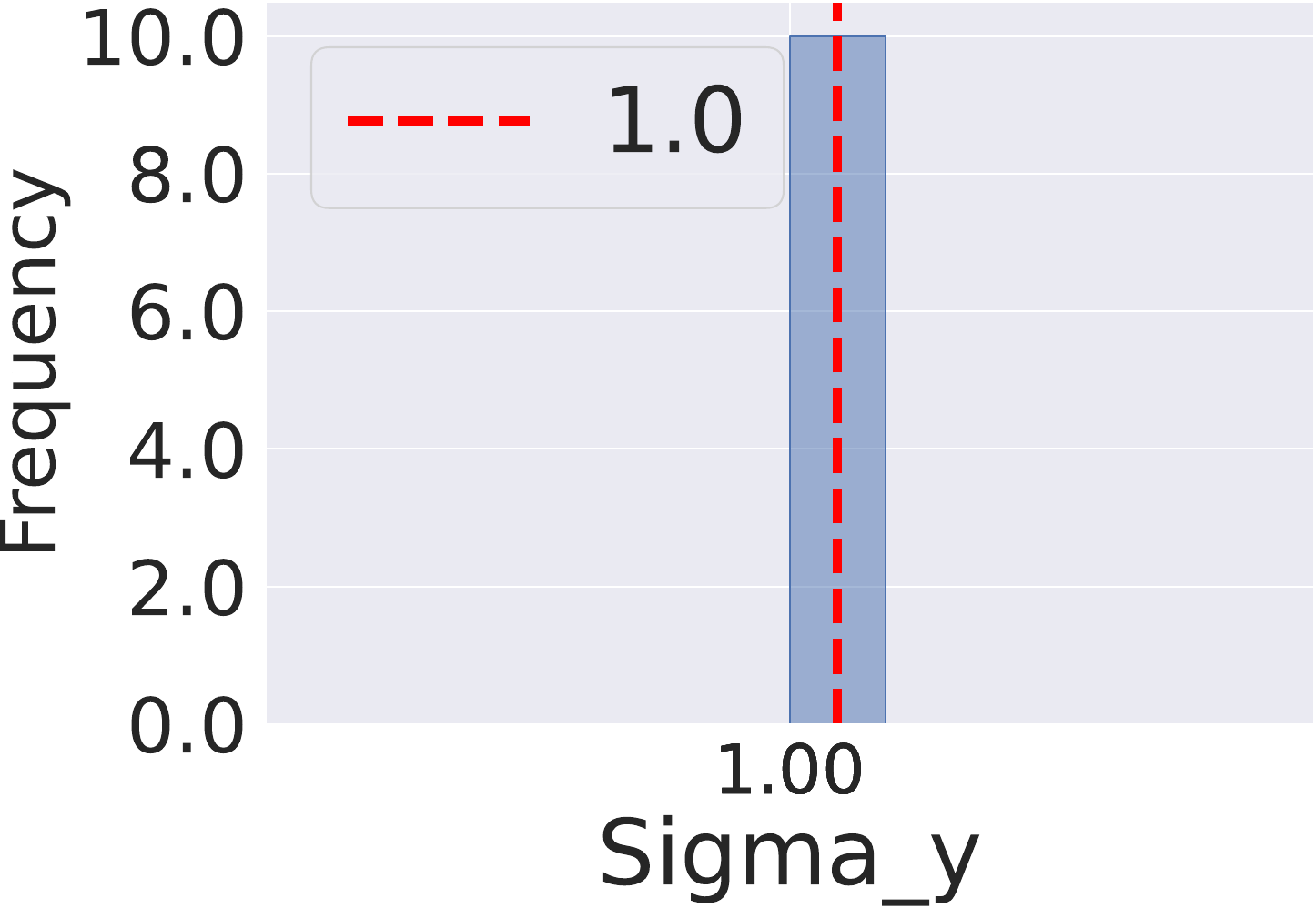}
        \end{minipage}
    }
    \subfigure{
        \begin{minipage}{0.23\linewidth}
            \includegraphics[width=\linewidth]{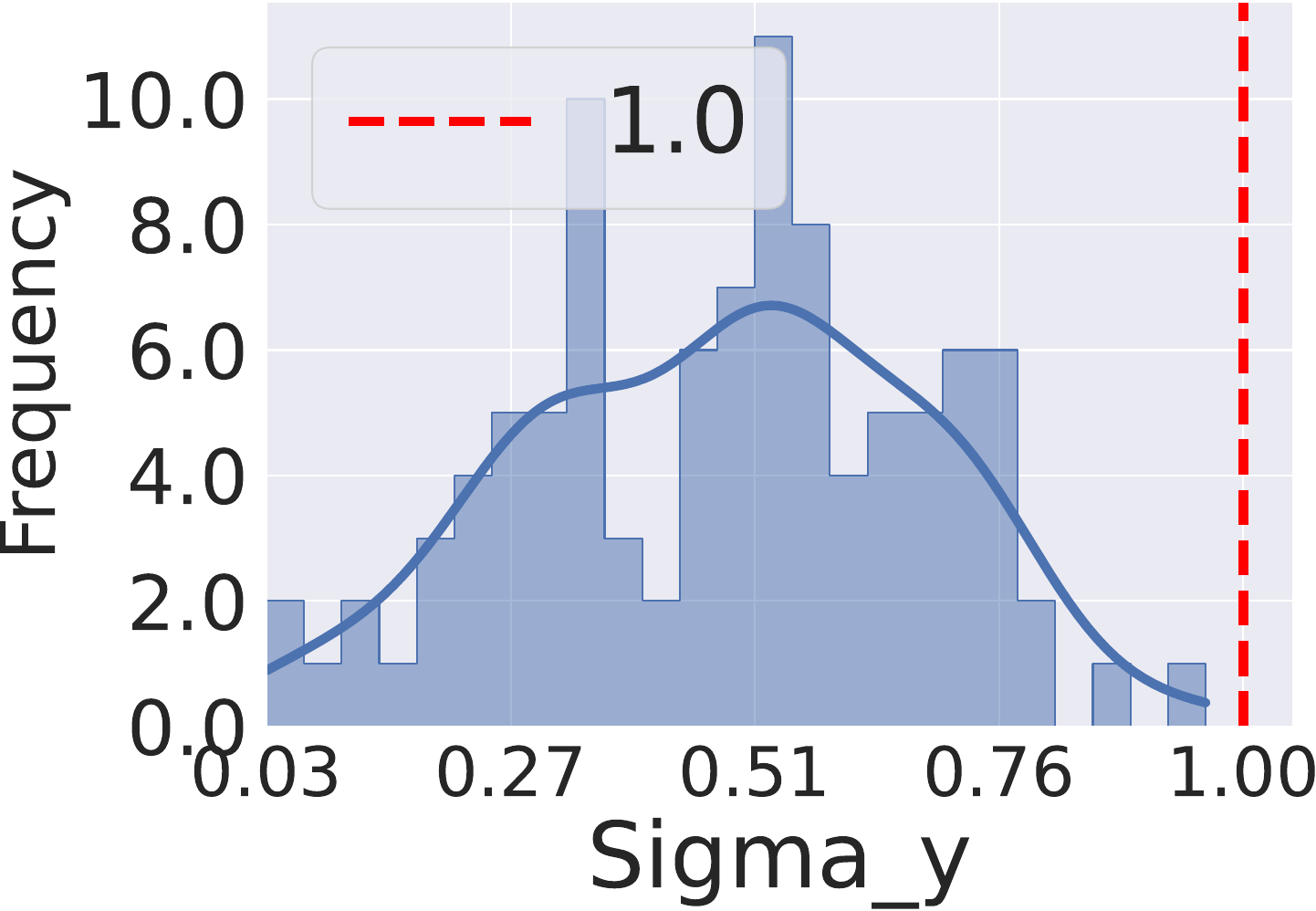}
        \end{minipage}
    }
    \subfigure{
        \begin{minipage}{0.23\linewidth}
            \includegraphics[width=\linewidth]{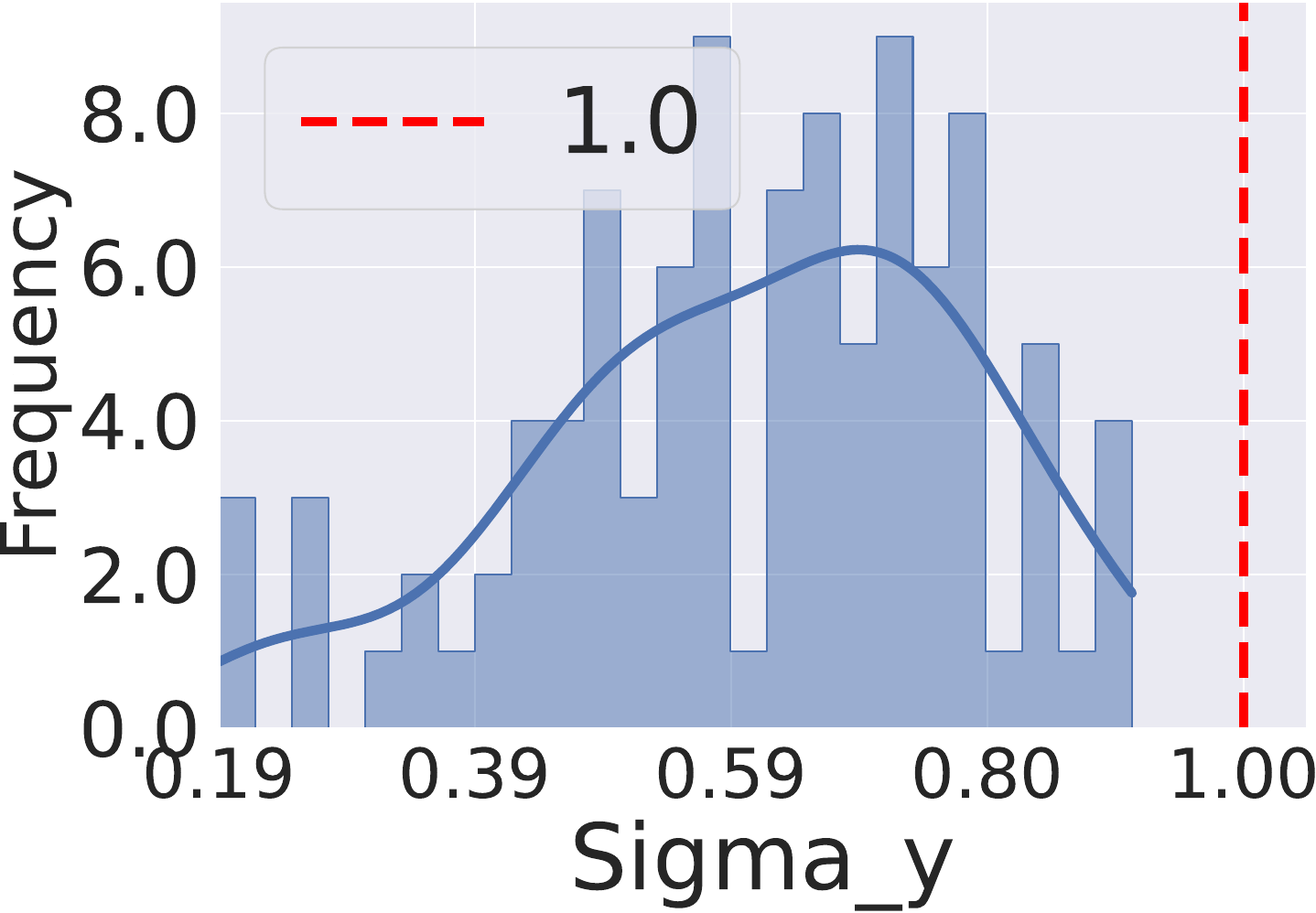}
        \end{minipage}
    }
    \subfigure{
        \begin{minipage}{0.23\linewidth}
            \includegraphics[width=\linewidth]{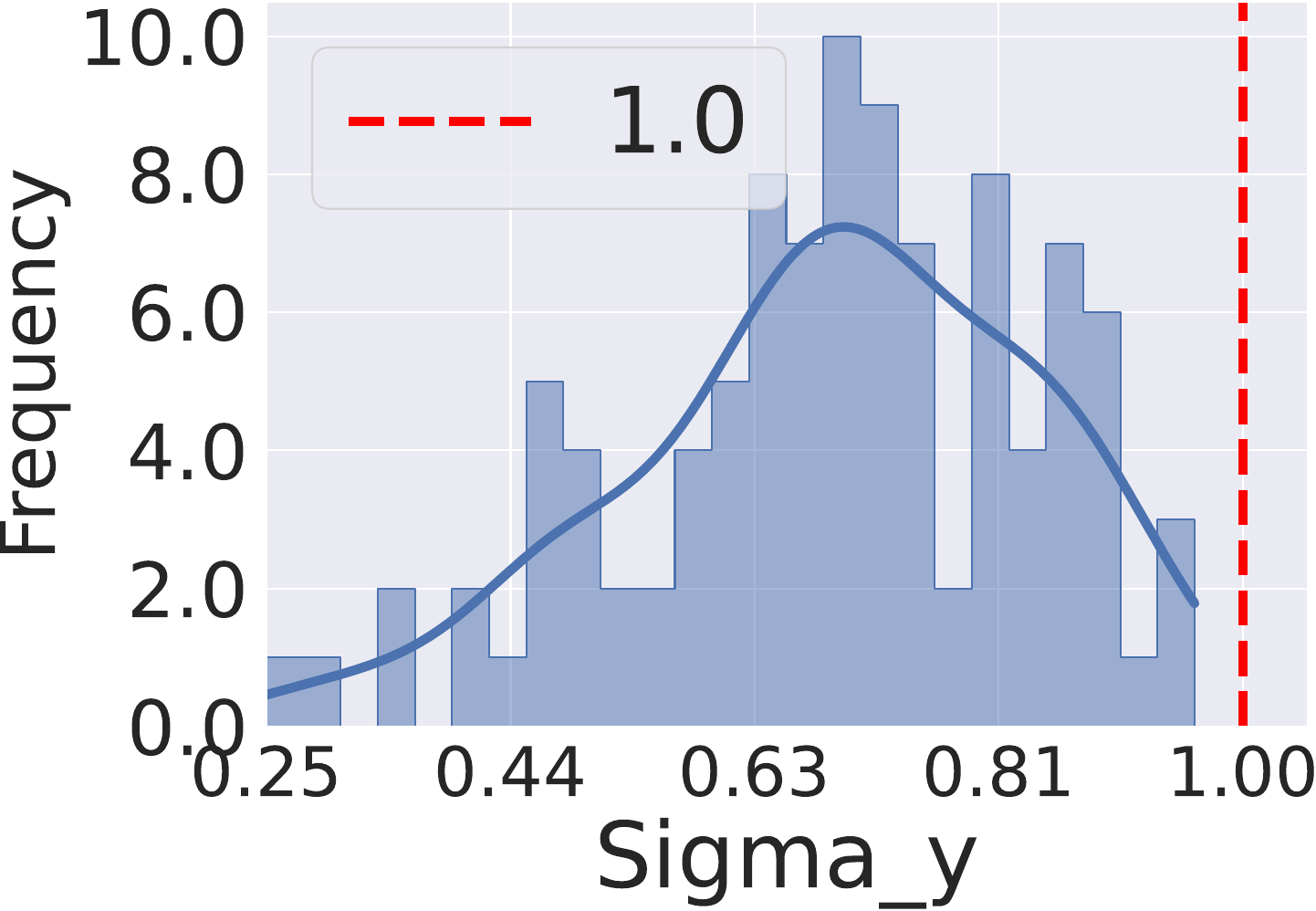}
        \end{minipage}
    }
    \caption{
    Verification of condition numbers $\{\sigma_y\}_{y=1}^C$ of Equation \ref{eq:sigma_y_defination} when epoch $=200$ and $\alpha = 0.1$ with $\rho=0.1$ \MAJ.
    Vertical dashed lines represent the value $1$, and we observe that all the condition numbers are smaller than $1$.
    This verifies the validity of the condition for Lemma \ref{lemma:RC3P_improved_efficiency}, and thus confirms that \texttt{\newCP}produces smaller prediction sets than \texttt{CCP} by the optimized trade-off between calibration on non-conformity scores and calibrated label ranks.
    }
    \label{fig:condition_number_sigma_y_maj_0.1}
\end{figure}

\begin{figure}[!ht]
    \centering
    \begin{minipage}{.24\textwidth}
        \centering
        (a) CIFAR-10
    \end{minipage}%
    \begin{minipage}{.24\textwidth}
        \centering
        (b) CIFAR-100
    \end{minipage}%
    \begin{minipage}{.24\textwidth}
        \centering
        (c) mini-ImageNet
    \end{minipage}%
    \begin{minipage}{.24\textwidth}
        \centering
        (d) Food-101
    \end{minipage}
    \subfigure{
        \begin{minipage}{0.23\linewidth}
            \includegraphics[width=\linewidth]{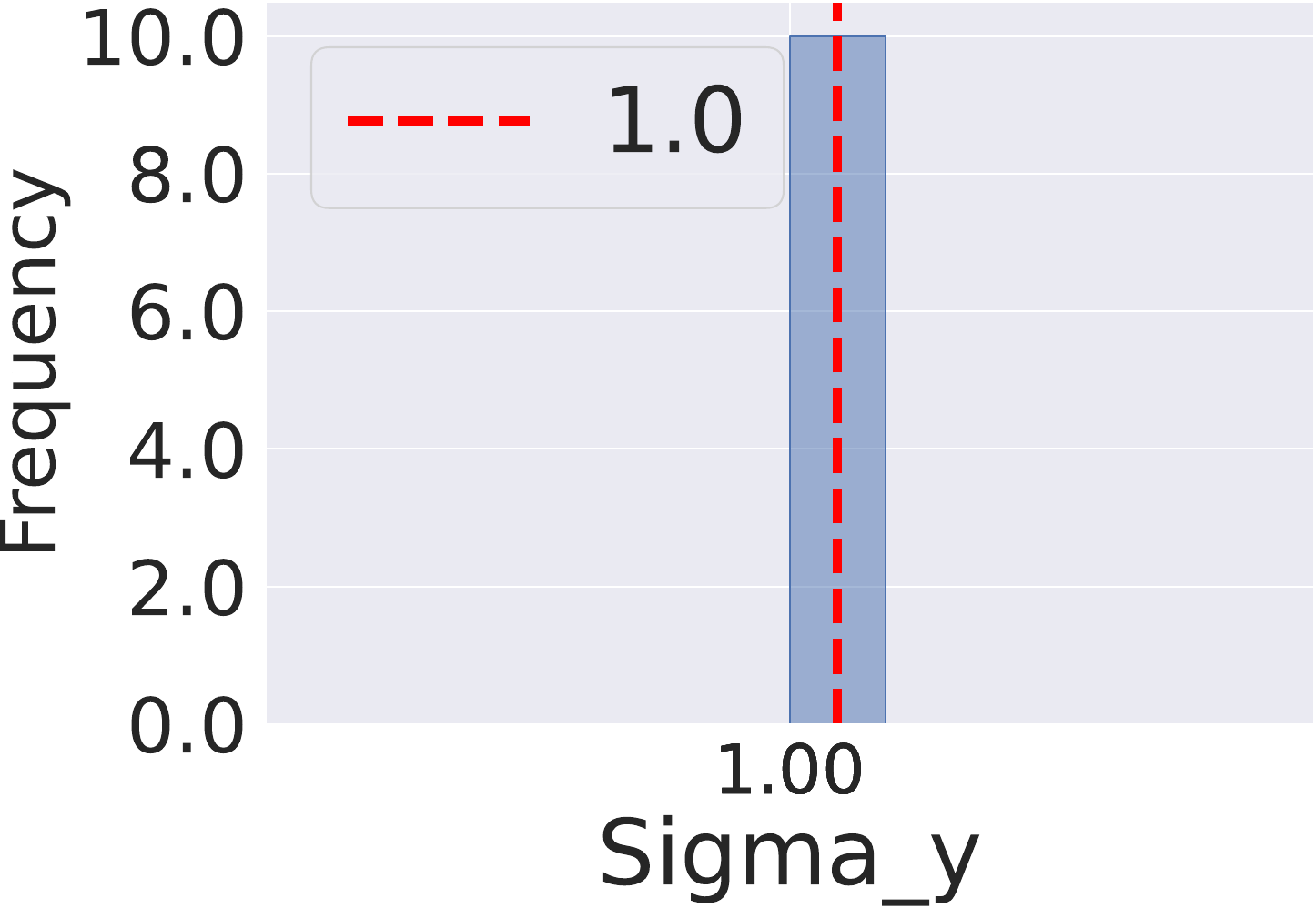}
        \end{minipage}
    }
    \subfigure{
        \begin{minipage}{0.23\linewidth}
            \includegraphics[width=\linewidth]{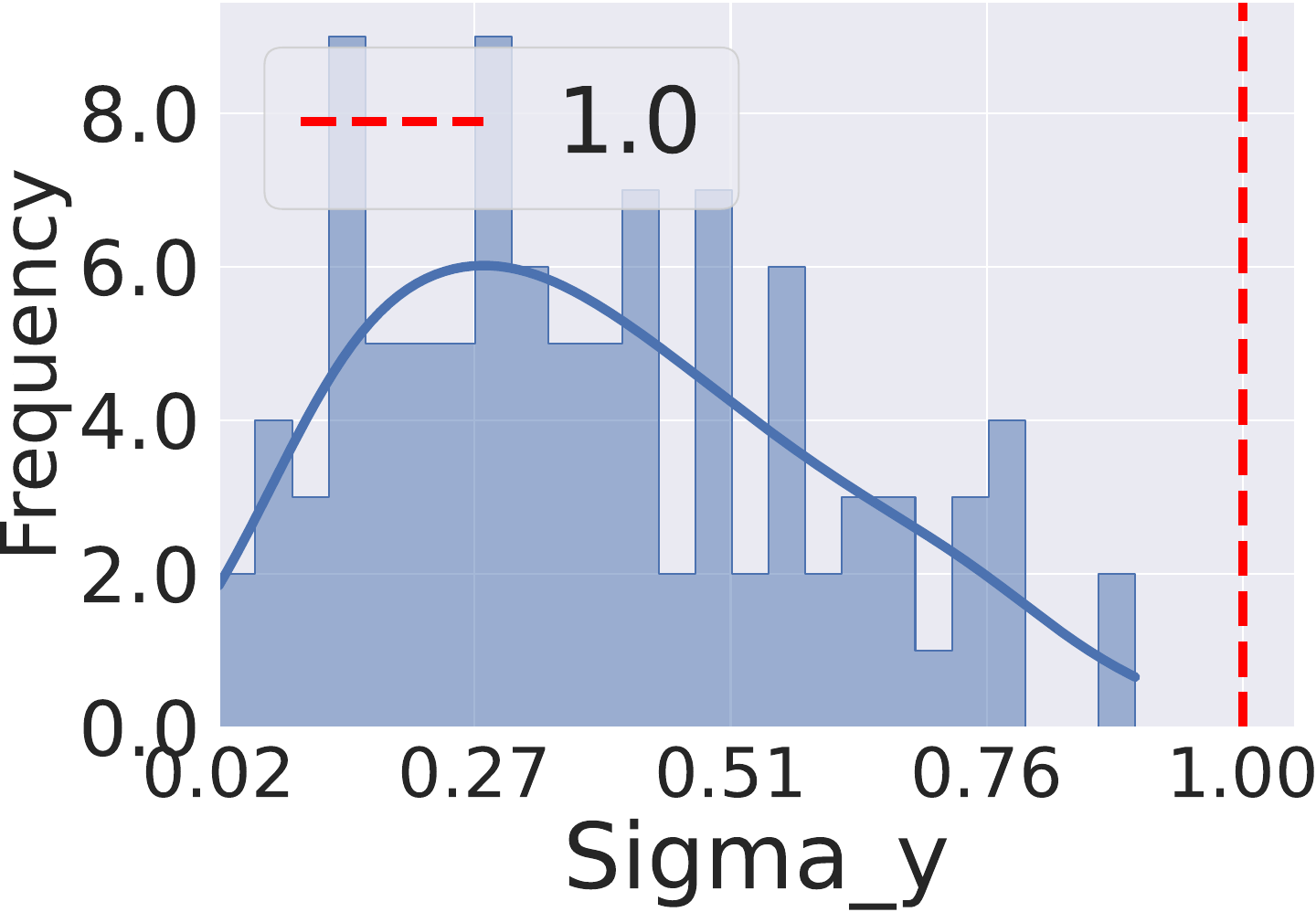}
        \end{minipage}
    }
    \subfigure{
        \begin{minipage}{0.23\linewidth}
            \includegraphics[width=\linewidth]{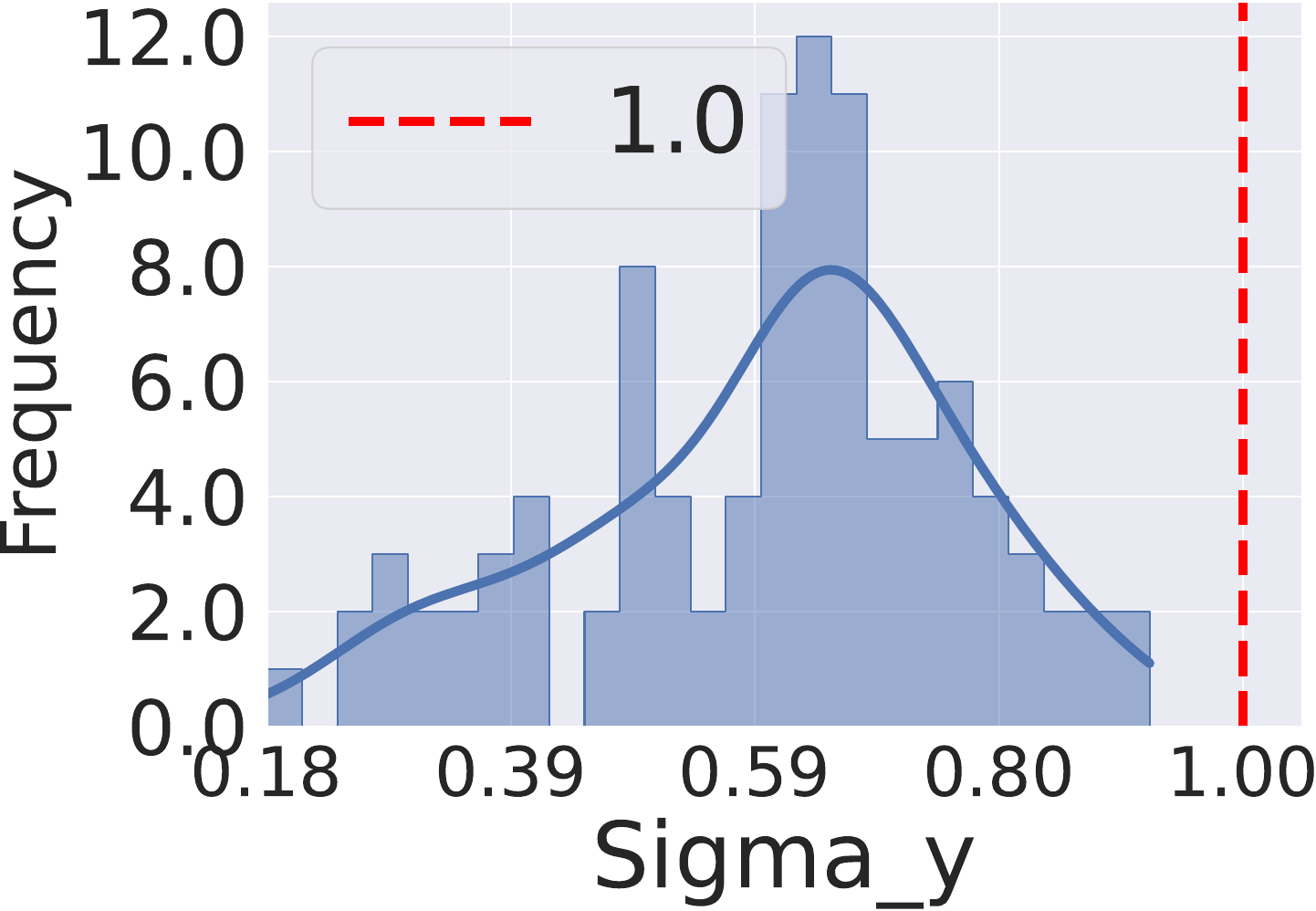}
        \end{minipage}
    }
    \subfigure{
        \begin{minipage}{0.23\linewidth}
            \includegraphics[width=\linewidth]{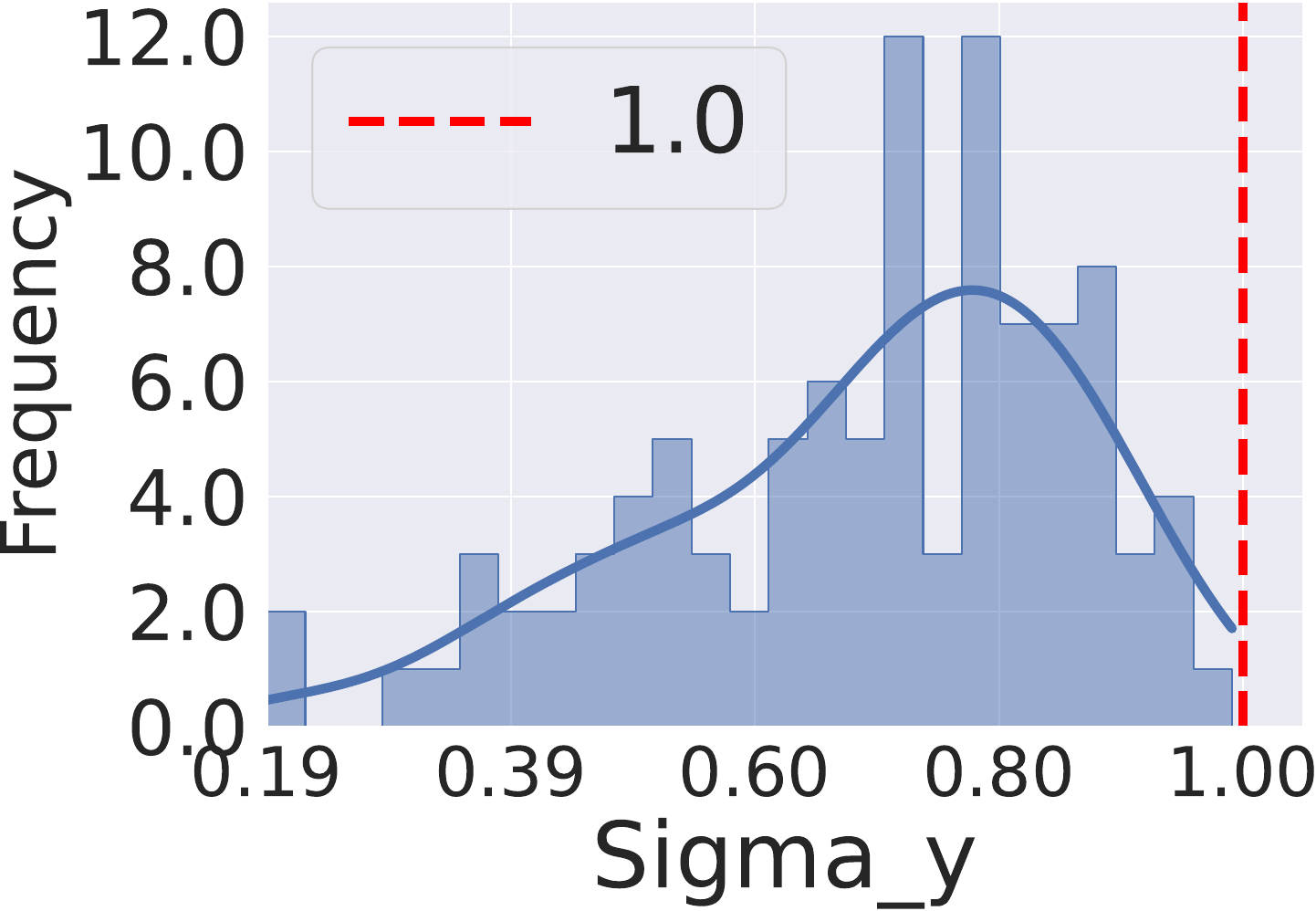}
        \end{minipage}
    }
    \caption{
    Verification of condition numbers $\{\sigma_y\}_{y=1}^C$ of Equation \ref{eq:sigma_y_defination} when epoch $=200$ and $\alpha = 0.1$ with $\rho=0.5$ \MAJ.
    Vertical dashed lines represent the value $1$, and we observe that all the condition numbers are smaller than $1$.
    This verifies the validity of the condition for Lemma \ref{lemma:RC3P_improved_efficiency}, and thus confirms that \texttt{\newCP}produces smaller prediction sets than \texttt{CCP} by the optimized trade-off between calibration on non-conformity scores and calibrated label ranks.
    }
    \label{fig:condition_number_sigma_y_maj_0.5}
\end{figure}

\begin{figure}[!ht]
    \centering
    \begin{minipage}{.24\textwidth}
        \centering
        (a) CIFAR-10
    \end{minipage}%
    \begin{minipage}{.24\textwidth}
        \centering
        (b) CIFAR-100
    \end{minipage}%
    \begin{minipage}{.24\textwidth}
        \centering
        (c) mini-ImageNet
    \end{minipage}%
    \begin{minipage}{.24\textwidth}
        \centering
        (d) Food-101
    \end{minipage}
    \subfigure{
        \begin{minipage}{0.23\linewidth}
            \includegraphics[width=\linewidth]{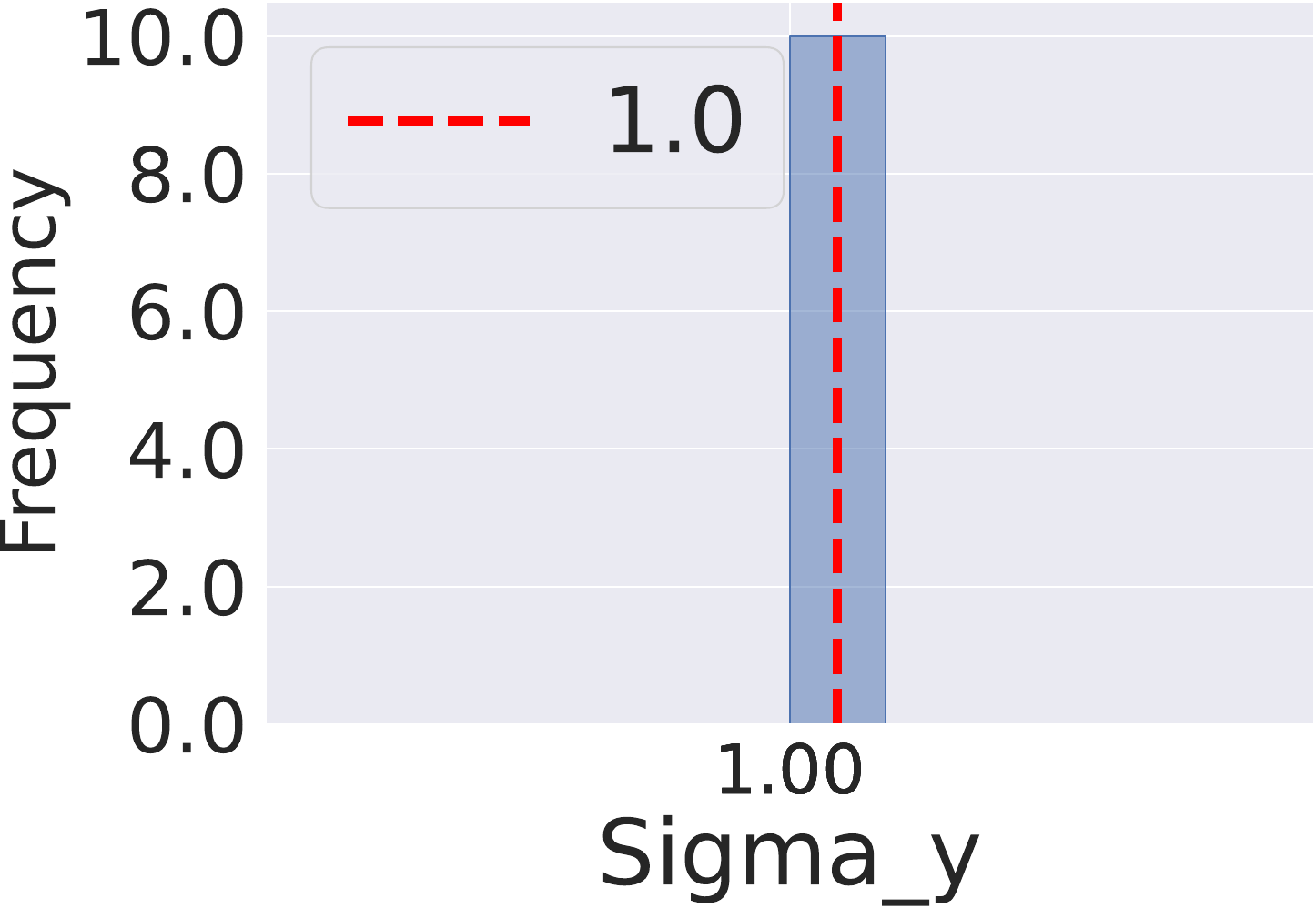}
        \end{minipage}
    }
    \subfigure{
        \begin{minipage}{0.23\linewidth}
            \includegraphics[width=\linewidth]{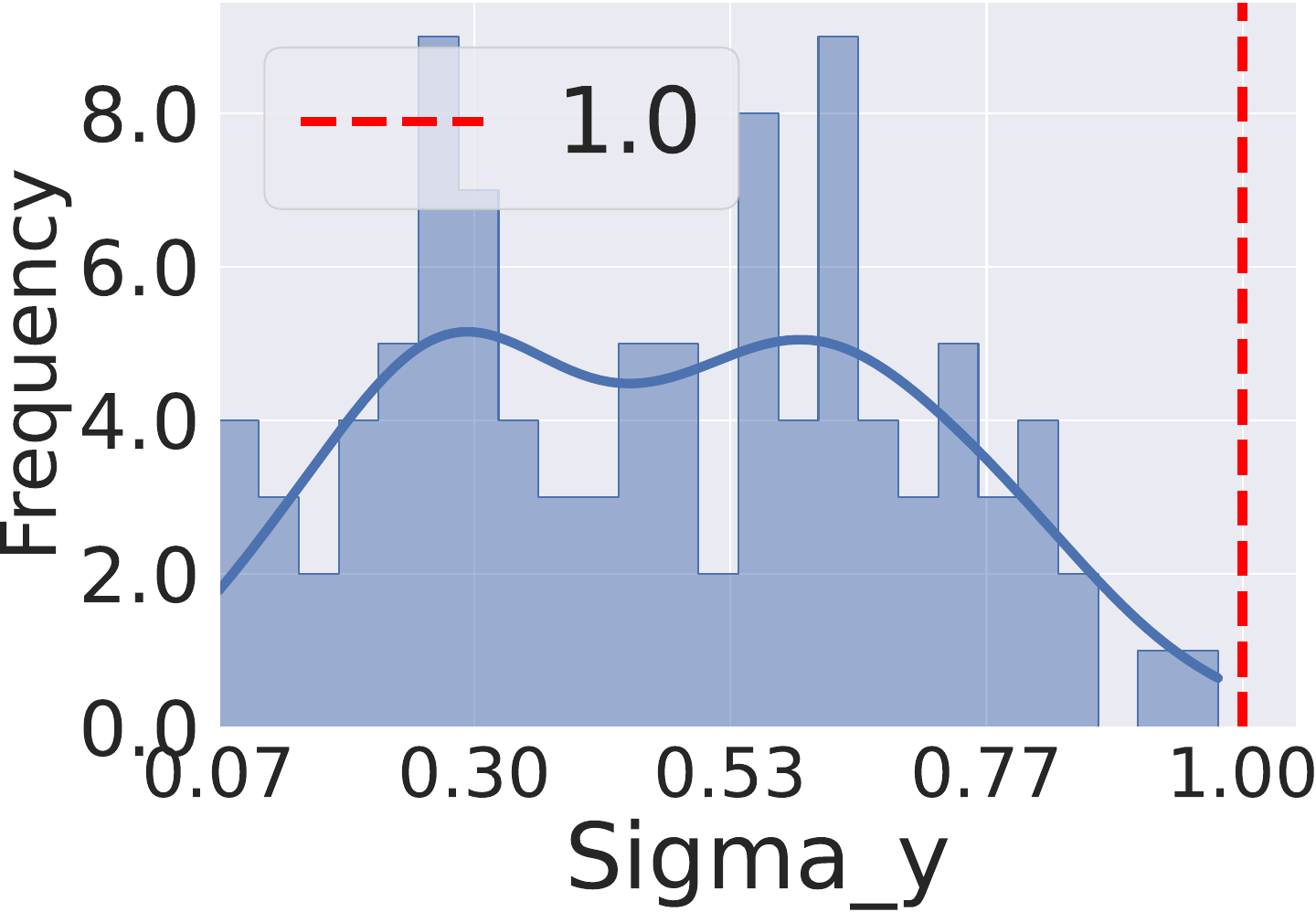}
        \end{minipage}
    }
    \subfigure{
        \begin{minipage}{0.23\linewidth}
            \includegraphics[width=\linewidth]{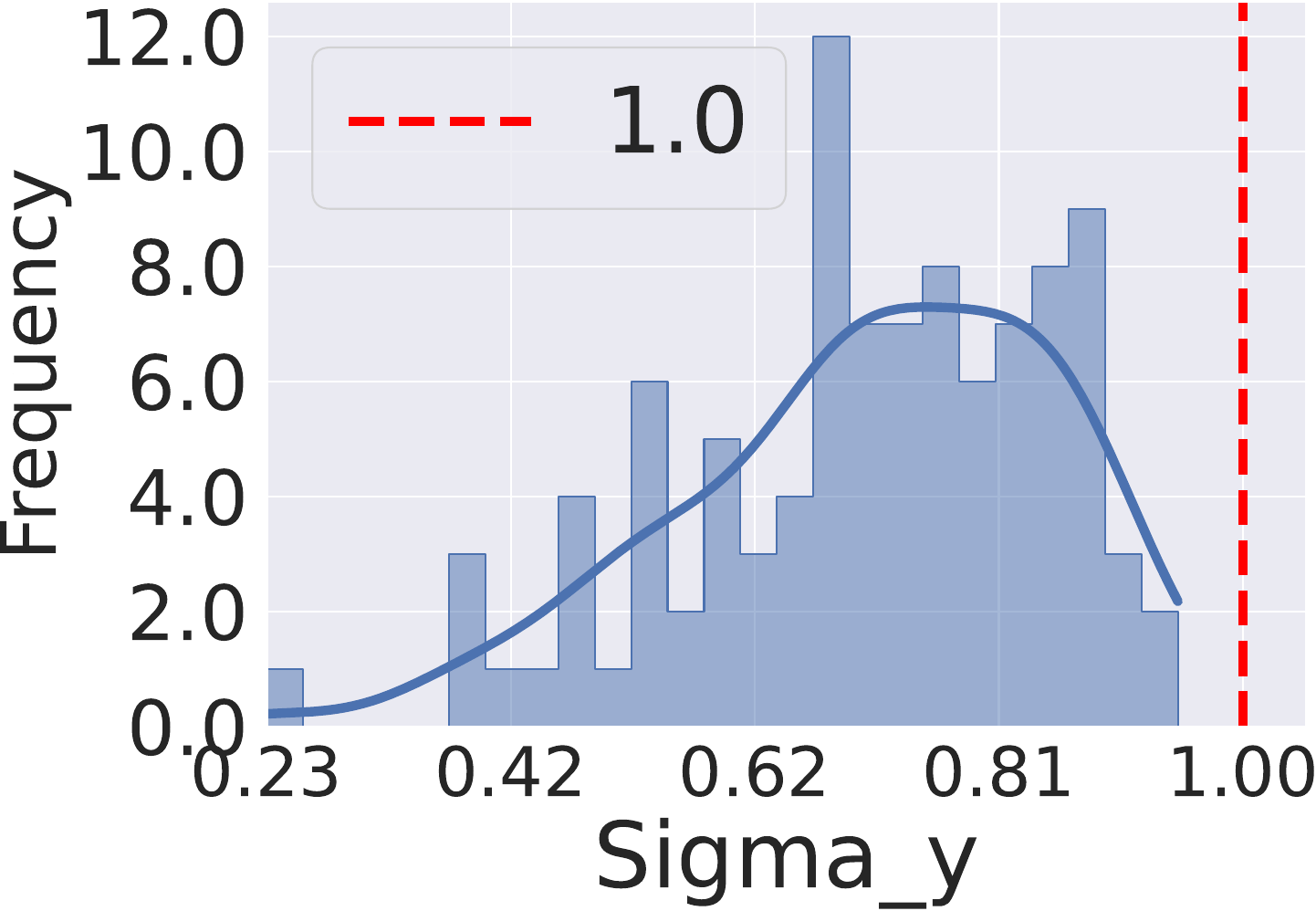}
        \end{minipage}
    }
    \subfigure{
        \begin{minipage}{0.23\linewidth}
            \includegraphics[width=\linewidth]{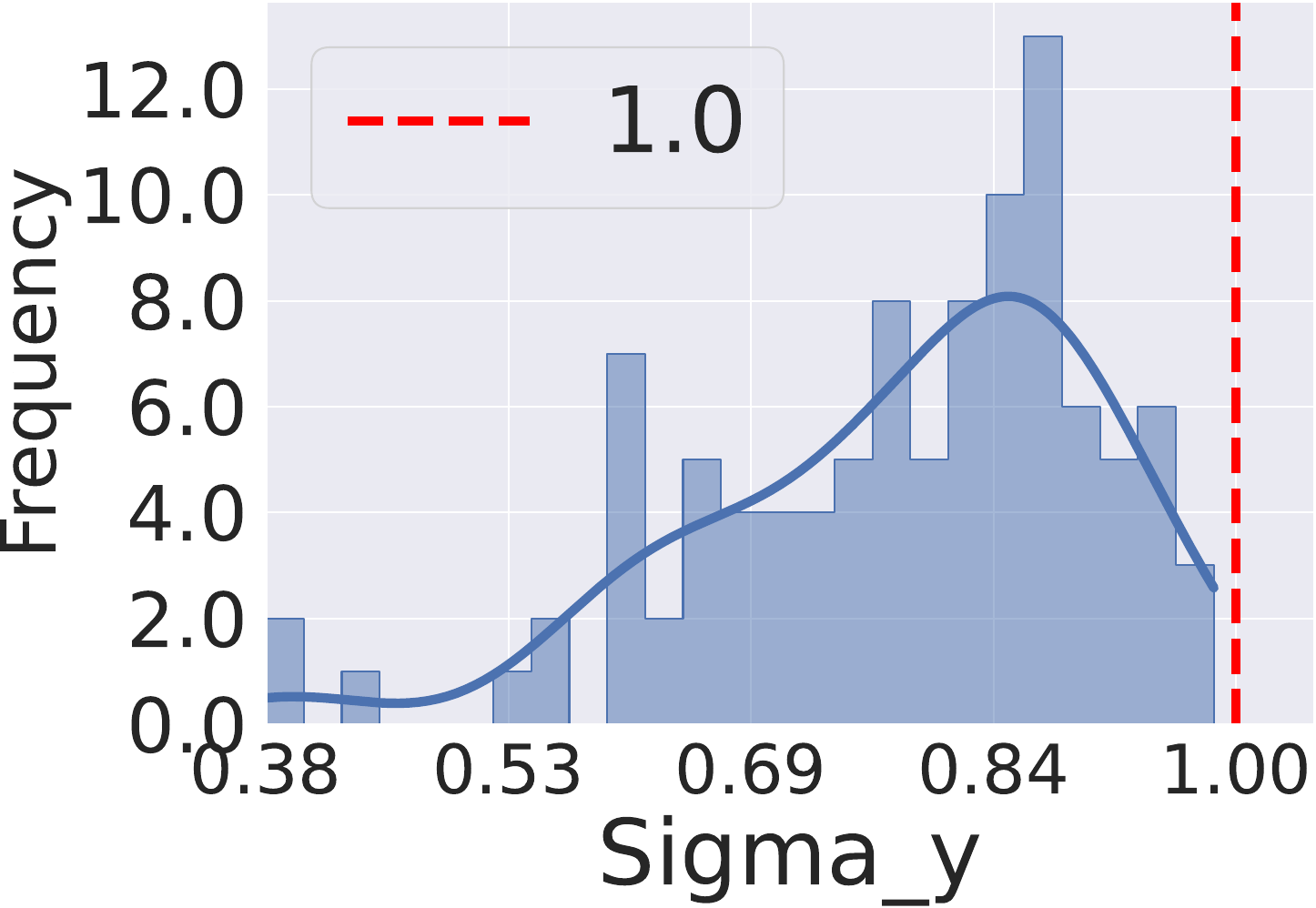}
        \end{minipage}
    }
    \caption{
    Verification of condition numbers $\{\sigma_y\}_{y=1}^C$ of Equation \ref{eq:sigma_y_defination} when epoch $=50$ and $\alpha = 0.1$ with $\rho=0.1$ \EXP.
    Vertical dashed lines represent the value $1$, and we observe that all the condition numbers are smaller than $1$.
    This verifies the validity of the condition for Lemma \ref{lemma:RC3P_improved_efficiency}, and thus confirms that \texttt{\newCP}produces smaller prediction sets than \texttt{CCP} by the optimized trade-off between calibration on non-conformity scores and calibrated label ranks.
    }
    \label{fig:condition_number_sigma_y_exp_0.1_50}
\end{figure}

\begin{figure}[!ht]
    \centering
    \begin{minipage}{.24\textwidth}
        \centering
        (a) CIFAR-10
    \end{minipage}%
    \begin{minipage}{.24\textwidth}
        \centering
        (b) CIFAR-100
    \end{minipage}%
    \begin{minipage}{.24\textwidth}
        \centering
        (c) mini-ImageNet
    \end{minipage}%
    \begin{minipage}{.24\textwidth}
        \centering
        (d) Food-101
    \end{minipage}
    \subfigure{
        \begin{minipage}{0.23\linewidth}
            \includegraphics[width=\linewidth]{Neurips_figure/cifar10_exp_0.5_SigHist_seed_9_tgap_0.5_cgap_0.5_cal.pdf}
        \end{minipage}
    }
    \subfigure{
        \begin{minipage}{0.23\linewidth}
            \includegraphics[width=\linewidth]{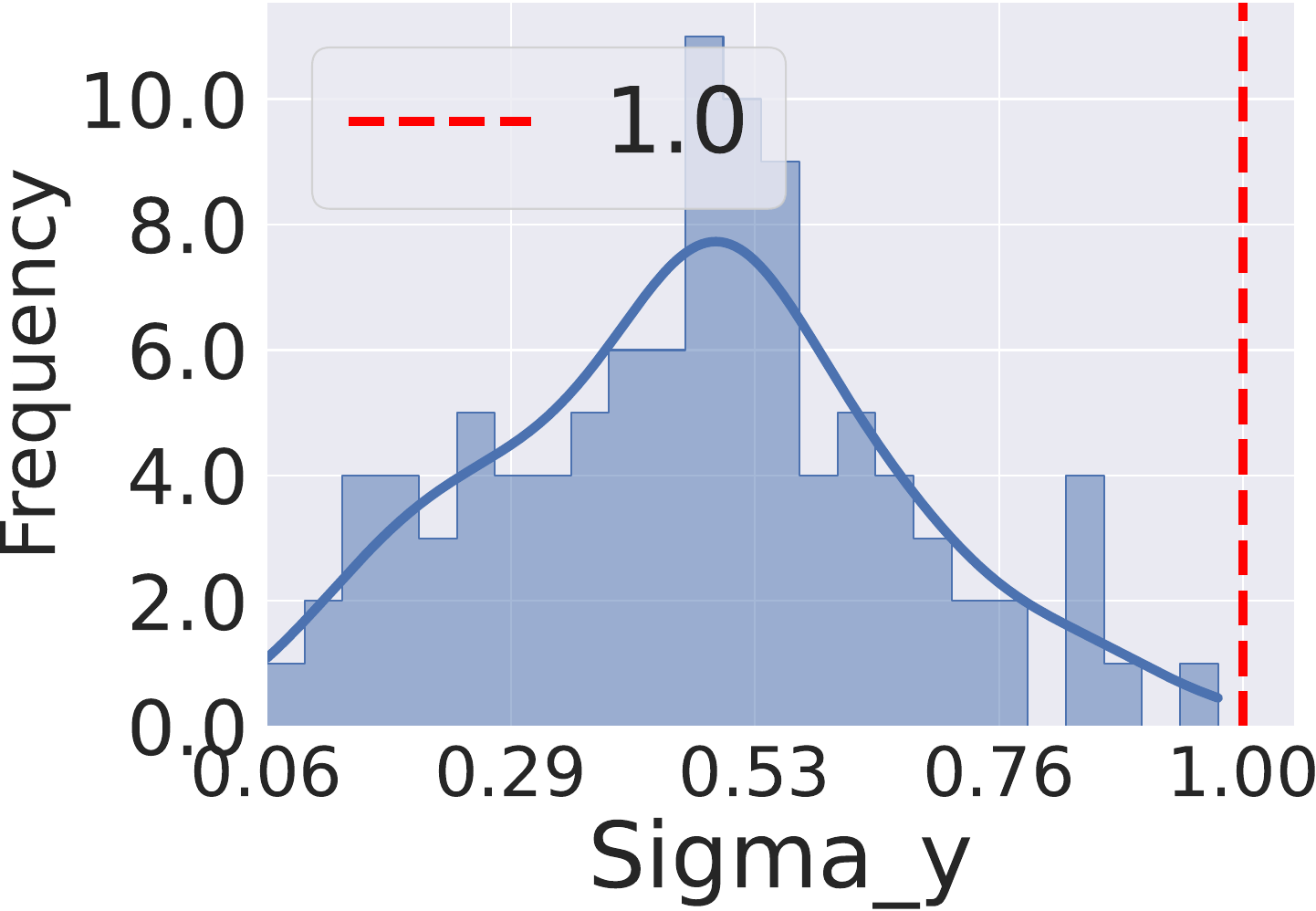}
        \end{minipage}
    }
    \subfigure{
        \begin{minipage}{0.23\linewidth}
            \includegraphics[width=\linewidth]{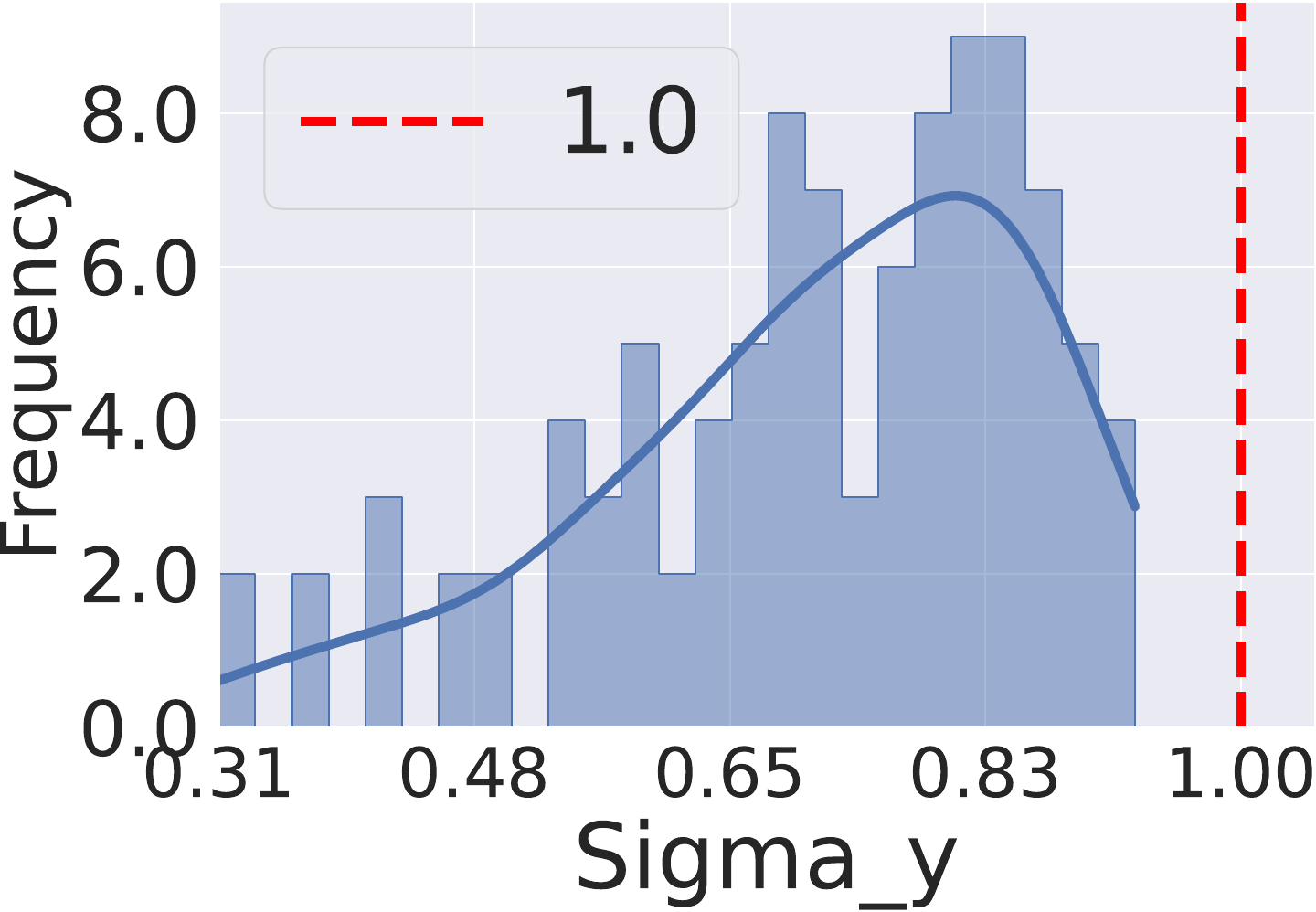}
        \end{minipage}
    }
    \subfigure{
        \begin{minipage}{0.23\linewidth}
            \includegraphics[width=\linewidth]{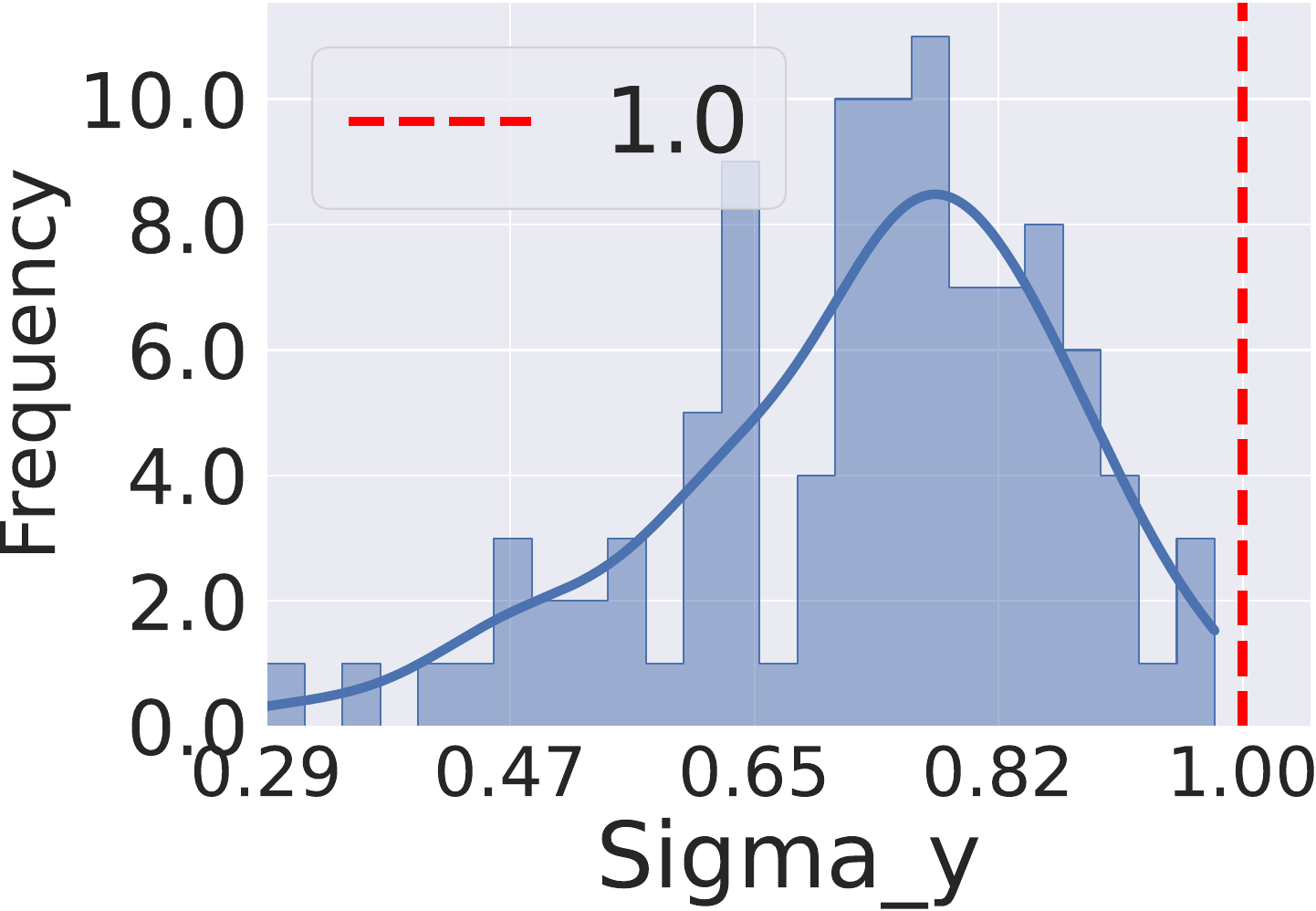}
        \end{minipage}
    }
    \caption{
    Verification of condition numbers $\{\sigma_y\}_{y=1}^C$ of Equation \ref{eq:sigma_y_defination} when epoch $=50$ and $\alpha = 0.1$ with $\rho=0.5$ \EXP.
    Vertical dashed lines represent the value $1$, and we observe that all the condition numbers are smaller than $1$.
    This verifies the validity of the condition for Lemma \ref{lemma:RC3P_improved_efficiency}, and thus confirms that \texttt{\newCP}produces smaller prediction sets than \texttt{CCP} by the optimized trade-off between calibration on non-conformity scores and calibrated label ranks.
    }
    \label{fig:condition_number_sigma_y_exp_0.5_50}
\end{figure}

\begin{figure}[!ht]
    \centering
    \begin{minipage}{.24\textwidth}
        \centering
        (a) CIFAR-10
    \end{minipage}%
    \begin{minipage}{.24\textwidth}
        \centering
        (b) CIFAR-100
    \end{minipage}%
    \begin{minipage}{.24\textwidth}
        \centering
        (c) mini-ImageNet
    \end{minipage}%
    \begin{minipage}{.24\textwidth}
        \centering
        (d) Food-101
    \end{minipage}
    \subfigure{
        \begin{minipage}{0.23\linewidth}
            \includegraphics[width=\linewidth]{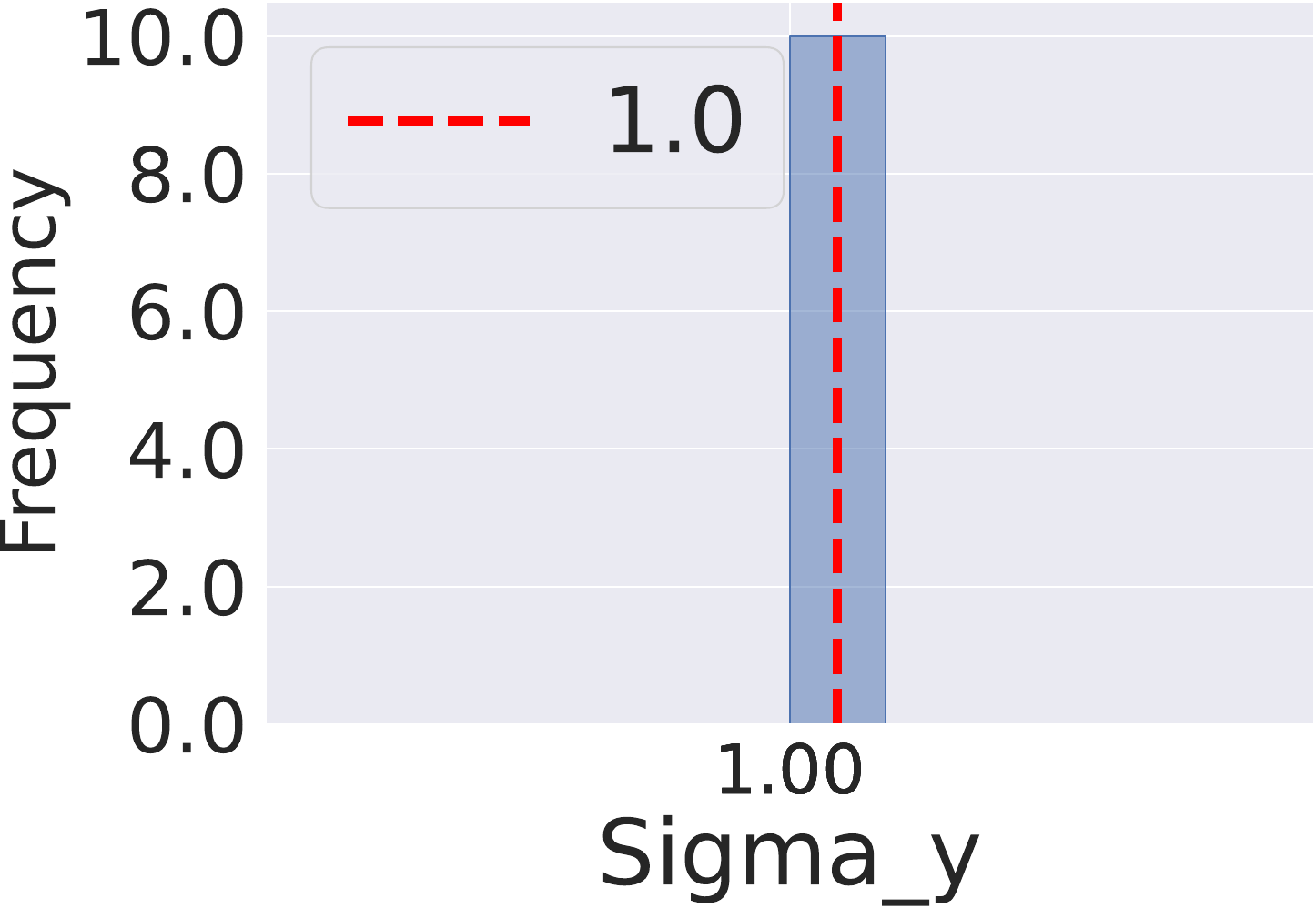}
        \end{minipage}
    }
    \subfigure{
        \begin{minipage}{0.23\linewidth}
            \includegraphics[width=\linewidth]{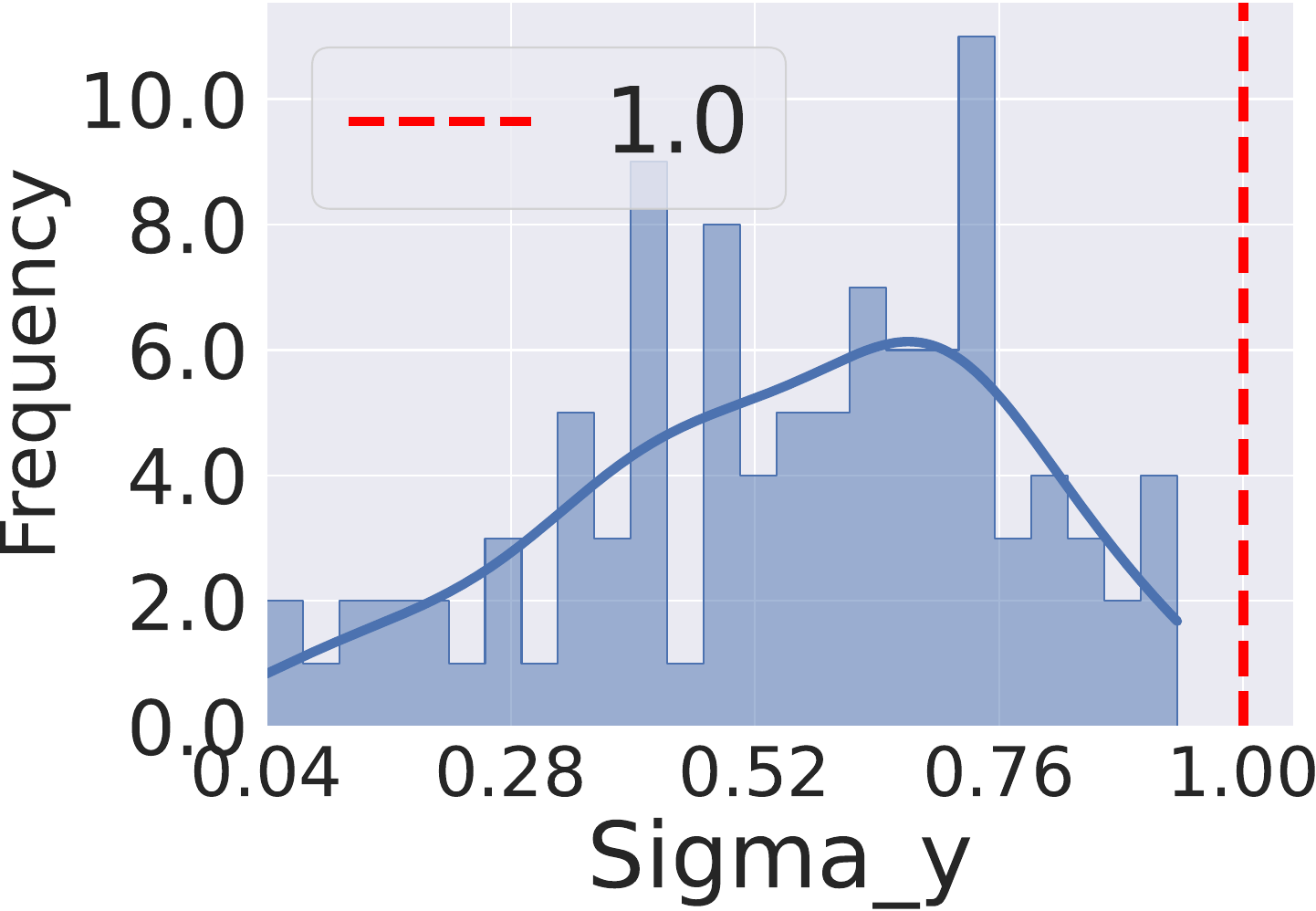}
        \end{minipage}
    }
    \subfigure{
        \begin{minipage}{0.23\linewidth}
            \includegraphics[width=\linewidth]{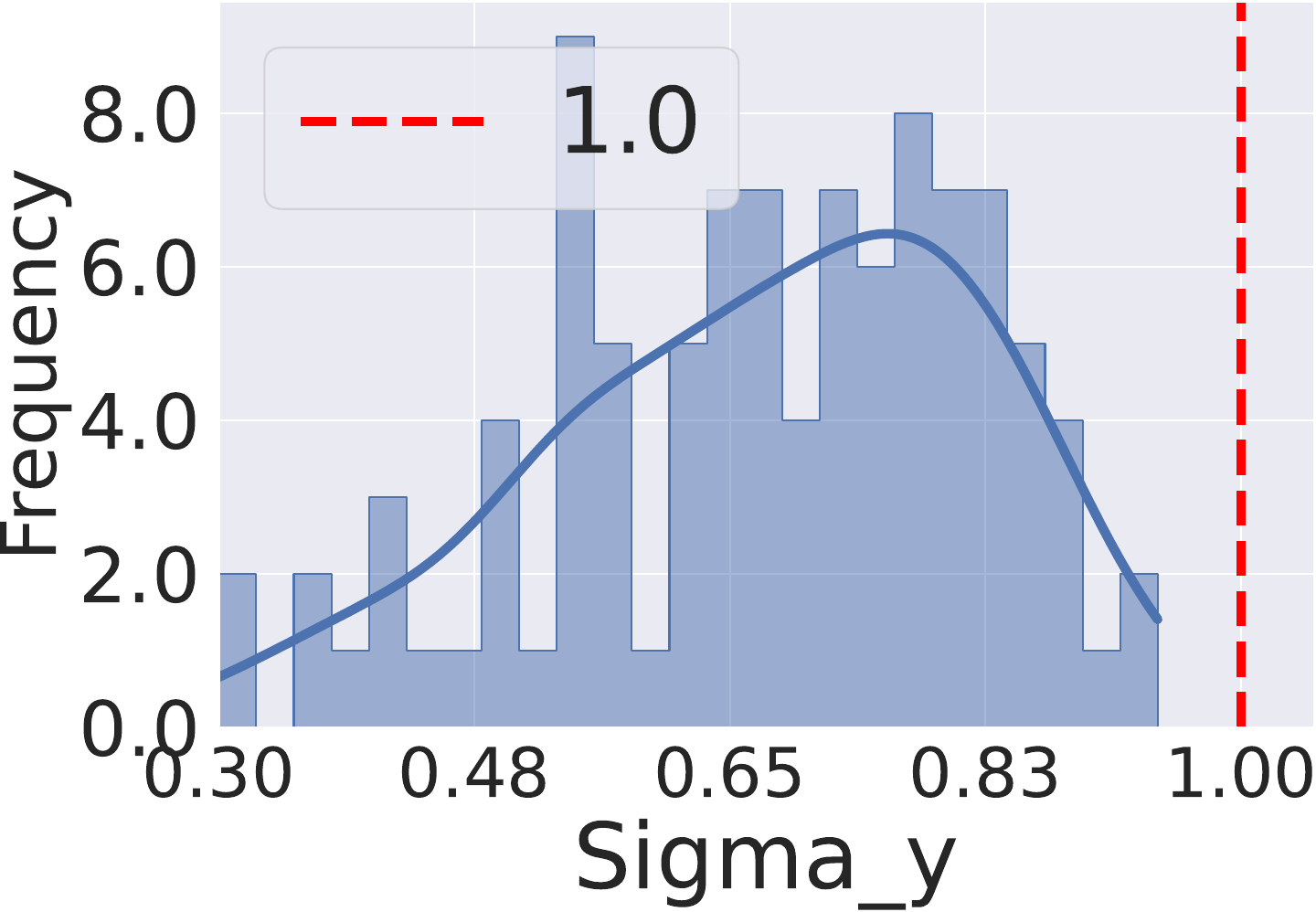}
        \end{minipage}
    }
    \subfigure{
        \begin{minipage}{0.23\linewidth}
            \includegraphics[width=\linewidth]{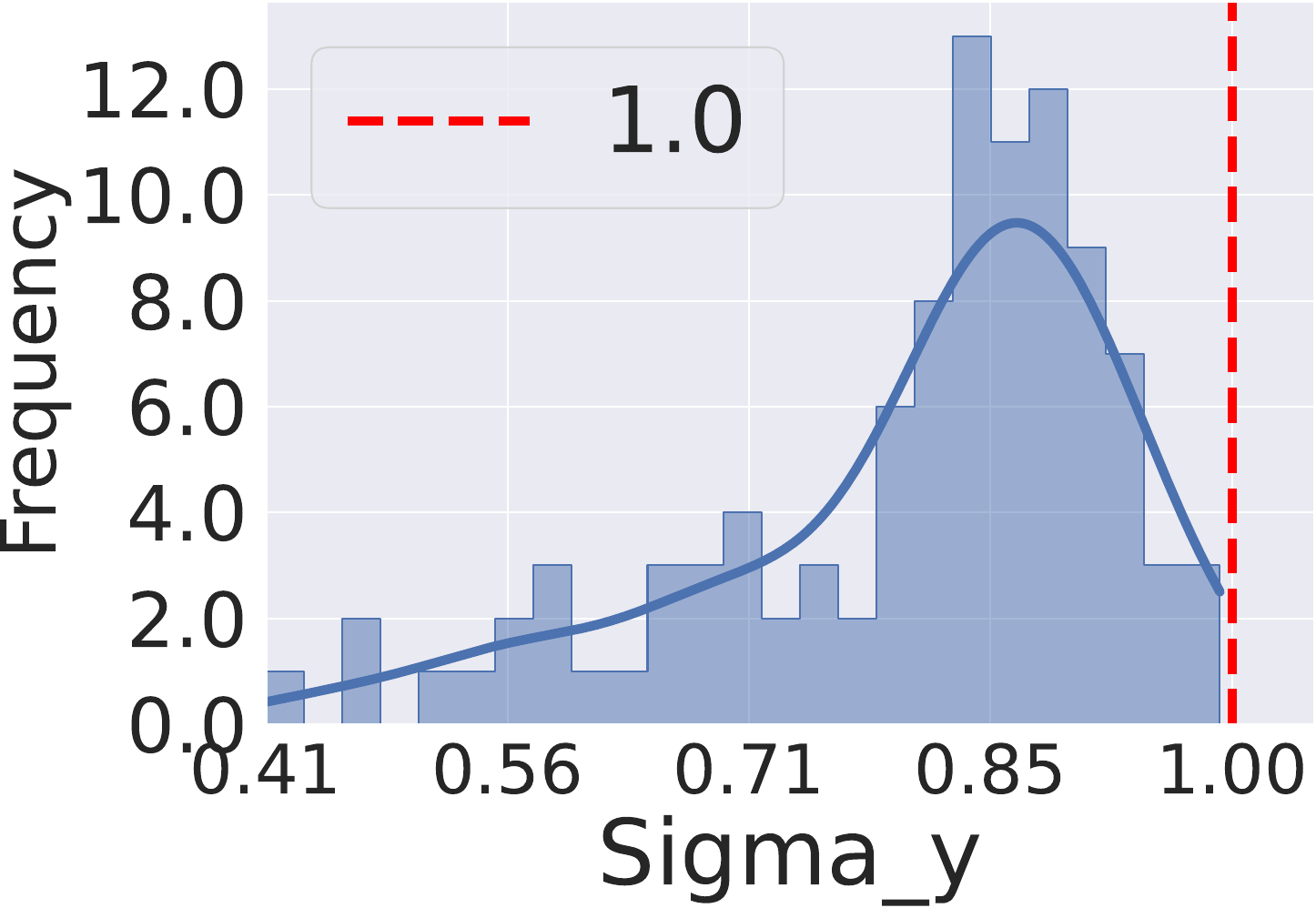}
        \end{minipage}
    }
    \caption{
    Verification of condition numbers $\{\sigma_y\}_{y=1}^C$ of Equation \ref{eq:sigma_y_defination} when epoch $=50$ and $\alpha = 0.1$ with $\rho=0.1$ \POLY.
    Vertical dashed lines represent the value $1$, and we observe that all the condition numbers are smaller than $1$.
    This verifies the validity of the condition for Lemma \ref{lemma:RC3P_improved_efficiency}, and thus confirms that \texttt{\newCP}produces smaller prediction sets than \texttt{CCP} by the optimized trade-off between calibration on non-conformity scores and calibrated label ranks.
    }
    \label{fig:condition_number_sigma_y_poly_0.1_50}
\end{figure}

\begin{figure}[!ht]
    \centering
    \begin{minipage}{.24\textwidth}
        \centering
        (a) CIFAR-10
    \end{minipage}%
    \begin{minipage}{.24\textwidth}
        \centering
        (b) CIFAR-100
    \end{minipage}%
    \begin{minipage}{.24\textwidth}
        \centering
        (c) mini-ImageNet
    \end{minipage}%
    \begin{minipage}{.24\textwidth}
        \centering
        (d) Food-101
    \end{minipage}
    \subfigure{
        \begin{minipage}{0.23\linewidth}
            \includegraphics[width=\linewidth]{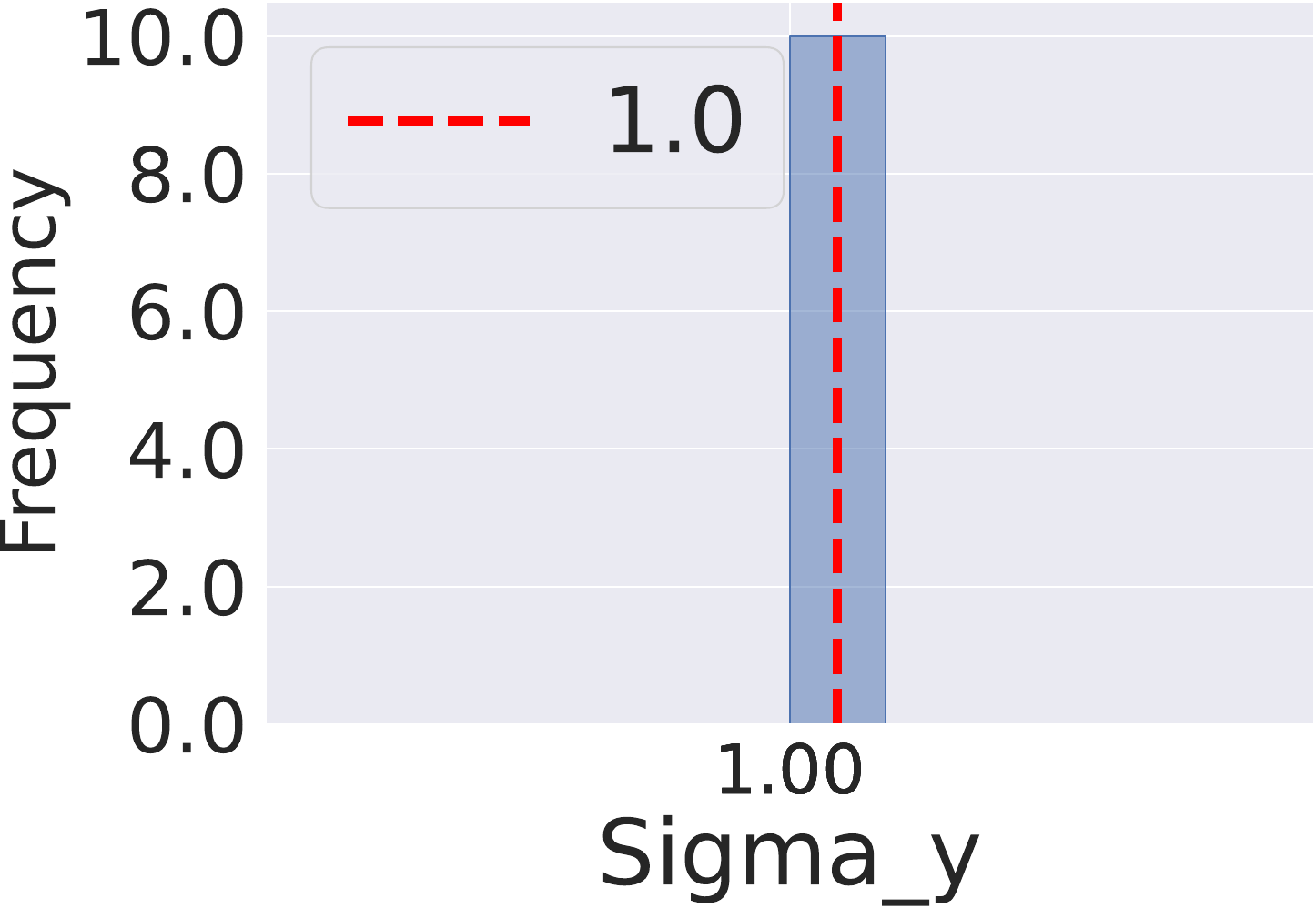}
        \end{minipage}
    }
    \subfigure{
        \begin{minipage}{0.23\linewidth}
            \includegraphics[width=\linewidth]{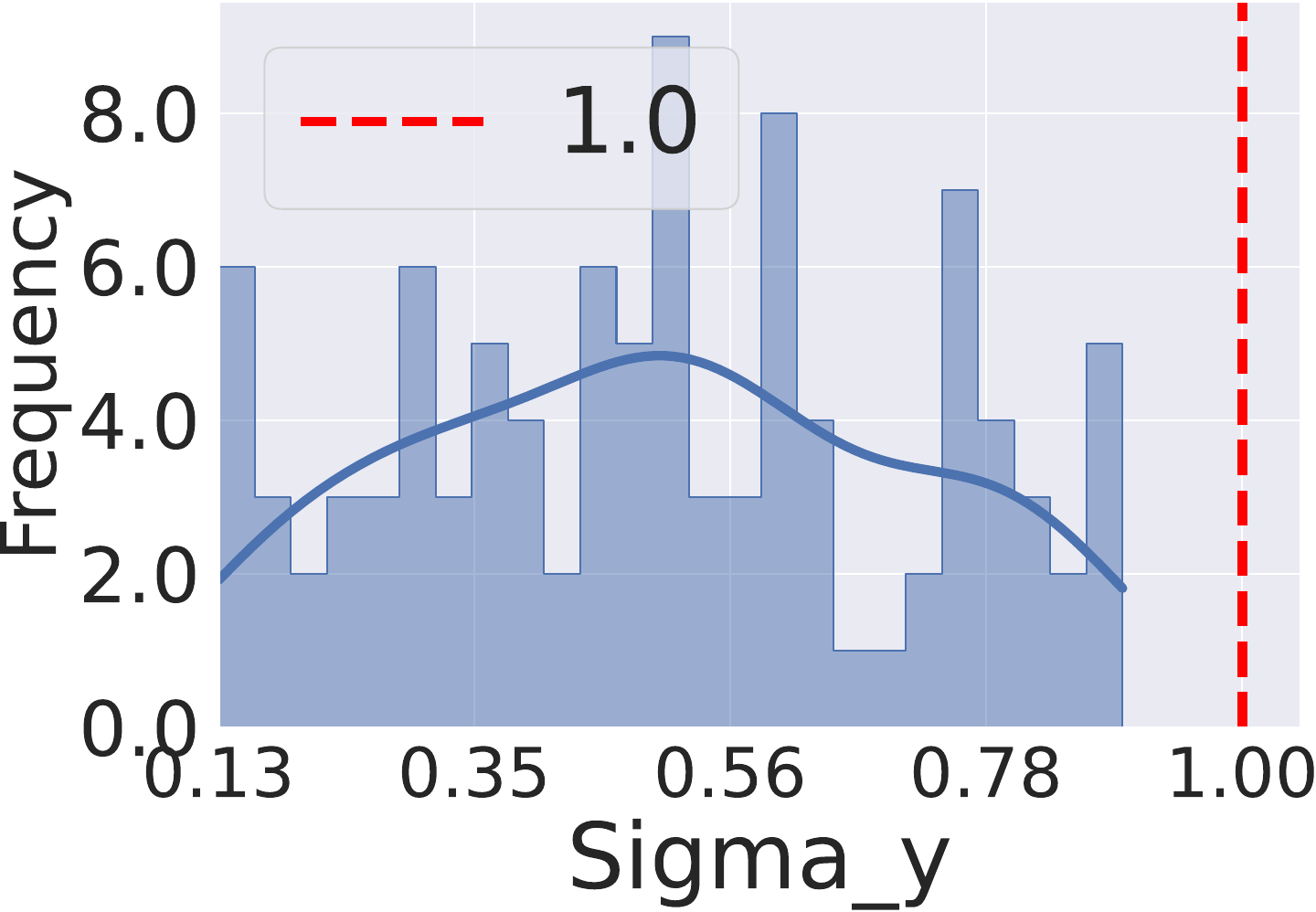}
        \end{minipage}
    }
    \subfigure{
        \begin{minipage}{0.23\linewidth}
            \includegraphics[width=\linewidth]{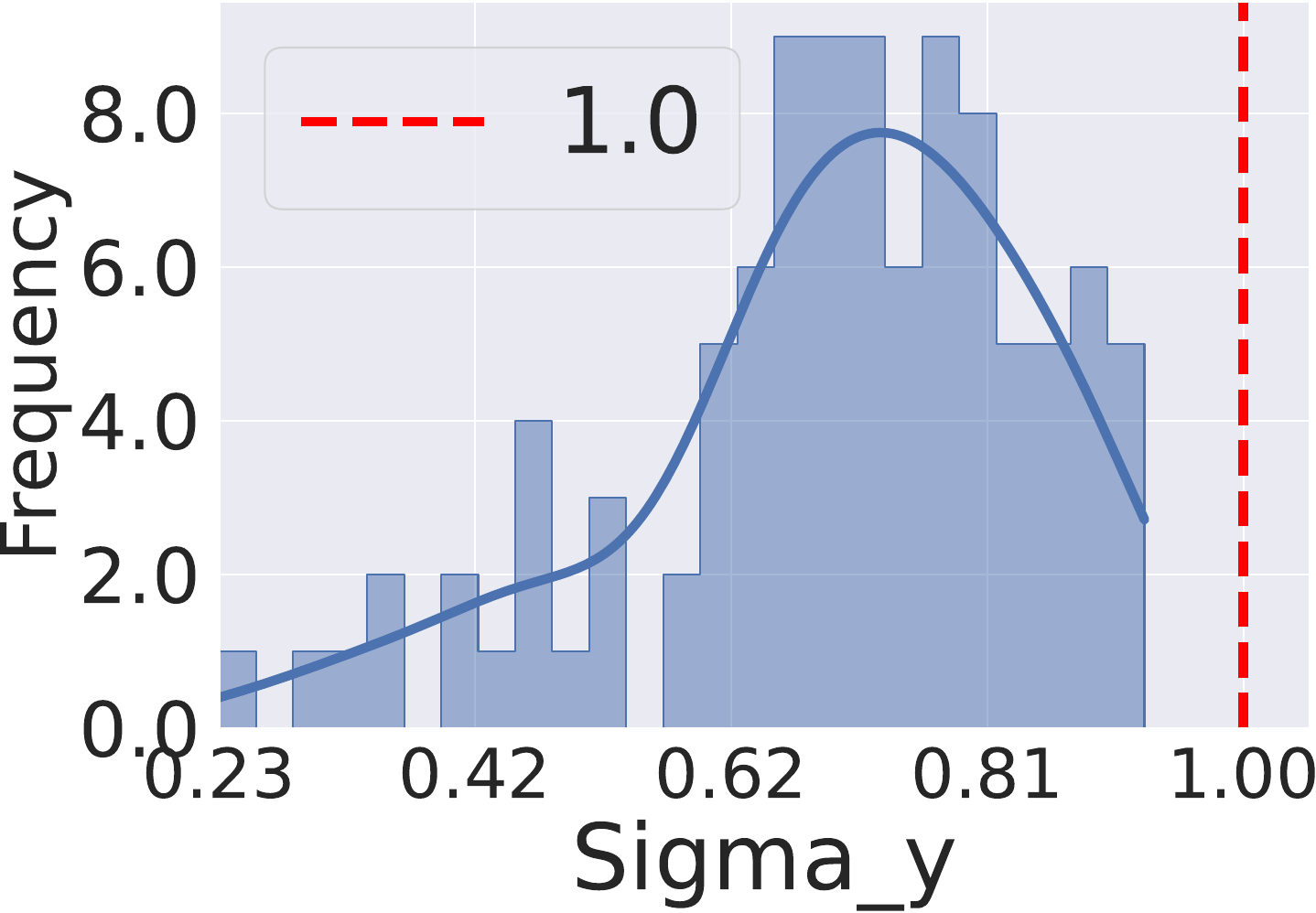}
        \end{minipage}
    }
    \subfigure{
        \begin{minipage}{0.23\linewidth}
            \includegraphics[width=\linewidth]{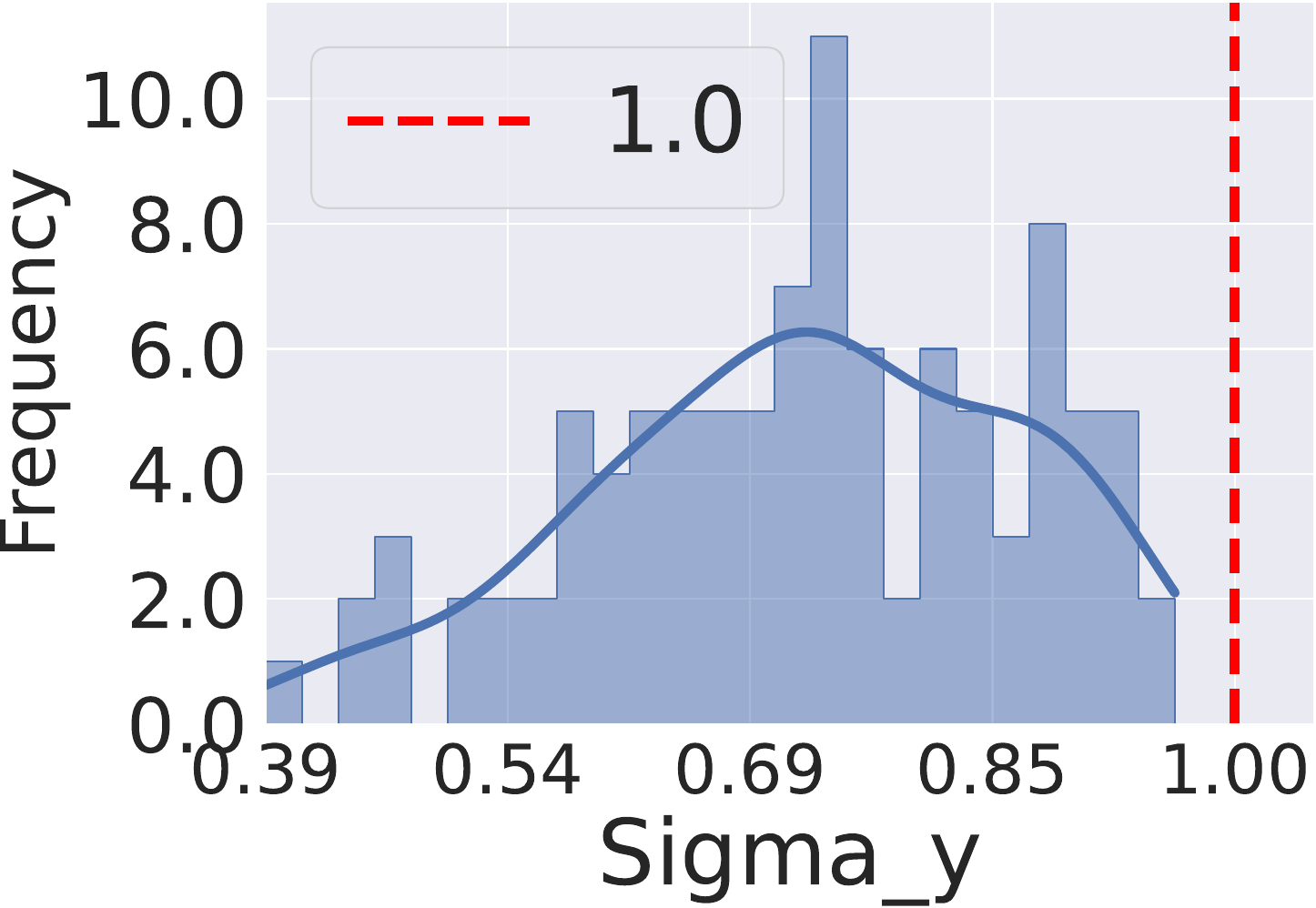}
        \end{minipage}
    }
    \caption{
    Verification of condition numbers $\{\sigma_y\}_{y=1}^C$ of Equation \ref{eq:sigma_y_defination} when epoch $=50$ and $\alpha = 0.1$ with $\rho=0.5$ \POLY.
    Vertical dashed lines represent the value $1$, and we observe that all the condition numbers are smaller than $1$.
    This verifies the validity of the condition for Lemma \ref{lemma:RC3P_improved_efficiency}, and thus confirms that \texttt{\newCP}produces smaller prediction sets than \texttt{CCP} by the optimized trade-off between calibration on non-conformity scores and calibrated label ranks.
    }
    \label{fig:condition_number_sigma_y_poly_0.5_50}
\end{figure}

\begin{figure}[!ht]
    \centering
    \begin{minipage}{.24\textwidth}
        \centering
        (a) CIFAR-10
    \end{minipage}%
    \begin{minipage}{.24\textwidth}
        \centering
        (b) CIFAR-100
    \end{minipage}%
    \begin{minipage}{.24\textwidth}
        \centering
        (c) mini-ImageNet
    \end{minipage}%
    \begin{minipage}{.24\textwidth}
        \centering
        (d) Food-101
    \end{minipage}
    \subfigure{
        \begin{minipage}{0.23\linewidth}
            \includegraphics[width=\linewidth]{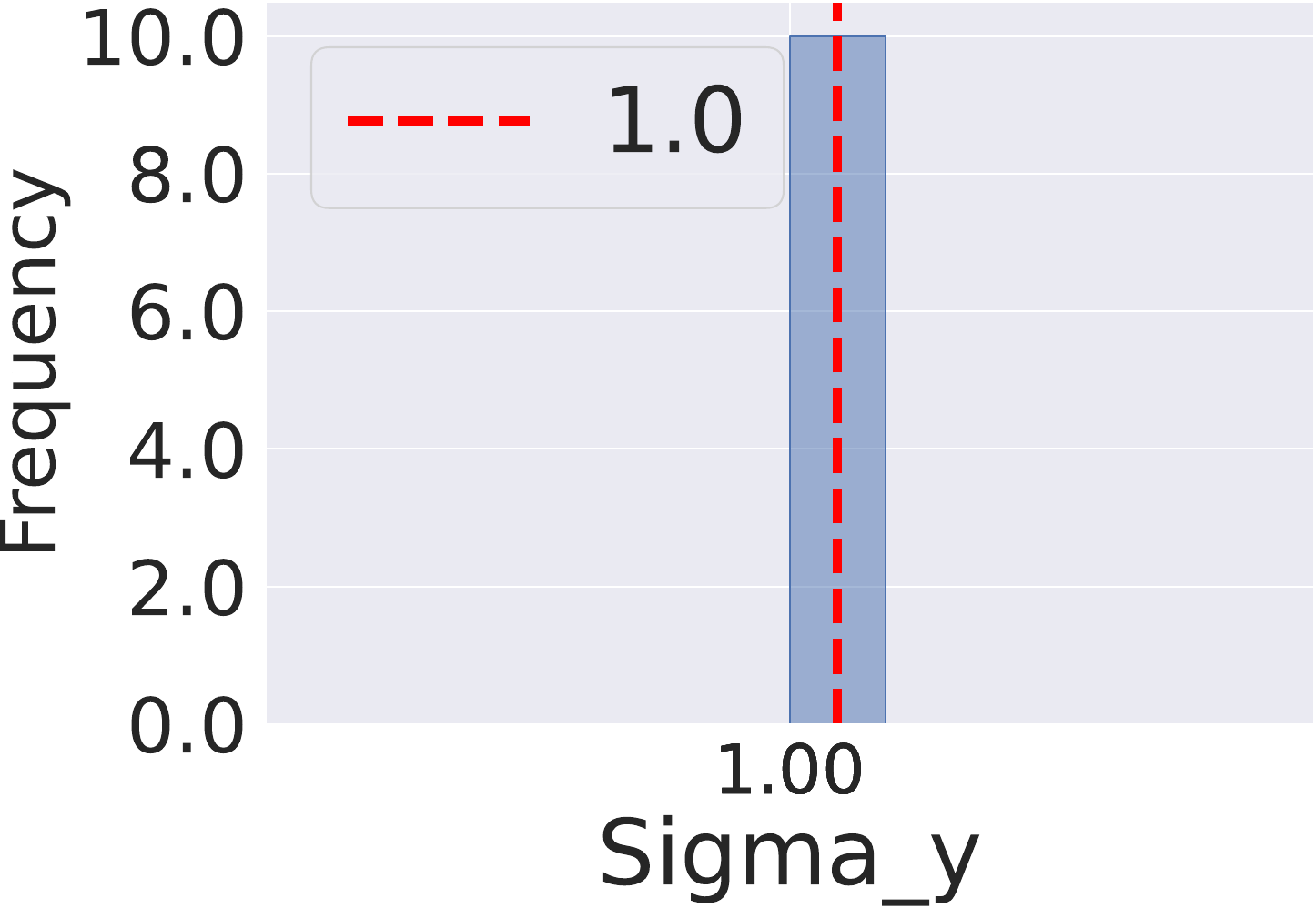}
        \end{minipage}
    }
    \subfigure{
        \begin{minipage}{0.23\linewidth}
            \includegraphics[width=\linewidth]{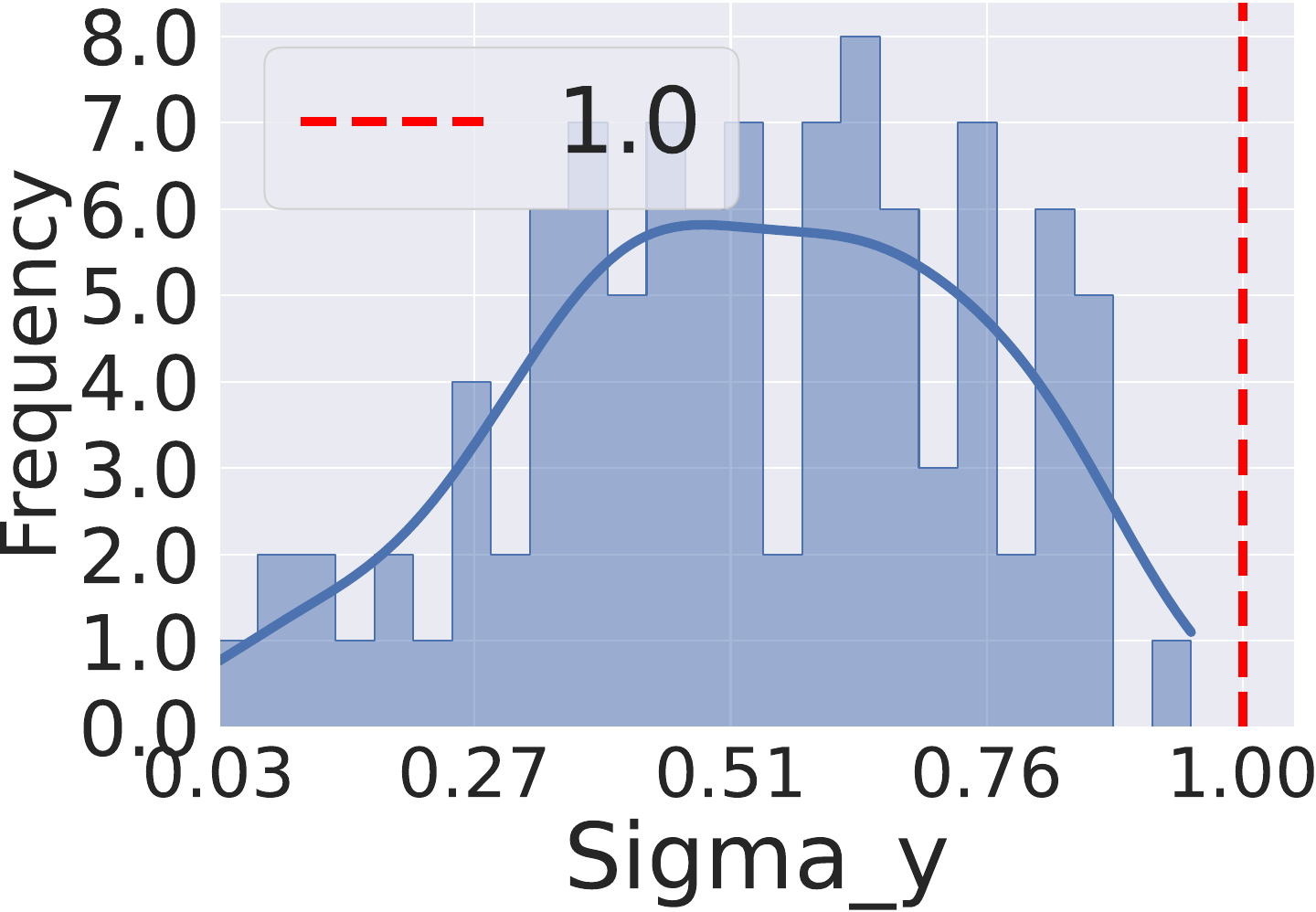}
        \end{minipage}
    }
    \subfigure{
        \begin{minipage}{0.23\linewidth}
            \includegraphics[width=\linewidth]{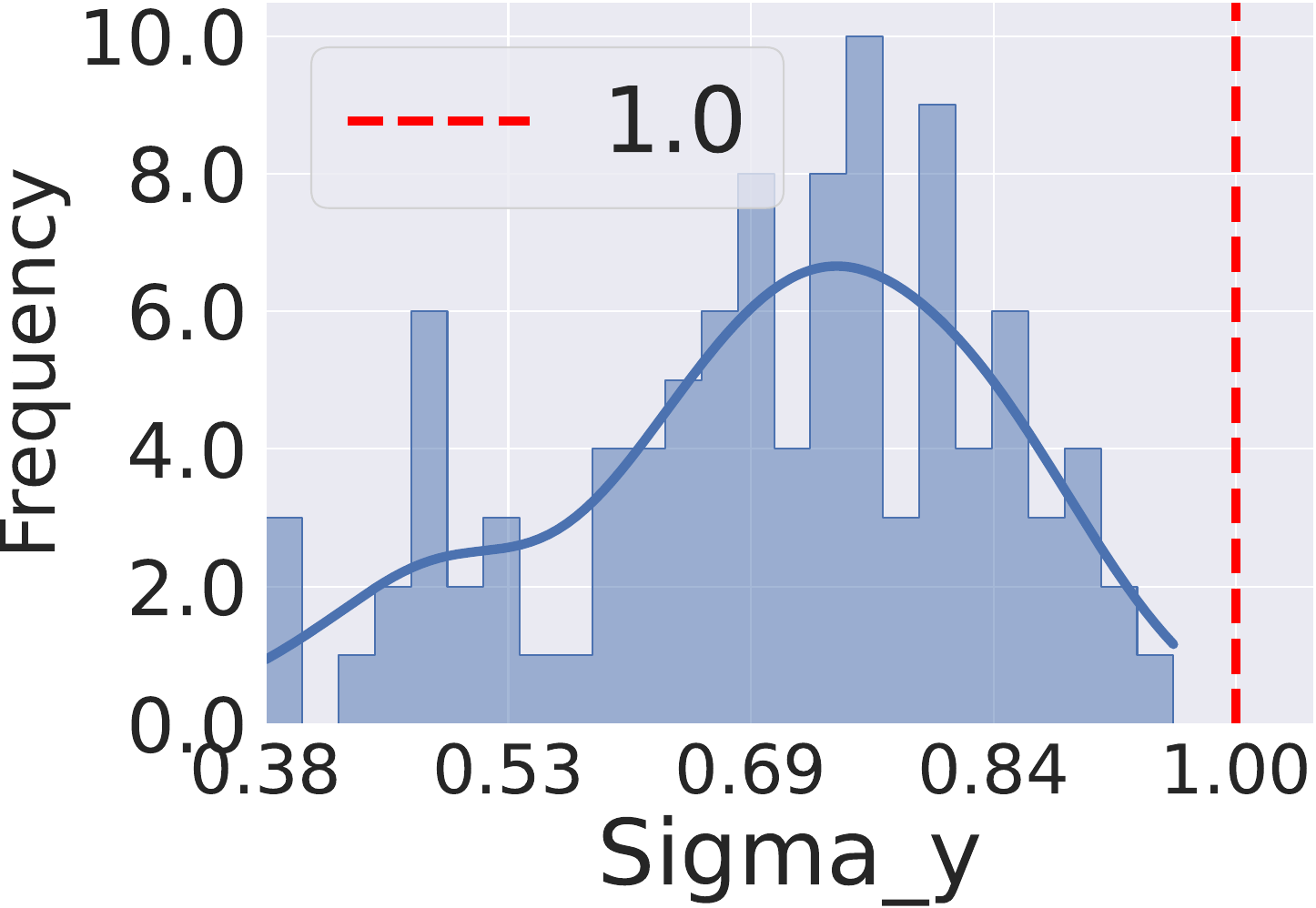}
        \end{minipage}
    }
    \subfigure{
        \begin{minipage}{0.23\linewidth}
            \includegraphics[width=\linewidth]{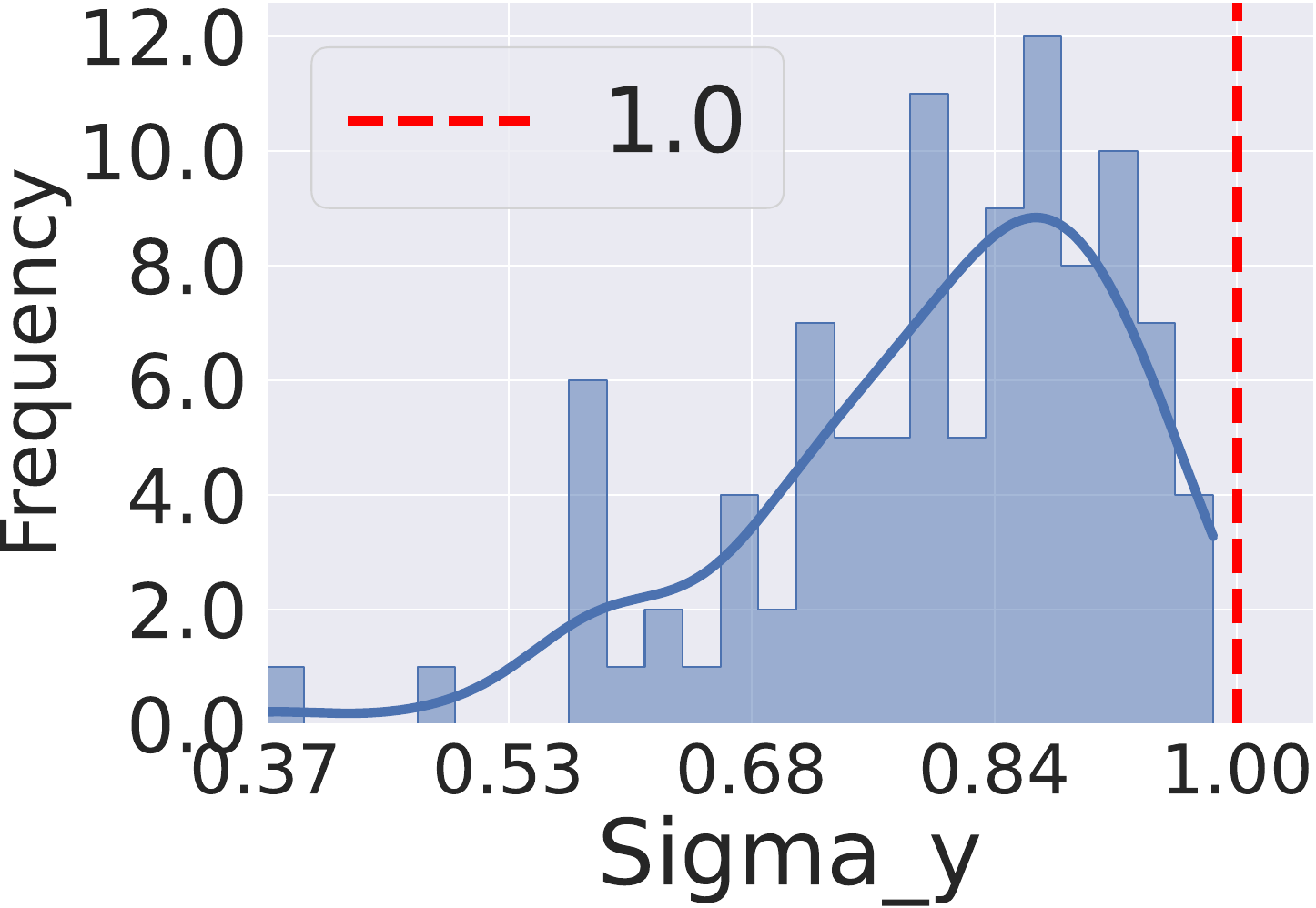}
        \end{minipage}
    }
    \caption{
    Verification of condition numbers $\{\sigma_y\}_{y=1}^C$ of Equation \ref{eq:sigma_y_defination} when epoch $=50$ and $\alpha = 0.1$ with $\rho=0.1$ \MAJ.
    Vertical dashed lines represent the value $1$, and we observe that all the condition numbers are smaller than $1$.
    This verifies the validity of the condition for Lemma \ref{lemma:RC3P_improved_efficiency}, and thus confirms that \texttt{\newCP}produces smaller prediction sets than \texttt{CCP} by the optimized trade-off between calibration on non-conformity scores and calibrated label ranks.
    }
    \label{fig:condition_number_sigma_y_maj_0.1_50}
\end{figure}

\begin{figure}[!ht]
    \centering
    \begin{minipage}{.24\textwidth}
        \centering
        (a) CIFAR-10
    \end{minipage}%
    \begin{minipage}{.24\textwidth}
        \centering
        (b) CIFAR-100
    \end{minipage}%
    \begin{minipage}{.24\textwidth}
        \centering
        (c) mini-ImageNet
    \end{minipage}%
    \begin{minipage}{.24\textwidth}
        \centering
        (d) Food-101
    \end{minipage}
    \subfigure{
        \begin{minipage}{0.23\linewidth}
            \includegraphics[width=\linewidth]{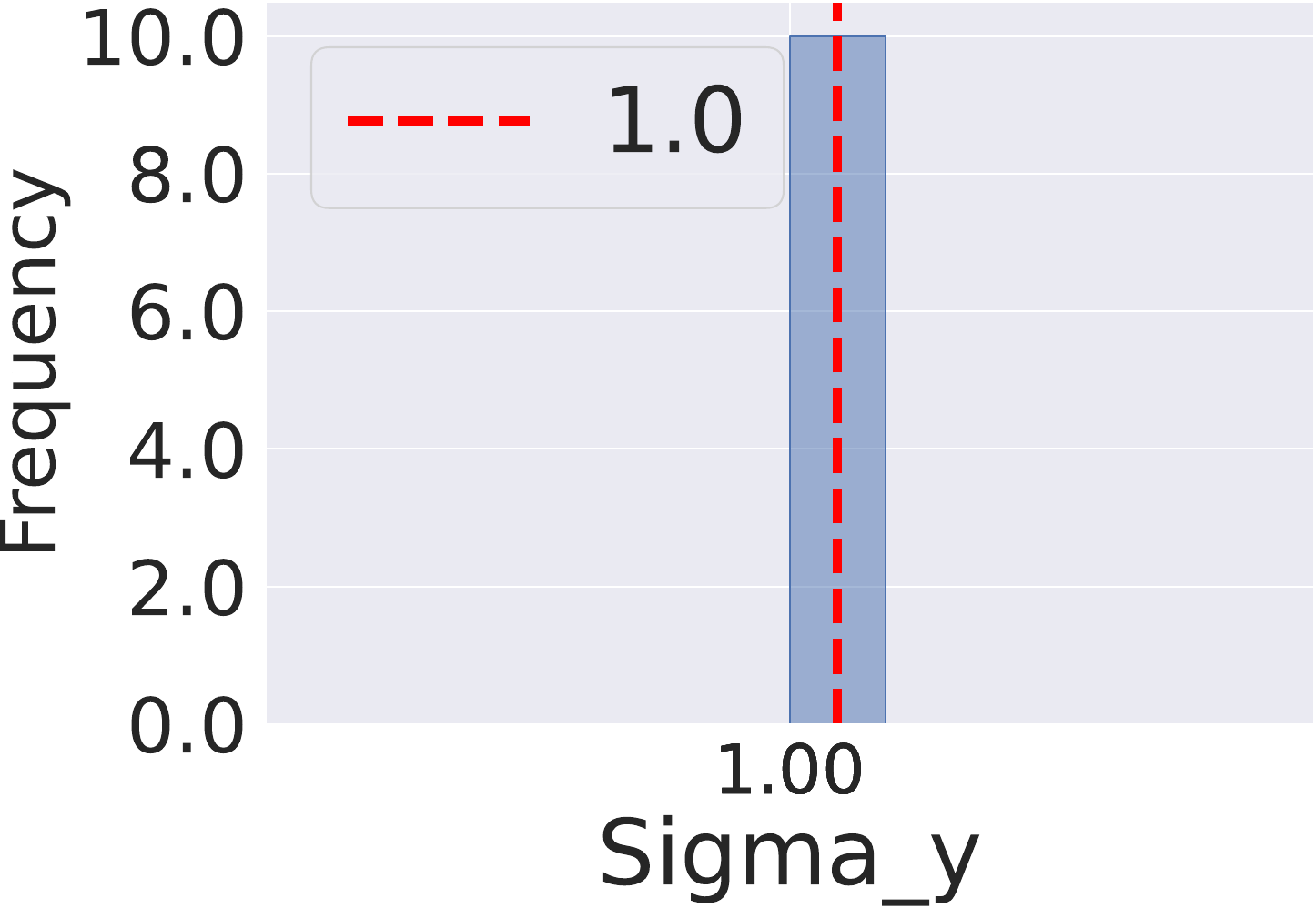}
        \end{minipage}
    }
    \subfigure{
        \begin{minipage}{0.23\linewidth}
            \includegraphics[width=\linewidth]{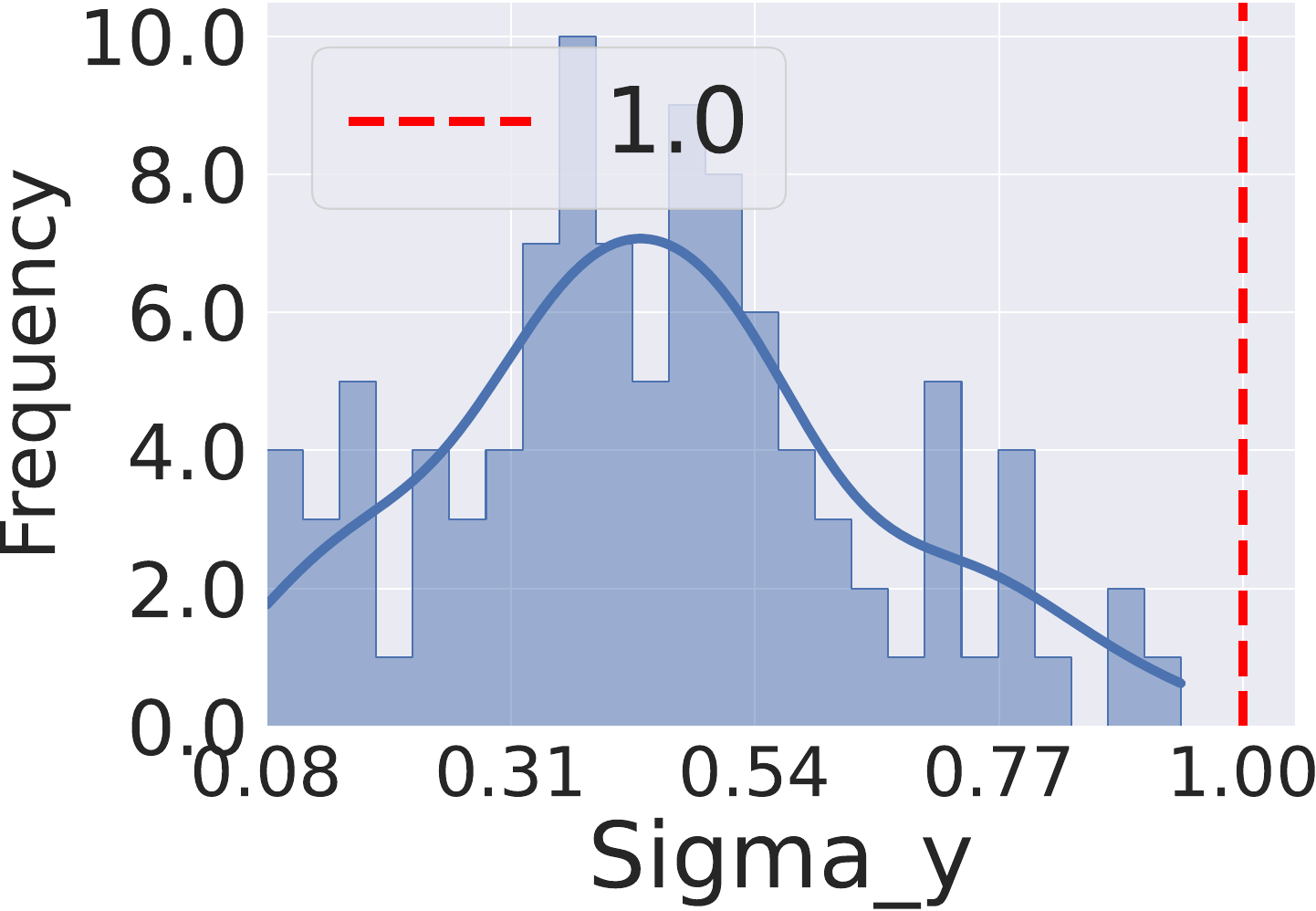}
        \end{minipage}
    }
    \subfigure{
        \begin{minipage}{0.23\linewidth}
            \includegraphics[width=\linewidth]{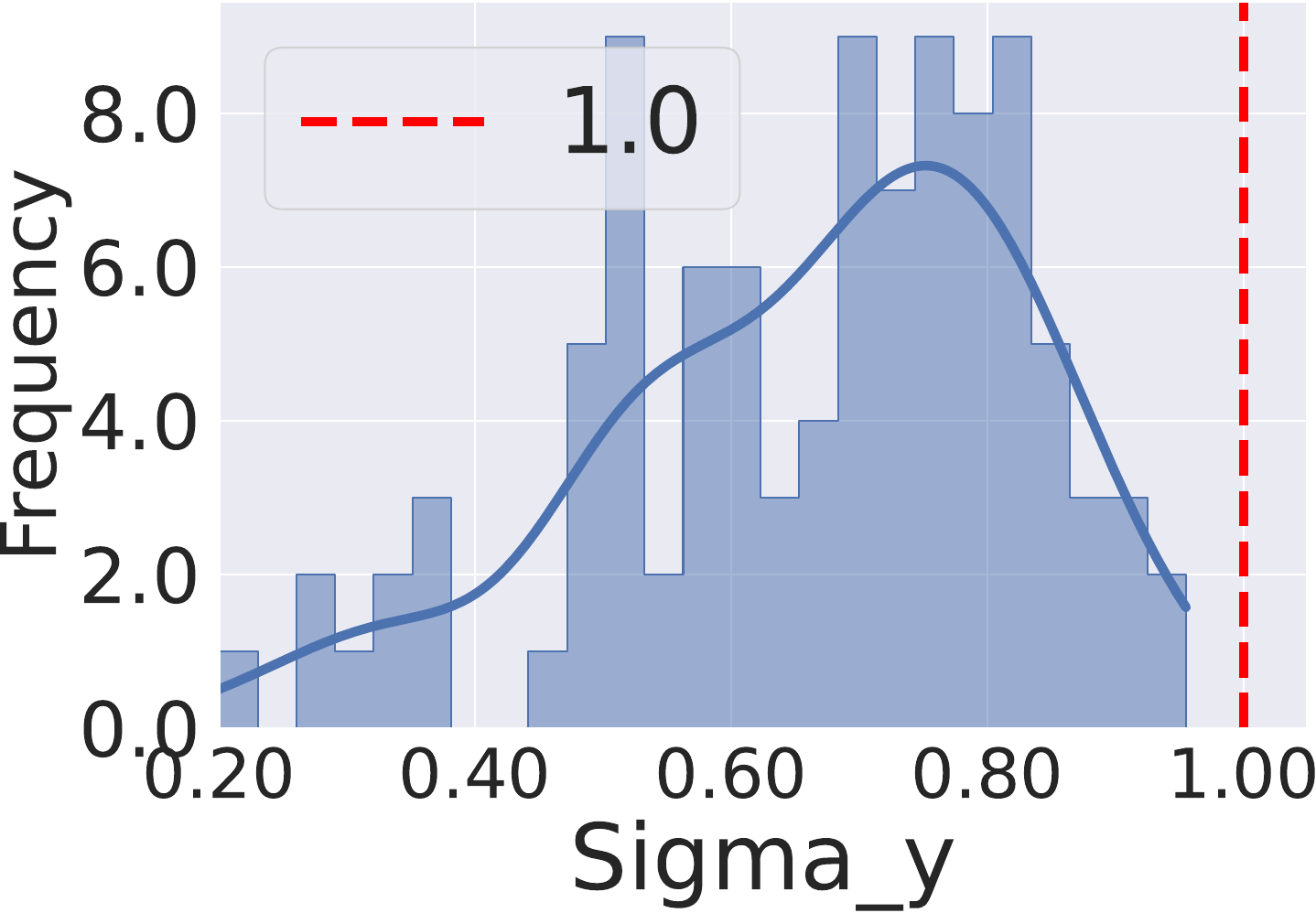}
        \end{minipage}
    }
    \subfigure{
        \begin{minipage}{0.23\linewidth}
            \includegraphics[width=\linewidth]{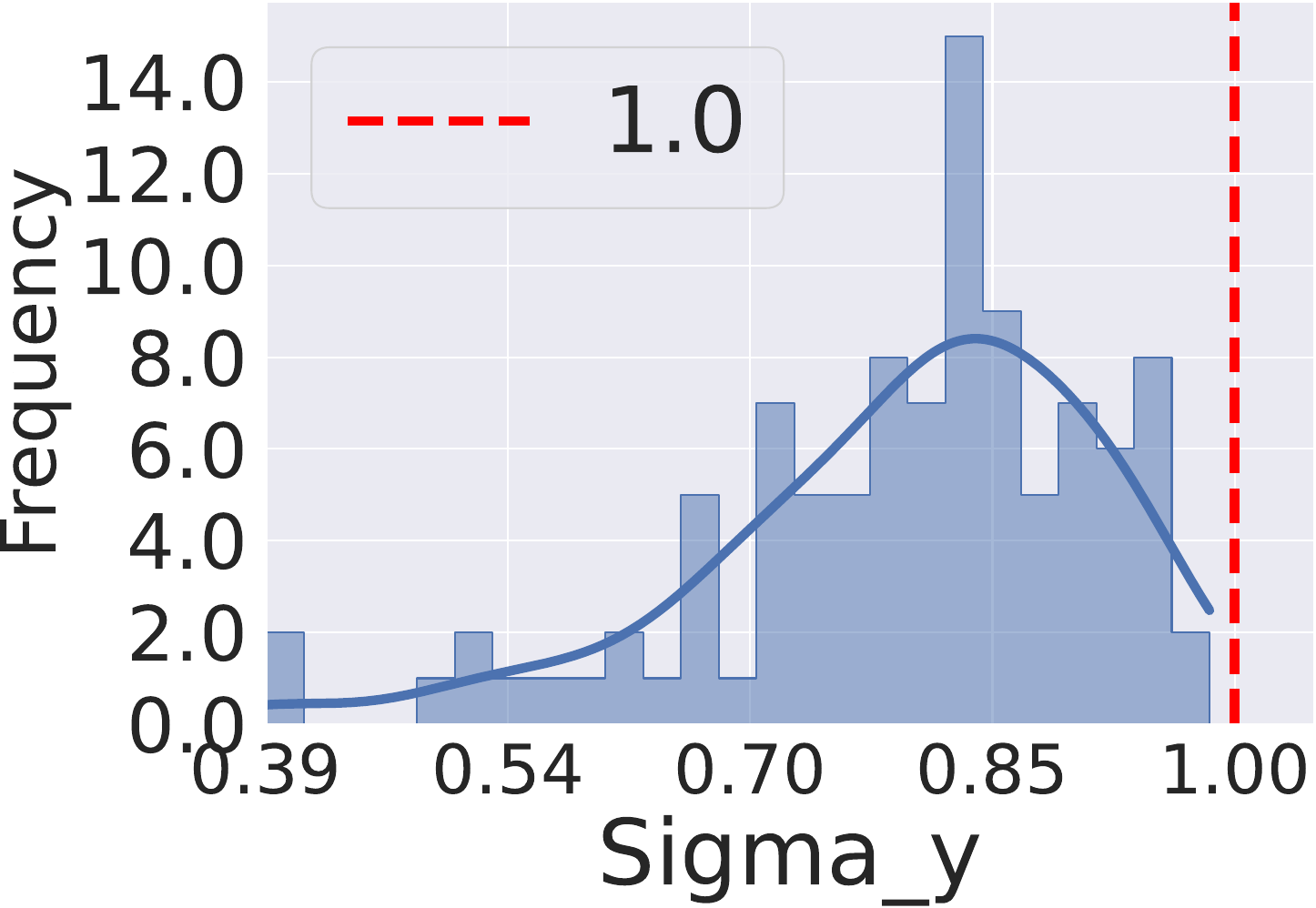}
        \end{minipage}
    }
    \caption{
    Verification of condition numbers $\{\sigma_y\}_{y=1}^C$ of Equation \ref{eq:sigma_y_defination} when epoch $=50$ and $\alpha = 0.1$ with $\rho=0.5$ \MAJ.
    Vertical dashed lines represent the value $1$, and we observe that all the condition numbers are smaller than $1$.
    This verifies the validity of the condition for Lemma \ref{lemma:RC3P_improved_efficiency}, and thus confirms that \texttt{\newCP}produces smaller prediction sets than \texttt{CCP} by the optimized trade-off between calibration on non-conformity scores and calibrated label ranks.
    }
    \label{fig:condition_number_sigma_y_maj_0.5_50}
\end{figure}

\clearpage
\newpage

\subsection{ Complete Experiment Results on Balanced Classification Datasets}
\label{subsection:appendix:complete experiment_results_balanced}

In this subsection, we report complete experimental results over four balanced datasets and $\alpha = 0.1$.
Specifically, 
Figure \ref{fig:overall_comparison_four_balanced_datasets}
shows the class-conditional coverage and the corresponding prediction set sizes. 
From the first row of Fig \ref{fig:overall_comparison_four_balanced_datasets}, the class-wise coverage bars of \texttt{CCP} and \texttt{\newCP}~distribute on the right-hand side of the target probability $1-\alpha$ (red dashed line).
Second, \texttt{\newCP}~outperforms \texttt{CCP} and \texttt{Cluster-CP} with $24.47\%$ (on four datasets) or $32.63\%$ (excluding CIFAR-10) on imbalanced datasets and $32.63\%$ on balanced datasets decrease in terms of average prediction set size the same class-wise coverage.
The second row of Figure \ref{fig:overall_comparison_four_balanced_datasets} shows
(i) \texttt{\newCP}~has more concentrated class-wise coverage distribution than \texttt{CCP} and \texttt{Cluster-CP};
(ii) the distribution of prediction set sizes produced by \texttt{\newCP}~is globally smaller than that produced by \texttt{CCP} and \texttt{Cluster-CP}, which is justified by a better trade-off number of $\{\sigma_y\}_{y=1}^K$ as shown in Figure \ref{fig:condition_number_sigma_y_exp_0.1}.

Figure \ref{fig:condition_number_rank_balanced}
illustrates the normalized frequency distribution of label ranks included in the prediction sets on balanced datasets.
It is evident that the distribution of label ranks in the prediction set generated by \texttt{\newCP}~tends to be lower compared to those produced by \texttt{CCP} and \texttt{Cluster-CP}. 
Furthermore, the probability density function tail for label ranks in the \texttt{\newCP}~prediction set is notably shorter than that of other methods. This indicates that \texttt{\newCP}~more effectively incorporates lower-ranked labels into prediction sets, as a result of its augmented rank calibration scheme.

Figure \ref{fig:condition_number_sigma_balanced}
verifies the condition numbers $\sigma_y$ on balanced datasets.
This result verifies the validity of Lemma \ref{lemma:RC3P_improved_efficiency} and Equation \ref{eq:sigma_y_defination} and confirm that the optimized trade-off between the coverage with inflated quantile and the constraint with calibrated rank leads to smaller prediction sets.

\begin{figure*}[!ht]
    \centering
    \begin{minipage}{.24\textwidth}
        \centering
        (a) CIFAR-100
    \end{minipage}%
    \begin{minipage}{.24\textwidth}
        \centering
        (b) Places365
    \end{minipage}%
    \begin{minipage}{.24\textwidth}
        \centering
        (c) iNaturalist
    \end{minipage}%
    \begin{minipage}{.24\textwidth}
        \centering
        (d) ImageNet
    \end{minipage}

    \subfigure{
        \begin{minipage}{0.23\linewidth}
            \includegraphics[width=\linewidth]{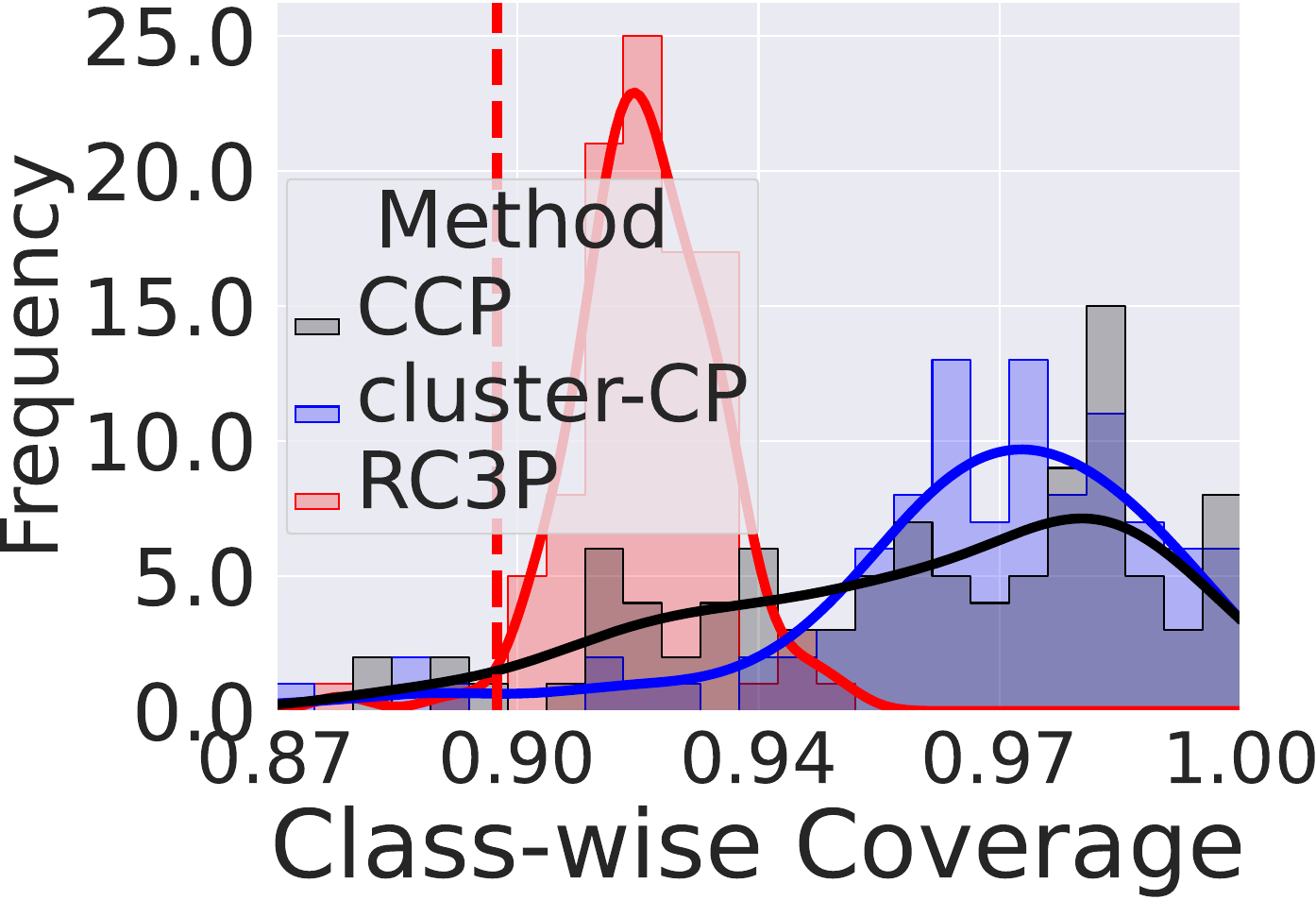}
        \end{minipage}
    }
    \subfigure{
        \begin{minipage}{0.23\linewidth}
            \includegraphics[width=\linewidth]{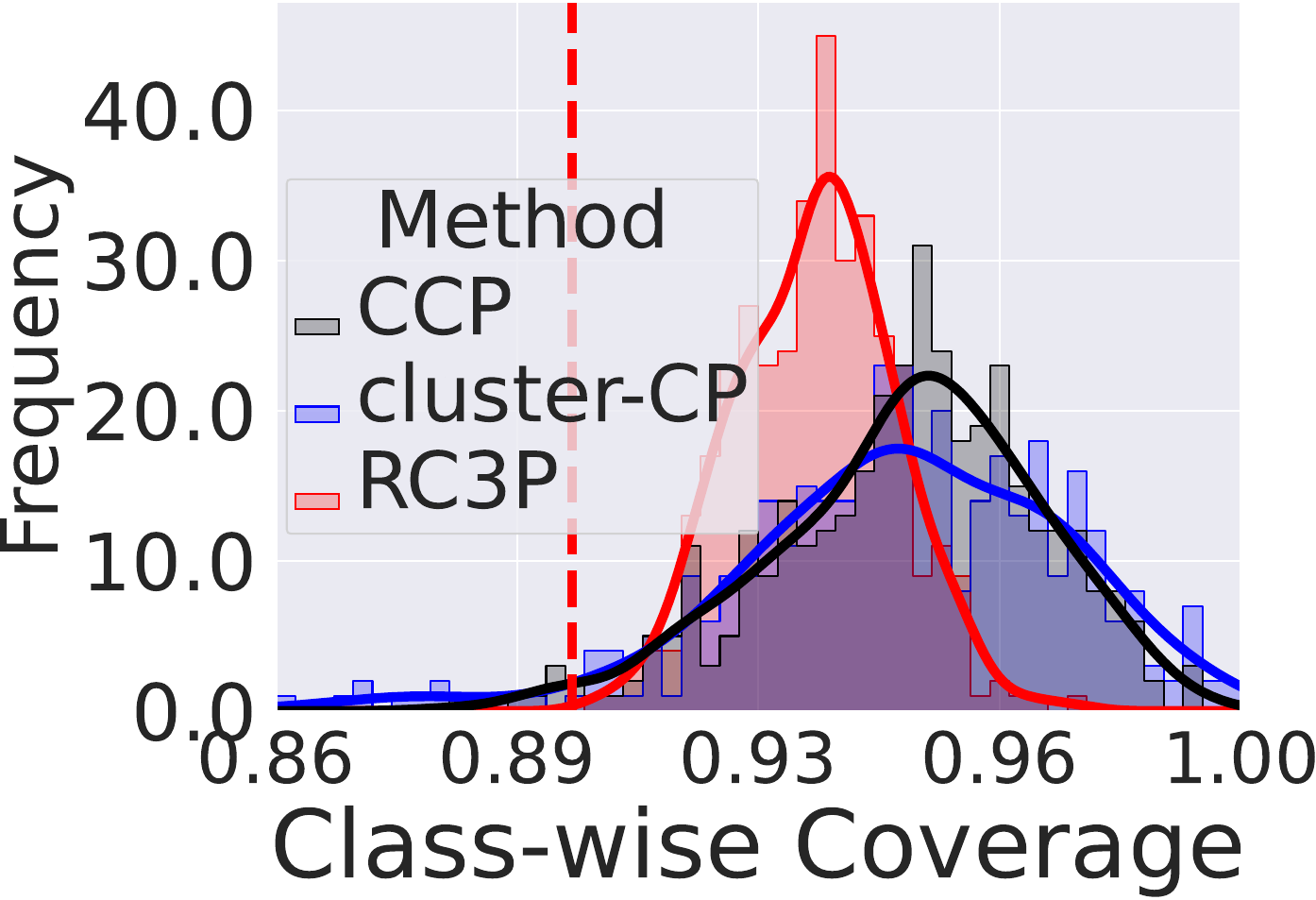}
        \end{minipage}
    }
    \subfigure{
        \begin{minipage}{0.23\linewidth}
            \includegraphics[width=\linewidth]{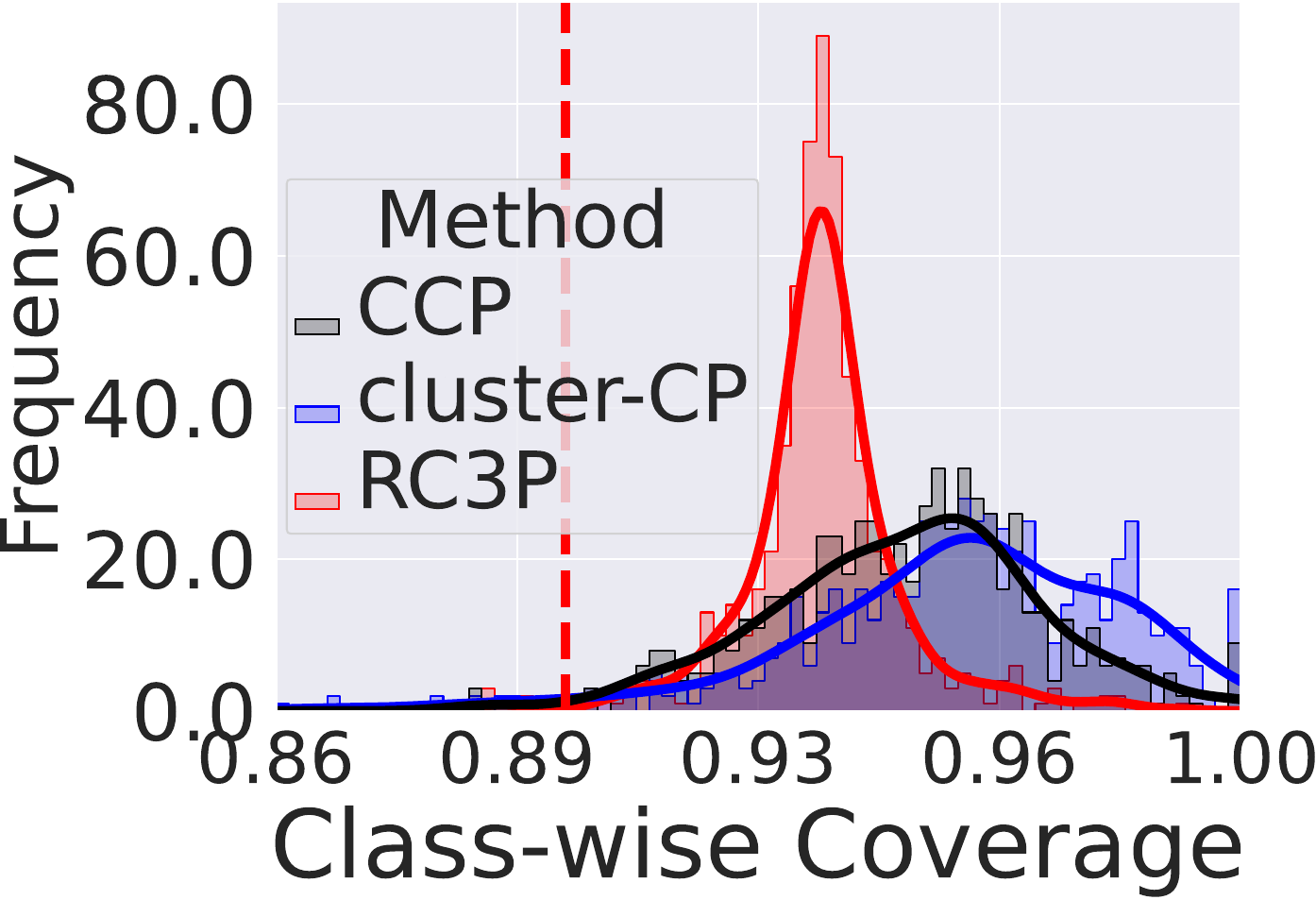}
        \end{minipage}
    }
    \subfigure{
        \begin{minipage}{0.23\linewidth}
            \includegraphics[width=\linewidth]{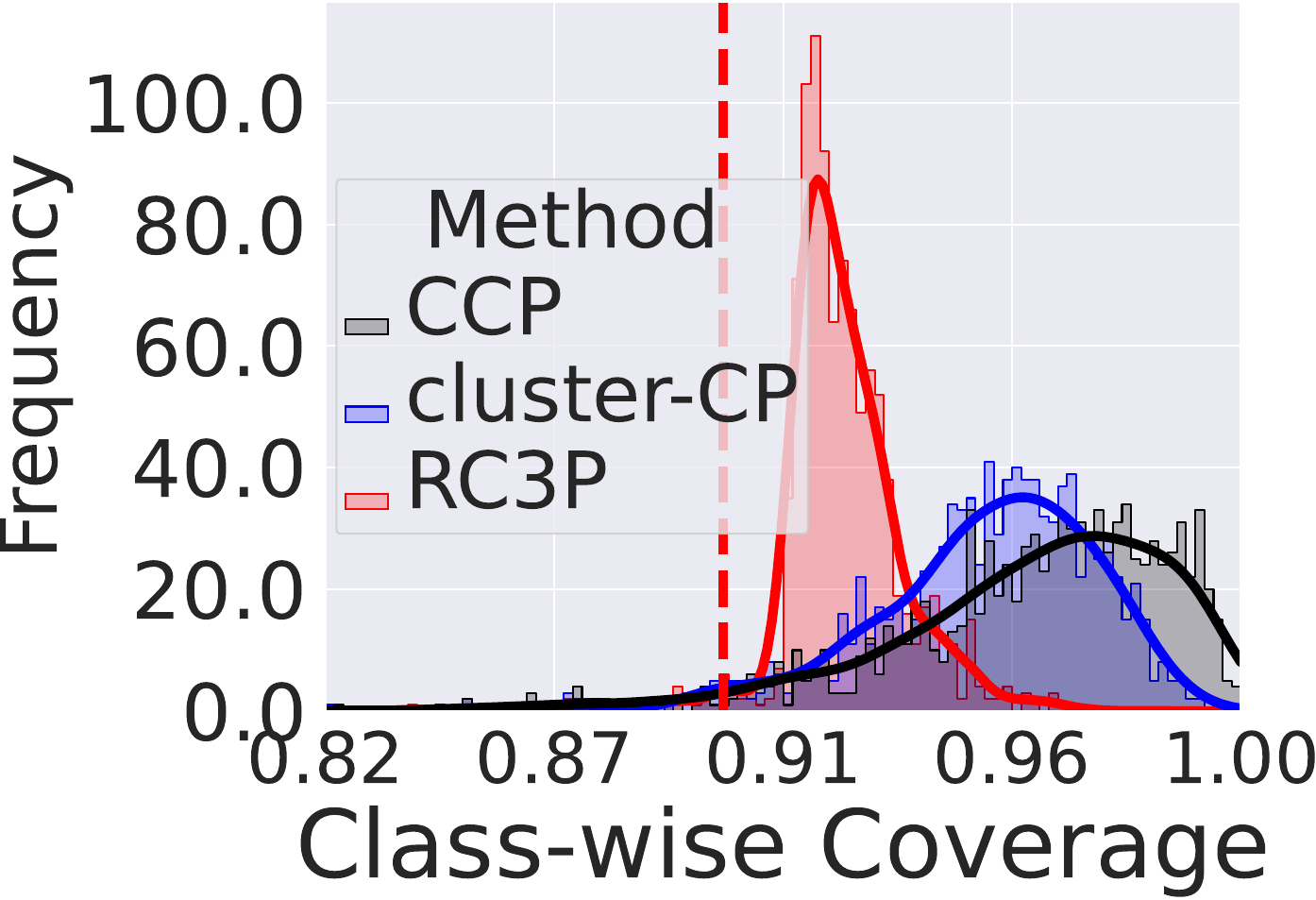}
        \end{minipage}
    }
    \subfigure{
        \begin{minipage}{0.23\linewidth}
            \includegraphics[width=\linewidth]{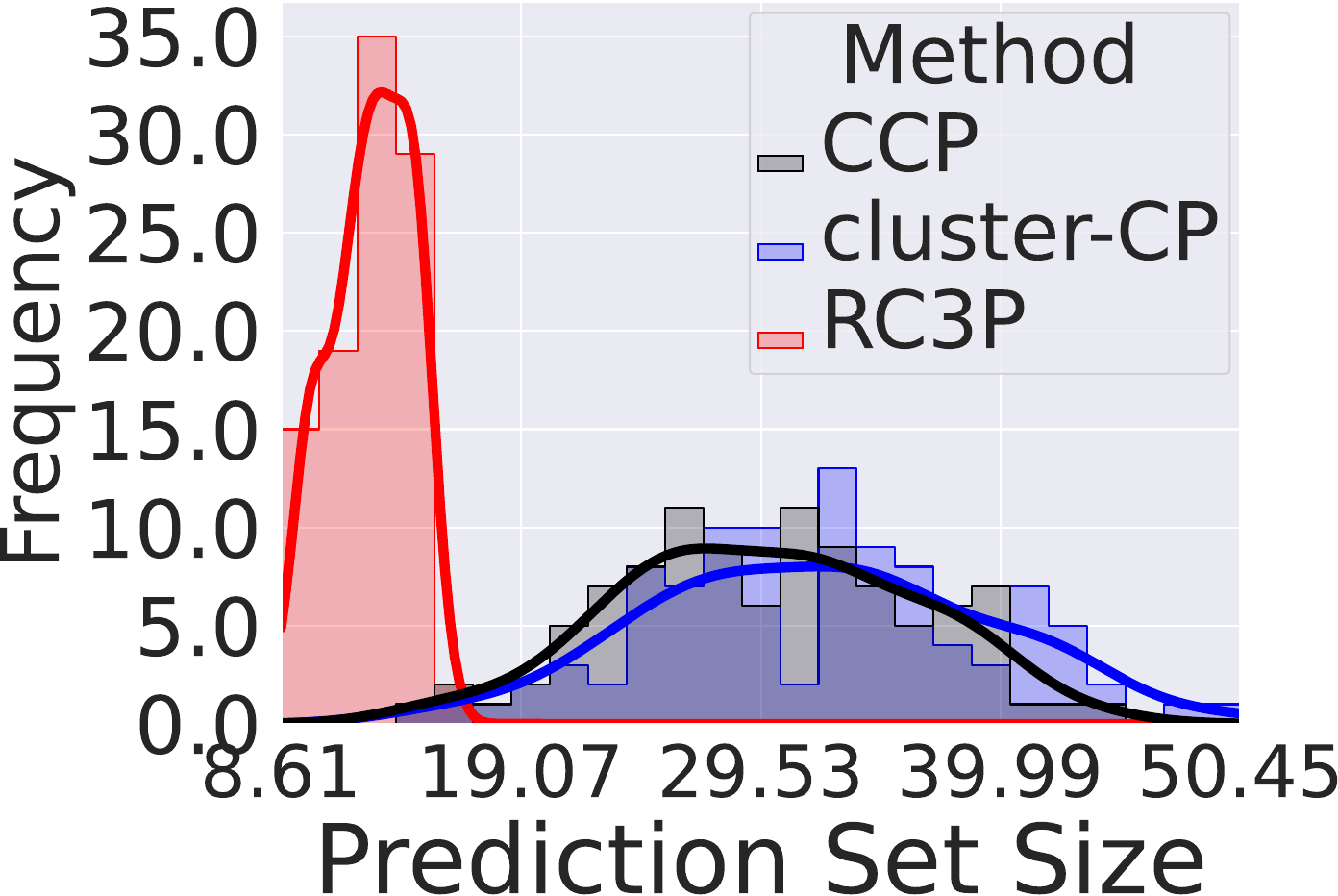}
        \end{minipage}
    }
    \subfigure{
        \begin{minipage}{0.23\linewidth}
            \includegraphics[width=\linewidth]{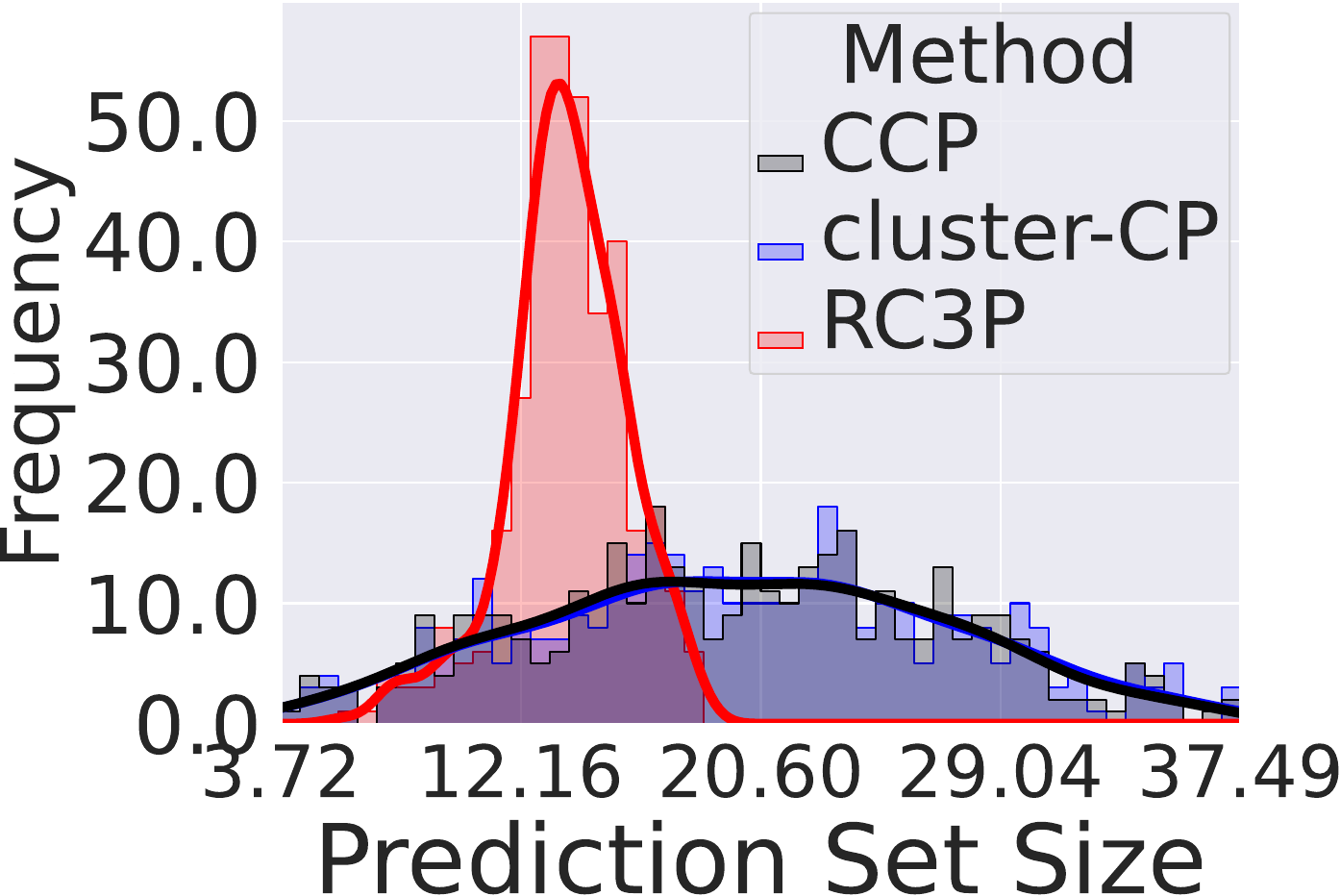}
        \end{minipage}
    }
    \subfigure{
        \begin{minipage}{0.23\linewidth}
            \includegraphics[width=\linewidth]{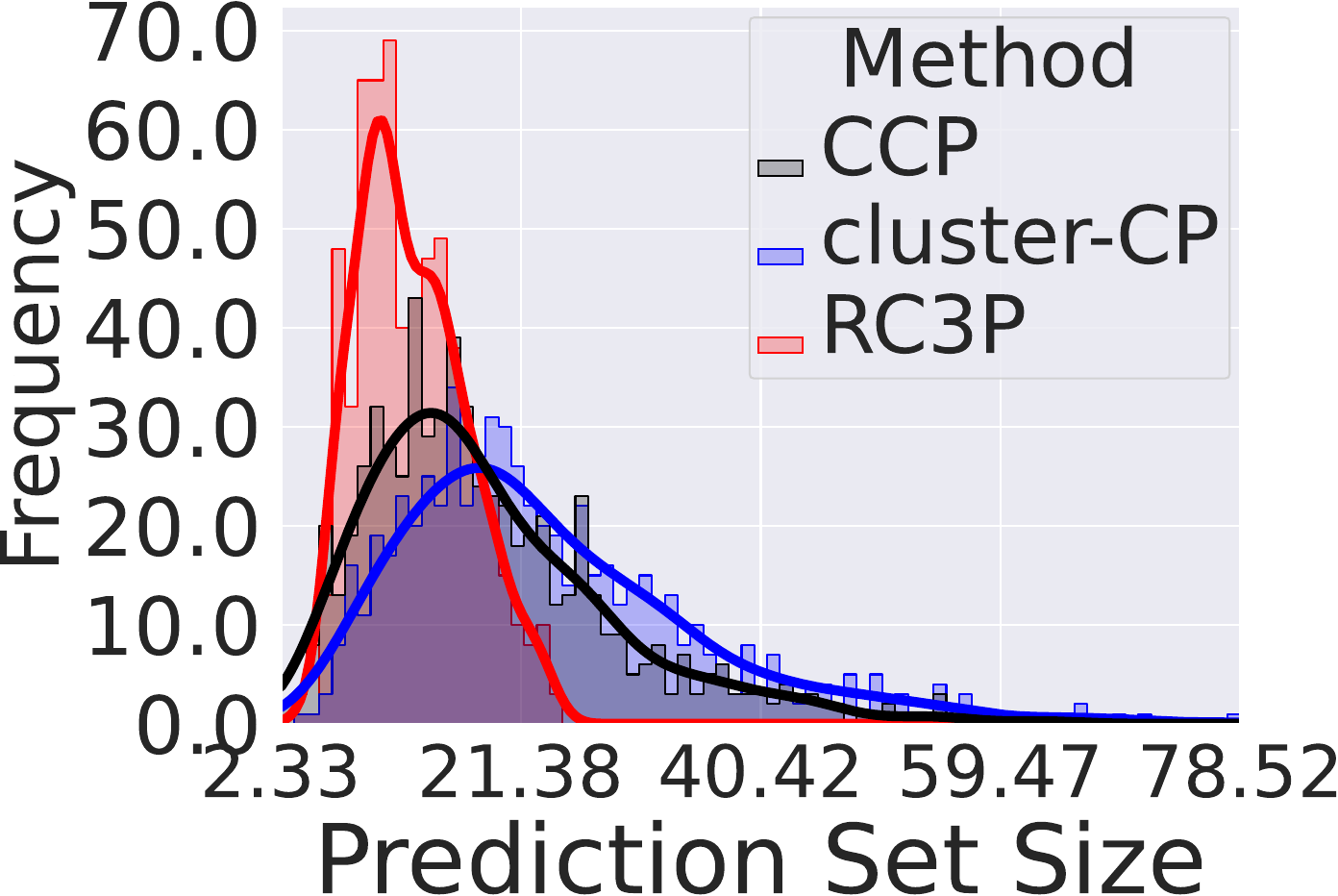}
        \end{minipage}
    }
    \subfigure{
        \begin{minipage}{0.23\linewidth}
            \includegraphics[width=\linewidth]{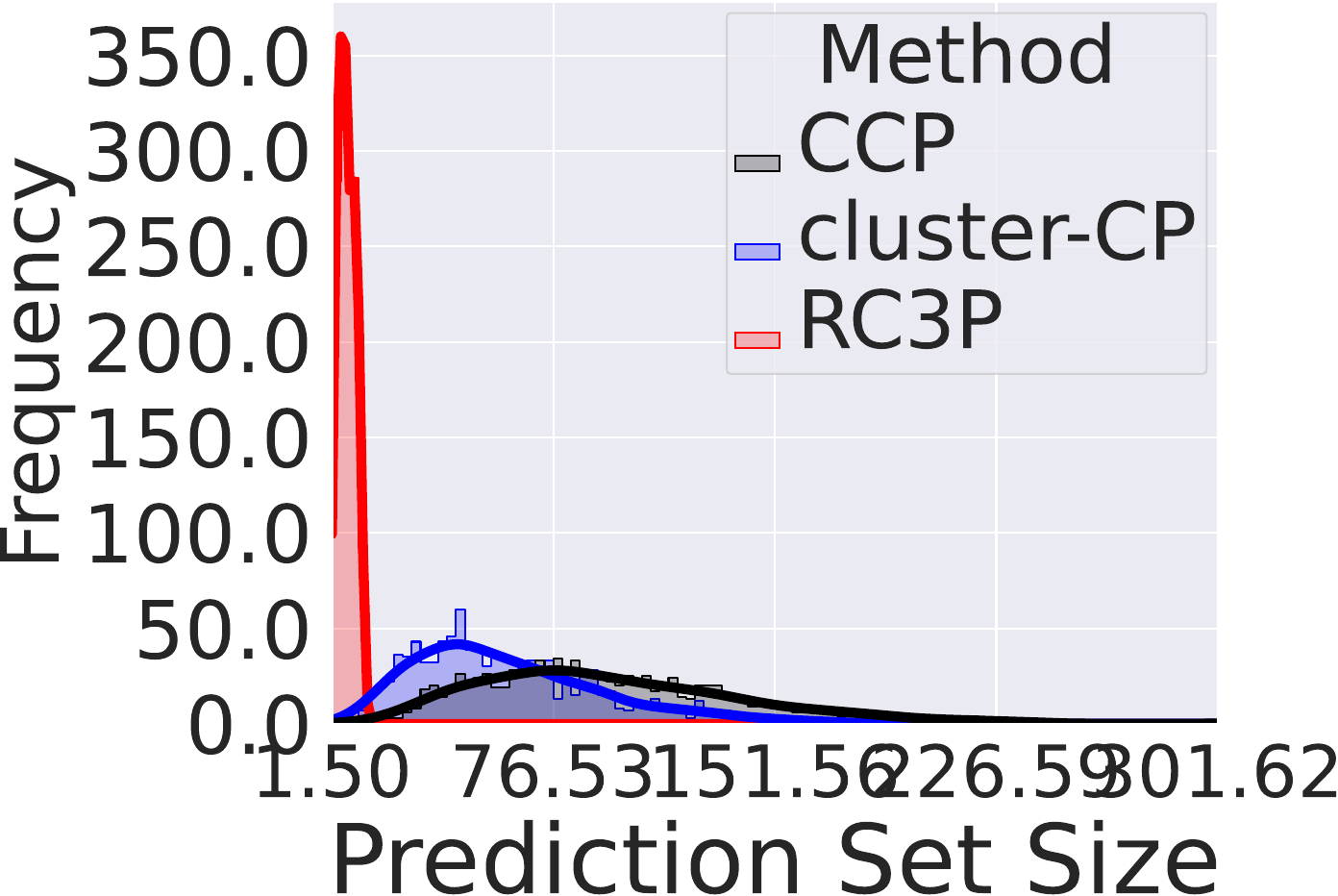}
        \end{minipage}
    }
    \caption{
    Class-conditional coverage (Top row) and prediction set size (Bottom row) achieved by \texttt{CCP}, \texttt{Cluster-CP}, and \texttt{\newCP} methods when $\alpha = 0.1$ on four balanced datasets.
    It is clear that \texttt{\newCP} has more densely distributed class-conditional coverage above $0.9$ (the target $1-\alpha$ class-conditional coverage) than \texttt{CCP} and \texttt{Cluster-CP} with significantly smaller prediction sets on all datasets.
    }
    \label{fig:overall_comparison_four_balanced_datasets}
\end{figure*}

\begin{figure*}[!ht]
    \centering
    \begin{minipage}{.24\textwidth}
        \centering
        (a) CIFAR-100
    \end{minipage}%
    \begin{minipage}{.24\textwidth}
        \centering
        (b) Places365
    \end{minipage}%
    \begin{minipage}{.24\textwidth}
        \centering
        (c) iNaturalist
    \end{minipage}%
    \begin{minipage}{.24\textwidth}
        \centering
        (d) ImageNet
    \end{minipage}
    \subfigure{
        \begin{minipage}{0.23\linewidth}
            \includegraphics[width=\linewidth]{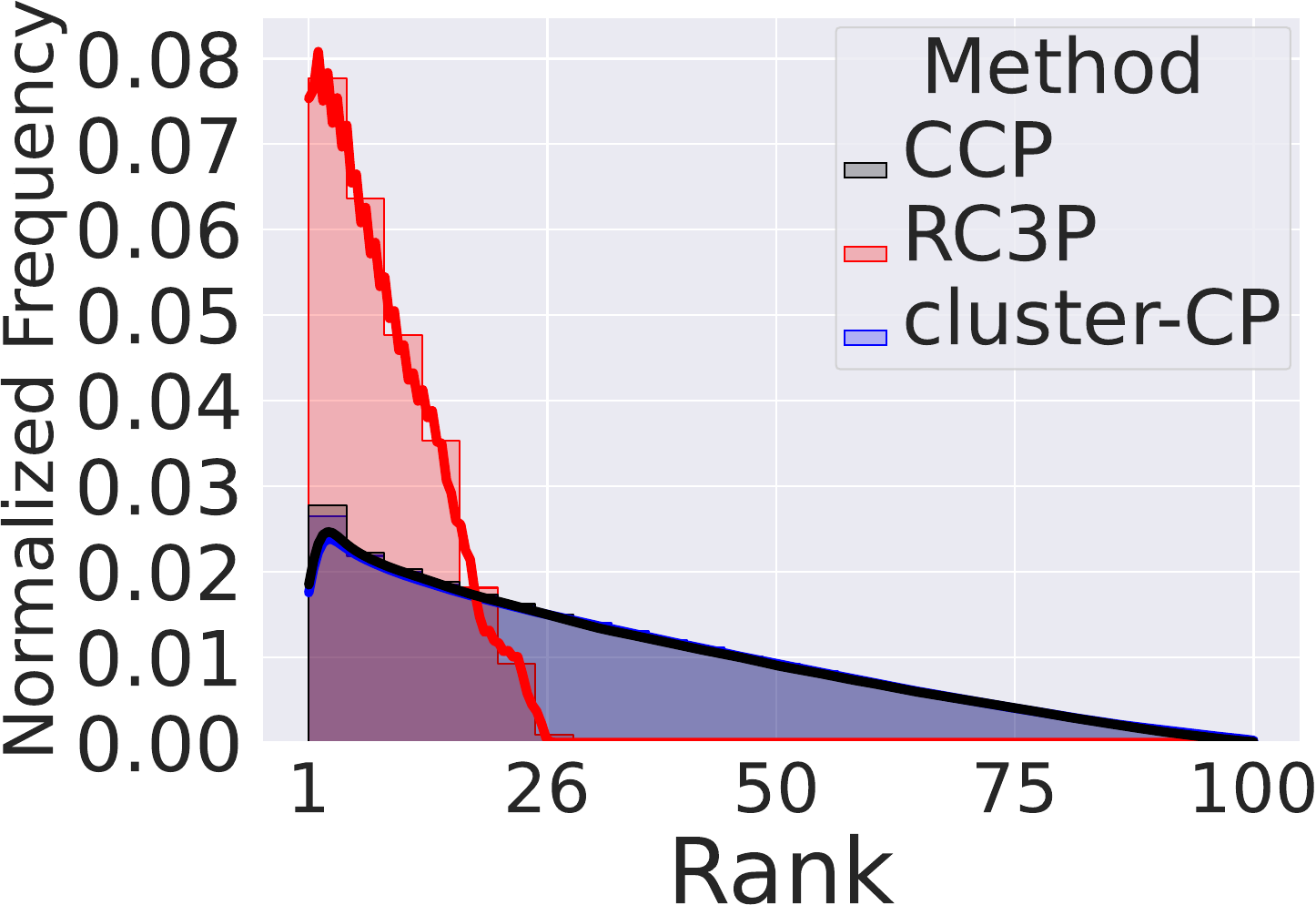}
        \end{minipage}
    }
    \subfigure{
        \begin{minipage}{0.23\linewidth}
            \includegraphics[width=\linewidth]{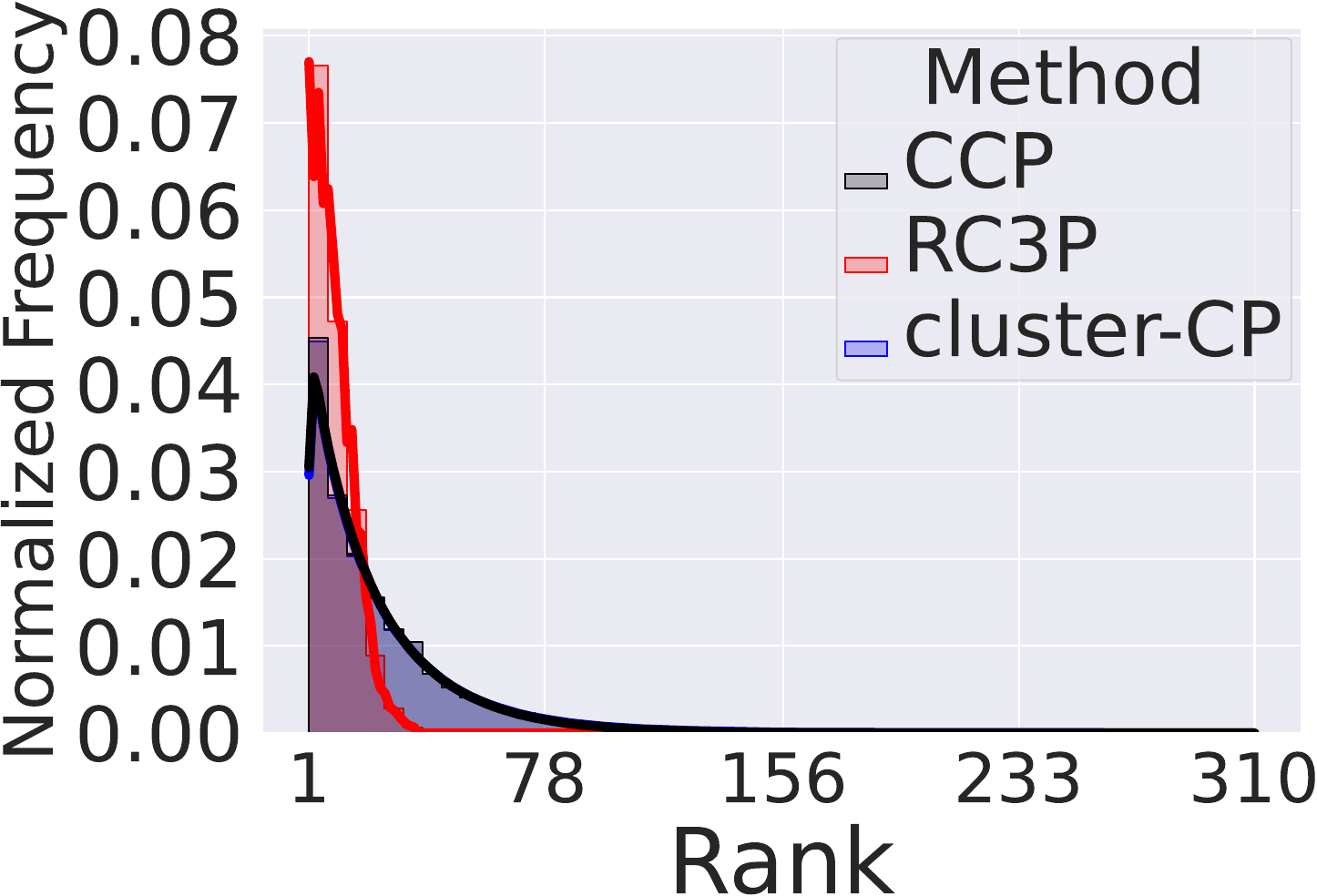}
        \end{minipage}
    }
    \subfigure{
        \begin{minipage}{0.23\linewidth}
            \includegraphics[width=\linewidth]{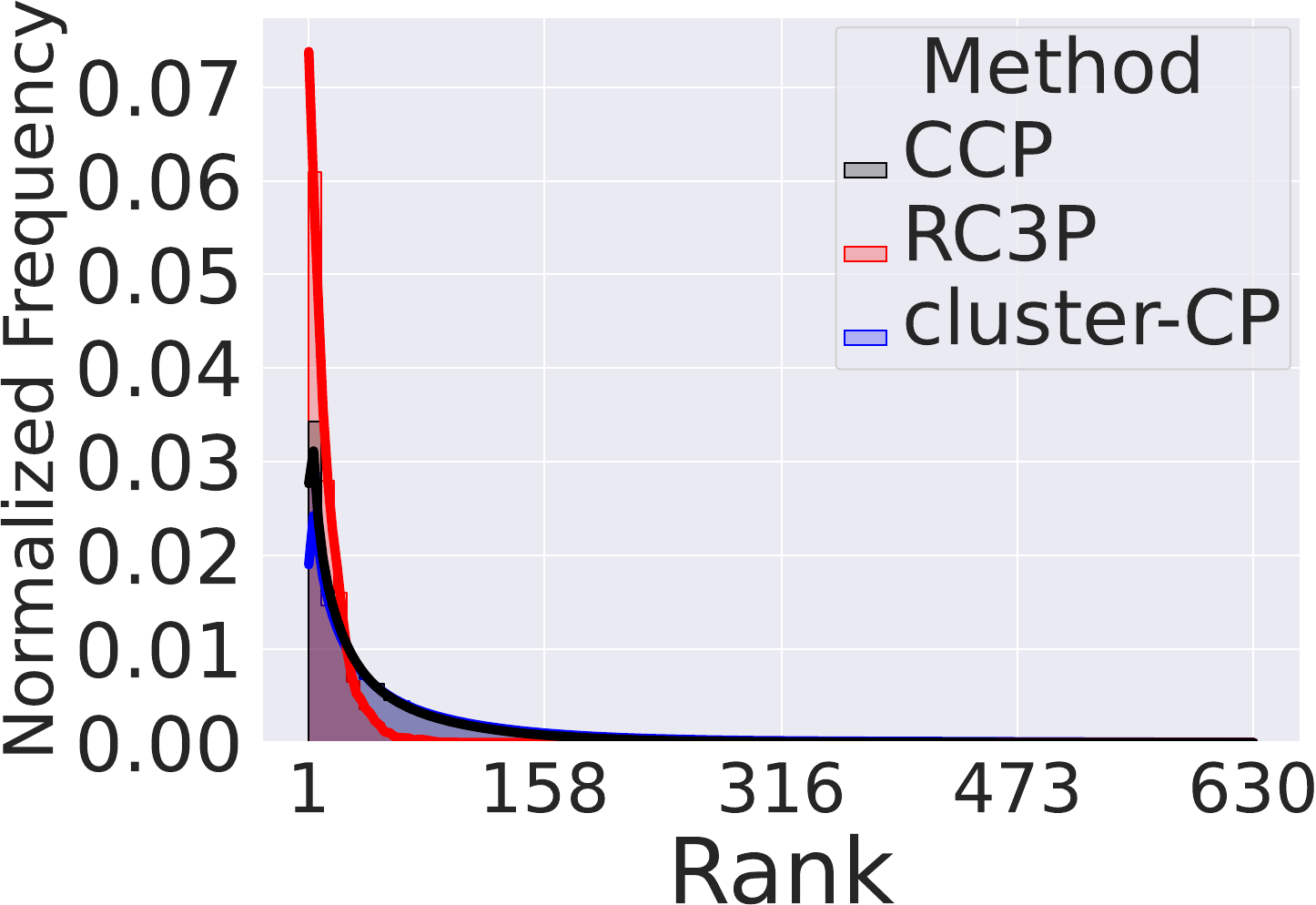}
        \end{minipage}
    }
    \subfigure{
        \begin{minipage}{0.23\linewidth}
            \includegraphics[width=\linewidth]{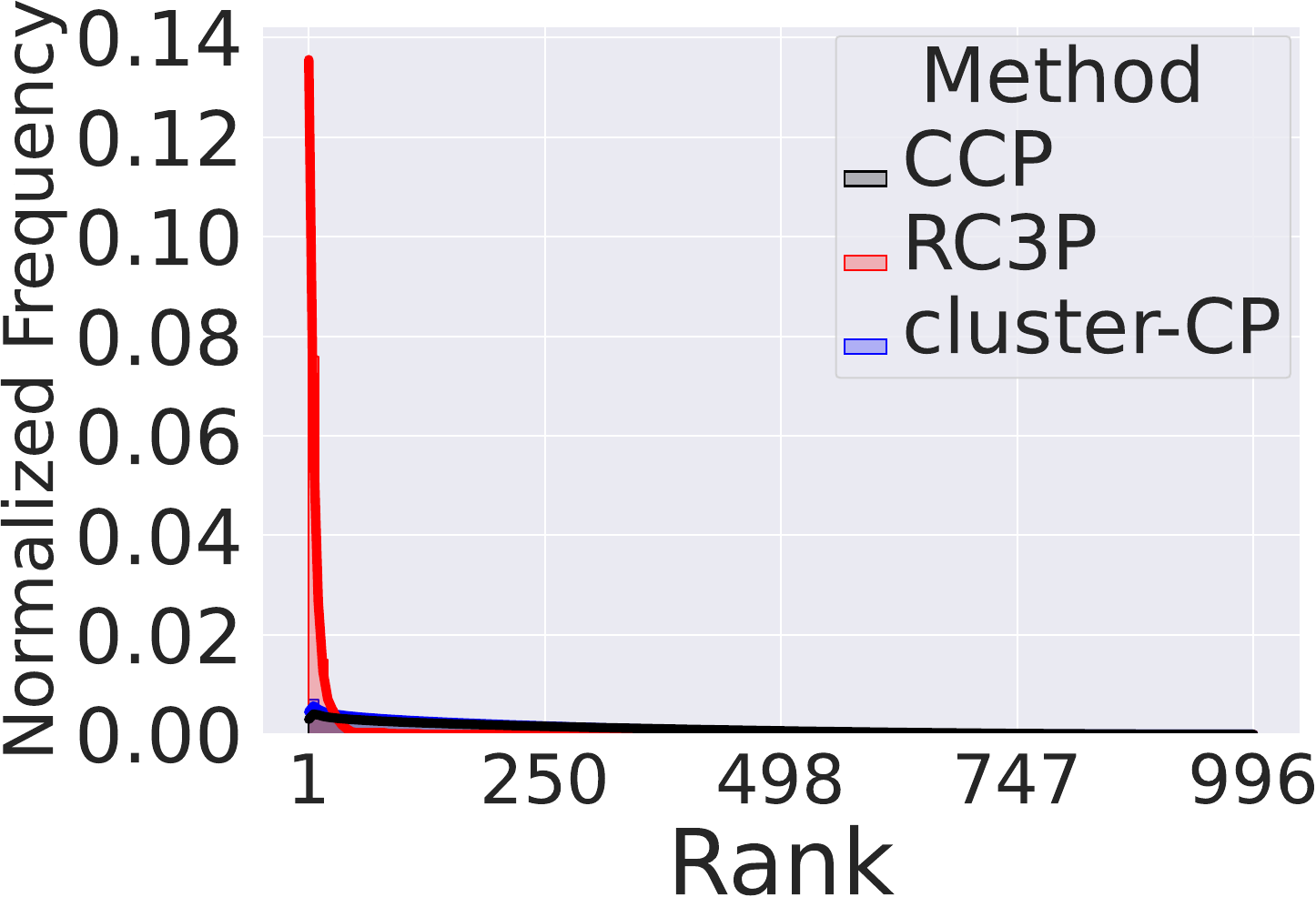}
        \end{minipage}
    }
    \caption{
    Visualization for the normalized frequency distribution of label ranks included in the prediction set of \texttt{CCP}, \texttt{Cluster-CP}, and \texttt{\newCP} with $\rho=0.1$ on balanced datasets.
    It is clear that the distribution of normalized frequency generated by \texttt{\newCP} tends to be lower compared to those produced by \texttt{CCP} and \texttt{Cluster-CP}.  
    Furthermore, the probability density function tail for label ranks in the \texttt{\newCP} prediction set is notably shorter than that of other methods.
    }
    \label{fig:condition_number_rank_balanced}
\end{figure*}

\begin{figure*}[!ht]
    \centering
    \begin{minipage}{.24\textwidth}
        \centering
        (a) CIFAR-100
    \end{minipage}%
    \begin{minipage}{.24\textwidth}
        \centering
        (b) Places365
    \end{minipage}%
    \begin{minipage}{.24\textwidth}
        \centering
        (c) iNaturalist
    \end{minipage}%
    \begin{minipage}{.24\textwidth}
        \centering
        (d) ImageNet
    \end{minipage}
    \subfigure{
        \begin{minipage}{0.23\linewidth}
            \includegraphics[width=\linewidth]{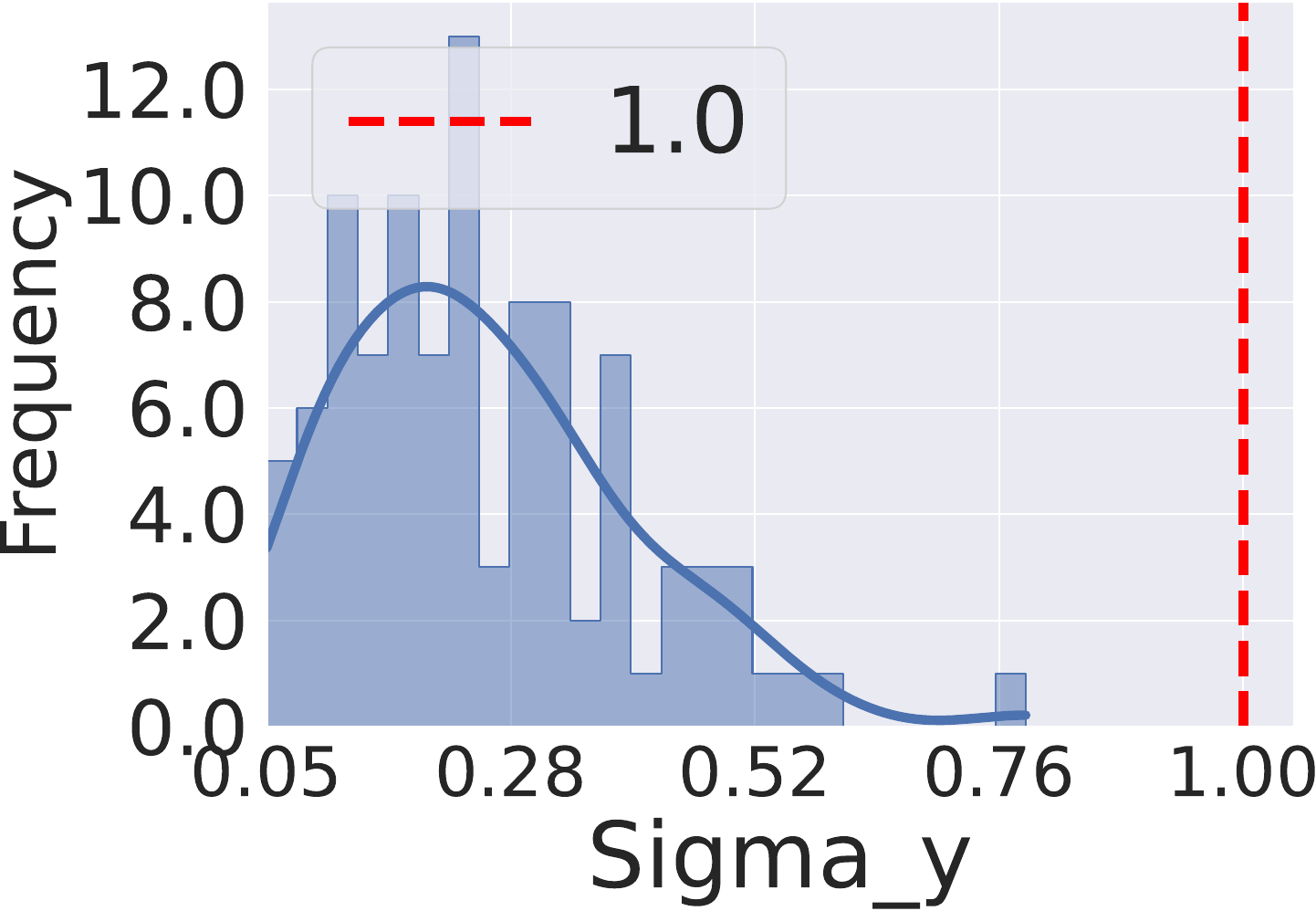}
        \end{minipage}
    }
    \subfigure{
        \begin{minipage}{0.23\linewidth}
            \includegraphics[width=\linewidth]{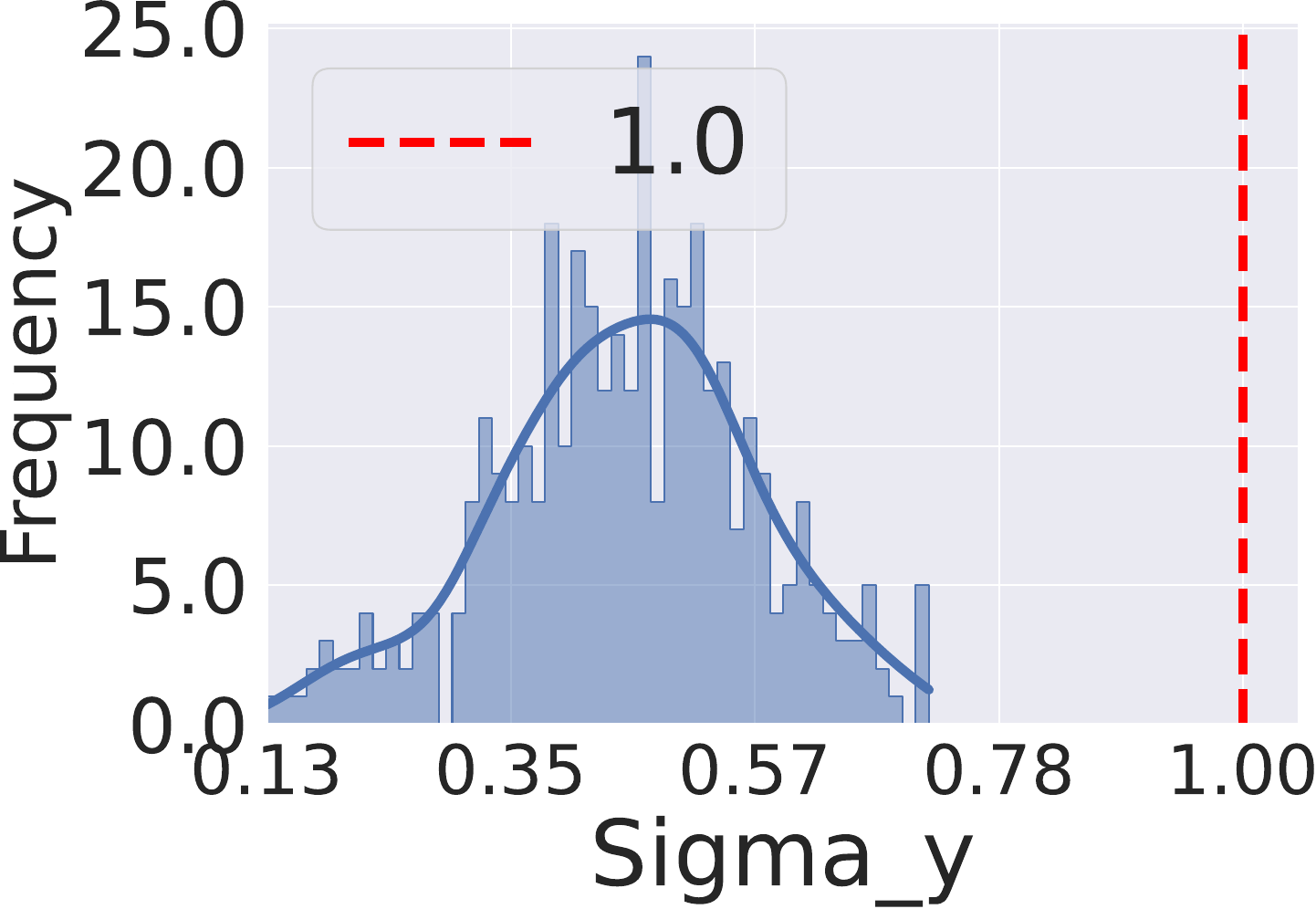}
        \end{minipage}
    }
    \subfigure{
        \begin{minipage}{0.23\linewidth}
            \includegraphics[width=\linewidth]{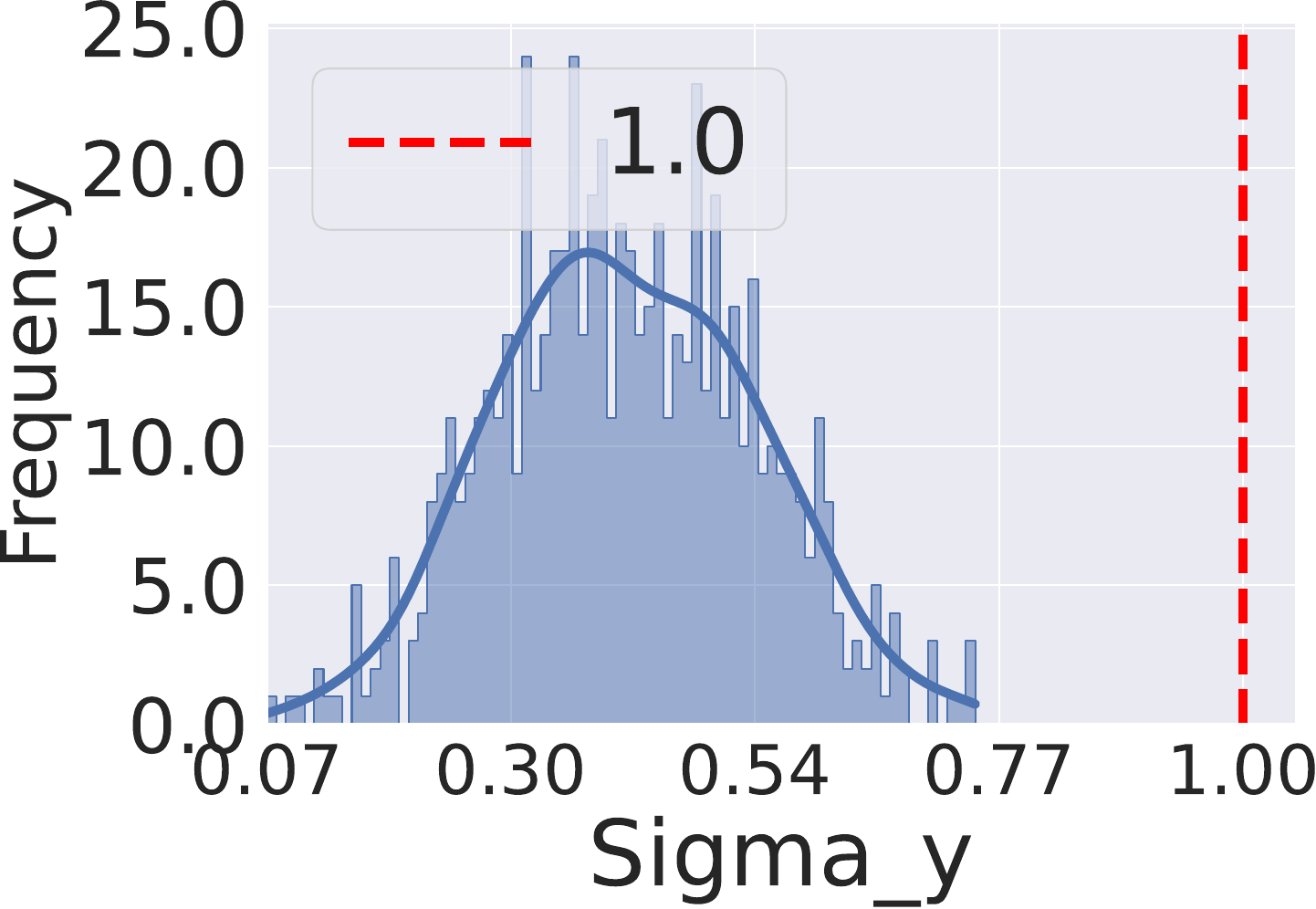}
        \end{minipage}
    }
    \subfigure{
        \begin{minipage}{0.23\linewidth}
            \includegraphics[width=\linewidth]{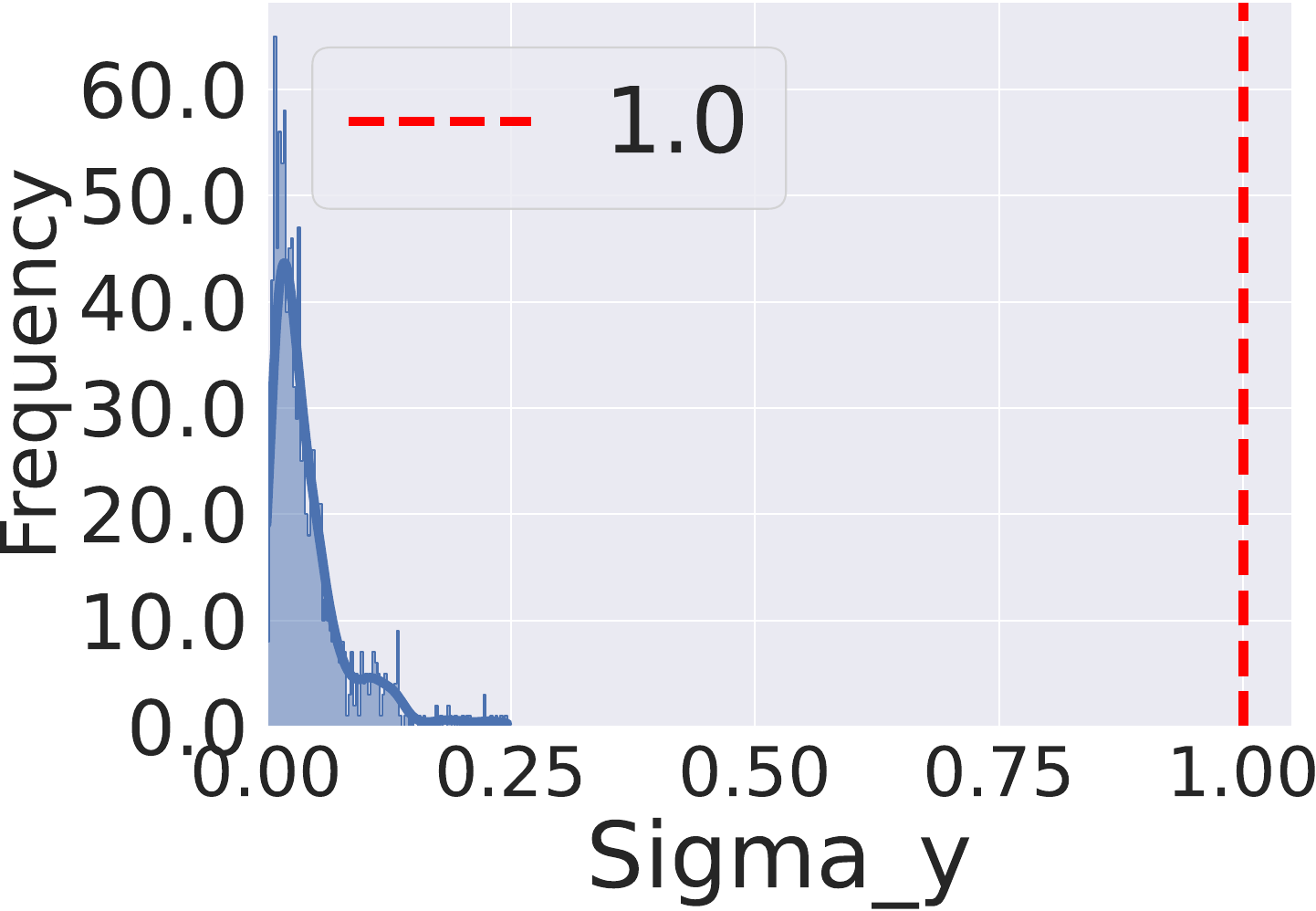}
        \end{minipage}
    }
    \caption{
    Verification of condition numbers $\{\sigma_y\}_{y=1}^K$ in Equation \ref{eq:sigma_y_defination} on balanced datasets.
    Vertical dashed lines represent the value $1$, and we observe that all the condition numbers are smaller than $1$.
    This verifies the validity of the condition for Lemma \ref{lemma:RC3P_improved_efficiency}, and thus confirms that \texttt{\newCP}~produces smaller prediction sets than \texttt{CCP} using calibration on both non-conformity scores and label ranks.
    }
    \label{fig:condition_number_sigma_balanced}
\end{figure*}

\clearpage
\newpage